\documentclass[oneside,12pt]{IIScthesisPSnPDF}
\usepackage{amssymb}
\usepackage{amsmath}
\usepackage{latexsym}
\usepackage{multicol}
\usepackage{booktabs}
\usepackage{multirow}
\usepackage{fancyhdr}
\usepackage[numbers]{natbib}
\usepackage{cleveref}
\usepackage{wrapfig}
\usepackage{algorithm}
\usepackage{algorithmic}
\usepackage{enumerate}
\usepackage{cancel}
\usepackage{mathrsfs}
\usepackage{tabularx}
\usepackage{color}
\usepackage{enumitem}
\usepackage[compact]{titlesec}
\usepackage{mdwlist}
\usepackage{tikz}
\usepackage{varwidth}
\usepackage[vertfit]{breakurl}
\usepackage{datetime}
\usepackage{pdfpages}
\usepackage{caption}
\usepackage{subcaption}
\usepackage{bm}
\usepackage{capt-of}
\usepackage{hyperref}
\usepackage{wrapfig}
\usepackage{subcaption}

\usetikzlibrary{shapes,arrows, trees}
\newtheorem{theorem}{Theorem}[chapter]

\newtheorem{proposition}{Proposition}[chapter]

\newcommand{\remove}[1]{}

\newenvironment{proof}{\noindent{\bf Proof:} \hspace*{1mm}}{\hfill $\Box$ }
\newcommand{\notes}[1]{}

\newcommand{\first}[1]{$1^{\mathrm{st}}$}
\newcommand{\second}[1]{$2^{\mathrm{nd}}$}

\newcommand{\squishlisttwo}{
\begin{list}{$\blacktriangleright$}
{ \setlength{\itemsep}{0.5pt}
\setlength{\parsep}{0pt}
\setlength{\topsep}{0pt}
\setlength{\partopsep}{0.5pt}
\setlength{\leftmargin}{1em}
\setlength{\labelwidth}{1em}
\setlength{\labelsep}{0.5em} } }

\newcommand{\squishend}{
\end{list} }
\allowdisplaybreaks[1]

\newdateformat{monthyeardate}{ \monthname[\THEMONTH], \THEYEAR}

\newcommand{\blankpage}{
\newpage
\thispagestyle{empty}
\mbox{}
\newpage
}

\crefname{observation}{observation}{observations}
\crefname{algorithm}{algorithm}{algorithms}
\crefname{align}{equation}{equations}
\crefname{eqnarray}{equation}{equations}

\makeatletter
\newcommand{\ostar}{\mathbin{\mathpalette\make@circled\star}}
\newcommand{\oast}{\mathbin{\mathpalette\make@circled\ast}}
\newcommand{\make@circled}[2]{%
	\ooalign{$\m@th#1\smallbigcirc{#1}$\cr\hidewidth$\m@th#1#2$\hidewidth\cr}
}
\newcommand{\smallbigcirc}[1]{%
	\vcenter{\hbox{\scalebox{0.77778}{$\m@th#1\bigcirc$}}}%
}
\makeatother

\begin{document}

\newcommand{\ttitle}{Neural Graph Embedding Methods for \\ Natural Language Processing}
\newcommand{\name}{Shikhar Vashishth}
\newcommand{\partOne}{Addressing Sparsity in Knowledge Graphs}
\newcommand{\partTwo}{Exploiting Graph Convolutional Networks in NLP}
\newcommand{\partThree}{GNN Enhancements}

\newcommand{\m}[1]{\CMcal{#1}}

\newcommand{\bmm}[1]{\bm{\CMcal{#1}}}
\newcommand{\real}[1]{\mathbb{R}^{#1}}

\newcommand{\reminder}[1]{\textcolor{red}{[[ #1 ]]}\typeout{#1}}
\newcommand{\reminderR}[1]{\textcolor{gray}{[[ #1 ]]}\typeout{#1}}

\newcommand{\wordgcnMethod}{SynGCN}
\newcommand{\wordgcnMethodSide}{SemGCN}

\newcommand{\datafb}{FB15k}
\newcommand{\datawn}{WN18}
\newcommand{\datafbn}{FB15k-237}
\newcommand{\datawnn}{WN18RR}
\newcommand{\datayago}{YAGO3-10}

\newcommand{\emb}[1]{\ensuremath{\bm{e}_{#1}}}
\newcommand{\remb}[1]{\ensuremath{\bm{#1}_{r}}}

\newcommand{\refalg}[1]{Algorithm \ref{#1}}
\newcommand{\refeqn}[1]{Equation \ref{#1}}
\newcommand{\reffig}[1]{Figure \ref{#1}}
\newcommand{\reftbl}[1]{Table \ref{#1}}
\newcommand{\refsec}[1]{Section \ref{#1}}

\newcommand*{\Scale}[2][4]{\scalebox{#1}{$#2$}}%
\newcommand*{\Resize}[2]{\resizebox{#1}{!}{$#2$}}%

\newcommand{\tensor}{\mathcal{X}}
\newcommand{\Real}{\ensuremath{\mathbb{R}}}
\newcommand{\Natural}{\ensuremath{\mathbb{N}}}
\newcommand{\Complex}{\ensuremath{\mathbb{C}}}
\newcommand{\tdot}[3]{\ensuremath{\langle #1, #2, #3 \rangle}}
\newcommand{\doubledot}[2]{\ensuremath{\langle #1, #2 \rangle}}
\newcommand{\bigO}[1]{\mathcal{O}(#1)}

\newcommand{\tuples}{\mathbb{T}}
\newcommand{\lstm}{GRU}

% RESIDE
\newcommand{\stepOne}{Syntactic Sentence Encoding}
\newcommand{\stepTwo}{Side Information Acquisition}
\newcommand{\stepThree}{Instance Set Aggregation}

% CESI
\newcommand{\gold}{L}
\newcommand{\elink}[2]{\mathcal{E}_{#1}(#2)}
\newcommand{\rlink}[2]{\mathcal{R}_{#1}(#2)}

\newcommand*{\Comb}[2]{{}^{#1}C_{#2}}
\newcommand{\NP}{NP}
\newcommand{\RP}{relation phrase}
\newcommand{\sideConst}[2]{\lambda_{\text{#1}, #2}}
\newcommand{\sidePairs}[2]{\mathcal{Z}_{\text{#1}, #2}}

\newcommand{\stepRep}{Source based NP segregation}
\newcommand{\stepSide}{Side Information Acquisition}
\newcommand{\stepEmbed}{Embedding NP and Relation Phrases}
\newcommand{\stepCluster}{Clustering Embeddings and Canonicalization}
\newcommand{\stepCanon}{Selecting Canonical Representative from Cluster}

\newcommand{\myData}{ReVerb45K}

\newcommand{\gidf}{Gal\'arraga-IDF}
\newcommand{\gstrsim}{Gal\'arraga-StrSim}
\newcommand{\gattr}{Gal\'arraga-Attr}

% INTERACTE
\newcommand\norm[1]{\left\lVert#1\right\rVert}

\newcommand\blfootnote[1]{%
	\begingroup
	\renewcommand\thefootnote{}\footnote{#1}%
	\addtocounter{footnote}{-1}%
	\endgroup
}

%\newtheorem{theorem}{Theorem}[section]
%\newtheorem{claim}[theorem]{Claim}

%\theoremstyle{definition}
%\newtheorem{definition}{Definition}[section]

%
%\theoremstyle{proposition}
%\newtheorem{proposition}{Proposition}[section]
%\newtheorem*{lemma*}{Lemma}
%
%\theoremstyle{remark}
%\newtheorem*{remark}{Remark}
%\newtheorem{case}{Case}
%\newtheorem{subcase}{Case}
%\numberwithin{subcase}{case}

\title{\ttitle{}} 

\submitdate{\monthyeardate\today} 
%\mtech
\phd
%\mtechresearch
%\degree{Master of Technology} 
\dept{Computer Science and Automation}
\faculty{Faculty of Engineering}
\author{\name{}}

% Using the watermark package which is in StyleFiles/
% and to remove DRAFT COPY ONLY appearing on the top of all pages comment out below line
%\watermark{DRAFT COPY ONLY}

\maketitle

\begin{center}
	\LARGE{\underline{\textbf{Declaration of Originality}}}
\end{center}
\noindent I, \textbf{\name{}}, with SR No. \textbf{04-04-00-15-12-16-1-13374} hereby declare that
the material presented in the thesis titled

\begin{center}
	\textbf{\ttitle{}}
\end{center}

\noindent represents original work carried out by me in the \textbf{Department of Computer Science and Automation} at \textbf{Indian Institute of Science} during the years \textbf{2016-2019}.

\noindent With my signature, I certify that:
\begin{itemize}
	\item I have not manipulated any of the data or results.
	\item I have not committed any plagiarism of intellectual
	property.
	I have clearly indicated and referenced the contributions of
	others.
	\item I have explicitly acknowledged all collaborative research
	and discussions.
	\item I have understood that any false claim will result in severe
	disciplinary action.
	\item I have understood that the work may be screened for any form
	of academic misconduct.
\end{itemize}

\vspace{20mm}

\noindent {\footnotesize{Date:	\hfill	Student Signature}} \qquad

\vspace{20mm}

\noindent In my capacity as supervisor of the above-mentioned work, I certify
that the above statements are true to the best of my knowledge, and 
I have carried out due diligence to ensure the originality of the
report.

\vspace{20mm}

\noindent  {\footnotesize{Advisor Name: \hfill Advisor Signature}} \qquad

\blankpage

\vspace*{\fill}
\begin{center}
	\large\bf \textcopyright \ \name{}\\
	\large\bf \monthyeardate\today\\
	\large\bf All rights reserved
\end{center}
\vspace*{\fill}
\thispagestyle{empty}

\blankpage

\vspace*{\fill}
\begin{center}
DEDICATED TO \\[2em]
\Large\it My Teachers\\[2em]
\Large\it who enlightened me with all knowledge.
\end{center}
\vspace*{\fill}
\thispagestyle{empty}

%\blankpage
%\includepdf[pages={1}]{declaration.pdf}

%\vspace*{\fill}
%\begin{tabular}{p{0.4\columnwidth}p{0.5\columnwidth}}
% {\em Signature of the Author}: & \dotfill \\
% & Your Name \\
% & Dept.\ of Computer Science and Automation \\ 
% & Indian Institute of Science, Bangalore \vspace{1in}\\
% {\em Signature of the Thesis Supervisor}: & \dotfill \\
% & Your Advisor's Name \\
% & Professor \\
% & Dept.\ of Computer Science and Automation \\ 
% & Indian Institute of Science, Bangalore
%\end{tabular}
%\vspace*{\fill}
%\thispagestyle{empty}

%\blankpage

%set the number of sectioning levels that get number and appear in the contents
\setcounter{secnumdepth}{3}
\setcounter{tocdepth}{3}

\frontmatter % book mode only
\pagenumbering{roman}

\prefacesection{Acknowledgements}
I want to offer my sincere thanks to my advisors Dr. Partha Pratim Talukdar and Prof. Chiranjib Bhattacharyya, who gave me the freedom to work in my area of interest and have been very supportive throughout the course of my PhD. I also want to thank Dr. Manaal Faruqui, who accepted the role of being my mentor and guided me in my research. I am very grateful to all my teachers at Indian Institute of Science for giving me a clear understanding of the basics which were essential for completing this work. 

I feel fortunate to get the opportunity to collaborate with several researchers during my PhD. I want to thank Prince Jain, Shib Sankar Das, Swayambhu Nath, Rishabh Joshi, Sai Suman, Manik Bhandari, Prateek Yadav, Soumya Sanyal, Vikram Nitin, Shyam Upadhyay, Gaurav Singh Tomar and all the members of MALL Lab for their support and help. I am also thankful to my parents and friends for their support throughout my stay in Bangalore. Finally, I would like to thank Almighty God for all His blessings without which this would not have been possible.

\prefacesection{Abstract}
Graphs are all around us, ranging from citation and social networks to Knowledge Graphs (KGs). They are one of the most expressive data structures which have been used to model a variety of problems. Embedding graphs involve learning a representation of all nodes (and relations) in the graph which allows to effectively utilize graphs for various downstream problems. In this thesis, we explore such techniques for alleviating sparsity problem in knowledge graphs and for tasks such as document timestamping and word representation learning. We also list some of the limitations of existing graph embedding methods and propose solutions for addressing them. 

Knowledge graphs are structured representations of facts in a graph, where nodes represent entities and edges represent relationships between them. Recent research has resulted in the development of several large KGs; examples include DBpedia, YAGO, NELL, and Freebase. However, all of them tend to be sparse with very few relations associated with each entity. For instance, NELL KG consists of only 1.34 facts per entity. In the first part of the thesis, we explore two solutions to alleviate this problem through neural graph embedding based techniques: (1) KG Canonicalization, i.e., identifying and merging duplicate entities in a KG, (2) Relation Extraction which involves densifying the knowledge graph by automatically extracting more relationships from unstructured text.
%, and (3) Link prediction which includes inferring missing facts based on the known facts in a KG. 
For KG Canonicalization, we propose CESI (Canonicalization using Embeddings and Side Information), a novel approach that performs canonicalization over learned embeddings of KGs. The method extends recent advances in KG embedding by incorporating relevant NP and relation phrase side information in a principled manner. For relation extraction, we propose RESIDE, a distantly-supervised neural relation extraction method which utilizes additional side information from KGs for improved relation extraction. 
Both the approaches demonstrate the effectiveness of utilizing relevant side information along with graph embedding techniques for addressing knowledge graph sparsity. 
%Finally, for link prediction, we propose InteractE which extends ConvE, a convolutional neural network-based link prediction method. InteractE increases the number of feature interaction through three key ideas – feature permutation, a novel feature reshaping, and circular convolution. 
Through extensive experiments on multiple datasets, we demonstrate the effectiveness of our proposed methods. 

Traditional Neural Networks like Convolutional Networks and Recurrent Neural Networks are constrained to learn representation of Euclidean data. However, graphs in Natural Language Processing (NLP) are prominent. Recently, Graph Convolutional Networks (GCNs) have been proposed to allow Deep Learning models to exploit graph structures by embedding them in an end-to-end fashion. GCNs 
%address this shortcoming 
have been successfully applied for several problems in NLP and computer vision. In the second part of the thesis, we explore application of GCNs in two prominent tasks: (1) Document Timestamping problem, which forms an essential component of tasks like document retrieval, and summarization. For this, we propose NeuralDater which leverages GCNs for jointly exploiting syntactic and temporal graph structures of document for obtaining state-of-the-art performance on the problem. (2) Word representation learning, which has been widely adopted across several NLP tasks. We propose SynGCN, a flexible Graph Convolution based method for learning word embeddings which utilize the dependency context of a word instead of linear context for learning more meaningful word embeddings. 

Finally, in the last part of the thesis, we address two limitations of existing GCN models, viz, (1) The standard neighborhood aggregation scheme puts no constraints on the number of nodes that can influence the representation of a target node. This leads to a noisy representation of hub-nodes which covers almost the entire graph in a few hops. To address this shortcoming, we propose ConfGCN (Confidence-based GCN) which estimates confidences to determine the importance of a node on another during aggregation, thus restricting its influence neighborhood. (2) Most of the existing GCN models are limited to handle undirected graphs. However, a more general and pervasive class of graphs are relational graphs where each edge has a label and direction associated with it. Existing approaches to handle such graphs suffer from over-parameterization and are restricted to learning representation of nodes only. We propose CompGCN, a novel Graph Convolutional framework which jointly embeds entity and relations in a relational graph. CompGCN is parameter efficient and scales with the number of distinct relation types. It leverages a variety of entity-relation composition operations from KG Embedding techniques and achieves demonstrably superior results on node classification, link prediction, and graph classification tasks.

\prefacesection{Publications based on this Thesis}

The work in this dissertation is primarily related to the following peer-reviewed articles: \blfootnote{* Equal Contribution}
\begin{enumerate}
	\item \textbf{Shikhar Vashishth}, Prince Jain, and Partha Talukdar. ``CESI: Canonicalizing Open Knowledge Bases using Embeddings and Side Information". In Proceedings of the World Wide Web Conference (WWW), 2018.
	\item \textbf{Shikhar Vashishth}, Shib Shankar Dasgupta, Swayambhu Nath Ray, and Partha Talukdar. ``Dating Documents using Graph Convolution Networks". In Proceedings of the 56th Annual Meeting of the Association for Computational Linguistics (ACL), 2018.
	\item \textbf{Shikhar Vashishth}, Rishabh Joshi, Sai Suman Prayaga, Chiranjib Bhattacharyya, and Partha Talukdar. ``RESIDE: Improving Distantly-Supervised Neural Relation Extraction using Side Information". In Proceedings of the 2018 Conference on Empirical Methods in Natural Language Processing (EMNLP), 2018.
	\item \textbf{Shikhar Vashishth}$^*$, Prateek Yadav$^*$, Manik Bhandari$^*$, and Partha Talukdar. ``Confidence-based Graph Convolutional Networks for Semi-Supervised Learning". In Proceedings of the International Conference on Artificial Intelligence and Statistics (AISTATS), 2019.
	\item \textbf{Shikhar Vashishth}, Manik Bhandari, Prateek Yadav, Piyush Rai, Chiranjib Bhattacharyya, and Partha Talukdar. ``Incorporating Syntactic and Semantic Information in Word Embeddings using Graph Convolutional Networks". In Proceedings of the 57th Annual Meeting of the Association for Computational Linguistics (ACL), 2019.	
	\item \textbf{Shikhar Vashishth}$^*$, Soumya Sanyal$^*$, Vikram Nitin, and Partha Talukdar. ``Composition-based Multi-Relational Graph Convolutional Networks". In Proceedings of 8th International Conference on Learning Representations (ICLR), 2020.
\end{enumerate}

\newpage
\noindent The following articles have also been completed over the course of the PhD but are not discussed in the thesis:
\begin{enumerate}[resume]
	\item \textbf{Shikhar Vashishth}$^*$, Soumya Sanyal$^*$, Vikram Nitin, and Partha Talukdar. ``InteractE: Improving Convolution-based Knowledge Graph Embeddings by Increasing Feature Interactions". In Proceedings of 34th Conference on Artificial Intelligence (AAAI), 2020.
	\item Zhiqing Sun$^*$, \textbf{Shikhar Vashishth}$^*$, Soumya Sanyal$^*$, Partha Talukdar, and Yiming Yang. ``A Re-evaluation of Knowledge Graph Completion Methods". In Proceedings of the 58th Annual Meeting of the Association for Computational Linguistics (ACL), 2020.
	\item \textbf{Shikhar Vashishth}, Shyam Upadhyay, Gaurav Singh Tomar, and Manaal Faruqui. ``Attention Interpretability Across NLP Tasks". \textit{arXiv preprint arXiv:1909.11218}, 2019.
	\item Prateek Yadav, Madhav Nimishakavi, Naganand Yadati, \textbf{Shikhar Vashishth}, Arun Rajkumar and Partha Talukdar. ``Lovasz Convolutional Networks". In Proceedings of the International Conference on Artificial Intelligence and Statistics (AISTATS), 2019.

\end{enumerate}

\tableofcontents
%\blankpagewithnumber
\listoffigures
\listoftables
%\blankpagewithnumber
% \printnomenclature  %% Print the nomenclature
% \addcontentsline{toc}{chapter}{Nomenclature}

\mainmatter % book mode only
\setcounter{page}{1}
\chapter{Introduction}
\label{chap:introduction}

Graphs are pervasive data structures that have been used to model a variety of problems. Embedding graphs involves learning a vector representation for each node (and relation) in the graph. This enables to effectively make inference over the graph and solve various tasks such as link prediction and node classification. Such techniques have also been extensively explored for Knowledge Graphs (KGs), which are a large storehouse of world facts in a graph format. Instances of KGs include Freebase \cite{freebase}, WordNet \cite{wordnet}, YAGO \cite{yago}, and NELL \cite{nell}. KGs find application in a variety of tasks, such as relation extraction \cite{distant_supervision2009}, question answering \cite{qa_kg_1,qa_kg_2}, recommender systems \cite{kb-recommender}, and dialog systems \cite{kg_in_dialog}. Formally, Knowledge graphs are defined as a structured representation of facts in a graph, where nodes represent entities and edges denote relationships between them. This can be represented as a collection of triples \((s, r, o)\), each representing a relation \(r\) between a ``subject-entity" \(s\) and an ``object-entity" \(o\). A KG example is presented in Figure \ref{fig:kg_example}. Here, \textit{John Lennon}, \textit{Beatles}, and \textit{Liverpool} are nodes in the knowledge graph, and edges indicate different facts about them. For instance, the fact that \textit{``John Lennon was born in Liverpool"} is expressed through \textit{(John Lennon, bornIn, Liverpool)} triple in the KG.

\begin{figure}[!t]
	\centering
	\includegraphics[width=0.4\textwidth]{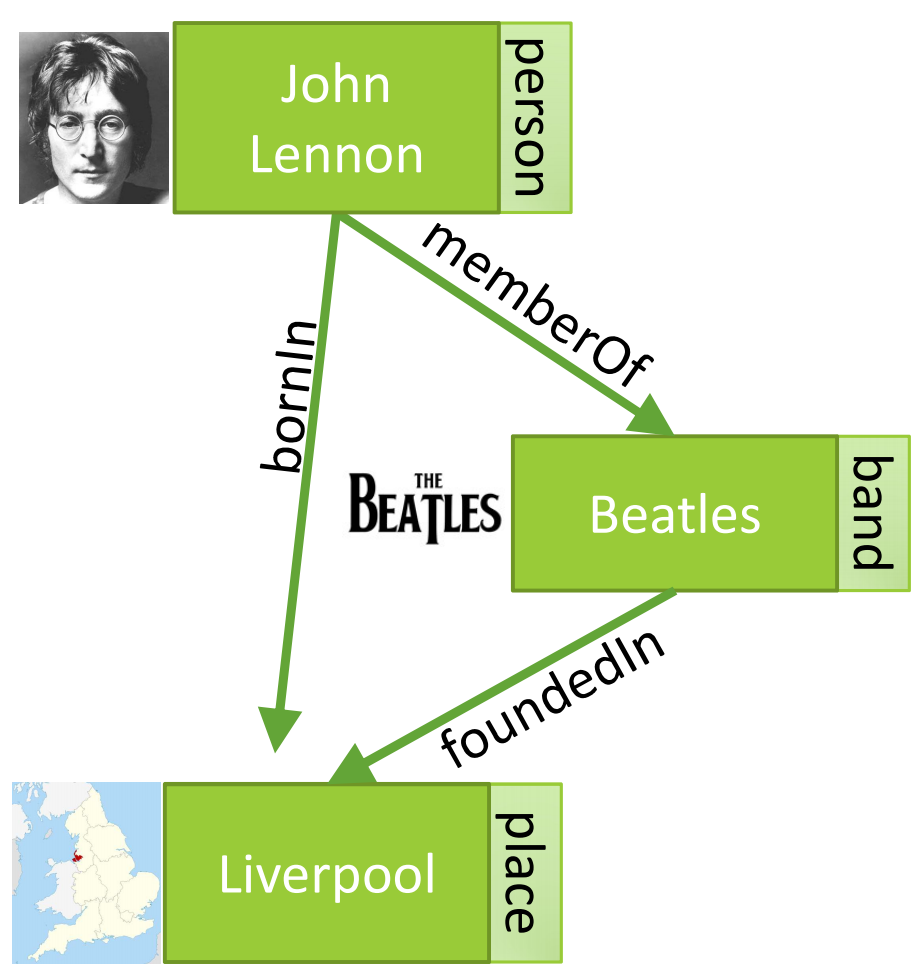}
	\caption{\small \label{fig:kg_example}An instance of Knowledge Graph$^1$. Here, nodes represent entities and edges indicate different relationships between them.}  
\end{figure}

Recently, neural graph embeddings are also being exploited for several Natural Language Processing applications. Graphs are prevalent in NLP, starting from dependency parse to temporal event ordering graphs, constituency parse graphs, etc. Neural graph embeddings methods allow utilizing such graph structures for a variety of downstream applications. An instance of such methods is Graph Convolutional Networks (GCNs) \cite{Defferrard2016,Kipf2016}, which allows embedding graphs based on the final objective in an end-to-end fashion. GCNs are a generalization of Convolutional Neural Networks (CNNs), which have been pivotal to achieve beyond human-like performance for tasks such as object recognition \cite{alexnet,obj_detect_survey} and speech recognition \cite{speech_survey}. GCNs have been shown to be effective for tasks such as neural machine translation \cite{gcn_nmt}, semantic role labeling \cite{gcn_srl}, event detection \cite{gcn_event}. However, the scope of their applicability for other NLP tasks is still an area of research.
\footnotetext[1]{Image credits: https://kgtutorial.github.io/}
Furthermore, improving existing graph embedding methods also serves as an important research problem. Recent works \cite{Xu2018,gin} have identified several limitations of Graph Convolutional Networks and have theoretically analyzed their potential. 

Knowledge graphs have been utilized for a variety of tasks \cite{kg_in_dialog,kg_incomp1}. However, most of the KGs are highly sparse with very few edges per entity, as also observed by \cite{transe}. For instance, NELL KG consists of only 1.34 facts per entity. This severely restricts their usage for several real-life applications. In the first part of the thesis, we present two graph embedding based techniques for addressing the sparsity problem in Knowledge Graphs. The proposed methods also demonstrate the effectiveness of utilizing relevant side information in a principled manner. We discuss each of them below.

 %In this thesis, we aim to explore neural graph embedding techniques for alleviating this problem. %Firstly, we explore their applicability for alleviating sparsity problem in Knowledge Graphs (KGs), 

\begin{enumerate}[itemsep=2pt,parsep=0pt,partopsep=0pt,leftmargin=*]
	\item \textbf{Knowledge Graph Canonicalization} involves identifying duplicate or redundant nodes in a KG and merging them as a single node. This can be explained through a concrete example. Given two triples in a KG: \textit{(Barack Obama, was president of, US)} and \textit{(Obama, born in, US)}. Identifying that \textit{Barack Obama} and \textit{Obama} refer to the same entity increases the overall facts per entity ratio. In our work, we focus on addressing such issues in Open KGs, which are automatically constructed knowledge graphs (using OpenIE algorithms \cite{ollie,ollie1,ollie3}) without any pre-specified ontology. In spite of its importance, canonicalization is a relatively unexplored problem, especially in the case of Open KGs. In this work, we propose CESI (Canonicalization using Embeddings and Side Information), a novel approach which performs canonicalization over learned embeddings of Open KBs. CESI extends recent advances in KB embedding by incorporating relevant noun and relation phrase side information in a principled manner. More details are provided in Chapter \ref{chap_cesi}.
	
	\item \textbf{Relation Extraction} involves automatically extracting semantic relationships between entity pairs from unstructured text. Most of the existing KGs like Wikidata and Freebase are human-curated. Relation extraction offers a mechanism for automatically constructing these KGs without any supervision. This allows it for further densifying existing KGs by extracting new facts from unstructured text. Since most supervised relation extraction methods require sizeable labeled training data which is expensive to construct, we utilize Distant Supervision (DS) \cite{distant_supervision2009} for automatically constructing a dataset. DS is based on the assumption that if two entities have a relationship in a KB, then all sentences mentioning those entities express the same relation. We propose a novel distantly-supervised neural relation extraction method, RESIDE which utilizes additional side information from KBs for improved relation extraction. It uses entity type and relation alias information for imposing soft constraints while predicting relations. RESIDE employs Graph Convolution Networks (GCN) to encode syntactic information from a text and improves performance even when limited side information is available. Please refer to Chapter \ref{chap_reside} for more details. 
	
%	\item  \textbf{Link Prediction} is the task of inferring missing facts based on the known facts in a KG. For instance, if we have two triples in the KG: \textit{(Michelle Obama, spouse of,  Barack Obama)} and \textit{(Sasha Obama, child of, Michelle Obama)}. Now, based on these two facts \textit{(Sasha Obama, child of, Barack Obama)} can be inferred with high confidence.  A popular approach for solving this problem involves learning a low-dimensional representation for all entities and relations and utilizing them to predict new facts. In general, most existing link prediction methods learn to embed KGs by optimizing a score function which assigns higher scores to actual facts than invalid ones. ConvE \cite{conve}, a recently proposed approach, applies convolutional filters on 2D reshapings of entity and relation embeddings to capture rich interactions between their components. However, the number of interactions that ConvE can capture is limited. In this work, we analyze how increasing the number of these interactions affects link prediction performance, and utilize our observations to propose InteractE which is based on three key ideas – feature permutation, a novel feature reshaping, and circular convolution. For more details, please refer to Chapter \ref{chap_interacte}.
\end{enumerate}

In the second part of the thesis, we focus on leveraging recently proposed Graph Convolutional Networks (GCNs) \cite{Defferrard2016,Kipf2016} for exploiting different graph structures in NLP. Traditional neural network architectures like Convolutional Neural Networks (CNNs) \cite{cnn_paper} and Recurrent Neural Networks (RNNs) \cite{lstm} are limited to handle Euclidean data. GCNs have been proposed to address this shortcoming  and have been successfully employed  for improving performance on tasks such as  semantic role labeling \citep{gcn_srl}, neural machine translation \citep{gcn_nmt}, relation extraction \cite{gcn_re_stanford}, shape segmentation \citep{yi2016syncspeccnn}, and action recognition \citep{huang2017deep}. In this work, we begin with utilizing GCNs for exploiting syntactic and event-ordering graph structures in Document Timestamping problem which involves predicting creation date of a given document. The task is at the core of many essential tasks, such as, information retrieval \cite{ir_time_usenix,ir_time_li}, temporal reasoning \cite{temp_reasoner1,temp_reasoner2}, text summarization \cite{text_summ_time}, and analysis of historical text \cite{history_time}. For this, we propose NeuralDater, a GCN-based approach which is to the best of our knowledge, the first application of deep learning for the problem. The model is more elaborately described in Chapter \ref{chap_neuraldater}. Next, we propose to use GCNs for utilizing syntactic context while learning word embeddings. Most existing word embedding methods are restricted to using the sequential context of a word. In this work, we overcome this problem by proposing SynGCN, a flexible Graph Convolution based method which utilizes the dependency context of a word without increasing the vocabulary size. We also propose SemGCN, an effective framework for incorporating diverse semantic knowledge for further enhancing learned word representations. Refer to Chapter \ref{chap_wordgcn} for details. 

In the third part of the thesis, we address some of the significant limitations of the current Graph Convolution based models. Most of the existing GCN methods are an instantiation of \textit{Message Passing Neural Networks} \cite{mpnn} which uses neighborhood aggregation scheme which puts no constraints on the number of nodes that can influence the representation of a given target node. In a $k$-layer model, each node is influenced by all the nodes in its k-hop neighborhood. This becomes a concern for hub nodes which covers almost the entire graph with a few hop neighbors. To alleviate this shortcoming, we propose ConfGCN, a Graph Convolutional Network which models label distribution and their confidences for each node in the graph. ConfGCN utilizes label confidences to estimate the influence of one node on another in a label-specific manner during neighborhood aggregation, thus controlling the influence neighborhood of nodes during GCN learning. Please refer to Chapter \ref{chap_confgcn} for details. Apart from this, we also propose an extension of GCN models for multi-relational graphs. Most of the existing GCN models are limited to handle undirected graphs. However, a more general and pervasive class of graphs are relational graphs where each edge has a label and direction associated with it. Existing approaches to handle such graph data suffer from overparameterization and are restricted to learning representation of nodes only. We propose CompGCN, a novel Graph Convolutional framework which jointly embeds entity and relations in a relational graph. CompGCN is parameter efficient and scales with the number of relations. It leverages a variety of entity-relation composition operations from Knowledge Graph Embedding techniques. CompGCN allows the application of GCNs for a problem which requires both node and edge embeddings such as drug discovery and KG link prediction. Through extensive experiments, we demonstrate the effectiveness of our proposed approaches. More details are presented in Chapter \ref{chap_compgcn}.

\section{Summary of Contributions}
Our contributions in the thesis can be grouped into the following three parts:

\textbf{\partOne{}:}
For addressing the sparsity problem in knowledge graphs, first we propose CESI (Canonicalization using Embeddings and Side Information), a novel method for canonicalizing Open KBs using learned embeddings. To the best of our knowledge, this is the first approach to use learned embeddings and side information for canonicalizing an Open KB. CESI models the problem of noun phrase (NP) and relation phrase canonicalization jointly using relevant side information in a principled manner. This is unlike prior approaches where NP and relation phrase canonicalization were performed sequentially. For densifying existing knowledge graphs using unstructured text, we propose RESIDE, a novel neural method which utilizes additional supervision from KB in a principled manner for improving distant supervised RE. RESIDE uses Graph Convolution Networks (GCN) for modeling syntactic information and has been shown to perform competitively even with limited side information. 
%Finally, for inferring new relations based on the existing ones, we propose InteractE, a method that augments the expressive power of ConvE through three key ideas – feature permutation, "checkered" feature reshaping, and circular convolution. We provide a precise definition of interaction, and theoretically analyze InteractE to show that it increases interactions compared to ConvE. Further, we establish a correlation between the number of \textit{heterogeneous interactions} and link prediction performance.
Through extensive evaluation on various benchmark datasets, we demonstrate the effectiveness of our proposed approaches. 

\textbf{\partTwo{}:}
We leverage recently proposed Graph Convolutional Networks for exploiting several graph structures in NLP to improve performance on two tasks: Document Timestamping and Word embeddings. We propose NeuralDater, a Graph Convolution Network (GCN)-based approach for document dating. To the best of our knowledge, this is the first application of GCNs, and more broadly deep neural network-based methods, for the document dating problem. NeuralDater is the first document dating approach which exploits the syntactic as well as temporal structure of the document, all within a principled joint model. Next, we propose SynGCN, a Graph Convolution based method for learning word embeddings. Unlike previous methods, SynGCN utilizes syntactic context for learning word representations without increasing vocabulary size. We also present SemGCN, a framework for incorporating diverse semantic knowledge (e.g., synonymy, antonymy, hyponymy, etc.) in learned word embeddings, without requiring relation-specific special handling as in previous methods. Through experiments on multiple intrinsic and extrinsic tasks, we demonstrate that our proposed methods obtain substantial improvement over state-of-the-art approaches, and also yield advantage when used with methods such as ELMo. 

\textbf{\partThree{}:}
Finally, we address two limitations in existing Graph Convolutional Network (GCN) based methods. For this, We propose ConfGCN, a Graph Convolutional Network framework for semi-supervised learning which models label distribution and their confidences for each node in the graph. To the best of our knowledge, this is the first confidence enabled formulation of GCNs. ConfGCN utilizes label confidences to estimate the influence of one node on another in a label-specific manner during neighborhood aggregation of GCN learning. Next, we propose CompGCN, a novel framework for incorporating multi-relational information in Graph Convolutional Networks which leverages a variety of composition operations from knowledge graph embedding techniques. Unlike previous GCN based multi-relational graph embedding methods, COMPGCN jointly learns embeddings of both nodes and relations in the graph. Through extensive experiments on multiple tasks, we demonstrate the effectiveness of our proposed method. Through extensive experiments on multiple tasks, we demonstrate the effectiveness of our proposed methods. 

\section{Organization of Thesis}
The rest of the thesis is organized as follows: In Chapter \ref{chap:background}, we review some background on Knowledge Graphs and Graph Convolutional Networks. Then in Part 1 of the thesis, we present two methods for addressing sparsity problem in Knowledge Graph: Canonicalization (Chapter \ref{chap_cesi}) and Relation Extraction (Chapter \ref{chap_reside}). In Part 2, we present two novel applications of Graph Convolutional Networks in NLP for Document Timestamping (Chapter \ref{chap_neuraldater}) and Word embedding (Chapter \ref{chap_wordgcn}) tasks. Then, we address two limitations in existing GCN models in Part 3. We present ConfGCN for controlling influence neighborhood in GCN learning in Chapter \ref{chap_confgcn} and an extension of GCNs for relational graphs in Chapter \ref{chap_compgcn}. Finally, we conclude in Chapter \ref{chap:conclusion} by summarizing our contributions and discussing future directions. 
\chapter{Background: Graph Convolutional Networks}
\label{chap:background}

In this chapter, we provide an overview of Graph Convolutional Networks (GCNs) which forms a necessary background material required for understanding the subsequent chapters of this work. %We provide an introduction to  which is  

\section{Introduction}
Convolutional Neural Networks (CNNs) have lead to major breakthroughs in the era of deep learning because of their ability to extract highly expressive features. However, CNNs are restricted to Euclidean data such as images and text which have grid like structure. Non-Euclidean data such as graphs are more expressive and have been used to model a variety of problems. Graph Convolutional Networks (GCNs) address this shortcoming by generalizing CNNs' property of local receptive field, shared weights, and multiple layers to graphs. GCNs have been successfully applied to several domains such as social networks \cite{fastgcn}, knowledge graphs \cite{r_gcn}, natural language processing \cite{gcn_srl,gcn_nmt}, drug discovery \cite{deep_chem} and natural sciences \cite{gcn_comp_synthesis,gcn_protein_interaction}. In this chapter. we describe how CNN model for Euclidean graphs can be generalized for non-Euclidean data using Spectral Graph theory \cite{spectral_graph_theory}. We acknowledge that most of the content of this chapter is adopted from \citet{emerging_field_graph_signal,Defferrard2016,Kipf2016}.

%\subsection{Notation}
%\label{sec:notation}
%In this work, 
%For multi-relational graph, an edge from node $u$ to $v$ with label $r$ is denoted as $(u,v,r)$ and $\m{E} = \{(u,v,r) \ | \ u, v \in \m{V}\}$.

%indicates the set of edges and $\m{X} \in \real{|V| \times d_0}$ refers to $d_0$-dimensional initial node features

\section{Preliminaries}
\label{sec:prelim}

\noindent \textbf{Notations:} We denote an undirected and connected graph as $\m{G} = (\m{V}, \m{E}, \bm{W})$, where $\m{V}$ refers to the set of nodes ($N = |\m{V}|$), $\m{E} = \{(u, v) \ | \ u, v \in \m{V}\}$ indicates the set of edges, and $\bm{W}$  is a weighted adjacency matrix of the graph. If there does not exist an edge between node $i$ and $j$ then $W_{ij}$ is set to $0$.  

\ \\

\noindent \textbf{Graph Signal} refers to a function defined on the vertices of a graph $\m{G}$, i.e., $f: \m{V} \rightarrow \real{}$. For the entire graph, it can be represented as a vector $\bm{x} \in \real{N}$, where $x_i$ denotes the function value at the $i^{th}$ vertex.  Figure \ref{fig:back_graph_signal} shows an illustration of a graph signal over a graph.  

\begin{figure*}[h]
	\centering
	\includegraphics[width=0.4\textwidth]{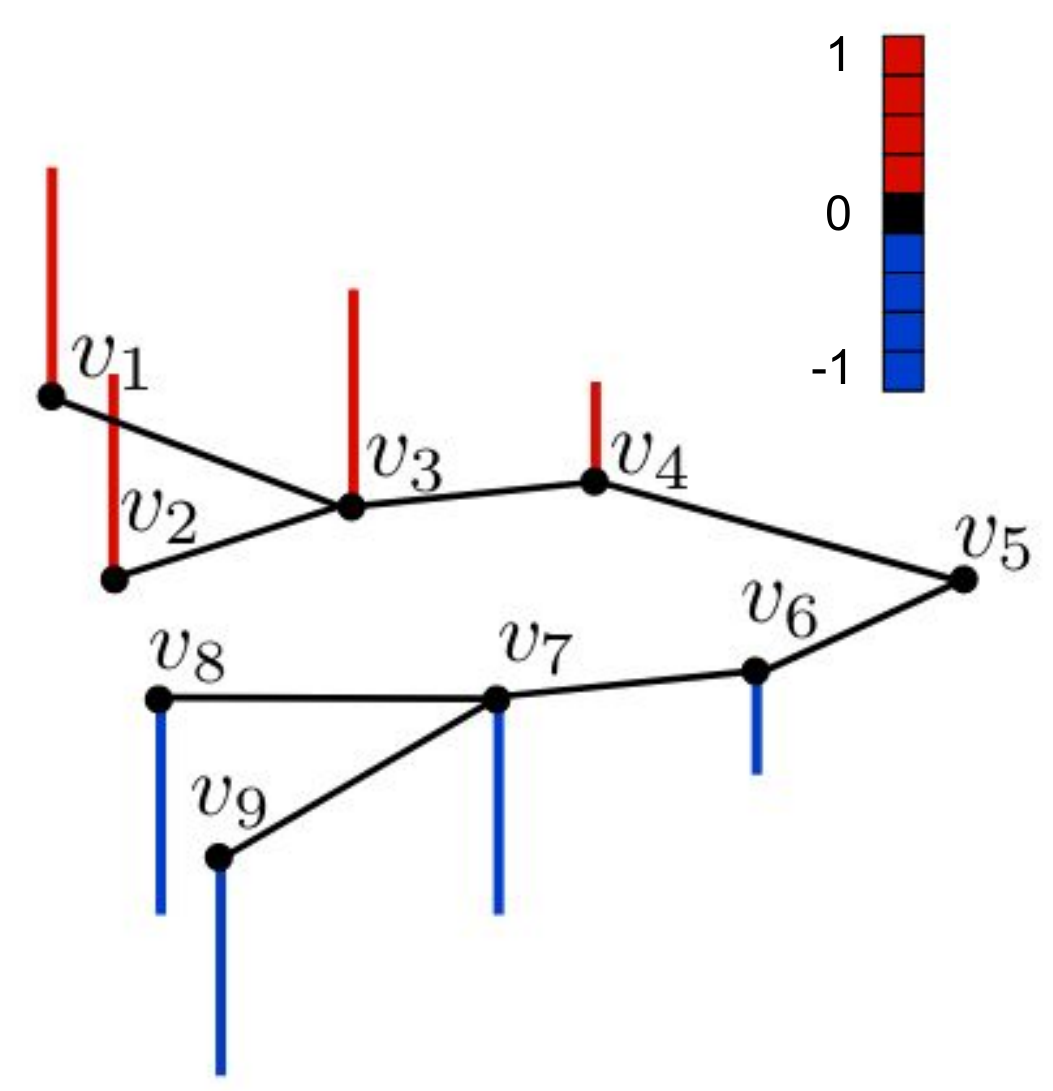}
	\caption{\small \label{fig:back_graph_signal}An Illustration of Graph signal over a graph. Refer to Section \ref{sec:prelim} for details.}  
\end{figure*}

\noindent \textbf{Graph Laplacian ($\Delta$)} for any graph signal $f$ is defined as:
\[
(\Delta f)(i) = \sum_{j \in \m{N}_i}W_{i,j} [f(i)-f(j)],
\]
where $\m{N}_i$ is the set of immediate neighbors of vertex $i$ in $\m{G}$. Graph Laplacian measures the difference between $f$ and its local average. It is small for a smooth signal, i.e., connected vertices have similar values and is large when $f$ frequently oscillates between connected vertices. Graph Laplacian can be represented as Laplacian matrix, i.e.,
\[
\Delta = D - \bm{W},
\]
where $D$ is a degree matrix, i.e., $D = diag\left(\sum_{i \neq j}W_{i,j}\right)$. $\Delta$ is a real symmetric matrix, therefore, it has a complete set of orthonormal eigenvectors which we denote by $\{\phi_0, \phi_1, ..., \phi_{N-1}\}$. Moreover, all its eigenvalues are real and non-negative, i.e., $\lambda_0, \lambda_1, ..., \lambda_{N-1} \geq 0$. Further, graph Laplacian $(\Delta)$ can be decomposed (Spectral Decomposition) as 
\[
\Delta = \Phi^T\Lambda\Phi,
\]
where $\Phi = [\phi_0, \phi_1, ..., \phi_{N-1}]$ and $\Lambda = diag(\lambda_0, \lambda_1, ..., \lambda_{N-1})$. In the graph setting, eigenvalues and eigenvectors provide a notion of frequency. The eigenvector corresponding to smaller eigenvalues are smoother compared to the eigenvectors with larger eigenvalues. For instance, if we count cross edges ($\m{Z}_{\m{G}}(f))$, i.e.,  the number of edges connecting vertices with opposite signal value which is defined as:
\[
\m{Z}_{\m{G}}(f)) = \{e = (i,j) \in \m{E} : f(j)f(j) < 0 \},
\]
then, we obtain the plot as shown in Figure \ref{fig:back_eigenvalues}. This shows that with the increase in eigenvalue, the number of such edges also increases. 

\begin{figure*}[h]
	\centering
	\includegraphics[width=0.8\textwidth]{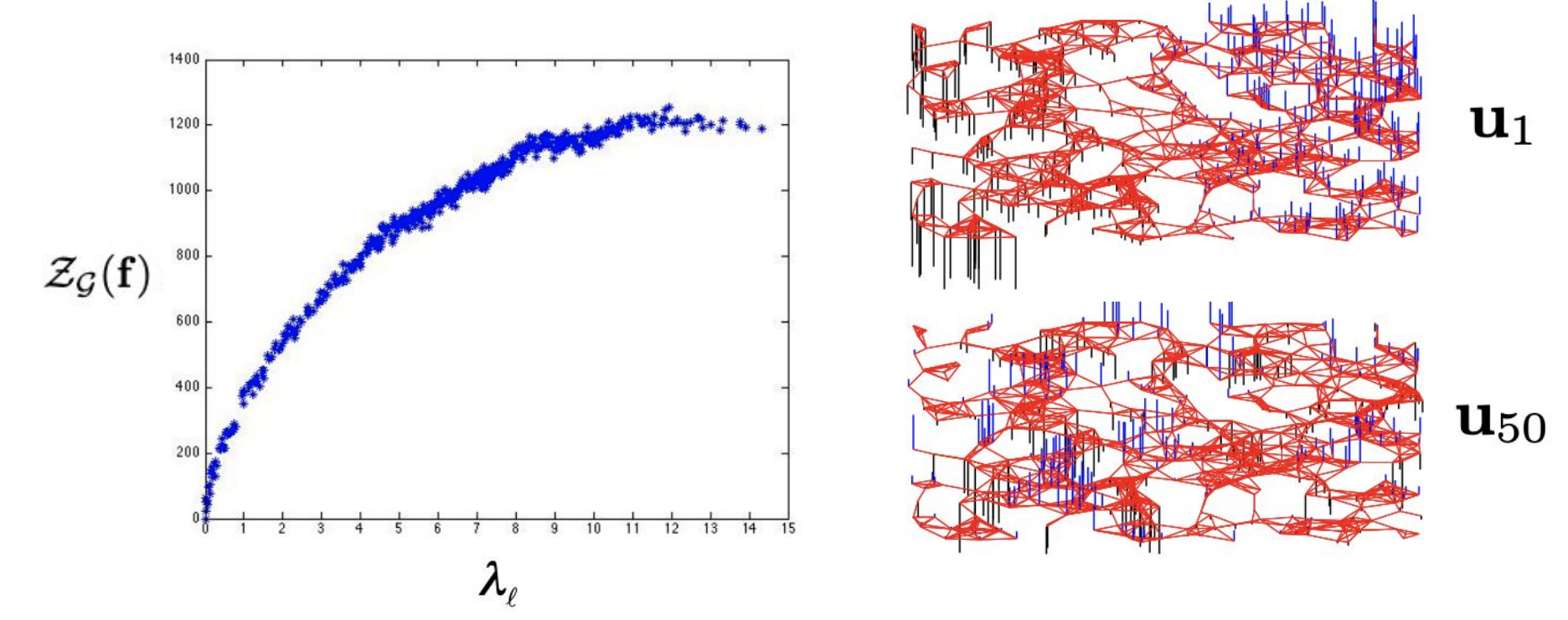}
	\caption{\small \label{fig:back_eigenvalues} \textbf{(Left)} Shows change in cross edges with the increase in eigenvalues of graph Laplacian. \textbf{(Right)} demonstrates that an eigenvector corresponding to a smaller eigenvalue is smoother compared to the eigenvector corresponding to a larger eigenvalue. }  
\end{figure*}

\section{Convolution in Euclidean space}
Given two functions $f, g : [-\pi, \pi] \rightarrow \real{}$, their convolution is defined as
\begin{equation}
\label{eqn:conv_def}
(f \star g)(x) = \int_{-\pi}^{\pi} f(t) g(x-t) dt,
\end{equation}

the above formulation satisfies the following properties:
\begin{enumerate}
	\item \textbf{Shift-invariance:} The convolution result remains unchanged on translating either of the function, i.e., $f(x-x_0) \star g(x) = (f \star g )(x-x_0)$.
	\item \textbf{Convolutional Theorem:} Fourier transform diagonalizes the convolution operator which allows it to be computed in the Fourier domain as 
	\[
	\widehat{(f \star g)} = \hat{f} \cdot \hat{g},
	\]
	where $\hat{\cdot}$ indicates the fourier transform of the function. Similarly, the convolution of two vectors $\bold{f} = (f_1, f_2, ..., f_n)$ and $\bold{g} = (g_1, g_2, .., g_n)$ can be defined as
	\begin{equation}
	\bold{f} \star \bold{g} = \bm{\Phi}(\bm{\Phi}^T\bold{g} \circ \bm{\Phi}^T\bold{f}) \\
	\end{equation}
	\item \textbf{Computational Efficient:} Using Fast-Fourier Transform (FFT) \cite{fft}, the Fourier transform can be computed efficiently in $O(\log{n})$.
\end{enumerate}

\section{Convolution in non-Euclidean space}
\label{sec:undirected_gcn}

The definition of convolution as given in Equation \ref{eqn:conv_def} cannot be directly generalized for the graph setting. This is because translation, $f(x-t)$ is not defined on graphs. However, by analogy, one can define convolution operation for two vectors $\bold{f}, \bold{g}: \m{V} \rightarrow \real{N}$ as 
\begin{align*}
	\bold{f} \star \bold{g} &= \bm{\Phi}(\bm{\Phi}^T\bold{g} \circ \bm{\Phi}^T\bold{f}) \\
	&= \bm{\Phi}\ \mathrm{diag}(\hat{g}_1, \hat{g}_2, ..., \hat{g}_N) \bm{\Phi}^T\bold{f} \\
	& = \bm{\Phi} \hat{g}(\bm{\Lambda})\bm{\Phi}^T\bold{f} \\
	& = \hat{g}(\bm{\Phi} \bm{\Lambda} \bm{\Phi}^T) \bold{f} \\
	& = \hat{g}(\Delta)\bold{f}
\end{align*}

The above formulation unlike for Euclidean space suffers from lack of shift invariance. Moreover, the filter coefficients $(\hat{g}(\Delta))$ depend on Fourier basis $\{\phi_1, \phi_2, ..., \phi_N\}$, which is expensive to compute $O(n^2)$ as FFT algorithm is not directly applicable. The above reduction allows to address this shortcoming, as $\hat{g}(\Lambda)$ can be approximated as a polynomial function of Laplacian eigenvalues, i.e.,
\[
\hat{g}(\Lambda)  =\sum_{k=1}^{K} \alpha_k \lambda^k, 
\]
where $\bm{\alpha} = (\alpha_1 \in \real{}, \alpha_2 \in \real{}, ..., \alpha_K \in \real{})^T$ is a vector of filter parameters. The filters represented by $K^{th}$-order polynomials of the Laplacian are exactly $K$-localized, i.e., restricted to capture $K$-hop neighborhood. Moreover, this also reduces the learning complexity to $O(K)$, the support size of the filter which is the same complexity as the standard CNNs. The above formulation, however requires $O(n^2)$ operation as $ \bm{\Phi} \hat{g}(\bm{\Lambda})\bm{\Phi}^T\bold{f}$ involves multiplication with Fourier basis. One solution for this is to use Chebyshev polynomial to parameterize $\hat{g}(\Delta)$ and recursively compute it from $\Delta$, i.e.,
\begin{equation}
\label{eqn:chebyshev}
\hat{g}(\Delta) = \sum_{k=0}^{K}\theta_k T_k(\tilde{\Delta}) \bold{f}.
\end{equation}

Here, $T_{k+1}(x) = 2xT_k(x) - T_{k-1}(x)$ with $T_0(x) = 1$ and $T_1(x)=x$. $T_k(\tilde{\Delta})$ denotes the $k^{th}$-order Chebyshev polynomial evaluated at $\tilde{\Delta} = 2 \Delta / \lambda_{max} - I_N$, Laplacian with eigenvalues constrained to $[-1, 1]$. This reduces the time complexity from $O(n^2)$ to $O(K|\m{E}|)$ as it involves $K$ multiplication with sparse $\Delta$ matrix.

\citet{Kipf2016} defines a first-order approximation of the above formulation by taking $K=1$. This reduces Equation \ref{eqn:chebyshev} to
\begin{equation*}
\hat{g}(\Delta) \bold{f} = (\theta_0 + \theta_1\tilde{\Delta})\bold{f}.
\end{equation*}
Now, approximating $\lambda_{max} \approx 2$  as neural network parameters will adapt to this change in scale during training and taking $\theta_0 = - \theta_1 = \theta$ gives $(\theta_0 + \theta_1\tilde{\Delta}) = (\theta - \theta(2 \Delta / \lambda_{max} - I_N)) = \theta(I_N - (\Delta - I_N))$. 
Thus, the above equation reduces to 
\begin{equation}
\label{eqn:kipf_1}
\hat{g}(\Delta) \bold{f} =  \theta(I_N + D^{-1/2}AD^{-1/2})\bold{f},
\end{equation}
here, $\Delta$ is replaced with $I_N - D^{-1/2}AD^{-1/2}$, the normalized Laplacian operator. Note that since $I_N + D^{-1/2}AD^{-1/2}$ has eigenvalues in range $[0, 2]$, repeated application of this operator can lead to numerical instabilities. To address this, we use re-normalization which replaces $I_N + D^{-1/2}AD^{-1/2}$ with $\tilde{D}^{-1/2} \tilde{A} \tilde{D}^{-1/2}$, where $\tilde{A} = A + I_N$ and $\tilde{D}_{ii} = \sum_{j}\tilde{A}_{ij}$. Thus, Equation \ref{eqn:kipf_1} is reduces to
\[
\hat{g}(\Delta) \bold{f} = \theta (\tilde{D}^{-1/2} \tilde{A} \tilde{D}^{-1/2})\bold{f}.
\]
The above formulation can be generalized for a graph signal $X \in \real{N \times d}$ with $d$-dimensional feature vector for every node and $F$ filters as
\begin{equation}
\label{eqn:kipf_main}
\bm{H} = f(\tilde{D}^{-1/2} \tilde{A} \tilde{D}^{-1/2}X\bm{W}),
\end{equation}
where $\bm{W} \in \real{d \times F}$ is a filter parameter, $f$ is any non-linearity and $\bm{H} \in \real{N \times F}$ is the convoluted signal matrix. For an undirected graph $\m{G}$, the above equation can be re-written as
\begin{equation}
h_{v} = f\left(\sum_{u \in \m{N}(v)}\left(\bm{W} x_{u} + b\right)\right),~~~\forall v \in \m{V} .
\end{equation}
Here, $\m{N}(v)$ refers to the set of neighbors of $v$ and and $b \in \mathbb{R}^{F}$ are learned in a task-specific setting using first-order gradient optimization. In order to capture nodes many hops away, multiple GCN layers may be stacked one on top of another. In particular, $h^{k+1}_{v}$, representation of node $v$  after $k^{th}$ GCN layer can be formulated as
\begin{equation}
h^{k+1}_{v} = f\left(\sum_{u \in \m{N}(v)}\left(W^{k} h^{k}_{u} + b^{k}\right)\right), \forall v \in \m{V} . 
\end{equation}

\section{GCNs for Directed and Labeled Graphs}
\label{sec:directed_gcn}

In this section, we consider GCN formulations over graphs where each edge is labeled as well as directed, proposed by \citet{gcn_srl}. In this setting, an edge from node $u$ to $v$ with label $l(u, v)$ is denoted as $(u, v, l(u, v))$. 
% While a few recent works focus on GCN over directed graphs \cite{gcn_summ,gcn_srl}, none of them consider labeled edges. %We handle both direction and label by incorporating label and direction specific filters.
Based on the assumption that the information in a directed edge need not only propagate along its direction, following \citet{gcn_srl} we define an updated edge set $\m{E'}$ which expands the original set $\m{E}$ by incorporating inverse, as well self-loop edges.
\[
\m{E'} = \m{E} \cup \{(v,u,l(u,v)^{-1})~|~ (u,v,l(u,v)) \in \m{E}\} \cup \{(u, u, \top)~|~u \in \m{V})\} \label{neuraldater_eqn:updated_edges}.
\]

%Here, $l(u,v)^{-1}$ is the inverse edge label corresponding to label $l(u,v)$, and $\top$ is a special empty relation symbol for self-loop edges. We now define $h_{v}^{k+1}$ as the embedding of node $v$ after $k^{th}$ GCN layer applied over the directed and labeled graph as:
\begin{equation}
\label{eqn:gcn_directed}
h_{v}^{k+1} = f \left(\sum_{u \in \m{N}(v)}\left(W^{k}_{l(u,v)}h_{u}^{k} + b^{k}_{l(u,v)}\right)\right).
\end{equation}

We note that the parameters $W^{k}_{l(u,v)}$ and $b^{k}_{l(u,v)}$ in this case are edge label specific.
\ \\

\noindent \textbf{Incorporating Edge Importance: }
In many practical settings, we may not want to give equal importance to all the edges. For example, in case of automatically constructed graphs, some of the edges may be erroneous and we may want to automatically learn to discard them. Edge-wise gating may be used in a GCN to give importance to relevant edges and subdue the noisy ones. \citet{gcn_event, gcn_srl} used gating for similar reasons and obtained large performance gain. At $k^{th}$ layer, we compute gating value for a particular edge $(u,v, l(u,v))$ as:
\[
g^{k}_{u,v} = \sigma \left( h^{k}_u \cdot \hat{w}^{k}_{l(u,v)} + \hat{b}^{k}_{l(u,v)} \right),
\]
where, $\sigma(\cdot)$ is the sigmoid function, $\hat{w}^{k}_{l(u,v)}$ and $ \hat{b}^{k}_{l(u,v)}$ are label specific gating parameters. Thus, gating helps to make the model robust to the noisy labels and directions of the input graphs. GCN embedding of a node while incorporating edge gating may be computed as follows.
\begin{equation}
\label{eqn:gated_gcn}
h^{k+1}_{v} = f\left(\sum_{u \in \m{N}(v)} g^{k}_{u,v} \times \left({W}^{k}_{l(u,v)} h^{k}_{u} + b^{k}_{l(u,v)}\right)\right).
\end{equation}

We utilize the GCN formulation for directed and labeled graph with (Equation \ref{eqn:gated_gcn}) and without edge-wise gating (Equation \ref{eqn:gcn_directed}) for most of the works in this thesis.

\part{\partOne{}}
\chapter{Open Knowledge Base Canonicalization using Embeddings and Side Information}
\label{chap_cesi}

% !TeX spellcheck = en_ZA
\section{Introduction}
\label{cesi_sec:intro}

In this chapter, we present our first solution to address the sparsity problem in Knowledge Graphs. Recent research has resulted in the development of several large \emph{Ontological} Knowledge Bases (KBs), examples include  DBpedia \cite{Auer:2007:DNW:1785162.1785216}, YAGO \cite{yago}, and Freebase \cite{freebase}. These KBs are called  ontological as the knowledge captured by them conform to a fixed ontology, i.e., pre-specified Categories (e.g., \textit{person, city}) and Relations (e.g., \textit{mayorOfCity(Person, City)}). Construction of such ontological KBs require significant human supervision. Moreover, due to the need for pre-specification of the ontology, such KB construction methods can't be quickly adapted to new domains and corpora. While other ontological KB construction approaches such as NELL \cite{NELL-aaai15} learn from limited human supervision, they  still suffers from the quick adaptation bottleneck.

In contrast, Open Information Extraction (OpenIE) methods need neither supervision nor any pre-specified  ontology. Given unstructured text documents, OpenIE methods readily extract triples of the form \textit{(noun phrase, relation phrase, noun phrase)} from them, resulting in the development of large Open Knowledge Bases (Open KBs). Examples of Open KBs include TextRunner \cite{textrunner}, ReVerb \cite{reverb}, and OLLIE \cite{ollie1,ollie2,ollie3}. While this makes OpenIE methods highly adaptable, they suffer from the following shortcoming: unlike Ontological KBs, the Noun Phrases (NPs) and \RP{}s in Open KBs are not \emph{canonicalized}. This results in storage of redundant and ambiguous facts.

%This extraction paradigm contrasts with Closed IE techniques which rely on a fixed ontology. Several KBs like Wikipedia, DBpedia \cite{Auer:2007:DNW:1785162.1785216}, YAGO \cite{yago}, and Freebase \cite{freebase} have adopted such techniques, their too much reliance on human inputs sometimes hampers their applicability and scalability in several domains. Unlike Ontological KBs, the Noun Phrases (NPs) and relations in open KBs are not canonicalized.

Let us explain the need for canonicalization through a concrete example. Please consider the two sentences below.

\begin{center}
\textit{Barack Obama was the president of US.} \\
\textit{Obama was born in Honolulu.}
\end{center}

Given the two sentences above, an OpenIE method may extract the two triples below and store them in an Open KB.

\begin{center}
\textit{(Barack Obama, was president of, US)} \\
\textit{(Obama, born in, Honolulu)}
\end{center}

Unfortunately, neither such OpenIE methods nor the associated Open KBs have any knowledge that both \textit{Barack Obama} and \textit{Obama} refer to the same person. This can be a significant problem as Open KBs will not return all the facts associated with \textit{Barack Obama} on querying for it. Such KBs will also contain redundant facts, which is undesirable. Thus, there is an urgent need to \emph{canonicalize} noun phrases (NPs) and relations in Open KBs.

In spite of its importance, canonicalization of Open KBs is a relatively unexplored problem. In \cite{Galarraga:2014:COK:2661829.2662073}, canonicalization of Open KBs is posed as a clustering problem over \emph{manually} defined feature representations. Given the costs and sub-optimality involved with manual feature engineering, and inspired by recent advances in knowledge base embedding \cite{transe,hole}, we pose canonicalization of Open KBs as a clustering over \emph{automatically learned} embeddings. We make the following contributions in this chapter.
\begin{itemize}
	\item We propose Canonicalization using Embeddings and Side Information (CESI), a novel method for canonicalizing Open KBs using learned embeddings. To the best of our knowledge,
	this is the first approach to use learned embeddings and side information
	for canonicalizing an Open KB.
	\item {CESI} models the problem of noun phrase (NP) and relation phrase canonicalization \emph{jointly} using relevant side information  %(refer Section \ref{cesi_sec:side-info}) 
	in a principled manner. This is unlike  prior approaches where \NP{} and \RP{} canonicalization were performed sequentially.
	\item We build and experiment with \myData{}, a new dataset for Open KB canonicalization. \myData{} consists of 20x more NPs than the previous biggest dataset for this task. Through extensive experiments on this and other real-world datasets, we demonstrate CESI's  effectiveness (\refsec{cesi_sec:experiments}).
\end{itemize}

CESI's source code and datasets used in the chapter are available at \texttt{\url{https://github.com/malllabiisc/cesi}}.
% !TeX spellcheck = en_US
\section{Related Work}
 \label{cesi_sec:related}

{\bf Entity Linking}: One traditional approach to canonicalizing noun phrases is to map them to an existing KB such as Wikipedia or Freebase. This problem is known as Entity Linking (EL) or Named Entity Disambiguation (NED). Most approaches generate a list of candidate entities for each NP and re-rank them using machine learning techniques. Entity linking has been an active area of research in the NLP community \cite{Trani:2014:DOS:2878453.2878558,Lin:2012:ELW:2391200.2391216,Ratinov:2011:LGA:2002472.2002642}. A major problem with these kind of approaches is that many NPs may refer to new and emerging entities which may not exist in KBs. One approach to resolve these noun phrases is to map them to NIL or an OOKB (Out of Knowledge Base) entity, but the problem still remains as to how to cluster these NIL mentions. Although entity linking is not the best approach to NP canonicalization, we still leverage signals from entity linking systems for improved canonicalization in CESI.

{\bf Canonicalization in Ontological KBs}: Concept Resolver \cite{Krishnamurthy:2011:NPD:2002472.2002545} is used for clustering NP mentions in NELL \cite{NELL-aaai15}. It makes ``one sense per category" assumption which states that a noun phrase can refer to at most one concept in each category of NELL's ontology. For example, the noun phrase ``Apple" can either refer to a company or a fruit, but it can refer to only one company and only one fruit. Another related problem to NP canonicalization is Knowledge Graph Identification \cite{Pujara:2013:KGI:2717129.2717164}, where given a noisy extraction graph, the task is to produce a consistent Knowledge Graph (KG) by performing entity resolution, entity classification and link prediction jointly. Pujara et al. \cite{Pujara:2013:KGI:2717129.2717164} incorporate information from multiple extraction sources and use ontological information to infer the most probable knowledge graph using probabilistic soft logic (PSL) \cite{Brocheler:2010:PSL:3023549.3023558}. However, both of these approaches require additional information in the form of an ontology of relations, which is not available in the Open KB setting.

{\bf Relation Taxonomy Induction}: SICTF \cite{D16-1040} tries to learn relation schemas for different OpenIE relations. It is built up on RESCAL \cite{Nickel:2011:TMC:3104482.3104584}, and uses tensor factorization methods to cluster noun phrases into \emph{categories} (such as ``person", ``disease", etc.). We, however, are interested in clustering noun phrases into entities.

There has been relatively less work on the task of relation phrase canonicalization. Some of the early works include DIRT \cite{Lin:2001:DSI:502512.502559}, which proposes an unsupervised method for discovering inference rules of the form ``\textit{X is the author of Y }  $\approx$ \textit{X wrote Y}" using paths in dependency trees; and the PATTY system \cite{Nakashole:2012:PTR:2390948.2391076}, which tries to learn subsumption rules among relations (such as \textit{son-of} $\subset$ \textit{child-of}) using techniques based on frequent itemset mining. These approaches are more focused on finding a taxonomy of relation phrases, while we are looking at finding equivalence between relation phrases.

{\bf Knowledge Base Embedding}: KB embedding techniques such as TransE \cite{transe}, HolE \cite{hole} try to learn vector space embeddings for entities and relations present in a KB. TransE makes the assumption that for any \textit{$\langle$subject, relation, object$\rangle$} triple, the relation vector is a translation from the subject vector to the object vector. HolE, on the other hand, uses non-linear operators to model a triple. These embedding methods have been successfully applied for the task of link prediction in KBs. In this work, we build up on HolE while exploiting relevant side information for the task of Open KB canonicalization. We note that, even though KB embedding techniques like HolE have been applied to ontological KBs, CESI might be the first attempt to use them in the context of Open KBs.

{\bf Canonicalizing Open KBs}: The RESOLVER system \cite{Yates:2009:UMD:1622716.1622724} uses string similarity based features to cluster phrases in TextRunner \cite{textrunner} triples. String similarity features, although being effective, fail to handle synonymous phrases which have completely different surface forms, such as \textit{Myopia} and \textit{Near-sightedness}. 

KB-Unify \cite{dellibovi-espinosaanke-navigli:2015:EMNLP} addresses the problem of unifying multiple Ontological and Open KBs into one KB. However, KB-Unify requires a pre-determined sense inventory which is not available in the setting CESI operates. 

%KB-Unify \cite{delli2015knowledge}  addresses a related problem of NP and \RP{} canonicalization using a shared semantic representation which is followed by a disambiguation step. The objective of their method is to canonicalize NPs and relations across several KBs, while our approach is focused on removing redundancy within an independent KB. 

The most closely related work to ours is \cite{Galarraga:2014:COK:2661829.2662073}. They perform NP canonicalization by performing Hierarchical Agglomerative Clustering (HAC) \cite{Tan:2005:IDM:1095618} over manually-defined feature spaces, and subsequently perform  relation phrase  clustering by using the AMIE algorithm \cite{Galarraga:2013:AAR:2488388.2488425}. CESI significantly outperforms this prior method (\refsec{cesi_sec:experiments}). %We compare CESI against their method in experiments section (Section \ref{cesi_sec:experiments}).

\section{Proposed Approach: CESI}
\label{cesi_sec:approach}

\begin{figure}
	\begin{center}
	\includegraphics[scale=0.6]{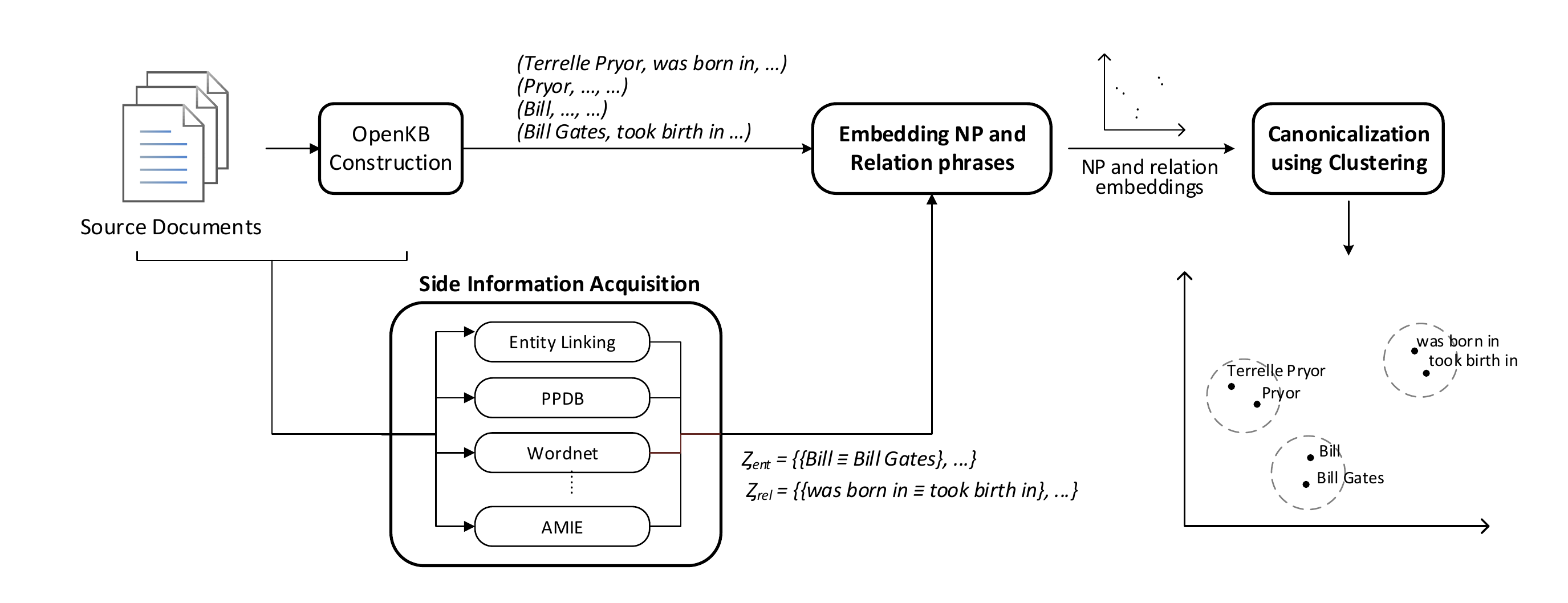}
	\caption{\label{cesi_fig:overview}\small Overview of CESI. CESI first acquires side information of noun and relation phrases of Open KB triples. In the second step, it learns embeddings of these NPs and relation phrases while utilizing the side information obtained in previous step. In the third step, CESI performs clustering over the learned embeddings to canonicalize NP and relation phrases. Please see \refsec{cesi_sec:approach} for more details.}
	\end{center}
\end{figure}

\subsection{Overview}
Overall architecture and dataflow of CESI is shown in Figure \ref{cesi_fig:overview}. The input to CESI is an un-canonicalized Open Knowledge Base (KB) with source information for each triple. The output is a list of canonicalized noun and relation phrases, which can be used to identify equivalent entities and relations or canonicalize the KB. CESI achieves this through its three step procedure:

\begin{enumerate}
	\item \textbf{\stepSide:} The goal of this step is to gather various NP and \RP{} side information for each triple in the input by running several standard algorithms on the source text of the triples. More details can be found in Section \ref{cesi_sec:side-info}.
	\item \textbf{\stepEmbed:} In this step, CESI learns specialized vector embeddings for all NPs and \RP{}s in the input by making principled use of side information available from the previous step. 
	\item \textbf{\stepCluster:} Goal of this step is to cluster the NPs and \RP{}s on the basis of their distance in the embedding space. Each cluster represents a specific entity or relation. Based on certain relevant heuristics, we assign a representative to each NP and \RP{} cluster.
\end{enumerate}

Details of different steps of CESI are described next.

% !TeX spellcheck = en_US

\subsection{\stepSide}
\label{cesi_sec:side-info}

Noun and relation phrases in Open KBs often have relevant side information in the form of useful context in the documents from which the triples were extracted. Sometimes, such information may also be present in other related KBs. Previous Open KB canonicalization methods \cite{Galarraga:2014:COK:2661829.2662073} ignored such available side information and performed canonicalization in isolation focusing only on the Open KB triples. 
%The idea behind using side information for canonicalization is to make use of the contextual information from the source documents from which the Open IE triples were extracted. Previous  methods like \cite{Galarraga:2014:COK:2661829.2662073} only use the given triples for canonicalizing NPs and \RP{}s, although the dataset \cite{reverb} on which \cite{Galarraga:2014:COK:2661829.2662073} was evaluated comes with the source information which was ignored completely. 
CESI attempts to exploit such side information to further improve the performance on this problem. In {CESI}, we make use of five types of NP side information to get equivalence relations of the form $e_1 \equiv e_2$ between two entities $e_1$ and $e_2$. Similarly, \RP{} side information is used to derive relation equivalence, $r_1 \equiv r_2$. All equivalences are used as soft constraints in later steps of CESI (details in Section \ref{cesi_sec:embedding}).

\subsubsection{Noun Phrase side Information}
\label{cesi_subsec:np-side}

In the present version of {CESI}, we make use of the following five types of NP side information:
\begin{enumerate}
	
	\item \textbf{Entity Linking}:  
	Given unstructured text, entity linking algorithms identify entity mentions and link them to Ontological KBs such as  Wikipedia, Freebase etc. We make use of Stanford CoreNLP entity linker which is based on \cite{SPITKOVSKY12.266} for getting NP to Wikipedia entity linking. Roughly, in about 30\% cases, we get this information for NPs. If two NPs are linked to the same Wikipedia entity, we assume them to be equivalent as per this information. For example, \textit{US} and \textit{America} can get linked to the same Wikipedia entity \textit{United\_States}.
	
	\item \textbf{PPDB Information}: 
	We make use of PPDB 2.0 \cite{ppdb2}, a large collection of paraphrases in English, for identifying equivalence relation among NPs. We first extracted high confidence paraphrases from the dataset while removing duplicates. Then, using union-find, we clustered all the phrases which are stated equivalent with high confidence and randomly assigned a representative to each cluster. Using an index created over the obtained clusters, we find cluster representative for each NP. If two NPs have the same cluster representative then they are considered to be equivalent. NPs not present in the dataset are skipped. This information helps us identifying equivalence between NPs such as \textit{management} and \textit{administration}. 
	
	\item \textbf{WordNet with Word-sense Disambiguation}:
	Using word-sense disambiguation \cite{Banerjee2002} with Wordnet \cite{wordnet}, we identify possible synsets for a given NP. If two NPs share a common synset, then they are marked as similar as per this side information. For example, \textit{picture} and \textit{image} can get linked to the same synset \textit{visualize.v.01}.
	
	\item \textbf{IDF of Overlapping Toke}: 
	NPs sharing infrequent terms give a strong indication of them referring to the same entity. For example, it is very likely for \textit{Warren Buffett} and \textit{Buffett} to refer to the same person. In \cite{Galarraga:2014:COK:2661829.2662073}, IDF token overlap was found to be the most effective feature for canonicalization. We assign a score for every pair of NPs based on the standard IDF formula:
		$$score_{{idf}}(n, n') = \dfrac{\sum_{x \in w(n) \cap w(n')}{\log{(1+f(x))^{-1}}}}{\sum_{x \in w(n) \cup w(n')}{\log{(1+f(x))^{-1}}}}$$
		
	Here, $w(\cdot)$ for a given NP returns the set of its terms, excluding stop words. $f(\cdot)$ returns the document frequency for a token.  
	
	\item \textbf{Morph Normalization}: 
	We make use of multiple morphological normalization operations like tense removal, pluralization, capitalization and others as used in \cite{reverb} for finding out equivalent NPs. We show in Section \ref{cesi_sec:ablation} that this information helps in improving performance. 
	
	%For getting Named Entity Recognition (NER) type of a given NP, we use an ensemble of Stanford CoreNLP \cite{finkel2005incorporating} and Illinois NER tagger \cite{finkel2005incorporating}. The motivation for using this information is to avoid canonicalizing two NPs together which belong to two different NER types. For example, avoid canonicalizing a person with a location or an organization. 
\end{enumerate}

\subsubsection{Relation Phrase Side Information}
Similar to noun phrases, we make use of PPDB and WordNet side information for \RP{} canonicalization as well. Apart from these, we use the following two additional types of side information involving relation phrases.

\begin{enumerate}
	\item \textbf{AMIE Information}: 
	AMIE algorithm \cite{Galarraga:2013:AAR:2488388.2488425} tries to learn implication rules between two relations $r$ and $r'$ of the form $r \Rightarrow r'$. These rules are detected based on statistical rule mining \cite{Galarraga:2014:COK:2661829.2662073}. It declares two relations $r$ and $r'$ to be equivalent if both $r \Rightarrow r'$ and  $r' \Rightarrow r$ satisfy support and confidence thresholds. AMIE accepts a semi-canonicalized KB as input, i.e., a KB where NPs are already canonicalized. Since this is not the case with Open KBs, we first canonicalized NPs morphologically and then applied AMIE over the NP-canonicalized KB. We chose morphological normalization for this step as such normalization is available for all NPs, and also because we found this side information to be quite effective in large Open KBs. %We found that for large KBs, morphologically normalized NPs allows algorithm to extract some relevant knowledge which can be used as a side information. 

	\item \textbf{KBP Information}: 
	Given unstructured text, Knowledge Base Population (KBP) systems detect relations between entities and link them to relations in standard KBs. For example, \textit{``Obama was born in Honolulu"} contains \textit{``was born in"} relation between \textit{Obama} and \textit{Honolulu}, which can be linked to \textit{per:city\_of\_birth} relation in KBs. In {CESI}, we use Stanford KBP \cite{surdeanu2012multi} to categorize relations. If two relations fall in the same category, then they are considered equivalent as per this information.
	
\end{enumerate}

The given list can be further extended based on the availability of other side information. For the experiments in this chapter, we have used the above mentioned NP and \RP{} side information. Some of the equivalences derived from different side information might be erroneous, therefore, instead of using them as hard constraints, we try to use them as supplementary information as described in the next section. Even though side information might be available only for a small fraction of NPs and \RP{}s, the hypothesis is that it will result in better overall canonicalization. We find this to be true, as shown in Section \ref{cesi_sec:results}.

\subsection{\stepEmbed}
\label{cesi_sec:embedding}

For learning embeddings of NPs and \RP{}s in a given Open KB, {CESI} optimizes HolE's \cite{hole} objective function along with terms for penalizing violation of equivalence conditions from the NP and \RP{} side information. Since the conditions from side information might be spurious, a factor ($\lambda_{\text{ent}/\text{rel},\theta}$) is multiplied with each term, which acts as a hyper-parameter and is tuned on a held out validation set. We also keep a constant ($\lambda_{str}$) with HolE objective function, to make selective use of structural information from KB for canonicalization. We choose HolE because it is one of the best performing KB embeddings techniques for tasks like link prediction in knowledge graphs. Since KBs store only true triples, we generate negative examples using local closed world heuristic \cite{kg_incomp1}. To keep the rank of true triples higher than the non-existing ones, we use pairwise ranking loss function. The final objective function is described below.

\begin{equation*}
\begin{split}
\min_{\Theta} 
    & \hspace{2 mm} \lambda_{str} \sum_{i \in D_+} \sum_{j \in D_-}{ \text{max} (0, \gamma + \sigma(\eta_j) - \sigma(\eta_i))} \\    
    & +  \sum_{\theta \in \mathscr{C}_{\text{ent}}} 
    	 \dfrac{\sideConst{ent}{\theta}} {|\sidePairs{ent}{\theta}|} 
    	 \sum_{v, v^{'} \in \sidePairs{ent}{\theta}}{\|e_v - e_{v^{'}}\|^2} \\
    & +  \sum_{\phi \in \mathscr{C}_{\text{rel}}} 
    	 \dfrac{\sideConst{rel}{\phi}} {|\sidePairs{rel}{\phi}|} 
    	 \sum_{u, u^{'} \in \sidePairs{rel}{\phi}}{\|r_u - r_{u^{'}}\|^2} \\
    & + \lambda_{\text{reg}} \left( \sum_{v \in V} \|e_v\|^2 + \sum_{r \in R} \|e_r\|^2 \right).
\end{split}
\end{equation*}

The objective function, consists of three main terms, along with one term for regularization. Optimization parameter, $\Theta = \{e_v\}_{v \in V} \cup \{r_u\}_{u \in R}$, is the set of all NP ($e_v$) and \RP{} ($r_u$) $d$-dimensional embeddings, where, $V$ and $R$ denote the set of all NPs and \RP{}s in the input. In the first term, $D_+, D_-$ specify the set of positive and negative examples and $\gamma > 0$ refers to the width of the margin \cite{transe}. Further, $\sigma(\cdot)$ denotes the logistic function and for a triple $t_i$ $(s,p,o)$, $\eta_i = r_p^T(e_s \star e_o)$, where $\star : \real{d} \times \real{d} \rightarrow \real{d}$ is the circular correlation operator defined as follows.
\begin{align*}
%    &\star : R^d \times R^d \rightarrow R^d,\\
[a \star b]_{k} = \sum_{i=0}^{d-1} a_i b_{(k+i) \text{ mod} \hspace{1 mm} d}.
\end{align*}
The first index of $(a \star b)$ measures the similarity between $a$ and $b$, while other indices capture the interaction of features from $a$ and $b$, in a particular order. Please refer to \cite{hole} for more details.

In the second and third terms, $\mathscr{C}_{\text{ent}}$ and $\mathscr{C}_{\text{rel}}$ are the collection of all types of NP and relation side information available from the previous step (Section \ref{cesi_sec:side-info}), i.e., $\mathscr{C}_{\text{ent}} = \{ \text{Entity Linking, PPDB, ..} \}$ and $\mathscr{C}_{\text{rel}} = \{ \text{AMIE, KBP, ..} \}$. Further, $\sideConst{ent}{\theta}$ and $\sideConst{rel}{\phi}$ denote the constants associated with entity and relation side information. Their value is tuned using grid search on a held out validation set. The set of all equivalence conditions from a particular side information is denoted by $\sidePairs{ent}{\theta}$ and $\sidePairs{rel}{\phi}$. 
The rationale behind putting these terms is to allow inclusion of side information while learning embeddings, by enforcing two NPs or relations close together if they are equivalent as per the available side information. Since the side information is available for a fraction of NPs and \RP{}s in the input, including these terms in the objective does not slow down the training of embeddings significantly.

The last term adds L2 regularization on the embeddings. All embeddings are initialized by averaging GloVe vectors \cite{glove}. We use mini-batch gradient descent for optimization.

\subsection{\stepCluster}
\label{cesi_sec:clustering}

{CESI} clusters NPs and \RP{}s by performing Hierarchical Agglomerative Clustering (HAC) using cosine similarity over the embeddings learned in the previous step (Section \ref{cesi_sec:embedding}). HAC was preferred over other clustering methods because the number of clusters are not known beforehand. Complete linkage criterion is used for calculating the similarity between intermediate clusters as it gives smaller sized clusters, compared to single and average linkage criterion. This is more reasonable for canonicalization problem, where cluster sizes are expected to be small. The threshold value for HAC was chosen based on held out validation dataset.

%The choice of threshold parameter in HAC plays a major in its accuracy, therefore, instead of relying on user for guessing a golden value we make use of the Entity linking side information for NP clustering. We select a threshold such that it minimizes the number of violations of equivalence conditions derived using this side information. We chose to use only this side information because through emperical evaluation this information was found to be the most accurate as compared to others. Similarly, for relation clustering we make use KBP information for deciding the threshold. This approach can be followed in case the dataset is small otherwise, one can set threshold based on a held out validation set for getting better results. 
%

The time complexity of HAC with complete linkage criterion is $O(n^2)$ \cite{defays1977efficient}. For scaling up {CESI} to large knowledge graphs, one may go for modern variants of approximate Hierarchical clustering algorithms \cite{Kobren:2017:HAE:3097983.3098079} at the cost of some loss in performance. 

Finally, we decide a representative for each NP and \RP{} cluster. For each cluster, we compute a mean of all elements' embeddings weighted by the frequency of occurrence of each element in the input. NP or \RP{} which lies closest to the weighted cluster mean is chosen as the representative of the cluster.

% !TeX spellcheck = en_US
\section{Experiments}
\label{cesi_sec:experiments}

\subsection{Experimental Setup}

\subsubsection{Datasets}
\label{cesi_sec:datasets}

%\note{PPT: pls check the 15k number below. How can you have 12k triples over 15k nodes, as many nodes will be disconnected?}
\begin{table}[t]
	\centering
	\begin{tabular}{ccccc}
		\toprule
		Datasets 	& \# Gold & \#NPs 	& \#Relations 	& \#Triples \\
		& Entities & & & \\ 
		\midrule
		Base 		& 150 		& 290  & 3K 		& 9K  \\
		Ambiguous	& 446 		& 717  & 11K		& 37K \\
		\myData 	& 7.5K 		& 15.5K & 22K 		& 45K \\
		\bottomrule
		\addlinespace
	\end{tabular}
	\caption{\label{cesi_tb:datasets}\small Details of datasets used. \myData{} is the new dataset we propose in this chapter. Please see \refsec{cesi_sec:datasets} for details.}
\end{table}

%\begin{table}[thb]
%\small
%\centering
%\begin{tabular}{|l|c|c|}
%\hline
%Dataset & \# Relations & \#Triples \\
%\hline
%Reverb 5K & 57K & $\sim$ 127K  \\
%\hline
%ClusteredKB & $\sim$ 167K & $\sim$ 620K \\
%\hline
%\end{tabular}
%\caption{Relation canonicalization datasets}
%\end{table}

%For entity canonicalization, various F1 scores are compared as discussed earlier. But in case of relations, we did not have any gold labels, so evaluation had to be done manually. For relation canonicalization, a subset of clusters were sampled randomly and  Macro, Micro and Pairwise precision values were computed manually for the sampled clusters. Also, only non-singleton clusters were used to compute precision, because singleton clusters will always give a precision of 1. We sampled 50 clusters for Reverb 5K and 100 clusters for ClusteredKB.
Statistics of the three datasets used in the experiments of this chapter are summarized in \reftbl{cesi_tb:datasets}. We present below brief summary of each dataset.
%In our experiments, we have used the following datasets: 
\begin{enumerate}
	\item \textbf{Base and Ambiguous Datasets:} We obtained the Base and Ambiguous datasets from the authors of \citet{Galarraga:2014:COK:2661829.2662073}. Base dataset was created by collecting triples containing 150 sampled Freebase entities that appear with at least two aliases in ReVerb Open KB. The same dataset was further enriched with mentions of homonym entities to create the  Ambiguous dataset. Please see \citet{Galarraga:2014:COK:2661829.2662073} for more details.

	\item \textbf{\myData{}:} This is the new Open KB canonicalization dataset we propose in this work.  
	\myData{} is a significantly extended version of the Ambiguous dataset, containing more than 20x NPs. \myData{} is constructed by intersecting information from the following three sources: ReVerb Open KB \cite{reverb}, Freebase entity linking information from \cite{gabrilovich2013facc1}, and Clueweb09 corpus \cite{callan2009clueweb09}. %, we constructed our own dataset for  canonicalization evaluation. 
	Firstly, for every triple in ReVerb, we extracted the source text from Clueweb09 corpus from which the triple was generated. In this process, we rejected  triples for which we could not find any source text. Then, based on the entity linking information from \cite{gabrilovich2013facc1}, we linked all subjects and objects of triples to their corresponding Freebase entities. If we could not find high confidence linking information for both subject and object in a triple, then it was rejected. Further, following the dataset construction procedure adopted by \cite{Galarraga:2014:COK:2661829.2662073}, we selected triples associated with all Freebase entities with at least two aliases occurring as subject in our dataset. Through these steps, we obtained 45K high-quality triples which we used for evaluation. We call this resulting dataset {\myData}.
	
	In contrast to Base and Ambiguous datasets, the number of entities, NPs and \RP{}s in \myData{} are significantly larger. %$17$ times more and facts per entity is considerably smaller (more detailed 
	Please see \reftbl{cesi_tb:datasets} for a detailed comparison. This better mimics real-world KBs which tend to be sparse with very few edges per entity, as also observed by \cite{transe}.

%	This choice has been made to mimic real-world scenarios where very few facts are available per entity. This problem has been identified by \cite{west2014knowledge}. %Through these steps from 14 million ReVerb triples we obtained 118K triples out of which we selected 12K high quality triples which we have used for evaluation. 

	%\item \textbf{FB15K-237:} We also make use of \cite{toutanova2015representing} dataset in our evaluation. Since, the dataset does not come with unnormalized triples, we randomly replaced each entity Freebase id with its alias from Freebase datadump \cite{freebase:datadumps}. 	
\end{enumerate}

For getting test and validation set for each dataset, we randomly sampled 20\% Freebase entities and called all the triples associated with them as validation set and rest was used as the test set. 

\subsubsection{Evaluation Metrics}
\label{cesi_sec:metrics}

%For evaluating the quality of CESI's output, we make use of macro, micro and pairwise metrics as used by \cite{Galarraga:2014:COK:2661829.2662073}. Below, 
Following \cite{Galarraga:2014:COK:2661829.2662073}, we use macro-, micro- and pairwise metrics for evaluating Open KB canonicalization methods. 
We briefly describe below these metrics for completeness. In all cases, $C$ denotes the clusters produced by the  algorithm to be evaluated, and $E$ denotes the gold standard clusters. In all cases, F1 measure is given as the harmonic mean of precision and recall.

\noindent \textbf{Macro:} Macro precision ($P_{\mathrm{macro}}$) is defined as the fraction of pure clusters in $C$,  i.e., clusters in which all the NPs (or relations) are linked to the same gold entity (or relation). Macro recall ($R_{\mathrm{macro}}$) is calculated like macro precision but with the roles of $E$ and $C$ interchanged.
\begin{eqnarray*}
	P_{\mathrm{macro}}(C, E) &=& \dfrac{|\{c \in C:\exists e \in E : e \supseteq c\}|}{|C|} \\
	R_{\mathrm{macro}}(C, E) &=& P_{\mathrm{macro}}(E, C)
\end{eqnarray*}
\textbf{Micro:} Micro precision ($P_{\mathrm{micro}}$) is defined as the purity of $C$ clusters \cite{Manning:2008:IIR:1394399} based on the assumption that the most frequent gold entity (or relation) in a cluster is correct. Micro recall ($R_{\mathrm{micro}}$) is defined similarly as macro recall.
\begin{eqnarray*} 
	P_{\mathrm{micro}}(C, E) &=& \dfrac{1}{N} \sum_{c \in C} \max_{e \in E} |c \cap e| \\
	R_{\mathrm{micro}}(C, E) &=& P_{\mathrm{micro}}(E, C)
\end{eqnarray*}
\textbf{Pairwise:} Pairwise precision ($P_{\mathrm{pair}}$) is measured as the ratio of the number of hits in $C$ to the total possible pairs in $C$. 
Whereas, pairwise recall ($R_{\mathrm{pair}}$) is the ratio of number of hits in $C$ to all possible pairs in $E$. A pair of elements in a cluster in $C$ produce a hit if they both refer to the same gold entity (or relation).
\begin{eqnarray*}
	P_{\mathrm{pair}}(C, E) &=& \dfrac{\sum_{c \in C}{|\{ (v,v') \in e, \exists e \in E, \forall (v,v') \in c\}|}} {\sum_{c \in C}{\Comb{|c|}{2}}} \\ 
	R_{\mathrm{pair}}(C, E) &=& \dfrac{\sum_{c \in C}{|\{ (v,v') \in e, \exists e \in E, \forall (v,v') \in c\}|}} {\sum_{e \in E}{\Comb{|e|}{2}}} \\
\end{eqnarray*}

%\begin{table}[!t]
%	\centering
%	\begin{tabular}{cc}
%		\toprule
%		\multirow{3}{*}{Actual (E)}	& \textit{ \{ America, USA \} } $(e_1)$ \\
%		  							& \textit{ \{ New York, NY, NYC, New York city \} } $(e_2)$ \\
%									& \textit{ \{California \} } $(e_3)$ \\
%		\midrule
%		\multirow{3}{*}{Predicted (C)}	& \textit{ \{ America, USA, NY \}} $(c_1)$ \\
%										& \textit{ \{ New York, New York city, NYC \}} $(c_2)$ \\
%										& \textit{ \{ California \}} $(c_3)$ \\
%		\bottomrule
%		\addlinespace
%	\end{tabular}
%	\caption{Illustrative example for different evaluation metrics}
%	\label{cesi_tbl:metrics}
%\end{table}

\begin{figure}[t]
	\centering
	\includegraphics[width=4in]{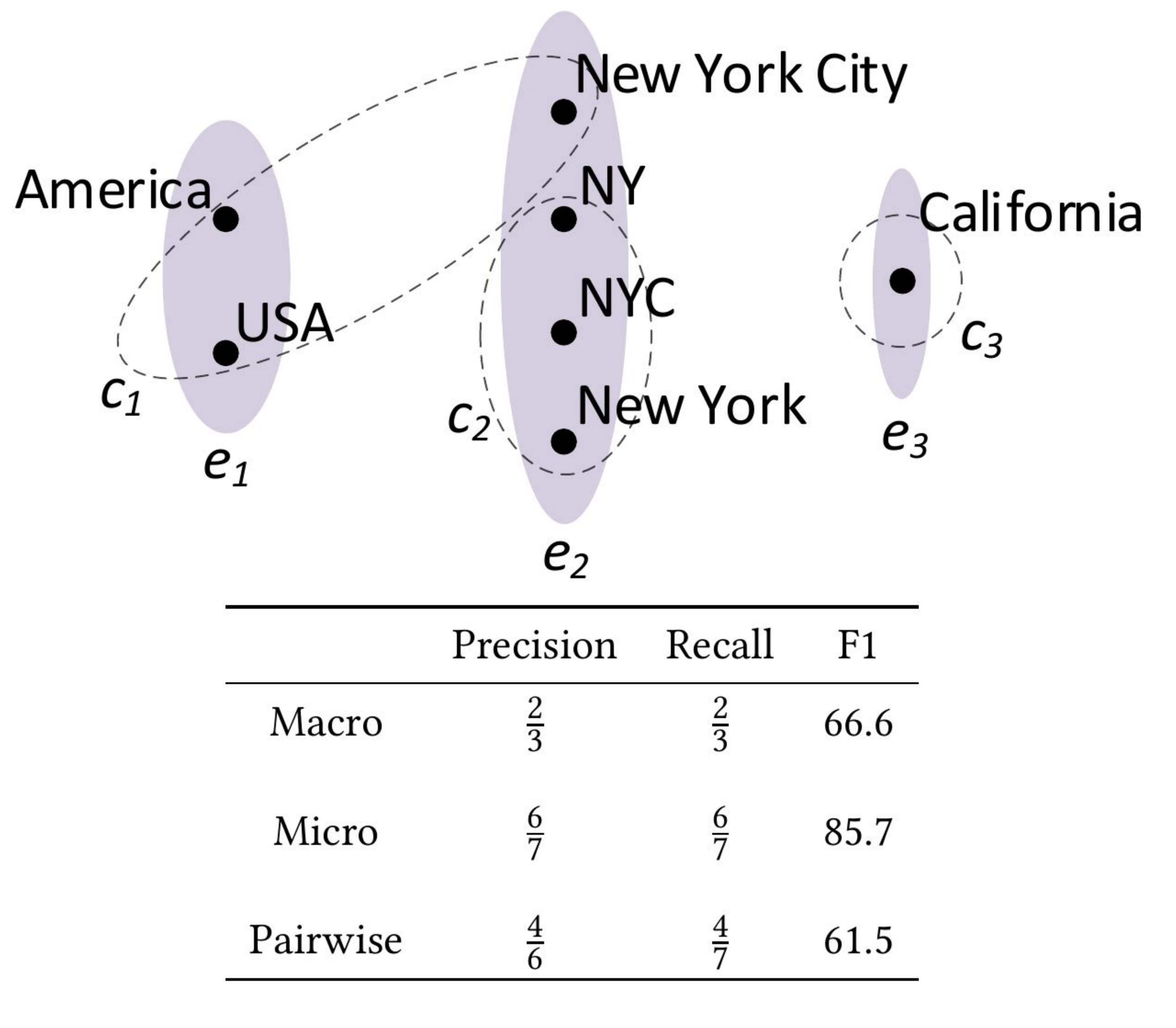}
	\caption{\label{cesi_tb:metric-example}\small
	Top: Illustrative example for different evaluation metrics. $e_i$ denotes actual clusters, whereas $c_i$ denotes predicted clusters. Bottom: Metric results for the above example. Please see \refsec{cesi_sec:metrics} for details.}
\end{figure}
%
%\begin{table}[t]
%	\begin{tabular}{cccc}
%		\toprule
%		 	& Precision 	& Recall 	& F1 \\ 
%		\midrule
%		Macro 		& $\frac{2}{3}$ 		& $\frac{2}{3}$  & 66.6 	\\ \\
%		Micro		& $\frac{6}{7}$ 		& $\frac{6}{7}$  & 85.7	\\ \\
%		Pairwise 	& $\frac{4}{6}$ 		& $\frac{4}{7}$  & 61.5 	\\
%		\bottomrule
%		\addlinespace
%	\end{tabular}
%	\caption{\label{cesi_tb:metric-example}
	%\label{cesi_fig:metric-venn}
%	Top: Illustrative example for different evaluation metrics. $e_i$ denotes actual clusters, whereas $c_i$ denotes predicted clusters. Bottom: Metric results for the above example. Please see \refsec{cesi_sec:metrics} for details. \note{PPT: reduce vertical space}}
%\end{table}

Let us illustrate these metrics through a concrete NP canonicalization example shown in Figure \ref{cesi_tb:metric-example}. In this Figure, we can see that only $c_2$ and $c_3$ clusters in C are pure because they contain mentions of only one entity, and hence, $P_{\mathrm{macro}} = \frac{2}{3}$. On the other hand, we have $e_1$ and $e_3$ as pure clusters if we interchange the roles of $E$ and $C$. So, $R_{\mathrm{macro}} = \frac{2}{3}$ in this case. For micro precision, we can see that \textit{America}, \textit{New York}, and \textit{California} are the most frequent gold entities in $C$ clusters. Hence, $P_{\mathrm{micro}} = \frac{6}{7}$. Similarly, $R_{\mathrm{micro}} = \frac{6}{7}$ in this case.
For pairwise analysis, we need to first calculate the number of hits in $C$. In $c_1$ we have 3 possible pairs out of which only 1, (\textit{America, USA}) is a hit as they belong to same gold cluster $e_1$. Similarly, we have 3 hits in $c_2$ and 0 hits in $c_3$. Hence, $P_{\mathrm{pair}} = \frac{4}{6}$.  To compute $R_{\mathrm{pair}}$, we need total number of pairwise decisions in $E$, which is $1 + 6 + 0$ , thus, $R_{\mathrm{pair}} = \frac{4}{7}$. All the results are summarized in \reftbl{cesi_tb:metric-example}.

For evaluating NP canonicalization, we use Macro, Micro and Pairwise F1 score. However, in the case of relations, where gold labels are not available, we use macro, micro and pairwise precision values based on the scores given by human judges. %Also, only non-singleton clusters were sampled, because singleton clusters will always give a precision of one.

\subsubsection{Methods Compared}
\textbf{Noun Phrase Canonicalization:}  
For \NP{} canonicalization, {CESI} has been compared against the following methods:

\begin{itemize}
	\item \textbf{Morphological Normalization:} As used in \cite{reverb}, this involves applying simple normalization operations like removing tense, pluralization, capitalization etc. over NPs and \RP{}s.
	\item {\bf Paraphrase Database (PPDB):} Using PPDB 2.0 \cite{ppdb2}, we clustered two NPs together if they happened to share a common paraphrase. NPs which could not be found in PPDB are put into singleton clusters.
	\item {\bf Entity Linking}: Since the problem of \NP{} canonicalization is closely related to entity linking, we compare our method against Stanford CoreNLP Entity Linker \cite{SPITKOVSKY12.266}. Two NPs linked to the same entity are clustered together.
	\item \textbf{\gidf{} \cite{Galarraga:2014:COK:2661829.2662073}}: IDF Token Overlap was the best performing method proposed in \cite{Galarraga:2014:COK:2661829.2662073} for NP canonicalization. In this method, IDF token similarity is defined between two NPs as in  \refsec{cesi_subsec:np-side}, and HAC is used to cluster the mentions.
	\item \textbf{\gstrsim{} \cite{Galarraga:2014:COK:2661829.2662073}}: This method is similar to Galarraga-IDF, but with similarity metric being the Jaro-Winkler \cite{winkler1999state} string similarity measure.
	\item \textbf{\gattr{} \cite{Galarraga:2014:COK:2661829.2662073}}: Again, this method is similar to the Galarraga-IDF, except that Attribute Overlap is used as the similarity metric between two NPs in this case. Attribute for a NP $n$, is defined as the set of relation-NP pairs which co-occur with $n$ in the input triples. Attribute overlap similarity between two NPs, is defined as the Jaccard coefficient of the set of attributes:
	$$ f_{\text{attr}}(n,n') = \dfrac{|A \cap A'|}{|A \cup A'|}$$
	where, $A$ and $A'$ denote the set of attributes associated with $n$ and $n'$.
	
	Since canonicalization methods using above similarity measures were found to be most effective in \cite{Galarraga:2014:COK:2661829.2662073}, even outperforming Machine Learning-based alternatives, we consider these three baselines as representatives of state-of-the-art in Open KB canonicalization. 
	
	\item {\bf GloVe}: %To evaluate the effectiveness of our proposed embedding for canonicalization, we compare it against GloVe \cite{glove}. 
	In this scheme, each NP and \RP{} is represented by a 300 dimensional GloVe embedding \cite{glove} trained on Wikipedia 2014 and Gigaword 5 \cite{parker2011english} datasets with 400k vocabulary size. Word vectors were averaged together to get embeddings for multi-word phrases. These GloVE embeddings were then clustered for final canonicalization.
	
	\item {\bf HolE}: In this method, embeddings of NPs and \RP{}s in an Open KB are obtained by applying HolE \cite{hole} over the Open KB. These embeddings are then clustered to obtain the final canonicalized groupings. Based on the initialization of embeddings, we differentiate between \textbf{HolE(Random)} and \textbf{HolE(GloVe)}. 
%	For highlighting the importance of using side information, we compare {CESI} against vanilla HolE \cite{hole} embeddings, while keeping the other steps same. Based on the initialization embeddings, we differentiate between \textbf{HolE(Random)} and \textbf{HolE(GloVe)}.
	\item {\bf CESI}: This is the method proposed in this chapter, please see \refsec{cesi_sec:approach} for more details. 

\end{itemize}

{\bf Hyper-parameters}: Following \cite{Galarraga:2014:COK:2661829.2662073}, we used Hierarchical Agglomerative Clustering (HAC) as the default clustering method across all methods (wherever necessary). For all methods, grid search over the hyperparameter space was performed, and results for the best performing setting are reported. This process was repeated for each dataset.

%\note{PPT: is this still valid?}
\subsubsection{Relation Phrase Canonicalization} AMIE \cite{Galarraga:2013:AAR:2488388.2488425} was found to be effective for \RP{} canonicalization in \cite{Galarraga:2014:COK:2661829.2662073}. We thus consider AMIE\footnote{We use support and confidence values of 2 and 0.2 for all the experiments in this chapter.} as the state-of-the-art baseline for \RP{} canonicalization and compare against CESI. We note that AMIE requires NPs of the input Open KB to be already canonicalized. In all our evaluation datasets, we already have \emph{gold} NP canonicalization available. We provide this gold NP canonicalization information as input to AMIE. Please note that CESI doesn't require such pre-canonicalized NP as input, as it performs \emph{joint} NP and \RP{} canonicalization. Moreover, providing gold NP canonicalization information to AMIE puts CESI at a disadvantage. We decided to pursue this choice anyways in the interest of stricter evaluation. However, in spite of starting from this disadvantageous position, CESI significantly outperforms AMIE in \RP{} canonicalization, as we will see in \refsec{cesi_sec:rel-eval}.

For evaluating performance of both  algorithms, we randomly sampled 25 non-singleton relation clusters for each of the three datasets and gave them to five different human evaluators\footnote{Authors did not participate in this evaluation.} for assigning scores to each cluster. The setting was kept blind, i.e., identity of the algorithm producing a cluster was not known to the evaluators. Based on the average of evaluation scores, precision values were calculated. Only non-singleton clusters were sampled, as singleton clusters will always give a precision of one.

%\note{PPT: include details on how many annotators etc. also that it was blind, hyperparameter tuning details etc.}

%!TeX spellcheck = en_US

\subsection{Results}
\label{cesi_sec:results}

In this section, we evaluate the following questions.

\begin{itemize}
	\item[Q1.] Is CESI effective in Open KB canonicalization? (\refsec{cesi_sec:overall-perf})
	\item[Q2.] What is the effect of side information in CESI's performance? (\refsec{cesi_sec:ablation})
	\item[Q3.] Does addition of entity linking side information degrade CESI's ability to canonicalize unlinked NPs (i.e., NPs missed by the entity linker)? (\refsec{cesi_sec:unlinked-eval})
\end{itemize}

Finally, in \refsec{cesi_sec:qualitative}, we present qualitative examples and discussions.

\begin{table*}[t]
	\begin{small}
		\resizebox{\textwidth}{!}{
			\begin{tabular}{lccc|ccc|ccc|c}
				\toprule
				Method & \multicolumn{3}{c}{Base Dataset} & \multicolumn{3}{c}{Ambiguous Dataset} & \multicolumn{3}{c}{\myData} & \\ 
				\cmidrule(r){2-4} \cmidrule(r){5-7} \cmidrule(r){8-10} \cmidrule(r){11-11} 
				& Macro & Micro & Pair. & Macro & Micro & Pair. & Macro & Micro & Pair. & Row Average\\
				\midrule
				Morph Norm	& 58.3 & 88.3 & 83.5 & 49.1 & 57.2 & 70.9 & 1.4  & 77.7 & 75.1 & 62.3\\
				PPDB       	& 42.4 & 46.9 & 32.2 & 37.3 & 60.2 & 69.3 & 46.0 & 45.4 & 64.2 & 49.3\\
				EntLinker   	& 54.9 & 65.1 & 75.2 & 49.7 & 83.2 & 68.8 & 62.8 & 81.8 & 80.4 & 69.1\\
				\gstrsim{} & 88.2 & 96.5 & 97.7 & 66.6 & 85.3 & 82.2 & 69.9 & 51.7 & 0.5  & 70.9\\
				\gidf{} & 94.8 & 97.9 & 98.3 & 67.9 & 82.9 & 79.3 & 71.6 & 50.8 & 0.5  & 71.5\\
				\gattr{}	& 76.1 & 51.4 & 18.1 & \textbf{82.9} & 27.7 & 8.4 & \textbf{75.1} & 20.1 & 0.2  & 40.0\\
				GloVe 		& 95.7 & 97.2 & 91.1 & 65.9 & 89.9 & 90.1 & 56.5 & 82.9 & 75.3 & 82.7\\
				HolE (Random)		& 69.5 & 91.3 & 86.6 & 53.3 & 85.0 & 75.1 & 5.4  & 74.6 & 50.9 & 65.7\\
				HolE (GloVe)    & 75.2 & 93.6 & 89.3 & 53.9 & 85.4 & 76.7 & 33.5 & 75.8 & 51.0 & 70.4\\
				CESI 	& \textbf{98.2} & \textbf{99.8} & \textbf{99.9} & 66.2 & \textbf{92.4} & \textbf{91.9} & 62.7 & \textbf{84.4} & \textbf{81.9} & \textbf{86.3}\\
				\bottomrule
				\addlinespace
			\end{tabular}
		}
		\caption{\label{cesi_tb:np_canonicalization}\small NP Canonicalization Results. CESI outperforms all other methods across datasets (Best in 7 out of 9 cases. \refsec{cesi_sec:np_results})}
	\end{small}
\end{table*}

\subsubsection{Evaluating Effectiveness of CESI in Open KB Canonicalization}
\label{cesi_sec:overall-perf}

\textbf{Noun Phrase Canonicalization:}
\label{cesi_sec:np_results}
Results for \NP{} canonicalization are summarized in \reftbl{cesi_tb:np_canonicalization}. Overall, we find that CESI performs well consistently across the datasets.
Morphological Normalization failed to give competitive performance in presence of homonymy.  
PPDB, in spite of being a vast reservoir of paraphrases, lacks information about real-world entities like people, places etc. Therefore, its performance remained weak throughout all datasets. Entity linking methods %, rather than looking at each NP in isolation, 
make use of contextual information from source text of each triple to link a NP to a KB entity. But their performance is limited because they are restricted by the entities in KB.

String similarity also gave decent performance in most cases but since they solely rely on surface form of NPs, they are bound to fail with NPs having dissimilar mentions.

Methods such as \gidf{}, \gstrsim{}, and \gattr{} performed poorly on \myData{}. Although, their performance is considerably better on the other two datasets. This is because of the fact that in contrast to Base and Ambiguous datasets, {\myData} has considerably large number of entities and comparatively fewer triples (\reftbl{cesi_tb:datasets}). \gidf{} token overlap is more likely to put two NPs together if they share an uncommon token, i.e., one with high IDF value. Hence, accuracy of the method relies heavily on the quality of document frequency estimates which may be quite misleading when we have smaller number of triples. Similar is the case with \gattr{} which decides similarity of NPs based on the set of shared attributes. Since, attributes for a NP is defined as a set of relation-NP pairs occurring with it across all triples, sparse data also results in poor performance for this method. %so in case when triples are few the attributes per NP remain sparse and hence impractical for NP canonicalization.

GloVe captures semantics of NPs and unlike string similarity it doesn't rely on the surface form of NPs. Therefore, its performance has been substantial across all the datasets. HolE captures structural information from the given triples and uses it for learning embeddings. Through our experiments, we can see that solely structural information from KB is quite effective for NP canonicalization. {CESI} performs the best across the datasets in 7 out of the 9 settings, as it incorporates the strength of all the listed methods. The superior performance of {CESI} compared to HolE clearly indicates that the side information is indeed helpful for canonicalization task. Results of GloVe, HolE and {CESI} suggest that embeddings based method are much more  effective for Open KB canonicalization.

\begin{table}[!t]
	\centering
	\small 
	\begin{tabular}{ccccc}
		\toprule
		& Macro & Micro & Pairwise & Induced \\
		& Precision & Precision & Precision &  Relation \\
		&  &  &  &  Clusters \\
		\midrule		
		\addlinespace
		\multicolumn{5}{c}{\textbf{Base Dataset}} \\
		AMIE    & 42.8 & 63.6 & 43.0 & 7 \\
		CESI & \textbf{88.0} & \textbf{93.1} & \textbf{88.1} & \textbf{210} \\	
		\addlinespace
		\hline
		\addlinespace
		\multicolumn{5}{c}{\textbf{Ambiguous Dataset}} \\ 
		AMIE    & 55.8 & 64.6 & 23.4 & 46 \\
		CESI & \textbf{76.0} & \textbf{91.9} & \textbf{80.9} & \textbf{952}\\					
		\addlinespace
		\hline
		\addlinespace 
		\multicolumn{5}{c}{\textbf{\myData}} \\ 
		AMIE 	& 69.3 & 84.2 & 66.2 & 51 \\
		CESI & \textbf{77.3} &\textbf{87.8} & \textbf{72.6} & \textbf{2116} \\
		\bottomrule
		\addlinespace
	\end{tabular}
	\caption{\label{cesi_tb:rel_canonicalization}\small Relation canonicalization results. Compared to AMIE, CESI canonicalizes more number of \RP{}s at higher precision. Please see \refsec{cesi_sec:rel-eval} for details.}
\end{table}

\ \\
\noindent \textbf{Relation Phrase Canonicalization}
\label{cesi_sec:rel-eval}
Results for \RP{} canonicalization are presented in \reftbl{cesi_tb:rel_canonicalization}. For all experiments, in spite of using  quite low values for minimum support and confidence, AMIE was unable to induce any reasonable number of non-singleton clusters (e.g., only 51 clusters out of the 22K \RP{}s in the \myData{} dataset). For relation canonicalization experiments, AMIE was evaluated on gold NP canonicalized data as the algorithm requires NPs to be already canonicalized. CESI, on the other hand, was tested on all the datasets without making use of gold NP canonicalization information. 
 
Based on the results in \reftbl{cesi_tb:rel_canonicalization}, it is quite evident that AMIE induces too few relation clusters to be of value in practical settings. 
On the other hand, {CESI} consistently performs well across all the datasets and induces  significantly larger number of clusters. %required for any real world application.

\subsubsection{Effect of Side Information in CESI}
\label{cesi_sec:ablation}

In this section, we evaluate the effect of various side information in CESI's performance. For this, we evaluated the performances of various  versions of CESI, each one of them obtained by ablating increasing amounts of side information from the full CESI model. Experimental results comparing these ablated versions on the \myData{} are presented in \reffig{cesi_fig:ablation}. From this figure, we observe that while macro performance benefits most from different forms of side information, micro and pairwise performance also show increased performance in the presence of various side information. This validates one of the central thesis of this chapter: side information, along with embeddings, can result in improved Open KB canonicalization.

\begin{table}[!t]
	\centering
	\small
	\begin{tabular}{lccc}
		\toprule
		& Macro F1 	& Micro F1 	& Pairwise F1 \\ 
		\midrule
		CESI 			& 81.7 & 87.6 & 81.5 	\\
		{CESI} w/o EL	& 81.3 & 87.3 & 80.7	\\
		\bottomrule
		\addlinespace
	\end{tabular}
	\caption{\label{cesi_tb:non-linked}\small CESI's performance in canonicalizing unlinked NPs, with and without Entity Linking (EL) side information, in the \myData{} dataset. We observe that CESI does not overfit to EL side information, and thereby helps prevent performance degradation in unlinked NP canonicalization (in fact it even helps a little). Please see \refsec{cesi_sec:unlinked-eval} for details.}
\end{table}

\subsubsection{Effect of Entity Linking Side Information on Unlinked NP}
\label{cesi_sec:unlinked-eval}

From experiments in \refsec{cesi_sec:ablation}, we find that Entity Linking (EL) side information (see \refsec{cesi_subsec:np-side}) is one of the most useful side information that CESI exploits. However, such side information is not available in case of unlinked NPs, i.e., NPs which were not linked by the entity linker. So, this naturally raises the following question: does CESI overfit to the EL side information and ignore the unlinked NPs, thereby resulting in poor canonicalization of such unlinked NPs?

In order to evaluate this question, we compared CESI's performance on unlinked NPs in the \myData{} dataset, with and without EL side information. We note that triples involving unlinked NPs constitute about 25\% of the entire dataset. Results are presented in \reftbl{cesi_tb:non-linked}. From this table, we observe that CESI doesn't overfit to EL side information, and it selectively uses such information when appropriate (i.e., for linked NPs). Because of this robust nature, presence of EL side information in CESI doesn't have an adverse effect on the unlinked NPs, in fact there is a small  gain in performance.

\begin{figure}[!t]
	%\centering
	\begin{minipage}{3.5in}
		\includegraphics[width=3.5in]{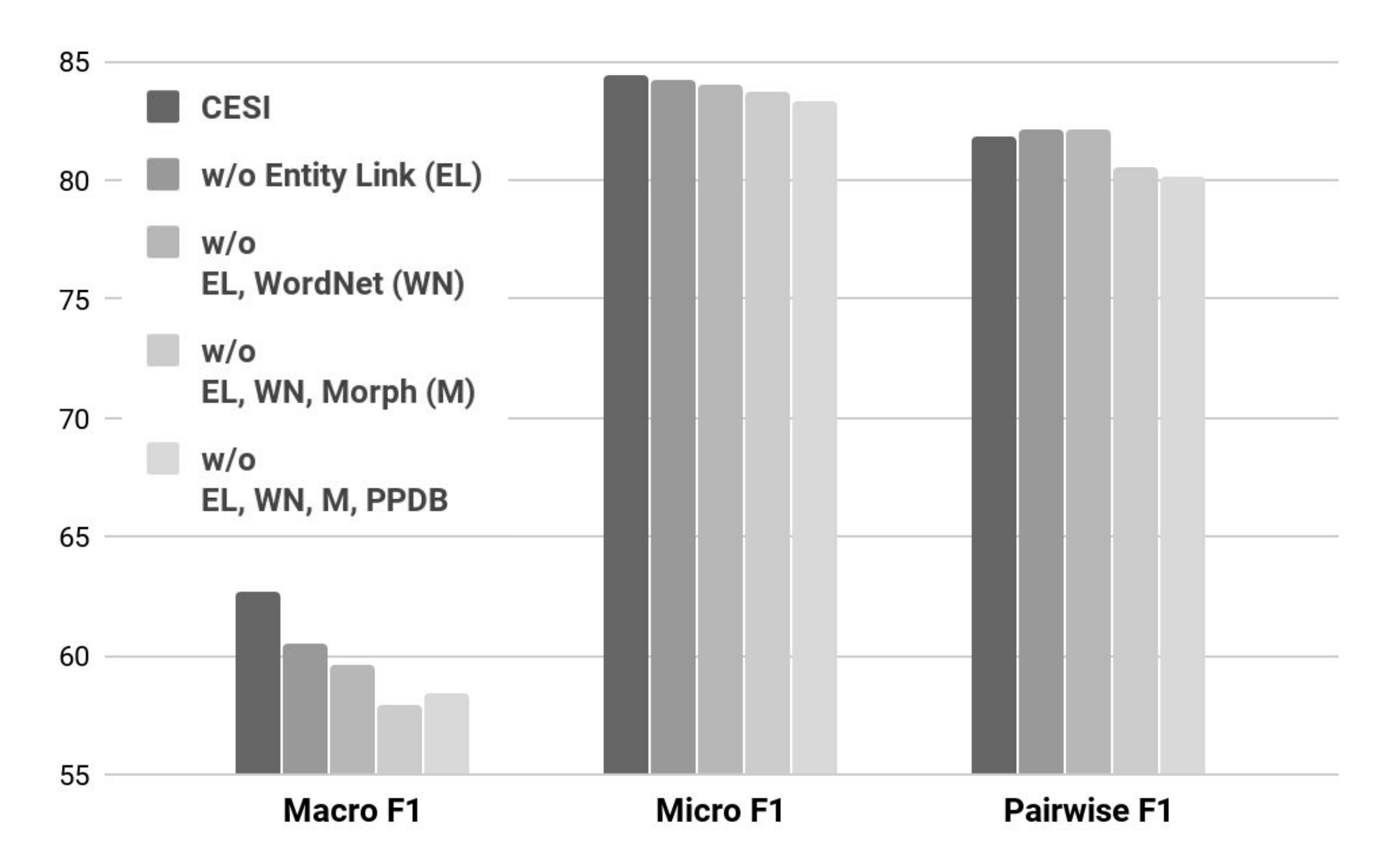}
		\subcaption{\label{cesi_fig:ablation}\small Performance comparison of various side information-ablated versions of CESI for NP canonicalization in the \myData{} dataset. Overall, side information helps CESI improve performance. Please see \refsec{cesi_sec:ablation} for details.}
	\end{minipage}
	\begin{minipage}{3.1in}
		\includegraphics[width=3.1in]{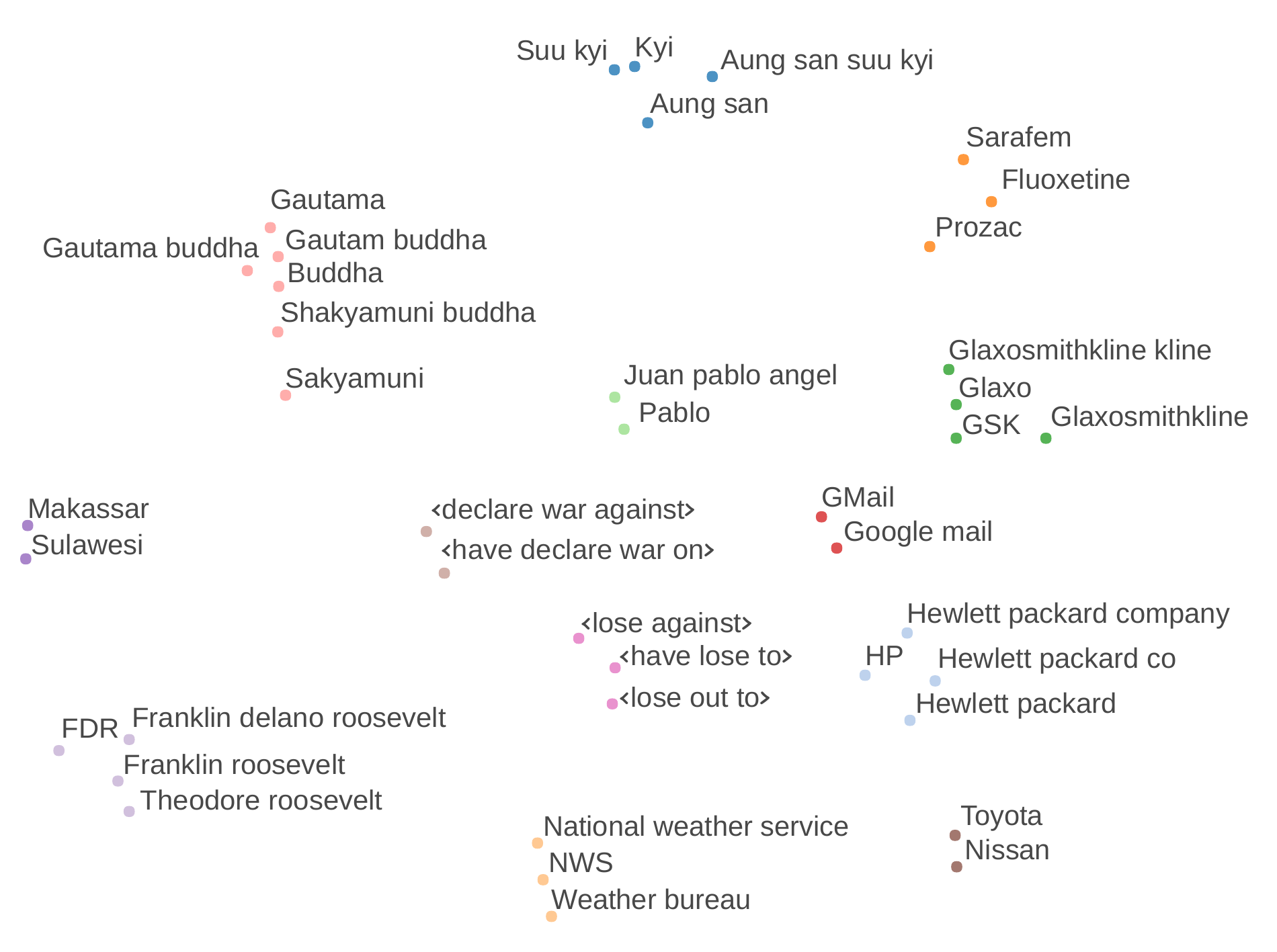}
		\subcaption{\label{cesi_fig:embeddings}t-SNE visualization of NP and relation phrase (marked in '$<\cdots>$') embeddings learned by {CESI} for {\myData} dataset. We observe that CESI is able to induce non-trivial canonical clusters. %, e.g., \{\textit{Prozac, Sarafem, Fluoxetine}\}. 
	More details in Sec. \ref{cesi_sec:qualitative}.}	
	\end{minipage}
\end{figure}

\subsubsection{Qualitative Evaluation}
\label{cesi_sec:qualitative}

Figure \ref{cesi_fig:embeddings} shows some of the NP and relation phrase clusters detected by {CESI} in {\myData} dataset. These results highlight the efficacy of algorithm in canonicalizing non-trivial NPs and relation phrases. The figure shows t-SNE \cite{maaten2008visualizing} visualization of NP and relation phrase (marked in '$<\cdots>$') embeddings for a few examples. We can see that the learned embeddings are actually able to capture equivalence of NPs and relation phrases. The algorithm is able to correctly embed \textit{Prozac}, \textit{Sarafem} and \textit{Fluoxetine} together (different names of the same drug), despite their having completely different surface forms. 

Figure \ref{cesi_fig:embeddings} also highlights the failures of {CESI}. For example, \textit{Toyota} and \textit{Nissan} have been embedded together although the two being different companies. Another case is with \textit{Pablo} and \textit{Juan Pablo Angel}, which refer to different entities. The latter case can be avoided by keeping track of the source domain type information of each NP for disambiguation. In this if we know that \textit{Juan Pablo Angel} has come from \textit{SPORTS} domain, whereas \textit{Pablo} has come from a different domain then we can avoid putting them together. We tried using DMOZ \cite{OMV93V_2016} dataset, which provide mapping from URL domain to their categories, for handling such errors. But, because of poor coverage of URLs in DMOZ dataset, we couldn't get significant improvement in canonicalization results. We leave this as a future work.

% !TeX spellcheck = en_US
\section{Conclusion}
\label{cesi_sec:conclusion}

%Due to lack of canonicalization, Open Knowledge Bases (KBs) often contain 

Canonicalizing Open Knowledge Bases (KBs) is an important but underexplored problem. In this chapter, we proposed CESI, a novel method for canonicalizing Open KBs using learned embeddings and side information. %To the best of our knowledge this is the first approach to use learned embeddings for Open KB canonicalization. 
CESI solves a joint objective to learn noun and relation phrase embeddings, while utilizing relevant side information in a principled manner. These learned embeddings are then clustered together to obtain canonicalized noun and relation phrase clusters. We also propose \myData{}, a new and larger dataset for Open KB canonicalization. Through extensive experiments on this and other real-world datasets, we demonstrate CESI's effectiveness over state-of-the-art baselines. CESI's source code and all data used in this work are publicly available.

%We propose CESI,  an approach that looks at \NP{} and \RP{} canonicalization simultaneously, and show that embeddings learnt by CESI contain rich information, which perform well on the task of canonicalization across datasets. For \RP{} canonicalization, CESI not only gives better precision, but also produces orders of magnitude more non-singleton clusters.

\chapter{Improving Distantly-Supervised  Relation Extraction using Graph Convolutional Networks and Side Information}
\label{chap_reside}

% !TeX spellcheck = en_US
\section{Introduction}

In this chapter, we present another solution to address sparsity problem in knowledge graph and also demonstrate how  relevant side information can be effectively utilized for achieving state-of-the-art results. The construction of large-scale Knowledge Bases (KBs) like Freebase \cite{freebase} and Wikidata \cite{wikidata_paper} has proven to be useful in many natural language processing (NLP) tasks like question-answering, web search, etc. However, these KBs have limited coverage. Relation Extraction (RE) attempts to fill this gap by extracting semantic relationships between entity pairs from plain text. This task can be modeled as a simple classification problem after the entity pairs are specified. Formally, given an entity pair ($e_1$,$e_2$) from the KB and an entity annotated sentence (or instance), we aim to predict the relation $r$, from a predefined relation set, that exists between $e_1$ and $e_2$. If no relation exists, we simply label it \textit{NA}.

%While supervised methods for relation extraction give good performance on the RE task, it is expensive to construct a large dataset that is required for supervision. 
Most supervised relation extraction methods require large labeled training data which is expensive to construct. Distant Supervision (DS) \cite{distant_supervision2009} helps with the construction of this dataset automatically, under the assumption that if two entities have a relationship in a KB, then all sentences mentioning those entities express the same relation. While this approach works well in generating large amounts of training instances, the DS assumption does not hold in all cases. \citet{riedel2010modeling,hoffmann2011knowledge,surdeanu2012multi} propose multi-instance learning to relax this assumption. However, they use manually defined NLP tools based features which can be sub-optimal. 

%To address the above issue of noisy data, \cite{riedel2010modeling} weakened the assumption of distant supervision by assuming that at least one of the sentences in the bag represents the relationship between the entity pair. They treat each sentence separately and convert it to a multi-instance learning problem. Later, \cite{hoffmann2011knowledge,surdeanu2012multi} extended the work to a multi-instance multi-label (MIML) setting. These models rely on elaborate handcrafted features which are expensive to construct and are prone to errors from other NLP methods.

Recently, neural models have demonstrated promising performance on RE. \citet{zeng2014relation,zeng2015distant} employ Convolutional Neural Networks (CNN) to learn representations of instances. For alleviating noise in distant supervised datasets, attention has been utilized by \cite{lin2016neural,bgwa_paper}.
Syntactic information from dependency parses has been used by \cite{distant_supervision2009,see_paper} for capturing long-range dependencies between tokens. Recently proposed Graph Convolution Networks (GCN) \cite{Defferrard2016} have been effectively employed for encoding this information \cite{gcn_srl,gcn_nmt}. However, all the above models rely only on the noisy instances from distant supervision for RE.

Relevant side information can be effective for improving RE. For instance, in the sentence, \textit{Microsoft was started by Bill Gates.}, the type information of \textit{Bill Gates (person)}  and \textit{Microsoft (organization)} can be helpful in predicting the correct relation \textit{founderOfCompany}. This is because every relation constrains the type of its target entities. 
%\citet{Liu2014ExploringFE,typeinfo2017} have utilized this information in their work. 
Similarly, relation phrase \textit{``was started by"} extracted using Open Information Extraction (Open IE) methods can be useful, given that the aliases of relation \textit{founderOfCompany}, e.g., \textit{founded, co-founded, etc.}, are available. KBs used for DS readily provide such information which has not been completely exploited by current models.

In this chapter, we propose RESIDE, a novel distant supervised relation extraction method which utilizes additional supervision from KB through its neural network based architecture.
%In this chapter, we propose RESIDE, a novel neural network based model which uses additional supervision from KB for improving neural relation extraction. 
RESIDE makes principled use of entity type and relation alias information from KBs, to impose soft constraints while predicting the relation. 
%It employs Graph Convolution Networks (GCN)  for encoding syntactic information of sentences and along with embedded side information, improves neural relation extraction.
It uses encoded syntactic information obtained from Graph Convolution Networks (GCN), along with embedded side information, to improve neural relation extraction.
%It also uses Open IE for extracting relation phrases between target entities which are utilized as a side information in the model. 
Our contributions can be summarized as follows:
\begin{itemize}[itemsep=2pt,parsep=0pt,partopsep=0pt,leftmargin=*]
%	\item We propose RESIDE, a novel Graph Convolution Network (GCN) based model, which is capable of modeling the syntactic information from instances as well as the dependencies between relations. \cite{feng2017effective}
	\item We propose RESIDE, a novel neural method which utilizes additional supervision from KB in a principled manner for improving distant supervised RE.
	\item RESIDE uses Graph Convolution Networks (GCN) for modeling syntactic information and has been shown to perform competitively even with limited side information.
%	\item RESIDE uses Graph Convolution Network (GCN) for encoding syntactic information of text and utilizes KB along with Open IE for getting relevant side information.
	\item Through extensive experiments on benchmark datasets, we demonstrate RESIDE's effectiveness over state-of-the-art baselines.
\end{itemize} 
RESIDE's source code and datasets used in the chapter are available at \url{http://github.com/malllabiisc/RESIDE}.
% !TeX spellcheck = en_GB
\section{Related Work}
\label{reside_sec:related_work}

\textbf{Distant supervision:} Relation extraction is the task of identifying the relationship between two entity mentions in a sentence. In supervised paradigm, the task is considered as a multi-class classification problem but suffers from lack of large labeled training data. To address this limitation, \cite{distant_supervision2009} propose distant supervision (DS) assumption for creating large datasets, by heuristically aligning text to a given Knowledge Base (KB). As this assumption does not always hold true, some of the sentences might be wrongly labeled. To alleviate this shortcoming, \citet{riedel2010modeling} relax distant supervision for multi-instance single-label learning. Subsequently, for handling overlapping relations between entities \cite{hoffmann2011knowledge,surdeanu2012multi} propose multi-instance multi-label learning paradigm. 

\textbf{Neural Relation Extraction:} The performance of the above methods strongly rely on the quality of hand engineered features. \citet{zeng2014relation} propose an end-to-end CNN based method which could automatically capture relevant lexical and sentence level features. This method is further improved through piecewise max-pooling by \cite{zeng2015distant}.
\citet{lin2016neural,candis_paper} use attention \cite{bahdanau2014} for learning from multiple valid sentences. We also make use of attention for learning sentence and bag representations. 

Dependency tree based features have been found to be relevant for relation extraction \cite{distant_supervision2009}. \citet{see_paper} use them for getting promising results through a recursive tree-GRU based model. In RESIDE, we make use of recently proposed Graph Convolution Networks \cite{Defferrard2016,kipf2016semi}, which have been found to be quite effective for modelling syntactic information \cite{gcn_srl,gcn_event,neuraldater}.

\textbf{Side Information in RE:} Entity description from KB has been utilized for RE \cite{entity_description}, but such information is not available for all entities. Type information of entities has been used by \citet{figer_paper,Liu2014ExploringFE} as features in their model. \citet{typeinfo2017} also attempt to mitigate noise in DS through their joint entity typing and relation extraction model. However, KBs like Freebase provide reliable type information which could be directly utilized. In our work, we make principled use of entity type and relation alias information obtained from KB. 
We also use unsupervised Open Information Extraction (Open IE) methods \cite{ollie,stanford_openie}, which automatically discover possible relations without the need of any predefined ontology, which is used as a side information as defined in \refsec{reside_sec:sideinfo}.
%\input{sections/reside/background}
%\section{Task Definition}
%The task of relation extraction in multi-instance setting involves a list of bags $\{B_1, B_2, ... , B_n\}$ for training, where each bag $B_i$ is labeled with a triplet from KB and contains multiple sentences (instances) $\{s_i^1, s_i^2, ... s_i^m\}$. The objective is to predict the relation in unseen bags between the given entity pair.

\begin{figure*}[t]
	\centering
	\includegraphics[width=\textwidth]{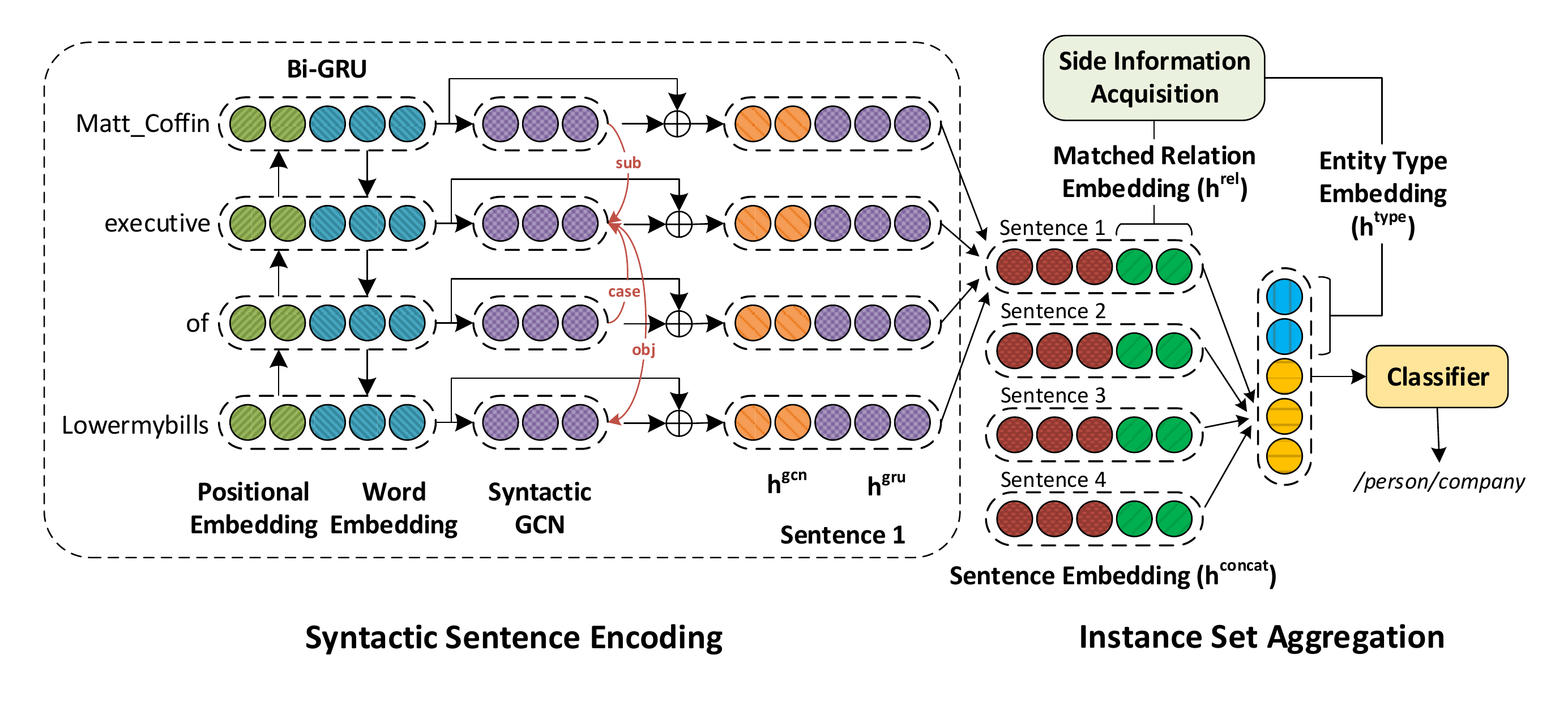}
	\caption{\small \label{reside_fig:model_overview}Overview of RESIDE. RESIDE first encodes each sentence in the bag by concatenating embeddings (denoted by $\oplus$) from Bi-\lstm{} and Syntactic GCN for each token, followed by word attention (denoted by arrows). Then, sentence embedding is concatenated with relation alias information, which comes from the Side Information Acquisition Section (\reffig{reside_fig:rel_alias}), before computing attention over sentences. Finally, bag representation with entity type information is fed to a softmax classifier. Please see \refsec{reside_sec:details} for more details.}  
\end{figure*}

\section{Proposed Method: RESIDE}
\label{reside_sec:details}

\subsection{Overview}
\label{reside_sec:overview}

In multi-instance learning paradigm, we are given a bag of sentences (or instances) $\{s_1, s_2, ... s_n\}$ for a given entity pair, the task is to predict the relation between them.
RESIDE consists of three components for learning a representation of a given bag, which is fed to a softmax classifier. We briefly present the components of RESIDE below. Each component will be described in detail in the subsequent sections. The overall architecture of RESIDE is shown in \reffig{reside_fig:model_overview}.

\begin{enumerate}[itemsep=2pt,parsep=0pt,partopsep=0pt,leftmargin=*]
	\item \textbf{\stepOne{}:} RESIDE uses a Bi-\lstm{} over the concatenated positional and word embedding for encoding the local context of each token. For capturing long-range dependencies, GCN over dependency tree is employed and its encoding is appended to the representation of each token. Finally, attention over tokens is used to subdue irrelevant tokens and get an embedding for the entire sentence. More details in \refsec{reside_sec:sent_encoder}.
	
	\item \textbf{\stepTwo{}:} In this module, we use additional supervision from KBs and utilize Open IE methods for getting relevant side information. This information is later utilized by the model as described in \refsec{reside_sec:sideinfo}.
	
	\item \textbf{\stepThree{}:} In this part, sentence representation from syntactic sentence encoder is concatenated with the \textit{matched relation embedding} obtained from the previous step. Then, using attention over sentences, a representation for the entire bag is learned. This is then concatenated with \textit{entity type embedding} before feeding into the softmax classifier for relation prediction. Please refer to \refsec{reside_sec:rel_ext} for more details.
\end{enumerate}
%\section{RESIDE Details}
%\label{reside_sec:details}

Below, we provide the detailed description of the components of RESIDE.

\subsection{\stepOne{}}
\label{reside_sec:sent_encoder}
For each sentence in the bag $s_i$ with $m$ tokens $\{w_1, w_2, ... w_m\}$, we first represent each token by $k$-dimensional GloVe embedding \cite{glove}. For incorporating relative position of tokens with respect to target entities, we use $p$-dimensional position embeddings, as used by \cite{zeng2014relation}. The combined token embeddings are stacked together to get the sentence representation $\m{H} \in \mathbb{R}^{m \times (k+2p)}$. Then, using Bi-\lstm{} \cite{gru_paper} over $\m{H}$, we get the new sentence representation $\m{H}^{gru} \in \mathbb{R}^{m \times d_{gru}}$, where $d_{gru}$ is the hidden state dimension. Bi-\lstm{}s have been found to be quite effective in encoding the context of tokens in several tasks \cite{seq2seq,rnn_speech_recog}.

Although Bi-\lstm{} is capable of capturing local context, it fails to capture long-range dependencies which can be captured through dependency edges. Prior works \cite{distant_supervision2009,see_paper} have exploited features from syntactic dependency trees for improving relation extraction. Motivated by their work, we employ Syntactic Graph Convolution Networks for encoding this information. For a given sentence, we generate its dependency tree using Stanford CoreNLP \cite{stanford_corenlp}. We then run GCN over the dependency graph and use Equation \ref{eqn:gated_gcn} for updating the embeddings, taking $\m{H}^{gru}$ as the input. Since dependency graph has 55 different edge labels, incorporating all of them over-parameterizes the model significantly. Therefore, following \cite{gcn_srl,gcn_event,neuraldater} we use only three edge labels based on the direction of the edge \{\textit{forward ($\rightarrow$), backward ($\leftarrow$), self-loop ($\top$)}\}. We define the new edge label $L_{uv}$ for an edge $(u,v,l_{uv})$ as follows:
\[ 
	L_{uv} =
	\begin{cases} 
		\rightarrow & \text{if edge exists in dependency parse} \\
		\leftarrow 	& \text{if edge is an inverse edge}	\\
		\top 		& \text{if edge is a self-loop}
	\end{cases}
\]
%\begin{itemize}[itemsep=2pt,parsep=0pt,partopsep=0pt,leftmargin=*]
%	\item $L_{uv} = \rightarrow$ if edge exists in the original dependency parse
%	\item $L_{uv} = \leftarrow$  if the edge is an inverse edge	
%	\item $L_{uv} = \top$ if the edge is a self-loop.
%\end{itemize}

For each token $w_i$, GCN embedding $h^{gcn}_{i_{k+1}} \in \mathbb{R}^{d_{gcn}}$ after $k^{th}$ layer is defined as:
\[
h^{gcn}_{i_{k+1}} = f \Bigg(\sum_{u \in \m{N}(i)} g_{iu}^{k} \times \left(W_{L_{iu}}^{k}h^{gcn}_{u_{k}} + b_{L_{iu}}^{k}\right) \Bigg).
\]
Here, $g_{iu}^{k}$ denotes edgewise gating as defined in Equation \ref{eqn:gated_gcn} and $L_{iu}$ refers to the edge label defined above. We use ReLU as activation function $f$, throughout our experiments. The syntactic graph encoding from GCN is appended to Bi-\lstm{} output to get the final token representation, $h^{concat}_i$ as $[h^{gru}_i; h^{gcn}_{i^{k+1}}]$.
Since, not all tokens are equally relevant for RE task, we calculate the degree of relevance of each token using attention as used in \cite{bgwa_paper}. For token $w_i$ in the sentence, attention weight $\alpha_{i}$ is calculated as:
$$ \alpha_{i} = \dfrac{\text{exp}(u_{i})}{\sum_{j=1}^{m}{\text{exp}(u_{j})}} \, \, \text{ where, } u_{i} = h^{concat}_{i} \cdot r .$$
where $r$ is a random query vector and $u_{i}$ is the relevance score assigned to each token. Attention values $\{\alpha_{i}\}$ are calculated by taking softmax over $\{u_{i}\}$. The representation of a sentence is given as a weighted sum of its tokens,
$ s = \sum_{j=1}^{m}{\alpha_{i}h^{concat}_i}$.

\subsection{\stepTwo{}}
\label{reside_sec:sideinfo}

Relevant side information has been found to improve performance on several tasks \cite{figer_paper,cesi_paper}. In distant supervision based relation extraction, since the entities are from a KB, knowledge about them can be utilized to improve relation extraction. Moreover, several unsupervised relation extraction methods (Open IE) \cite{stanford_openie,ollie} allow extracting relation phrases between target entities without any predefined ontology and thus can be used to obtain relevant side information. In RESIDE, we employ Open IE methods and additional supervision from KB for improving neural relation extraction. 

\begin{figure*}[t]
	\centering
	\includegraphics[width=\textwidth]{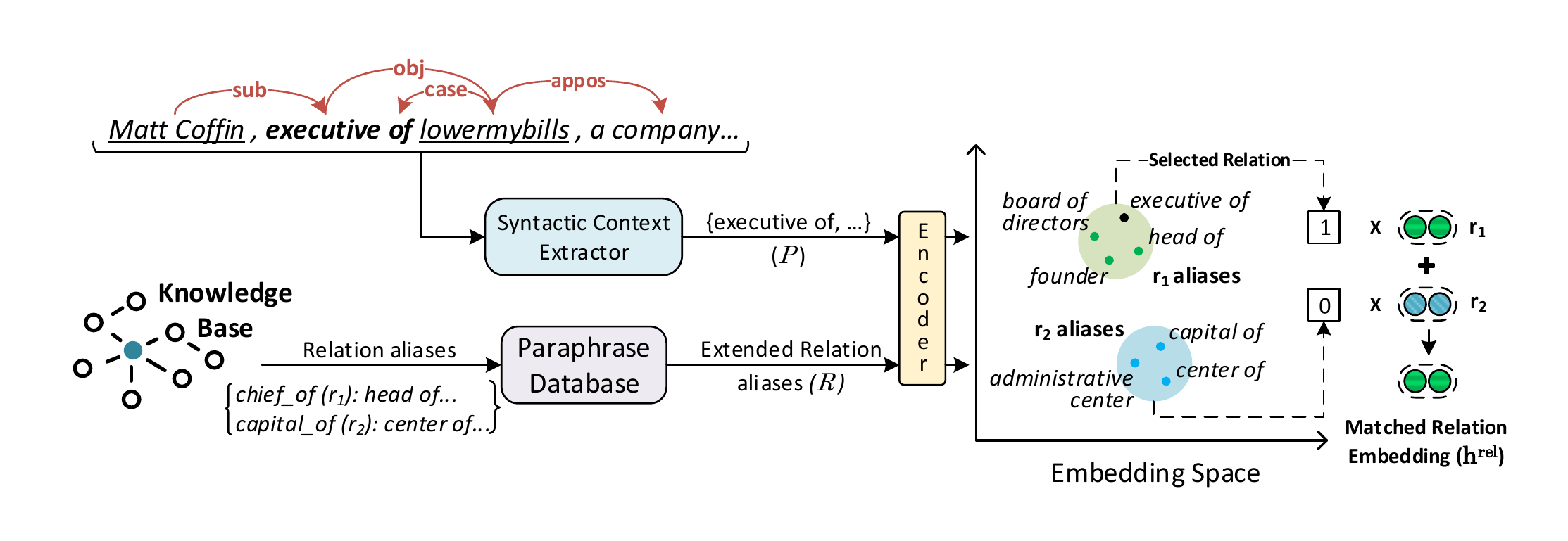}
	\caption{\label{reside_fig:rel_alias}\small
		Relation alias side information extraction for a given sentence. First, Syntactic Context Extractor identifies relevant relation phrases $\m{P}$ between target entities. They are then matched in the embedding space with the extended set of relation aliases $\m{R}$ from KB. Finally, the relation embedding corresponding to the closest alias is taken as relation alias information. Please refer \refsec{reside_sec:sideinfo}. } 
\end{figure*}

\subsubsection*{Relation Alias Side Information}
\label{reside_sec:alias_sideinfo}
%Among the set of sentences generated through distant supervision, few of them might be wrongly labeled. To alleviate this noise, we build upon the observation of \cite{semantic_eq_paper}, which identify that semantically similar sentences express the same relation. In RESIDE, we try to group semantically similar sentences together by bringing them closer in the embedding space.

RESIDE uses Stanford Open IE \cite{stanford_openie} for extracting relation phrases between target entities, which we denote by $\m{P}$. As shown in \reffig{reside_fig:rel_alias}, for the sentence \textit{Matt Coffin, executive of lowermybills, a company..}, Open IE methods extract \textit{``executive of"} between \textit{Matt Coffin} and \textit{lowermybills}. Further, we extend $\m{P}$ by including tokens at one hop distance in dependency path from target entities. Such features from dependency parse have been exploited in the past by \cite{distant_supervision2009,see_paper}. The degree of match between the extracted phrases in $\m{P}$ and aliases of a relation can give important clues about the relevance of that relation for the sentence. %The extracted phrases in $\m{P}$ give a strong indication that relation could be \textit{chief\_of} if its aliases are available. 
Several KBs like Wikidata provide such relation aliases, which can be readily exploited. In RESIDE, we further expand the relation alias set using Paraphrase database (PPDB) \cite{ppdb2}. We note that even for cases when aliases for relations are not available, providing only the names of relations give competitive performance. We shall explore this point further in  \refsec{reside_sec:results_rel_side}.

For matching $\m{P}$ with the PPDB expanded relation alias set $\m{R}$, we project both in a $d$-dimensional space using GloVe embeddings \cite{glove}. Projecting phrases using word embeddings helps to further expand these sets, as semantically similar words are closer in embedding space \cite{word2vec,glove}. Then, for each phrase $p \in \m{P}$, we calculate its cosine distance from all relation aliases in $\m{R}$ and take the relation corresponding to the closest relation alias as a matched relation for the sentence. We use a threshold on cosine distance to remove noisy aliases. In RESIDE, we define a $k_r$-dimensional embedding for each relation which we call as \textit{matched relation embedding} $(h^{rel})$. For a given sentence, $h^{rel}$ is concatenated with its representation $s$, obtained from syntactic sentence encoder (\refsec{reside_sec:sent_encoder}) as shown in \reffig{reside_fig:model_overview}. For sentences with $|\m{P}|>1$, we might get multiple matched relations. In such cases, we take the average of their embeddings. We hypothesize that this helps in improving the performance and find it to be true as shown in \refsec{reside_sec:results}. 

\subsubsection*{Entity Type Side Information}
\label{reside_sec:type_sideinfo}
Type information of target entities has been shown to give promising results on relation extraction \cite{figer_paper,typeinfo2017}. Every relation puts some constraint on the type of entities which can be its subject and object. For example, the relation \textit{person/place\_of\_birth} can only occur between a \textit{person} and a \textit{location}. Sentences in distance supervision are based on entities in KBs, where the type information is readily available.

In RESIDE, we use types defined by FIGER \cite{figer_paper} for entities in Freebase. For each type, we define a $k_t$-dimensional embedding which we call as \textit{entity type embedding} ($h^{type}$). For cases when an entity has multiple types in different contexts, for instance, \textit{Paris} may have types \textit{government} and \textit{location}, we take the average over the embeddings of each type. We concatenate the \textit{entity type embedding} of target entities to the final bag representation before using it for relation classification. To avoid over-parameterization, instead of using all fine-grained 112 entity types, we use 38 coarse types which form the first hierarchy of FIGER types.

\subsection{\stepThree{}}
\label{reside_sec:rel_ext}

For utilizing all valid sentences, following \cite{lin2016neural,bgwa_paper}, we use attention over sentences to obtain a representation for the entire bag. Instead of directly using the sentence representation $s_i$ from \refsec{reside_sec:sent_encoder}, we concatenate the embedding of each sentence with \textit{matched relation embedding} $h_i^{rel}$ as obtained from \refsec{reside_sec:sideinfo}. 
%This helps to bring the embeddings of semantically similar sentences closer which allows learning attention weight over sentences more effectively. 
The attention score $\alpha_i$ for $i^{th}$ sentence is formulated as:
$$ \alpha_{i} = \dfrac{\text{exp}(\hat{s}_i \cdot q)}{\sum_{j=1}^{n}{\text{exp}(\hat{s}_j \cdot q)}} \, \, \text{ where, } \hat{s}_{i} = [s_{i};h^{rel}_i] .$$

here $q$ denotes a random query vector. The bag representation $\m{B}$, which is the weighted sum of its sentences, is then concatenated with the \textit{entity type embeddings} of the subject ($h^{type}_{sub}$) and object ($h^{type}_{obj}$) from \refsec{reside_sec:sideinfo} to obtain $\hat{\m{B}}$.
%Here, $q$ denotes the query vector. The bag representation $\m{B}$ is defined as weighted sum of its sentences.
$$ \hat{\m{B}} = [\m{B};h^{type}_{sub}; h^{type}_{obj}]  \, \, \text{ where, } \m{B} = \sum_{i=1}^{n}{\alpha_i \hat{s}_i}.$$
Finally, $\hat{\m{B}}$ is fed to a softmax classifier to get the probability distribution over the relations.
$$p(y) = \mathrm{Softmax}(W \cdot \hat{\m{B}} + b).$$

%For each sentence, we have its dependency edges from . We run a Graph Convolution Network over the encodings from \lstm{} corresponding to the edges. Our hypothesis is that the GCN will be able to capture the syntactic information from the sentence which has been shown to be effective in discriminating the relation \cite{see_paper}.

%We use selective attention over the GCN layer to help the model target only the relevant parts of a sentence. Attention over the words has been shown to improve relation extraction (cite). The output of GCN is the sentence representation. 

%Not all sentences in a bag show the relation between two entities. Therefore, it is intuitive to use an attention model over the sentences to learn a bag representation. Since the data set is skewed and very few bags are there with more than 2 or 3 sentences (show exact figure), we believe that learning the sentence attention is hard. We introduce the side information of type as query vectors for sentence attention with the hypothesis that this way we can learn sentence attention better. A semantic encoding which is obtained from the alias side information of each sentence is also appended to the corresponding sentence representation before attention. This way we influence attention by the presence or absence of relation phrases that were obtained from PPDB dataset. The presence of a relation phrase in some proximity to the entities increases the likelihood of that relation being shown in the sentence.
% !TeX spellcheck = en_US

\begin{table}[t]
	\small
	\centering
	\begin{tabular}{ccccc}
		\toprule
		Datasets 	& Split & \# Sentences 	& \# Entity-pairs \\
		\midrule
		\multirow{3}{*}{\shortstack{Riedel\\(\# Relations: 53)}} & Train & 455,771 & 233,064 \\
		 & Valid & 114,317 & 58,635 \\
		 & Test  & 172,448 & 96,678 \\
		 \midrule
 		\multirow{3}{*}{\shortstack{GDS\\(\# Relations: 5)}} & Train & 11,297 & 6,498 \\
 		& Valid & 1,864 & 1,082 \\
 		& Test  & 5,663 & 3,247 \\
		\bottomrule
		\addlinespace
	\end{tabular}
	\caption{\label{reside_tb:datasets}\small Details of datasets used. Please see \refsec{reside_sec:datasets} for more details. }
\end{table}

\section{Experiments}

\subsection{Experimental Setup}
\subsubsection{Datasets}
\label{reside_sec:datasets}

In our experiments, we evaluate the models on Riedel and Google Distant Supervision (GDS) dataset. Statistics of the datasets is summarized in \reftbl{reside_tb:datasets}. Below we described each in detail.
\begin{enumerate}
	\item \textbf{Riedel:} The dataset is developed by \cite{riedel2010modeling} by aligning Freebase relations with New York Times (NYT) corpus, where sentences from the year 2005-2006 are used for creating the training set and from the year 2007 for the test set. The entity mentions are annotated using Stanford NER \cite{finkel2005incorporating} and are linked to Freebase. The dataset has been widely used for RE by \cite{hoffmann2011knowledge,surdeanu2012multi} and more recently by \cite{lin2016neural,feng2017effective,see_paper}. %For model selection, we took 20\% of the training data as validation set.
	
	\item \textbf{GIDS:} \citet{bgwa_paper} created Google IISc Distant Supervision (GIDS) dataset by extending the Google relation extraction corpus\footnote{\href{https://research.googleblog.com/2013/04/50000-lessons-on-how-to-read-relation.html}{https://research.googleblog.com/2013/04/50000-lessons-on-how-to-read-relation.html}} with additional instances for each entity pair. The dataset assures that the at-least-one assumption of multi-instance learning, holds. This makes automatic evaluation more reliable and thus removes the need for manual verification.
\end{enumerate}

%The training data contains 522,611 sentences, 281,270 entity pairs and 18,252 relation groundings extracted from the 2005-06 NYT data. The test data contains 172,448 sentences, 96,678 entity pairs and 1,950 relational groundings extracted from 2007 NYT data.

\begin{figure*}[t]
	\begin{minipage}{\textwidth}
		\captionsetup{type=figure} % -- This line added
		\centering
		\subcaptionbox{Riedel dataset}
		{\includegraphics[width=.495\textwidth]{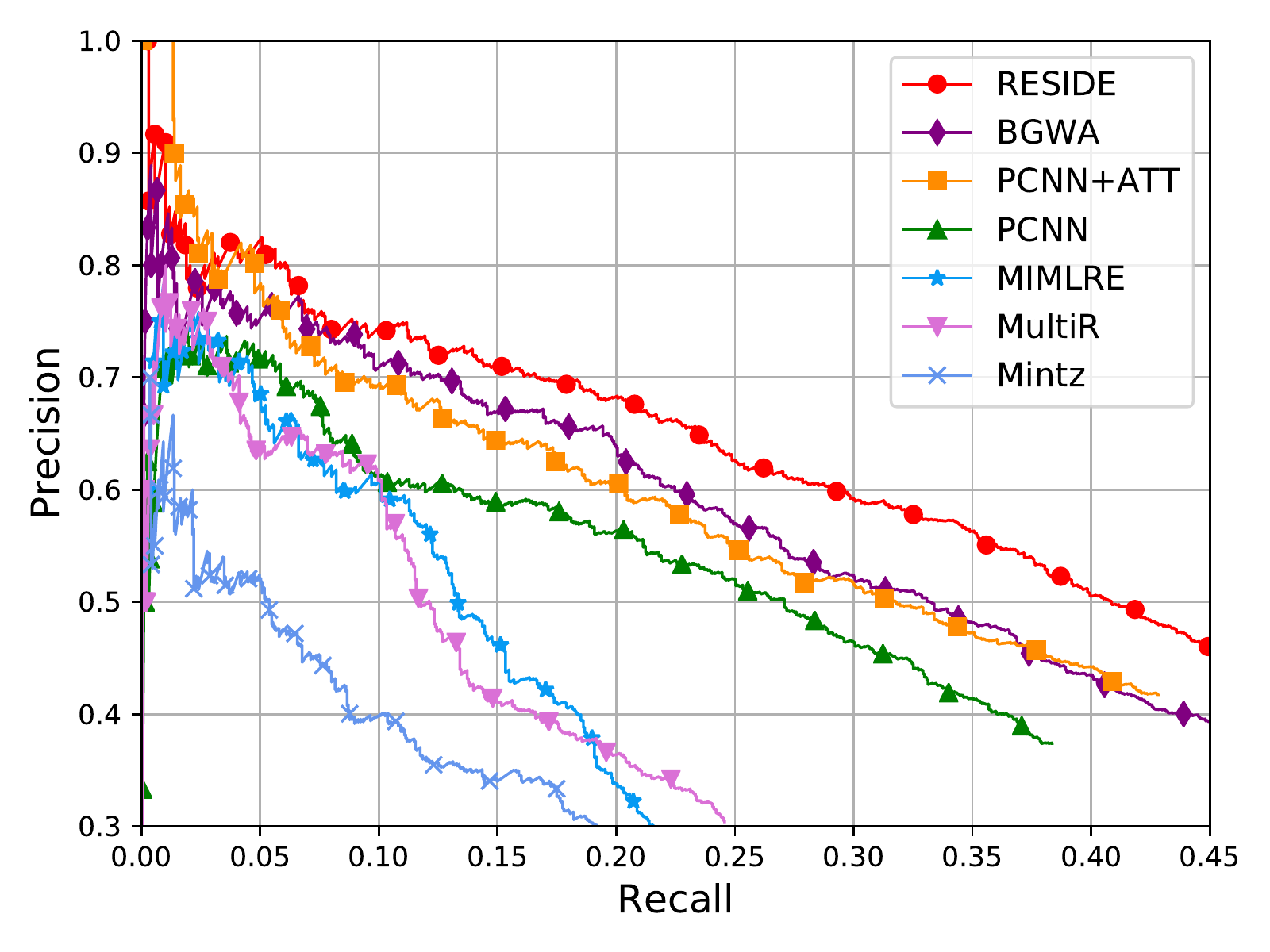}}
		\subcaptionbox{GIDS dataset}
		{\includegraphics[width=.495\textwidth]{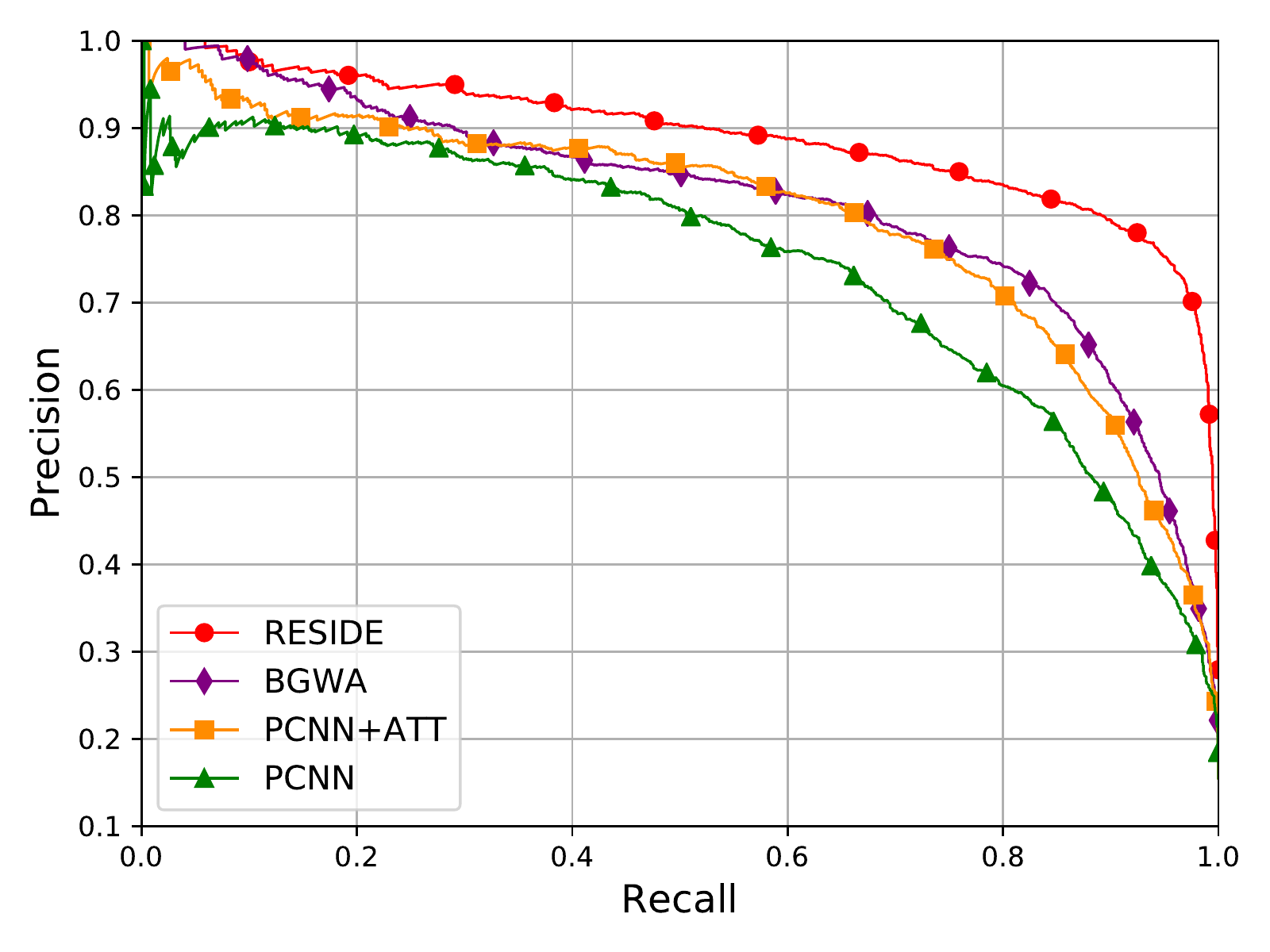}}
		\caption{\label{reside_fig:main_pr}Comparison of Precision-recall curve. RESIDE achieves higher precision over the entire range of recall than all the baselines on both datasets. Please refer \refsec{reside_sec:results_main} for more details.}
	\end{minipage}
\end{figure*}

\subsubsection{Baselines}
\label{reside_sec:baselines}
For evaluating RESIDE, we compare against the following baselines:

\begin{itemize}[itemsep=2pt,parsep=0pt,partopsep=0pt]
	\item \textbf{Mintz:} Multi-class logistic regression model proposed by \cite{distant_supervision2009} for distant supervision paradigm.
	\item \textbf{MultiR:} Probabilistic graphical model for multi instance learning by \cite{hoffmann2011knowledge} 
	%uses a probabalistic graphical model for multi-instance learning which handles overlapping relations. 
	\item \textbf{MIMLRE:} A graphical model which jointly models multiple instances and multiple labels. More details in \cite{surdeanu2012multi}.
	\item \textbf{PCNN:} A CNN based relation extraction model by \cite{zeng2015distant} which uses piecewise max-pooling for sentence representation.
	\item \textbf{PCNN+ATT:} A piecewise max-pooling over CNN based model which is used by \cite{lin2016neural} to get sentence representation followed by attention over sentences.
	\item \textbf{BGWA:} Bi-GRU based relation extraction model with word and sentence level attention \cite{bgwa_paper}.
	\item \textbf{RESIDE:} The method proposed in this chapter, please refer \refsec{reside_sec:details} for more details.
\end{itemize}

\subsubsection{Evaluation Criteria}
Following the prior works \cite{lin2016neural,feng2017effective}, we evaluate the models using held-out evaluation scheme. This is done by comparing the relations discovered from test articles with those in Freebase. We evaluate the performance of models with Precision-Recall curve and top-N precision (P@N) metric in our experiments.

%\subsection{Hyperparameters}
%We tune the hyper-parameters of our model on the validation dataset. Grid search was used to determine the optimal parameters. For token representation, we use 50-dimensional GloVe \cite{glove} embeddings and 16-dimensional positional embeddings. Bi-GRU hidden size was taken as $192$. Single layer GCN was employed with $16$-dimensional hidden layer. For relation alias and entity type side information embedding size of $32$ and $50$ respectively was used. Adam optimizer \cite{adam_opt} with $32$ batch size was used for training the network.

\subsection{Results}
\label{reside_sec:results}

In this section we attempt to answer the following questions:
\begin{itemize}[itemsep=2pt,parsep=0pt,partopsep=0pt]
	\item[Q1.] Is RESIDE more effective than existing approaches for distant supervised RE? (\ref{reside_sec:results_main})
	\item[Q2.] What is the effect of ablating different components on RESIDE's performance? (\ref{reside_sec:results_sideinfo})
	\item[Q3.] How is the performance affected in the absence of relation alias information? (\ref{reside_sec:results_rel_side})
\end{itemize}

\begin{table*}[t!]
	\centering
	\begin{small}
		\begin{tabular}{lccc|ccc|ccc}
			\toprule
			& \multicolumn{3}{c}{One} & \multicolumn{3}{c}{Two} & \multicolumn{3}{c}{All}\\ 
			\cmidrule(r){2-4} \cmidrule(r){5-7} \cmidrule(r){8-10} 
			& P@100 & P@200 & P@300 & P@100 & P@200 & P@300 & P@100 & P@200 & P@300 \\
			\midrule
			%CNN			& 68.3	& 60.7	& 53.8	& 70.3	& 62.7	& 55.8	& 67.3	& 64.7	& 58.1 \\
			PCNN		& 73.3	& 64.8	& 56.8	& 70.3	& 67.2	& 63.1	& 72.3	& 69.7	& 64.1 \\ 
			PCNN+ATT	& 73.3	& 69.2	& 60.8	& 77.2	& 71.6	& 66.1	& 76.2	& 73.1	& 67.4 \\
			BGWA		& 78.0	& 71.0	& 63.3	& 81.0	& 73.0	& 64.0	& 82.0	& 75.0	& 72.0 \\
			RESIDE	& \textbf{80.0}	& \textbf{75.5}	& \textbf{69.3}	& \textbf{83.0}	& \textbf{73.5}	& \textbf{70.6}	& \textbf{84.0}	& \textbf{78.5}	& \textbf{75.6} \\
			\bottomrule
			\addlinespace
		\end{tabular}
		\caption{\label{reside_tb:np_canonicalization}P@N for relation extraction using variable number of sentences in bags (with more than one sentence) in Riedel dataset. Here, One, Two and All represents the number of sentences randomly selected from a bag. RESIDE attains improved precision in all settings. More details in \refsec{reside_sec:results_main}}
	\end{small}
\end{table*}

%\begin{figure}[t]
%	\centering
%	\includegraphics[width=3.2in]{./sections/reside/images/plot_pr.pdf}
%	\caption{\label{reside_fig:pr_riedel} Comparision of Precision-recall curve on Riedel. RESIDE outperforms }
%\end{figure}
%
%\begin{figure}[t]
%	\centering
%	\includegraphics[width=3.2in]{./sections/reside/images/plot_pr_gids.pdf}
%	\caption{\label{reside_fig:pr_gds} Comparision of Precision-recall curves of RESIDE and other models for GIDS data}
%\end{figure}

%\begin{table}[t!]
%	\centering
%	\begin{small}
%		\begin{tabular}{lcccc}
%			\toprule
%			& Top 100 & Top 200 & Top 500 & Average\\ 
%			\midrule
%			Mintz		& 77.0	& 71.0	& 55.0 & 67.6 \\
%			MultiR		& 83.0	& 74.0	& 59.0 & 72.0 \\
%			MIML		& 85.0	& 75.0	& 61.0 & 73.7 \\
%			PCNN		& 84.0	& 77.0	& 64.0 & 75.0 \\
%			PCNN+ATT	& 86.0	& 82.0	& 68.0 & 79.7 \\
%			\textbf{RESIDE}	&  84.0	& 78.5 & 69.4 & 77.3\\
%			\bottomrule
%			\addlinespace
%		\end{tabular}
%		\caption{\label{reside_tb:pr_comparison}P@N for top 100, 200 and 500 discovered relation instances.}
%	\end{small}
%\end{table}

\subsubsection{Performance Comparison}
\label{reside_sec:results_main}
For evaluating the effectiveness of our proposed method, RESIDE, we compare it against the baselines stated in \refsec{reside_sec:baselines}. We use only the neural baselines on GDS dataset. The Precision-Recall curves on Riedel and GDS are presented in \reffig{reside_fig:main_pr}. Overall, we find that RESIDE achieves higher precision over the entire recall range on both the datasets. All the non-neural baselines could not perform well as the features used by them are mostly derived from NLP tools which can be erroneous. RESIDE outperforms PCNN+ATT and BGWA which indicates that incorporating side information helps in improving the performance of the model. The higher performance of BGWA and PCNN+ATT over PCNN shows that attention helps in distant supervised RE. Following \cite{lin2016neural,softlabel_paper}, we also evaluate our method with different number of sentences. Results summarized in \reftbl{reside_tb:np_canonicalization}, show the improved precision of RESIDE in all test settings, as compared to the neural baselines, which demonstrates the efficacy of our model.
%shows top-N (P@N) of the neural baselines and our model. We can see that RESIDE improves precision in all test settings, which demonstrates the effectiveness of our model.

%Overall, RESIDE gives a significant improvement compared to all other models over the entire range of recall and thus achieves state-of-the-art results for relation extraction task.

\subsubsection{Ablation Results}
\label{reside_sec:results_sideinfo}
In this section, we analyze the effect of various components of RESIDE on its performance. For this, we evaluate various versions of our model with cumulatively removed components. The experimental results are presented in \reffig{reside_fig:ablation}. We observe that on removing different components from RESIDE, the performance of the model degrades drastically. The results validate that GCNs are effective at encoding syntactic information. Further, the improvement from side information shows that it is complementary to the features extracted from text, thus validating the central thesis of this chapter, that inducing side information leads to improved relation extraction.

%\begin{figure}[t]
%	%\centering
%	\begin{minipage}{3.3in}
%		\includegraphics[width=3.2in]{./sections/reside/images/sideinfo.pdf}
%		\subcaption{\label{reside_fig:ablation}\small Performance comparison of different ablated version of RESIDE on Riedel dataset. Overall, GCN and side information helps RESIDE improve performance. Refer \refsec{reside_sec:results_sideinfo}.}
%	\end{minipage}
%	\begin{minipage}{3.2in}
%		\includegraphics[width=3.2in]{./sections/reside/images/sideinfo3.pdf}
%		\subcaption{\label{reside_fig:diff_aliases}\small Performance on settings defined in \refsec{reside_sec:results_rel_side} with respect to the presence of relation alias side information on Riedel dataset. RESIDE performs comparably in the absence of relations from KB.}
%	\end{minipage}
%\end{figure}

\subsubsection{Effect of Relation Alias Side Information}
\label{reside_sec:results_rel_side}
In this section, we test the performance of the model in setting where relation alias information is not readily available. For this, we evaluate the performance of the model on four different settings:

\begin{itemize}[itemsep=2pt,topsep=2pt,parsep=0pt,partopsep=0pt]
	\item \textbf{None:} Relation aliases are not available.
	\item \textbf{One:} The name of relation is used as its alias. 
	\item \textbf{One+PPDB:} Relation name extended using Paraphrase Database (PPDB). 
	\item \textbf{All:} Relation aliases from Knowledge Base\footnote{Each relation in Riedel dataset is manually mapped to corresponding Wikidata property for getting relation aliases. Few examples are presented in supplementary material.}
	
\end{itemize}

The overall results are summarized in \reffig{reside_fig:diff_aliases}. We find that the model performs best when aliases are provided by the KB itself. Overall, we find that RESIDE gives competitive performance even when very limited amount of relation alias information is available. We observe that performance improves further with the availability of more alias information.

\begin{figure}[t]
	%\centering
	\begin{minipage}{3.15in}
		\includegraphics[width=3.15in]{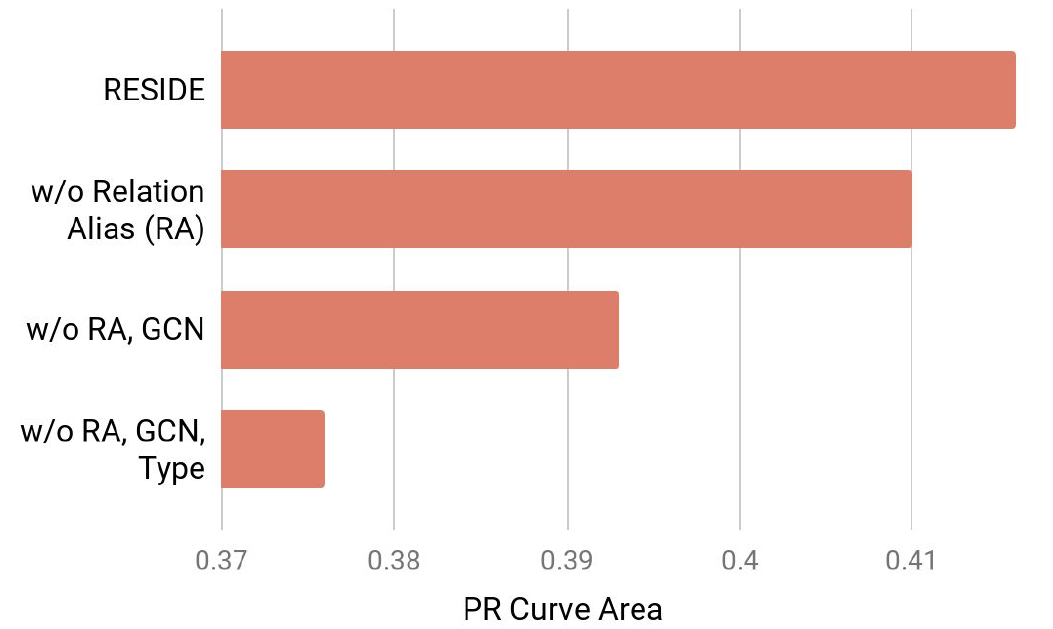}
		\subcaption{\label{reside_fig:ablation}\small Performance comparison of different ablated version of RESIDE on Riedel dataset. Overall, GCN and side information helps RESIDE improve performance. Refer \refsec{reside_sec:results_sideinfo}.}
	\end{minipage}
	\begin{minipage}{3.4in}
		\includegraphics[width=3.4in]{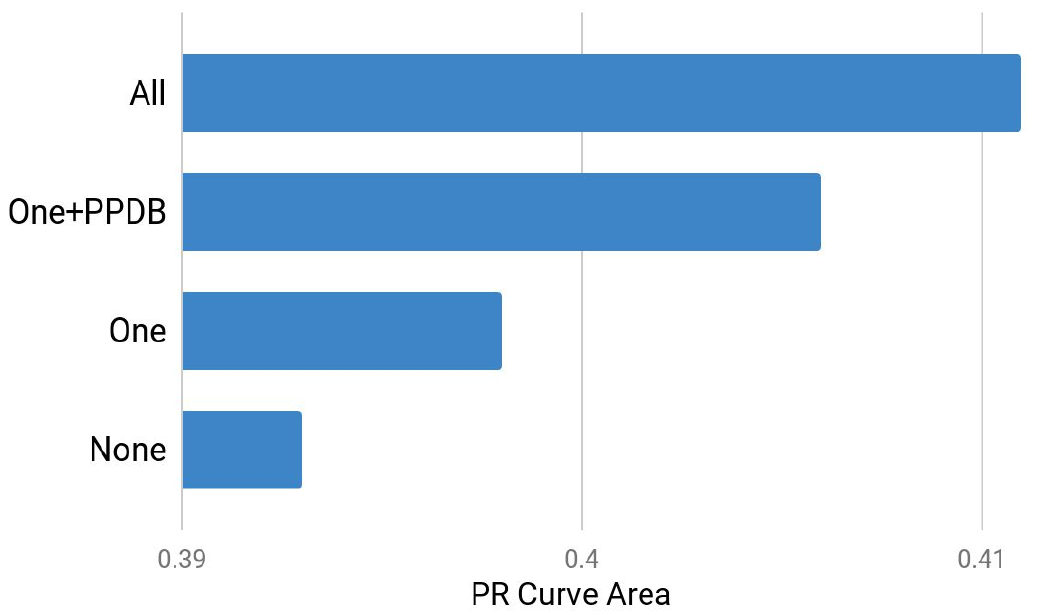}
		\subcaption{\label{reside_fig:diff_aliases}\small Performance on settings defined in \refsec{reside_sec:results_rel_side} with respect to the presence of relation alias side information on Riedel dataset. RESIDE performs comparably in the absence of relations from KB.}
	\end{minipage}
\end{figure}

%\subsection{Effect of sentence numbers}
%We perform held-out evaluation to give an empirical comparison of different methods to evaluate the performance of sentence attention. we select bags that contain more than one sentence and compare the performance of CNN/PCNN+ONE, CNN/PCNN+AVG, CNN/PCNN+ATT with our model in the following settings.
%\begin{itemize}
%	\item \textbf{one:} A randomly selected sentence for each entity pair is used for predicting the relation of the entities
%	\item \textbf{two:} Two randomly selected sentences for each entity pair are used for predicting the relation of the entities
%	\item \textbf{all:} All the sentences of the entity pair are used for predicting the relation 
%\end{itemize}
%We report the P@100, P@200 and p@300 and the mean of them of each model and compare them with our model. Table [\ref{reside_tb:np_canonicalization}] shows the metrics in 3 test settings of our model in comparison to the CNN/PCNN+ONE, CNN/PCNN+AVG, CNN/PCNN+ATT. The table shows that \textbf{RESIDE} performs better than all the attention based models in all the 3 different settings which demonstrate that side information is effective for improvement of sentence-level attention. We also provide various ablation studies which further prove this point.
\section{Conclusion}
\label{reside_sec:conclusion}

In this chapter, we propose RESIDE, a novel neural network based model which makes principled use of relevant side information, such as entity type and relation alias, from Knowledge Base, for improving distant supervised relation extraction. RESIDE employs Graph Convolution Networks for encoding syntactic information of sentences and is robust to limited side information. Through extensive experiments on benchmark datasets, we demonstrate RESIDE's effectiveness over state-of-the-art baselines. We have made RESIDE's source code publicly available to promote reproducible research.

\part{\partTwo{}}
\chapter{Documents Timestamping using Graph Convolutional Networks}
\label{chap_neuraldater}

% !TeX spellcheck = en_GB
\section{Introduction}
\label{neuraldater_sec:introduction}

In this chapter, we present our first application of Graph Convolutional Networks for solving document timestamping problem. 
Date of a document, also referred to as the Document Creation Time (DCT), is at the core of many important tasks, such as, information retrieval \cite{ir_time_usenix,ir_time_li,ir_time_dakka}, temporal reasoning \cite{temp_reasoner1,temp_reasoner2}, text summarization \cite{text_summ_time}, event detection \cite{event_detection}, and analysis of historical text \cite{history_time}, among others. In all such tasks, the document date is assumed to be available and also accurate -- a strong assumption, especially for arbitrary documents from the Web. Thus, there is a need to automatically predict the date of a document based on its content. This problem is referred to as \emph{Document Dating}.

\begin{figure}[t]
	%\centering
	\begin{minipage}{3.2in}
		\includegraphics[width=3.2in]{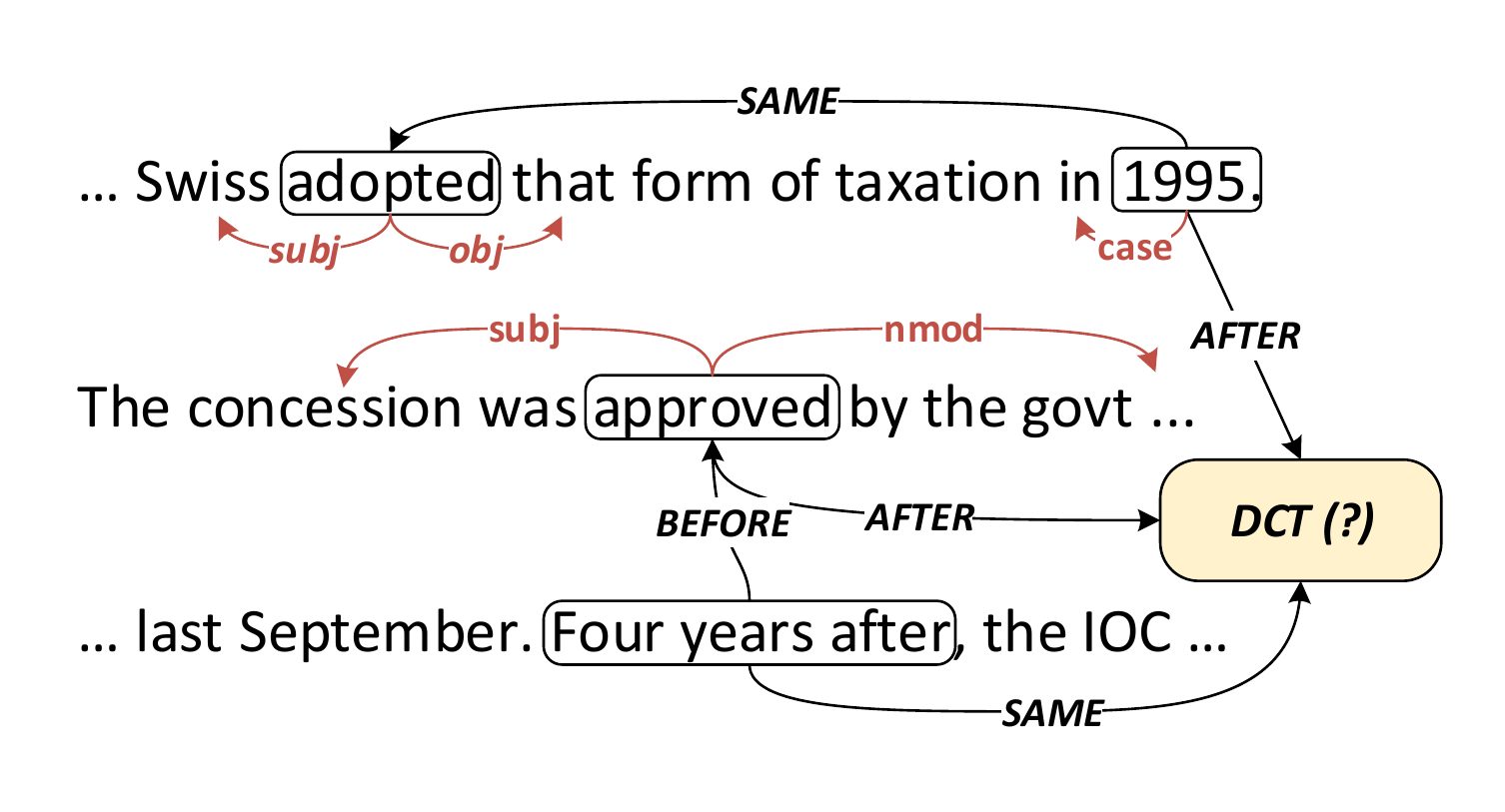}
	\end{minipage}
	\begin{minipage}{3.2in}
		\includegraphics[width=3.2in]{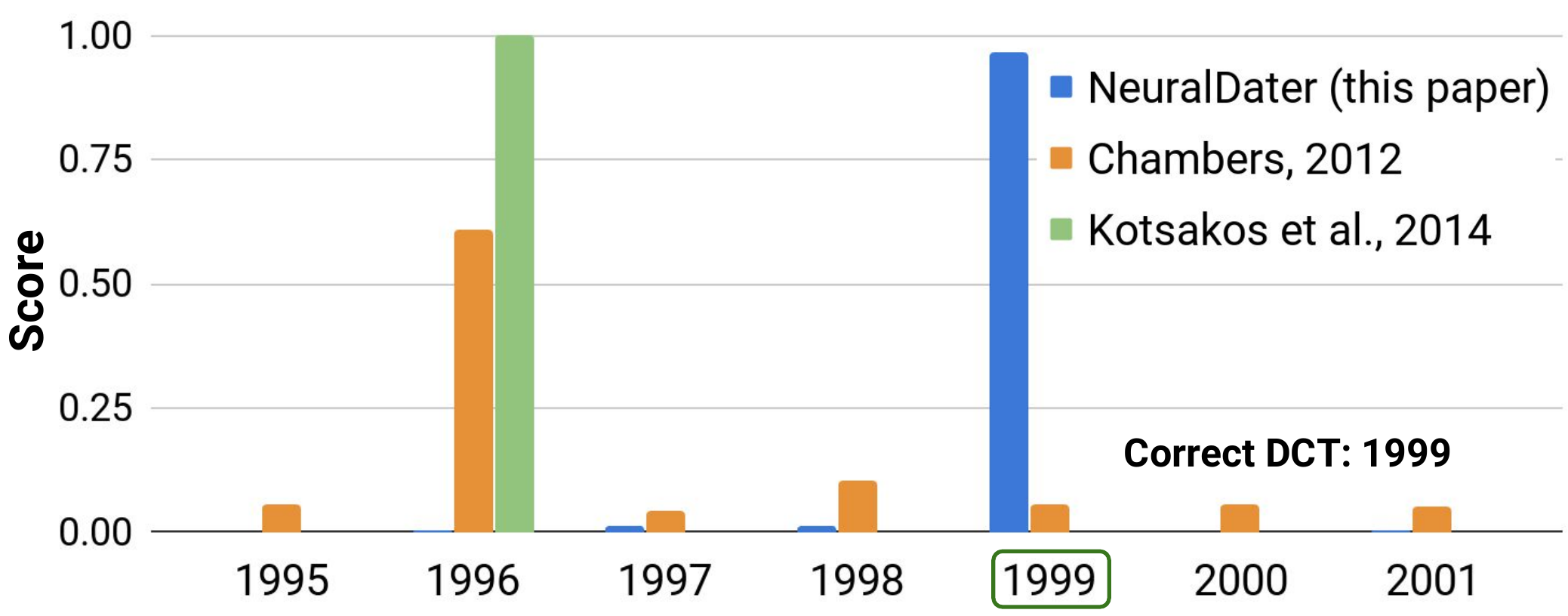}
	\end{minipage}
	\caption{\label{neuraldater_fig:motivation}\small \textbf{Top:} An example document annotated with syntactic and temporal dependencies. In order to predict the right value of 1999 for the Document Creation Time (DCT), inference over these document structures is necessary. \textbf{Bottom:} Document date prediction by two state-of-the-art-baselines and NeuralDater, the method proposed in this chapter. While the two previous methods are getting misled by the temporal expression (\textit{1995}) in the document, NeuralDater is able to use the syntactic and temporal structure of the document to predict the right value (\textit{1999}). %Using temporal structure our proposed method ({NeuralDater}) outperforms existing document dating techniques.
	}
\end{figure}

Initial attempts on automatic document dating started with generative models by \cite{history_time}. This model is
later improved by \cite{temporal_entropy} who incorporate additional features such as POS tags, collocations, etc. \citet{Chambers:2012:LDT:2390524.2390539} shows significant improvement over these prior efforts through their discriminative models using handcrafted temporal features. \citet{Kotsakos:2014:BAD:2600428.2609495} propose a statistical approach for document dating exploiting term burstiness \cite{Lappas:2009:BSD:1557019.1557075}. 
% proposed an entirely statistical based method which exploited advances in term burstiness \cite{Lappas:2009:BSD:1557019.1557075} for further improving on this problem. 

Document dating is a challenging problem which requires extensive reasoning over the temporal structure of the document. Let us motivate this through an example shown in \reffig{neuraldater_fig:motivation}. In the document, \textit{four years after} plays a crucial role in identifying the creation time of the document. The existing approaches give higher confidence for timestamp immediate to the year mention \textit{1995}. NeuralDater exploits the syntactic and temporal structure of the document to predict the right timestamp (1999) for the document. With the exception of \cite{Chambers:2012:LDT:2390524.2390539}, all prior works on the document dating problem ignore such informative temporal structure within the document.

Research in document event extraction and ordering have made it possible to extract such temporal structures involving events, temporal expressions, and the (unknown) document date in a document \cite{catena_paper,Chambers14}. While methods to perform reasoning over such structures exist \cite{tempeval07,tempeval10,tempeval13,tempeval15,timebank03}, none of them have exploited advances in deep learning \cite{alexnet,microsoft_speech,deep_learning_book}. In particular, recently proposed Graph Convolution Networks (GCN) \cite{Defferrard:2016:CNN:3157382.3157527,kipf2016semi} have emerged as a way to learn graph representation while encoding structural information and constraints represented by the graph. We adapt GCNs for the document dating problem and make the following contributions:

\begin{itemize}
	\item We propose NeuralDater, a Graph Convolution Network (GCN)-based approach for document dating. To the best of our knowledge, this is the first application of GCNs, and more broadly deep neural network-based methods, for the document dating problem.
	\item NeuralDater is the first document dating approach which exploits syntactic as well temporal structure of the document, all within a principled joint model.
	\item Through extensive experiments on multiple real-world datasets, we demonstrate NeuralDater's effectiveness over state-of-the-art baselines.
\end{itemize}

NeuralDater's source code and datasets used in the chapter are available at \url{http://github.com/malllabiisc/NeuralDater}.
% !TeX spellcheck = en_GB
\section{Related Work}
\label{neuraldater_sec:related_work}

{\bf Automatic Document Dating}:
\citet{history_time} propose the first approach for automating document dating through a statistical language model. \citet{temporal_entropy} further extend this work by incorporating semantic-based preprocessing and temporal entropy \cite{temporal_entropy} based term-weighting. 
\citet{Chambers:2012:LDT:2390524.2390539} proposes a MaxEnt based discriminative model trained on hand-crafted temporal features. He also proposes a model to learn probabilistic constraints between year mentions and the actual creation time of the document. We draw inspiration from his work for exploiting temporal reasoning for document dating. 
\citet{Kotsakos:2014:BAD:2600428.2609495} propose a purely statistical method which considers lexical similarity alongside burstiness \cite{Lappas:2009:BSD:1557019.1557075} of terms for dating documents. To the best of our knowledge, NeuralDater, our proposed method,  is the first method to utilize deep learning techniques for the document dating problem.

%\cite{temporal_entropy} proposed semantic-based preprocessing approach to improve the quality of existing time-stamping models. They also proposed some extensions of language model for time-stamping by incorporating internal and external knowledge. \cite{Jatowt:2013:EDF:2505515.2505655} used the above approach to build a web-based interface for document dating. For an input undated document, the system outputs the likelihood scores for every time-partition as queried by the user.

%\cite{Chambers:2012:LDT:2390524.2390539} proposed two different models for time-stamping documents. The first model is trained with features extracted from the time informations present in the text, whereas the second model learns probabilistic constraints between the time informations in the text and the actual document time. They also showed that, by imposing the learned constraints on the first discriminative model further improves the classification accuracy.

{\bf Event Ordering Systems}: Temporal ordering of events is a vast research topic in NLP. The problem is posed as a temporal relation classification between two given temporal entities. Machine Learned classifiers and well crafted linguistic features for this task are used in \cite{Chambers:2007:CTR:1557769.1557820, E14-1033}. \citet{N13-1112} use a hybrid approach by adding 437 hand-crafted rules. \citet{Chambers:2008:JCI:1613715.1613803, P09-1046} try to classify with many more temporal constraints, while utilizing integer linear programming and Markov logic. 

CAEVO, a CAscading EVent Ordering architecture \cite{Chambers14} use sieve-based architecture 
\cite{sieve_architecture} 
for temporal event ordering for the first time. They mix multiple learners according to their precision based ranks and use transitive closure for maintaining consistency of temporal graph. \citet{catena_paper} recently propose CATENA (CAusal and TEmporal relation extraction from NAtural language texts), the first integrated system for the temporal and causal relations extraction between pre-annotated events and time expressions. They also incorporate sieve-based architecture which outperforms existing methods in temporal relation classification domain. We make use of CATENA for temporal graph construction in our work. 

{\bf Graph Convolutional Networks (GCN)}:
GCNs generalize Convolutional Neural Network (CNN) over graphs. GCN is introduced by \cite{gcn_first_work}, and later extended by \cite{Defferrard:2016:CNN:3157382.3157527} with efficient localized filter approximation in spectral domain. \citet{kipf2016semi} propose a first-order approximation of localized filters through layer-wise propagation rule. GCNs over syntactic dependency trees have been recently exploited in the field of semantic-role labeling \cite{gcn_srl}, neural machine translation \cite{gcn_nmt}, event detection \cite{gcn_event}. In our work, we successfully use GCNs for document dating.

\section{Proposed Approach: NeuralDater}
\label{neuraldater_sec:method}

\subsection{Overview}
\begin{figure*}[!t]
	\centering
	\includegraphics[width=7in]{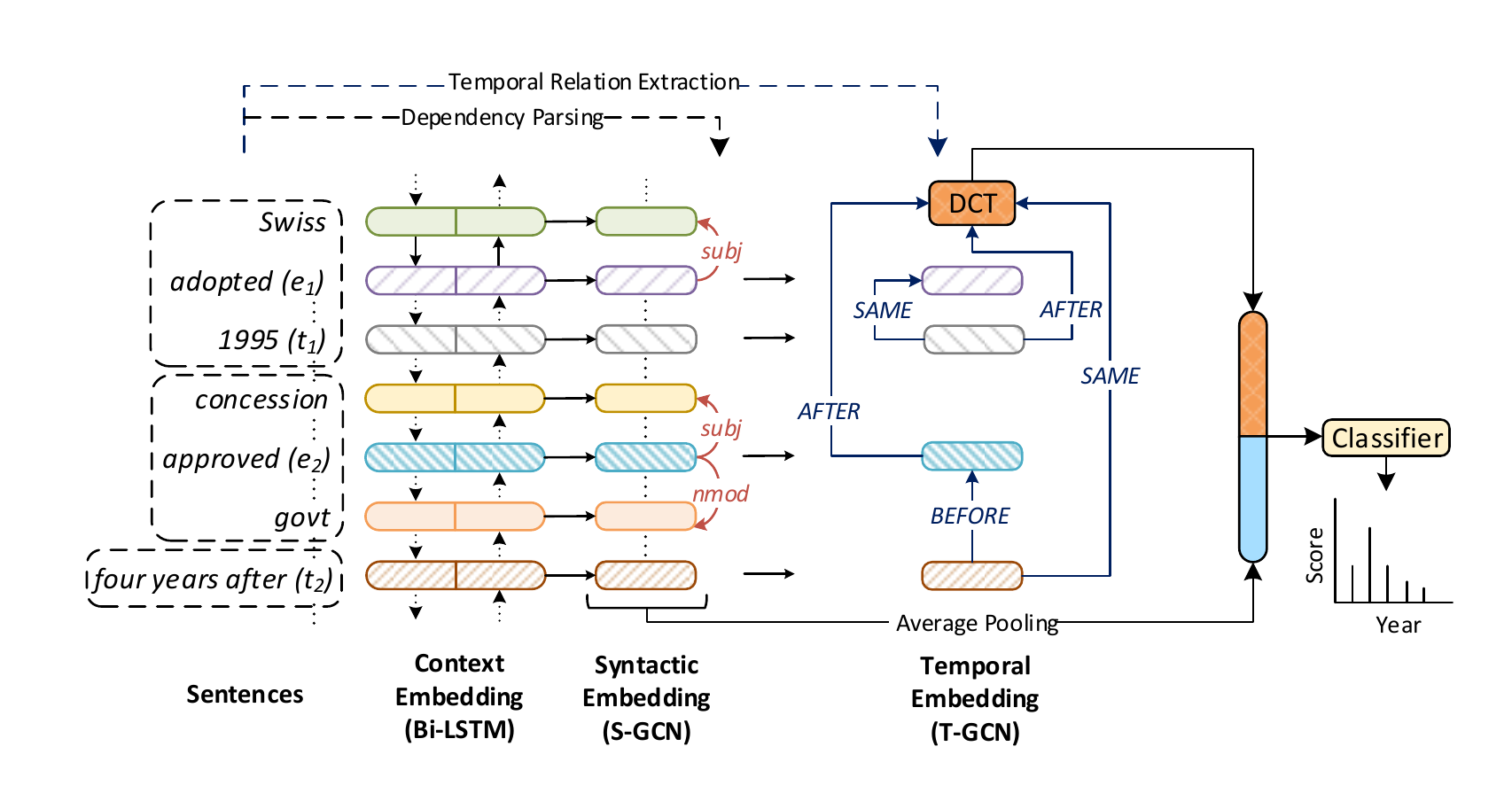}
	\caption{\label{neuraldater_fig:overview} \small Overview of NeuralDater. NeuralDater exploits syntactic and temporal structure in a document to learn effective representation, which in turn are used to predict the  document time. NeuralDater uses a Bi-directional LSTM (Bi-LSTM), two Graph Convolution Networks (GCN) -- one over the dependency tree and the other over the document's temporal graph -- along with a softmax classifier, all trained end-to-end jointly. Please see \refsec{neuraldater_sec:method} for more details.}	
\end{figure*}

The Documents Dating problem may be cast as a multi-class classification problem \cite{Kotsakos:2014:BAD:2600428.2609495,Chambers:2012:LDT:2390524.2390539}. In this section, we present an overview of NeuralDater, the document dating system proposed in this chapter. Architectural overview of NeuralDater is shown in \reffig{neuraldater_fig:overview}.

NeuralDater is a deep learning-based multi-class classification system. It takes in a document as input and returns its predicted date as output by exploiting the syntactic and temporal structure of document. 

NeuralDater network consists of three layers which learn an  embedding for the Document Creation Time (DCT) node corresponding to the document. This embedding is then fed to a softmax classifier which produces a distribution over timestamps. Following prior research \cite{Chambers:2012:LDT:2390524.2390539,Kotsakos:2014:BAD:2600428.2609495}, we work with year granularity for the experiments in this chapter.  We, however, note that NeuralDater can be trained for finer granularity with appropriate training data. The NeuralDater network is trained end-to-end using training data. We briefly present NeuralDater's various components below. Each component is described in greater detail in subsequent sections.

\begin{itemize}
	\item \textbf{Context Embedding}: In this layer, NeuralDater uses a Bi-directional LSTM (Bi-LSTM) to learn embedding for each token in the document. Bi-LSTMs have been shown to be quite effective in capturing local context inside token embeddings \cite{seq2seq}.
	\item \textbf{Syntactic Embedding}: In this step, NeuralDater revises token embeddings from the previous step by running a GCN over the dependency parses of sentences in the document. We refer to this GCN as \textbf{Syntactic GCN} or \textbf{S-GCN}. While the Bi-LSTM captures immediate local context in token embeddings, S-GCN augments them by capturing syntactic context. 
	\item \textbf{Temporal Embedding}: In this step, NeuralDater further refines embeddings learned by S-GCN to incorporate cues from temporal structure of event and times in the document. NeuralDater uses state-of-the-art causal and temporal relation extraction algorithm \cite{catena_paper} for extracting temporal graph for each document. A GCN is then run over this temporal graph to refine the embeddings from the previous layer. We refer to this GCN as \textbf{Temporal GCN} or \textbf{T-GCN}. In this step, a special DCT node is introduced whose embedding is also learned by the T-GCN.
	\item \textbf{Classifier}: Embedding of the DCT node along with average pooled embeddings learned by S-GCN are fed to a fully connected softmax classifier which makes the final prediction about the date of the document.
\end{itemize}

Even though the previous discussion is presented in a sequential manner, the whole network is trained in a joint end-to-end manner using backpropagation. Below, we present detailed description of various components of NeuralDater.

%Documents dating, or approximating timestamp of documents can be cast as a multi-classification problem \cite{Kotsakos:2014:BAD:2600428.2609495,Chambers:2012:LDT:2390524.2390539}. Given a document our method assigns a unique timestamp which reflects the time at which the document was published. Firstly, we identify the relevant events and time mentions in the document using the method employed by \cite{Chambers14} then exploiting temporal and causal relation extraction models \cite{catena_paper,caveo_plus}, we identify the causal relationships among events, time mentions and creation time of the document. This temporal graph structure is utilized by our Graph Convolutional \cite{Defferrard:2016:CNN:3157382.3157527,kipf2016semi} based Neural Network for predicting the timestamp of the document. We also incorporate syntactic GCNs \cite{gcn_srl,gcn_nmt} over dependency parse tree from Stanford CoreNLP \cite{stanford_corenlp} for encoding the content of the document. The whole model can be divided into two modules:
%
%\begin{itemize}
%	\item \textbf{Content extraction:} This modules captures the content of a given document and encodes it in a vector which is later utilized for time stamping.
%	\item \textbf{Temporal Inference:} By reasoning over the temporal relationships among temporal entities, this module generates an encoding of Document Creation Time (DCT) which incorporates the extracted temporal constraints.
%\end{itemize}
%
%
%The overall architecture and dataflow of the proposed method is summarized in Figure \ref{neuraldater_fig:caevo}.

\subsection{Context Embedding (Bi-LSTM)}
\label{neuraldater_sec:et_gcn}

Let us consider a document $D$ with $n$ tokens  $w_1, w_2, ..., w_n$.
%Content extraction module is motivated by the previously proposed content based dating approaches \cite{history_time,temporal_entropy} which predict creation time of a document by matching their content with others.%
We first represent each token by a $k$-dimensional word embedding. For the experiments in this chapter, we use GloVe \cite{glove} embeddings. These token embeddings are stacked together to get the document representation $\m{X} \in \mathbb{R}^{n \times k}$. We then employ a Bi-directional LSTM (Bi-LSTM) \cite{lstm} on the input matrix $\m{X}$ to obtain contextual embedding for each token. After stacking contextual embedding of all these tokens, we get the new document representation matrix $\m{H}^{cntx} \in \mathbb{R}^{n \times r_{cntx}}$. In this new representation, each token is represented in a $r_{cntx}$-dimensional space. Our choice of LSTMs for learning contextual embeddings for tokens is motivated by the previous success of LSTMs in this task \cite{seq2seq}.

%We then perform average pooling over the GCN encoding of all the tokens to get the content encoding of the entire document
\subsection{Syntactic Embedding (S-GCN)} 
\label{neuraldater_sec:syntax_gcn}

While the Bi-LSTM is effective at capturing immediate local context of a token, it may not be as effective in capturing longer range dependencies among words in a sentence. For example, in \reffig{neuraldater_fig:motivation}, we would like the embedding of token \textit{approved} to be directly affected by \textit{govt}, even though they are not immediate neighbors. A dependency parse may be used to capture such longer-range connections. In fact, similar features were exploited by \cite{Chambers:2012:LDT:2390524.2390539} for the document dating problem. NeuralDater captures such longer-range information by using another GCN run over the syntactic structure of the document. We describe this in detail below.

The context embedding, $\m{H}^{cntx} \in \mathbb{R}^{n \times r_{cntx}}$ learned in the previous step is used as input to this layer. For a given document, we first extract its syntactic dependency structure by applying the Stanford CoreNLP's dependency parser \cite{stanford_corenlp} on each sentence in the document individually. We now employ the Graph Convolution Network (GCN) over this dependency graph using the GCN formulation presented in Section \ref{sec:directed_gcn}. We call this GCN the Syntactic GCN or S-GCN, as mentioned in \refsec{neuraldater_sec:method}.

Since S-GCN operates over the dependency graph and uses Equation \ref{eqn:gated_gcn} for updating embeddings, the number of parameters in S-GCN is directly proportional to the number of dependency edge types. Stanford CoreNLP's dependency parser returns 55 different dependency edge types. This large number of edge types is going to significantly over-parameterize S-GCN, thereby increasing the possibility of overfitting. In order to address this, we use only three edge types in S-GCN. For each edge connecting nodes $w_i$ and $w_j$ in $\m{E'}$ (see \refeqn{neuraldater_eqn:updated_edges}), we determine its new type $L(w_i, w_j)$ as follows:
\begin{itemize}
	\item $L(w_i, w_j) = \rightarrow$ if $(w_i, w_j, l(w_i, w_j)) \in \m{E'}$, i.e., if edge is an original dependency parse edge
	\item $L(w_i, w_j) = \leftarrow$ if $(w_i, w_j, l(w_i, w_j)^{-1}) \in \m{E'}$, i.e., if the edges is an inverse edge	
	\item $L(w_i, w_j) = \top$ if $(w_i, w_j, \top) \in \m{E'}$, i.e., if the edge is a self-loop with $w_i = w_j$
\end{itemize}
S-GCN now estimates embedding $h^{syn}_{w_{i}} \in \mathbb{R}^{r_{syn}}$ for each token $w_{i}$ in the document using the formulation shown below.
\[
h^{syn}_{w_i} = f \Bigg(\sum_{w_j \in \m{N}(w_i)}\left(W_{L(w_i, w_j)}h^{cntx}_{w_j} + b_{L(w_i, w_j)}\right) \Bigg)
\]
Please note S-GCN's use of the new edge types $L(w_i, w_j)$  above, instead of the $l(w_i, w_j)$ types used in Equation \ref{eqn:gated_gcn}. By stacking embeddings for all the tokens together, we get the new embedding matrix $\m{H}^{syn} \in \mathbb{R}^{n \times r_{syn}}$ representing the document.

\textbf{AveragePooling}: We obtain an embedding $h_{D}^{avg}$ for the whole document by average pooling of every token representation.
\begin{equation}
h_{D}^{avg} = \frac{1}{n} \sum_{i = 1}^{n} h_{w_i}^{syn}
\label{neuraldater_eqn:avg-pool}.
\end{equation}

%The output matrix $H_{syntax}$ is being fed for further temporal processing. 

\subsection{Temporal Embedding (T-GCN)}
\label{neuraldater_sec:t-gcn}

In this layer, NeuralDater exploits temporal structure of the document to learn an embedding for the Document Creation Time (DCT) node of the document. First, we describe the construction of temporal graph, followed by GCN-based embedding learning over this graph.

\textbf{Temporal Graph Construction}: NeuralDater uses Stanford's SUTime tagger \cite{sutime_paper} for date normalization and the event extraction classifier of \cite{Chambers14} for event detection. The annotated document is then passed to CATENA \cite{catena_paper}, current state-of-the-art temporal and causal relation extraction algorithm, to obtain a temporal graph for each document. %\remove{NeuralDater first identifies event and time mentions in a given document using the event-time identification algorithm of CAEVO. NeuralDater then employs CATENA, current state-of-the-art temporal and causal relation extraction algorithm, to obtain a temporal graph for each document.}. 
Since our task is to predict the creation time of a given document, we supply DCT as unknown to CATENA. We hypothesize that the temporal relations extracted in absence of DCT are helpful for document dating and we indeed find this to be true, as shown in Section \ref{neuraldater_sec:results}. Temporal graph is a directed graph, where nodes correspond to events, time mentions, and the Document Creation Time (DCT). Edges in this graph represent causal and temporal relationships between them. Each edge is attributed with a label representing the type of the temporal relation. CATENA outputs 9 different types of temporal relations, out of which we selected five types, viz.,  \textit{AFTER}, \textit{BEFORE}, \textit{SAME}, \textit{INCLUDES}, and  \textit{IS\_INCLUDED}. The remaining four types were ignored as they were substantially infrequent. 

Please note that the temporal graph may involve only a small number of tokens in the document. For example, in the temporal graph in \reffig{neuraldater_fig:overview}, there are a total of 5 nodes: two temporal expression nodes (\textit{1995} and \textit{four years after}), two event nodes (\textit{adopted} and \textit{approved}), and a special DCT node. This graph also consists of temporal relation edges such as (\textit{four years after}, \textit{approved}, \textit{BEFORE}). 

%Our motivation behind extracting temporal graph is to utilize the research on TimeBank \cite{timebank03} and TempEval \cite{tempeval07,tempeval10,tempeval13,tempeval15} contests for document dating problem. Most systems participating in these campaigns focus predominantly on relation classification among pre-annotated temporal entities, viz. events and time mentions. So, first we identify event and time mentions in a given document using event-time identification algorithm of \cite{Chambers14}. Then we employ the current state-of-the-art temporal and causal relation extraction algorithm, CATENA \cite{catena_paper} to obtain a temporal graph for each document. Since our task is to predict the creation time of a given document, we supply DCT as unknown to CATENA. We hypothesize that the temporal relations extracted in absence of DCT are helpful for document dating and we indeed find this to be true, as shown in Section \ref{neuraldater_sec:results}. Temporal graph is a directed graph, where nodes corresponds to events, time mentions and Document creation time (DCT) and edges represent causal and temporal relationships between them. Each edge is attributed with a label representing the type of the temporal relation. CATENA outputs 9 different types of temporal relations, out of which we selected \textit{After}, \textit{Before}, \textit{Simultaneous}, \textit{Includes}, and  \textit{Is\_Included} for our purpose.

\textbf{Temporal Graph Convolution}: NeuralDater employs a GCN over the temporal graph constructed above. We refer to this GCN as the Temporal GCN or T-GCN, as mentioned in \refsec{neuraldater_sec:method}. T-GCN is based on the GCN formulation presented in Section \ref{sec:directed_gcn}. Unlike S-GCN, here we consider label and direction specific parameters as the temporal graph consists of only five types of edges.

Let $n_T$ be the number of nodes in the temporal graph. Starting with $\m{H}^{syn}$ (\refsec{neuraldater_sec:syntax_gcn}), T-GCN learns a $r_{temp}$-dimensional embedding for each node in the temporal graph. Stacking all these embeddings together, we get the embedding matrix $\m{H}^{temp} \in \mathbb{R}^{n_{T} \times r_{temp}}$. T-GCN embeds the temporal constraints induced by the temporal graph in $h_{DCT}^{temp} \in \mathbb{R}^{r_{temp}}$, embedding of the DCT node of the document. %This allows NeuralDater to mimic the previously proposed time constraint models \cite{Chambers:2012:LDT:2390524.2390539} for document dating. 

%The event time ordering algorithm \cite{catena_paper} extracts events and time from the document D, with tokens $w_{1}, w_{2}, ..., w_{n}$. Lets say it produces the sets , $Event\text{ } = \text{ } {e_{1},e_{2} ,\cdots,e_{n_{E}}}$ and $time\text{ } = \text{ }{t_{1},t_{2}, \cdots ,t_{n_{T}}}$. The ordering algorithm also spits out the graph by labeling the temporal relations between some of the elements in $Event$ and $Time$, for eg. in Figure \ref{neuraldater_fig:timeEg}, edge from \textit{hit} to \textit{Dec 2004} will be represented as (\textit{hit}, \textit{Dec 2004}, \textit{SAME}) in the temporal graph. 
%Just like Syntactic GCN (\ref{neuraldater_sec:syntax_gcn}) we add edges in opposite direction to the original arc and self loops for \textbf{T-GCN} as well. The embeddings for only the tokens that are in $Time$ and $Event$, play as inputs to this module, ie. the $H_{syntax}$ matrix is been reduced down from $\mathbb{R}^{n\times r}$ to $\mathbb{R}^{n_{T}+n_{E}\times r}$. T-GCN on Temporal Graph for this module uses the exact same formulation described in subsection \ref{neuraldater_sec:directed_gcn}. Unlike S-GCN, here we consider label and direction specific parameters. 
%The convolution operation embeds the temporal constraints induced by the temporal graph in the DCT node and outputs temporal representation of the document, $h_{DCT}$. This allows our model to mimic the previously proposed time constraint models \cite{Chambers:2012:LDT:2390524.2390539} for document dating. 

\subsection{Classifier}
Finally, the DCT embedding $h_{DCT}^{temp}$ and average-pooled syntactic representation $h_{D}^{avg}$ (see \refeqn{neuraldater_eqn:avg-pool}) of document $D$ are concatenated and fed to a fully connected feed forward network followed by a softmax. This allows the NeuralDater to exploit context, syntactic, and temporal structure of the document to  predict the final document date $y$.
\begin{eqnarray*}
	h_{D}^{avg+temp} &=& \text{ } [h_{DCT}^{temp}~;~h_{D}^{avg}] \\ 
	p(y \vert D) &=& \mathrm{Softmax}(W \cdot h_{D}^{avg+temp} + b).
\end{eqnarray*}
\section{Experiments}

\subsection{Experimental Setup}
\label{neuraldater_sec:experiments}

\begin{table}[t]
	\centering
	\begin{tabular}{cccc}
		\toprule
		Datasets 	& \# Docs & Start Year & End Year\\
		\midrule
		APW 		&  675k	& 1995  & 2010 \\
		NYT			&  647k	& 1987  & 1996 \\
		\bottomrule
		\addlinespace
	\end{tabular}
	\caption{\label{neuraldater_tb:datasets}Details of datasets used. Please see \refsec{neuraldater_sec:experiments} for details.}
\end{table}

%\subsection{Datasets}
%\label{neuraldater_sec:datasets}

\textbf{Datasets}: We experiment on Associated Press Worldstream (APW) and New York Times (NYT) sections of Gigaword corpus \cite{gigaword5th}. The original dataset contains around 3 million documents of APW and 2 million documents of NYT from span of multiple years. From both sections, we randomly sample around 650k documents while maintaining balance among years. Documents belonging to years with substantially fewer documents are omitted. Details of the dataset can be found in Table \ref{neuraldater_tb:datasets}. For train, test and validation  splits, the dataset was randomly divided in 80:10:10 ratio.

%\subsection{Evaluation Criteria}
\textbf{Evaluation Criteria}: Given a document, the model needs to predict the year in which the document was published. We measure performance in terms of overall accuracy of the model. 

%\subsection{Baselines}
%\label{neuraldater_sec:baselines}
\textbf{Baselines}: For evaluating NeuralDater, we compared against the following methods:

\begin{itemize}
	\item \textbf{BurstySimDater} \citet{Kotsakos:2014:BAD:2600428.2609495}:  This is a purely statistical method which uses lexical similarity and term burstiness \cite{Lappas:2009:BSD:1557019.1557075} for dating documents in arbitrary length time frame. For our experiments, we took the time frame length as 1 year. Please refer to \cite{Kotsakos:2014:BAD:2600428.2609495} for more details.
	\item \textbf{MaxEnt-Time-NER}: Maximum Entropy (MaxEnt) based classifier trained on hand-crafted temporal and Named Entity Recognizer (NER) based features. More details in \cite{Chambers:2012:LDT:2390524.2390539}. 
	\item \textbf{MaxEnt-Joint}: Refers to MaxEnt-Time-NER combined with year mention classifier as described in \cite{Chambers:2012:LDT:2390524.2390539}. 
	\item \textbf{MaxEnt-Uni-Time:} MaxEnt based discriminative model which takes bag-of-words representation of input document with normalized time expression as its features. 
	\item \textbf{CNN:} A Convolution Neural Network (CNN) \cite{cnn_paper} based text classification model proposed by \cite{yoon_kim}, which attained state-of-the-art results in several domains. 
	\item {\bf {NeuralDater}}: Our proposed method, refer Section \ref{neuraldater_sec:method}.
\end{itemize}

\textbf{Hyperparameters}: By default, edge gating (Section \ref{sec:directed_gcn}) is used in all GCNs. The parameter $K$ represents the number of layers in T-GCN (\refsec{neuraldater_sec:t-gcn}). We use 300-dimensional GloVe embeddings and 128-dimensional hidden state for both GCNs and BiLSTM with $0.8$ dropout. We used Adam \cite{adam_opt} with $0.001$ learning rate for training.
% !TeX spellcheck = en_US

\begin{table}[t]
	\centering
	\begin{tabular}{lcc}
		\toprule
		Method 			 & APW & NYT \\
		\midrule		
		\addlinespace
		BurstySimDater 		& 45.9 & 38.5 \\
		MaxEnt-Time+NER		& 52.5 & 42.3 \\
		MaxEnt-Joint		& 52.5 & 42.5 \\
		MaxEnt-Uni-Time		& 57.5 & 50.5 \\
		CNN 				& 56.3 & 50.4 \\
		NeuralDater			& \textbf{64.1} & \textbf{58.9} \\
		\bottomrule
	\end{tabular}
	\caption{\label{neuraldater_tb:result_main}Accuracies of different methods on APW and NYT datasets for the document dating problem (higher is better). NeuralDater significantly outperforms all other competitive baselines. This is our main result. Please see \refsec{neuraldater_sec:perf_comp} for more details.}

\end{table}

%\begin{figure}[t]
%	\centering
%	\includegraphics[width=3.1in]{sections/neuraldater/images/mean_abs_dev.pdf}
%	\caption{\label{neuraldater_fig:results_mean_dev}Mean absolute deviation (in years; lower is better) between a model's top prediction and the true year in the APW dataset. We find that NeuralDater, the proposed method, achieves the least deviation. Please see \refsec{neuraldater_sec:perf_comp} for details. %Evaluating performance of different methods on dating  documents with and without time mentions.
%	}
%\end{figure}

\begin{table}[!t]
%	\begin{small}
		\centering
		\begin{tabular}{lc}
			\toprule
			Method 			 & Accuracy \\
			\midrule		
			\addlinespace
			%		CNN									& 56.3 \\
			T-GCN 								& 57.3 \\
			S-GCN + T-GCN $(K=1)$				& 57.8 \\
			S-GCN + T-GCN $(K=2)$				& 58.8 \\
			S-GCN + T-GCN $(K=3)$				& \textbf{59.1} \\
			\midrule
			Bi-LSTM 							& 58.6 \\
			Bi-LSTM + CNN 						& 59.0 \\
			Bi-LSTM + T-GCN						& 60.5 \\
			Bi-LSTM + S-GCN + T-GCN (no gate)	& 62.7 \\
			Bi-LSTM + S-GCN + T-GCN $(K=1)$		& \textbf{64.1} \\
			Bi-LSTM + S-GCN + T-GCN $(K=2)$		& 63.8 \\
			Bi-LSTM + S-GCN + T-GCN $(K=3)$		& 63.3 \\
			\bottomrule
		\end{tabular}
	\caption{\label{neuraldater_tb:result_ablation}\small Accuracies of different ablated methods on the APW dataset. Overall, we observe that incorporation of context (Bi-LSTM), syntactic structure (S-GCN) and temporal structure (T-GCN) in NeuralDater achieves the best performance. Please see \refsec{neuraldater_sec:perf_comp} for details.}
%\end{small}
\end{table}

\subsection{Results}
\label{neuraldater_sec:results}

%\subsection{Comparison with other methods}
\subsubsection{Performance Comparison}
\label{neuraldater_sec:perf_comp}

In order to evaluate the effectiveness of NeuralDater, our proposed method, we compare it against existing document dating systems and text classification models. The final results are summarized in Table \ref{neuraldater_tb:result_main}. Overall, we find that NeuralDater outperforms all other methods with a significant margin on both datasets. Compared to the previous state-of-the-art in document dating, BurstySimDater  \cite{Kotsakos:2014:BAD:2600428.2609495}, we get 19\% average absolute improvement in accuracy across both datasets. We observe only a slight gain in the performance of MaxEnt-based model (MaxEnt-Time+NER) of \cite{Chambers:2012:LDT:2390524.2390539} on combining with temporal constraint reasoner (MaxEnt-Joint). This may be attributed to the fact that the model utilizes only year mentions in the document, thus ignoring other relevant signals which might be relevant to the task. BurstySimDater performs considerably better in terms of precision compared to the other baselines,  although it significantly underperforms in accuracy. We note that NeuralDater outperforms all these prior models both in terms of precision and accuracy. We find that even generic deep-learning based text classification models, such as CNN \cite{yoon_kim}, are quite effective for the problem. However, since such a model doesn't give specific attention to temporal features in the document, its performance remains limited. From \reffig{neuraldater_fig:results_mean_dev}, we observe that NeuralDater's top prediction achieves on average the lowest deviation from the true year.

\subsubsection{Ablation Comparisons}
\label{neuraldater_sec:ablation}

For demonstrating the efficacy of GCNs and BiLSTM for the problem, we evaluate different ablated variants of NeuralDater on the APW dataset. Specifically, we validate the importance of using syntactic and temporal GCNs and the effect of eliminating BiLSTM from the model. Overall results are summarized in Table \ref{neuraldater_tb:result_ablation}. The first block of rows in the table corresponds to the case when BiLSTM layer is excluded from NeuralDater, while the second block denotes the case when BiLSTM is included. We also experiment with multiple stacked layers of T-GCN (denoted by $K$) to observe its effect on the performance of the model. 

%We observe that an isolated Temporal-GCN (T-GCN) model performs quite worse in comparison to a BiLSTM+CNN model. This validates the fact that GCNs in themselves are inefficient in capturing dependencies but on stacking with BiLSTM, it performs significantly better. 
We observe that embeddings from Syntactic GCN (S-GCN) are much better than plain GloVe embeddings for T-GCN as S-GCN encodes the syntactic neighborhood information in event and time embeddings which makes them more relevant for document dating task.

Overall, we observe that including BiLSTM in the model improves performance significantly. Single BiLSTM model outperforms all the models listed in the first block of  \reftbl{neuraldater_tb:result_ablation}. Also, some gain in performance is observed on increasing the number of T-GCN layers ($K$) in absence of BiLSTM, although the same does not follow when BiLSTM is included in the model. This observation is consistent with \cite{gcn_srl}, as multiple GCN layers become redundant in the presence of BiLSTM. We also find that eliminating edge gating from our best model deteriorates its overall performance.

In summary, these results validate our thesis that joint incorporation of syntactic and temporal structure of a document in NeuralDater results in improved performance.

\begin{figure}[t]
	%\centering
	\begin{minipage}{3.4in}
		\includegraphics[width=3.4in]{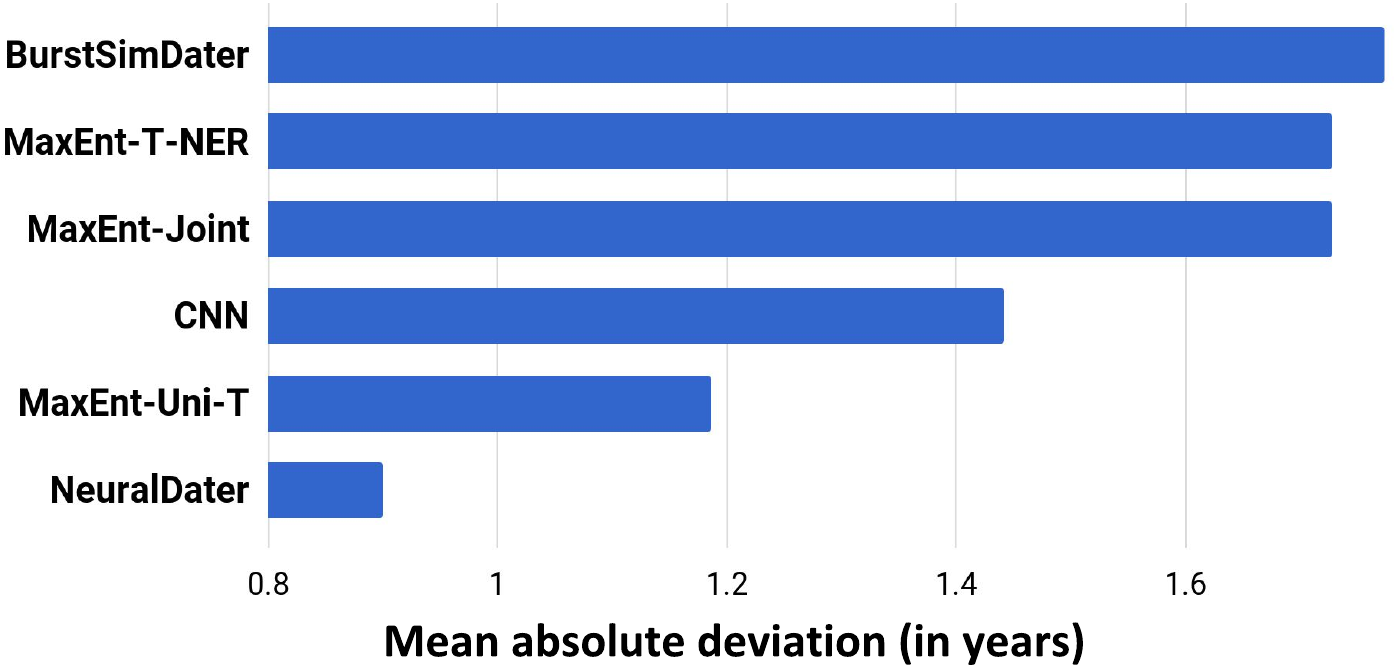}
		\subcaption{\label{neuraldater_fig:results_mean_dev}\small Mean absolute deviation (in years; lower is better) between a model's top prediction and the true year in the APW dataset. We find that NeuralDater, the proposed method, achieves the least deviation. Please see \refsec{neuraldater_sec:perf_comp} for details}
	\end{minipage}
	\begin{minipage}{3.0in}
		\includegraphics[width=3.0in]{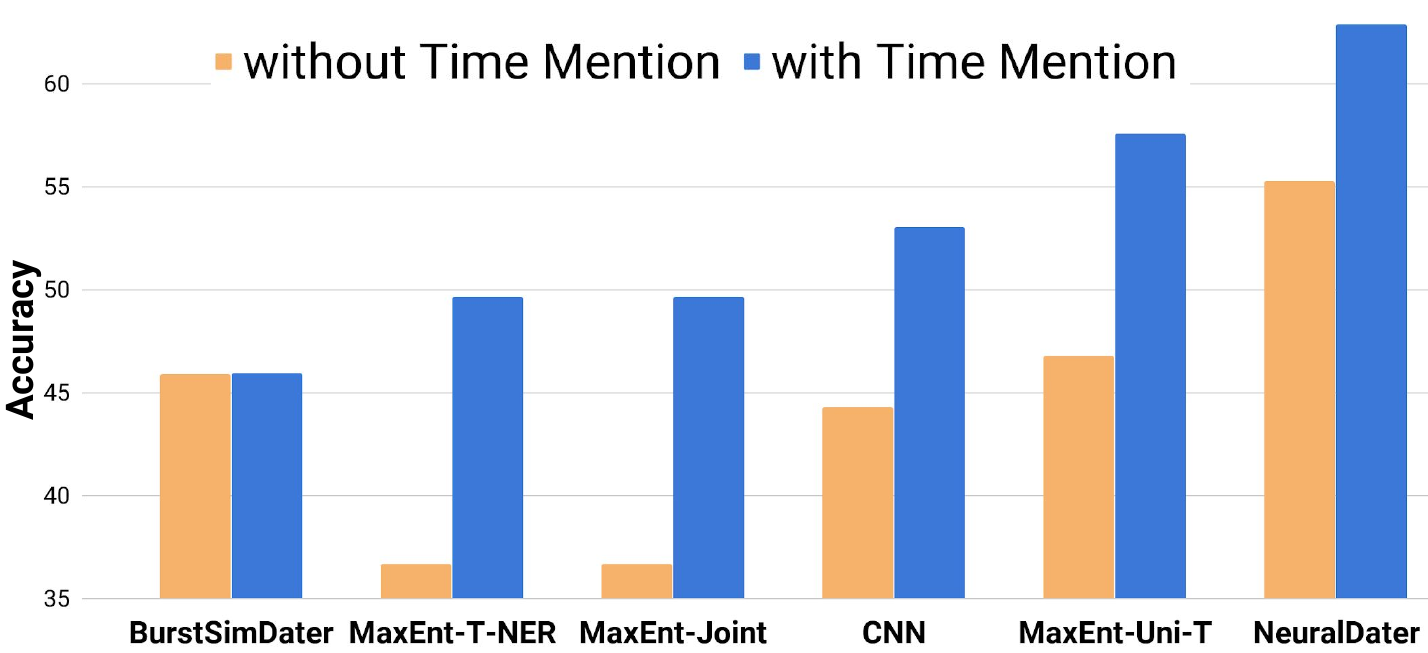}
		\subcaption{\label{neuraldater_fig:results_time_mention}\small Evaluating performance of different methods on dating  documents with and without time mentions. Please see \refsec{neuraldater_sec:discussion} for details.}
	\end{minipage}
\end{figure}

\subsubsection{Discussion and Error Analysis}
\label{neuraldater_sec:discussion}

%\begin{figure}[t]
%	\centering
%	\includegraphics[width=3.1in]{sections/neuraldater/images//time_info.pdf}
%	\caption{\label{neuraldater_fig:results_time_mention}Evaluating performance of different methods on dating  documents with and without time mentions. Please see \refsec{neuraldater_sec:discussion} for details.}
%\end{figure}

In this section, we list some of our observations while trying to identify pros and cons of NeuralDater, our proposed method. We divided the development split of the APW dataset into two sets -- those with and without any mention of time expressions (year). We apply NeuralDater and other methods to these two sets of documents and report accuracies in \reffig{neuraldater_fig:results_time_mention}. We find that overall, NeuralDater performs better in comparison to the existing baselines in both scenarios. Even though the performance of NeuralDater degrades in the absence of time mentions, its performance is still the best relatively. %But its performance significantly degrades when time mentions are not present. 
Based on other analysis, we find that NeuralDater fails to identify timestamp of documents reporting local infrequent incidents without explicit time mention. NeuralDater becomes confused in the presence of multiple misleading time mentions; it also loses out on documents discussing events which are outside the time range of the text on which the model was trained. In future, we plan to eliminate these pitfalls by incorporating additional signals from Knowledge Graphs about entities mentioned in the document. We also plan to utilize free text temporal expression \cite{temponym_paper} in documents for improving performance on this problem. % as a supplementary information to our model. 

\section{Conclusion}
\label{neuraldater_sec:conclusion}

We propose NeuralDater, a Graph Convolutional Network (GCN) based method for document dating which exploits syntactic and temporal structures in the document in a principled way. To the best of our knowledge, this is the first application of deep learning techniques for the problem of document dating. Through extensive experiments on real-world datasets, we demonstrate the effectiveness of NeuralDater over existing state-of-the-art approaches. 
We are hopeful that the representation learning techniques explored in this chapter will inspire further development and adoption of such techniques in the temporal information processing research community.
\chapter{Incorporating Syntactic and Semantic Information in Word Embeddings using Graph Convolutional Networks}
\label{chap_wordgcn}

% !TeX spellcheck = en_US
\section{Introduction}
\label{wordgcn_sec:intro}

As we saw in the last chapter, Graph Convolutional Networks prove to be very effective for exploiting different graph structure in NLP. Here, we utilize them for learning word representation. Representing words as real-valued vectors is an effective and widely adopted technique in NLP. Such representations capture %semantic and syntactic 
properties of words based on their usage and allow them to generalize across tasks. Meaningful word embeddings have been shown to improve performance on several relevant tasks, such as named entity recognition (NER) \cite{app_ner}, parsing \cite{parsing_app}, and part-of-speech (POS) tagging \cite{pos_tagging}. Using word embeddings for initializing Deep Neural Networks has also been found to be quite useful \cite{Collobert2011,multilingual_google,lisa_paper}. 
%\reminderR{PPT: include some more recent examples as well}

Most popular methods for learning word embeddings are based on the distributional hypothesis, which utilizes the co-occurrence statistics from \emph{sequential} context of words for learning word representations \cite{word2vec,glove}. More recently, this approach has been extended to include syntactic contexts \cite{deps_paper} derived from dependency parse of text. Higher order dependencies have also been exploited by \citet{ext_paper,Li2018a}. Syntax-based embeddings encode functional similarity (in-place substitutable words) rather than topical similarity (topically related words) which provides an advantage on specific tasks like question classification \cite{ext_paper}.
%\textcolor{officegreen}{However, all dependency context based methods severely expand vocabulary size \cite{deps_paper,ext_paper,Li2018a} which  limits their scalability to a large corpus. For instance, the context vocabulary in \cite{ext_paper} blows to 1.3 million for learn embeddings of 220k words.}
However, current approaches incorporate syntactic context by concatenating words with their dependency relations. % as a part of the vocabulary. 
%For example, \textit{discover} in Figure \ref{wordgcn_fig:wordgcn_overview} %, %for 
% has the following three context words: \textit{scientists\_subj}, \textit{water\_obj}, and \textit{mars\_nmod}. \reminder{
 For instance, in Figure \ref{wordgcn_fig:wordgcn_overview} \textit{scientists\_subj}, \textit{water\_obj}, and \textit{mars\_nmod} needs to be included as a part of vocabulary for utilizing the dependency context of \textit{discover}. 
 This severely expands the vocabulary, thus limiting the
scalability of models on large corpora. For instance, in \citet{deps_paper} and \citet{ext_paper}, the context vocabulary explodes to around 1.3 million for learning embeddings of 220k words.

Incorporating relevant signals from semantic knowledge sources such as WordNet \cite{wordnet}, FrameNet \cite{framenet}, and Paraphrase Database (PPDB) \cite{ppdb_paper} has been shown to improve the quality of word embeddings. Recent works utilize these by incorporating them in a neural language modeling objective function 
%\reminderR{PPT: obj of what?} 
\cite{Mo2014,japanese2018}, or as a post-processing step \cite{Faruqui2014,counter_paper}. Although existing approaches improve the quality of word embeddings, they require explicit modification for handling different types of semantic information. 

%For instance, \cite{Faruqui2014} models \textit{synonyms}, \textit{hyponyms} and \textit{hypernyms} word relationship through the same undirected edge. 
 
Recently proposed Graph Convolutional Networks (GCN) \cite{Defferrard2016,Kipf2016} have been found to be useful for encoding structural information in graphs. Even though GCNs have been successfully employed for several NLP tasks such as machine translation \cite{gcn_nmt}, semantic role labeling \cite{gcn_srl}, document dating \cite{neuraldater} and text classification \cite{gcn_text_class}, they have so far not been used for learning word embeddings, especially leveraging cues such as syntactic and semantic information. GCNs provide flexibility to represent diverse syntactic and semantic relationships between words all within one framework, without requiring relation-specific special handling as in previous methods. 
Recognizing these benefits, we make the following contributions in this chapter. 
 
\noindent
\begin{enumerate}[itemsep=2pt,parsep=0pt,partopsep=0pt,leftmargin=*,topsep=0.1pt]
	\item We propose \wordgcnMethod{}, a Graph Convolution based method for learning word embeddings. Unlike previous methods, \wordgcnMethod{} %which 
		utilizes syntactic context for learning word representations without increasing vocabulary size. 
	\item We also present \wordgcnMethodSide{}, a framework for incorporating diverse semantic knowledge (e.g., synonymy, antonymy, hyponymy, etc.) in learned word embeddings, without requiring relation-specific special handling as in previous methods.
	\item Through experiments on multiple intrinsic and extrinsic tasks, we demonstrate that our proposed methods obtain substantial improvement over state-of-the-art approaches,  
	and also yield an advantage when used in conjunction with methods such as ELMo \cite{elmo_paper}.
\end{enumerate}
% !TeX spellcheck = en_US
\section{Related Work}
\label{wordgcn_sec:related}

\textbf{Word Embeddings:} Recently, there has been much interest in learning meaningful word representations such as neural language modeling \cite{Bengio2003} based continuous-bag-of-words (CBOW) and skip-gram (SG) models \cite{word2vec}. 
%These models lead to representations with semantic regularities and perform well on similarity and analogy tasks without explicitly optimizing for them.
%\cite{Mikolov2013,Mikolov2013b} introduce continuous bag-of-words (CBOW) and skip-gram (SG) models based on neural language modeling \cite{Bengio2003}. 
%The learned embeddings possess nice structural properties which capture semantic and syntactic relationships between words.
This is further extended by \citet{glove} which learns embeddings by factorizing word co-occurrence matrix to leverage global statistical information.  
Other formulations for learning word embeddings include multi-task learning \cite{Collobert2011} and ranking frameworks \cite{Ji2015}.

%Several methods focus on the usage of negative samples \cite{word2vec,Chen2017,Guo2018}, interpretablity \cite{Subramanian2018}, task specific embeddings \cite{Maas2011,Jastrzebski2017}, alternate training of CBOW and SG models \cite{Song2017}, and incorporation of semantic knowledge by modifying loss function \cite{Mo2014,joint_app2015,Faruqui2014,Jastrzebski2017}. In this work, we propose another way of including different types of semantic knowledge while learning embeddings. \\

%Since then, there has been an interest in several different aspects of word representations. \citeauthor{Maas2011,Jastrzebski2017} focus on learning task specific word representations. Alternate training of CBOW and SG models has been explored by \cite{Song2017} while several other works explore the use of negative sampling to speed up training \cite{word2vec}. In this work, we use a more recent approach -- sampled softmax \cite{sampled_softmax} to reduce training time.\\ 

\textbf{Syntax-based Embeddings:} Dependency parse context based word embeddings is first introduced by \citet{deps_paper}. They allow encoding syntactic relationships between words and show improvements on tasks where functional similarity is more relevant than topical similarity. The inclusion of syntactic context is further enhanced through second-order \cite{ext_paper} and multi-order \cite{Li2018a} dependencies. However, in all these existing approaches, the word vocabulary is severely expanded for incorporating syntactic relationships. %limiting limits their scalability to a large corpus. %In contrast, \wordgcnMethod{}, our proposal in this chapter, uses GCNs for encoding syntactic word relationships without increasing word vocabulary. %In this chapter, we address this drawback by utilizing Graph Convolutional Networks for encoding syntactic word relationships without increasing word vocabulary.

%They used the edge labels of dependency parse and concatenated them with the target words to append the context vocabulary. They then predict the context words using this vocabulary. A follow up work by \cite{MacAvaney2018} used Universal dependencies and analysed various levels of enhancement for various dependency schemes. \cite{ext_paper} used second order dependencies. \cite{Li2018a} highlighted that the distribution of one hop context words is highly skewed which results in either the whole sentence being the context or just one or two words. They proposed using multi-hop relationship to address this issue. \cite{Qiu2015,Yin2016} used dependency based embeddings for \textit{Chinese Analogy Detection} and \textit{Aspect Term Extraction} respectively.

\textbf{Incorporating Semantic Knowledge Sources:}
%There have been several methods which aim to incorporate 
Semantic relationships such as \textit{synonymy}, \textit{antonymy}, \textit{hypernymy}, etc. from several semantic sources have been utilized for improving the quality of word representations. Existing methods either exploit them jointly \cite{rc_net_joint2014,joint_app2015,japanese2018} or as a post-processing step \cite{Faruqui2014,counter_paper}. %Our proposed method 
	\wordgcnMethod{} falls under the latter category and is more effective at incorporating semantic constraints (Section \ref{wordgcn_sec:diverse_side_results} and \ref{wordgcn_sec:rfgcn_re_stanfordsults}). 
%training word embeddings from scratch as well as 
%incorporating semantic knowledge both during training and as a post processing step. 

\textbf{Graph Convolutional Networks:} 
%The first neural network model on graphs is proposed by \cite{Scarselli2009}. The spectral and spatial construction of GCNs  is formulated by \cite{Bruna2013} which is later improved by \cite{Defferrard2016} through an efficient localized filter approximation. 
In this chapter, we use the first-order formulation of GCNs via a layer-wise propagation rule as proposed by \cite{Kipf2016}. Recently, some variants of GCNs have also been proposed \cite{lovasz_paper,confgcn}. A detailed description of GCNs and their applications can be found in \citet{Bronstein2017}. In NLP, GCNs have been utilized for semantic role labeling \cite{gcn_srl}, machine translation \cite{gcn_nmt}, and relation extraction \cite{reside}. Recently, \citet{gcn_text_class} use GCNs for text classification by jointly embedding words and documents. However, their learned embeddings are task specific whereas in our work we aim to learn task agnostic word representations.
%\reminderR{PPT: shorten GCN background, include RESIDE and neuraldater citations. Mention that they are focused on specific tasks, and not on general word embedding.} 

%To the best of our knowledge, \wordgcnMethod{} is the first application of GCNs for learning word embeddings. 
%\textcolor{officegreen}{Although GCNs have been used for tasks like relation extraction \cite{reside}, document dating and \cite{neuraldater} but these methods focused on specific tasks. GCNs are a natural fit for learning word embeddings as they provide a single framework to encode various types of relationships among words.	Although, \citet{gcn_text_class} jointly learns task specific work and document embeddings but WordGCN is the first application of GCN for learning task agnostic embeddings.}

%GCNs have been shown to improve performance on semantic role labeling \cite{gcn_srl}, event detection \cite{gcn_event}, document time-stamping \cite{neuraldater}, and machine translation \cite{gcn_nmt}.

%Generalized aggregation of neighborhood features \cite{graphsage} allows GCNs to generalize to unseen nodes. 

%\input{./sections/wordgcn/background}
%
\section{Proposed Methods: SynGCN and SemGCN}

\subsection{Overview}
\label{wordgcn_sec:overview}
%\begin{figure*}[t]
%	\centering
%	\includegraphics[width=\textwidth]{./sections/wordgcn/images/wordgcn_model.pdf}
%	\caption{\label{wordgcn_fig:model_overview}
%	}  
%\end{figure*}

%The task of learning word representations in unsupervised setting can be formulated as follows: given a text corpus, the aim is to learn a $d$-dimensional embedding for each word in the vocabulary. Our model captures sentence level relationships between words  like \textit{linear context, dependency context, hypernym, hyponym, holonym, and meronym} relationships, along with corpus wide properties like \textit{synonym}.
%%and \textit{paraphrase database}.	
%
%Given the two different levels of information we wish to include, our model employs a Sentence level GCN (S-GCN) which operates on sentence level and inspired by \cite{Faruqui2014}, a Retrofitting GCN (RF-GCN) which operates at corpus level for learning word representation. We briefly describe the two components of \wordgcnMethod{} below.

%\subsection{Overview}

The task of learning word representations in an unsupervised setting can be formulated as follows: Given a text corpus,
%= (s_1, s_2, \dots, s_m)$ comprised of sentences, 
the aim is to learn a $d$-dimensional embedding for each word in the vocabulary. 
%In \wordgcnMethod{}, for each sentence, we construct a directed labeled sentence graph, with the set of words as nodes and edges denoting word relationships. 
%The set $\mathcal{E}_s$ incorporates different types of directed relationships between words in the sentence. 
Most of the distributional hypothesis based approaches
% \cite{Mikolov2013b,glove} 
only utilize sequential context for each word in the corpus. However, this becomes suboptimal when the relevant context words lie beyond the window size. For instance in Figure \ref{wordgcn_fig:wordgcn_overview}, a relevant context word \textit{discover} for \textit{Mars} is missed if the chosen window size is less than $3$. %For instance, in Figure \ref{wordgcn_fig:wordgcn_overview}, \textit{discover}, a relevant context word for \textit{Mars} is missed if the chosen window size is less than $3$. 
On the contrary, a large window size might allow irrelevant words to influence word embeddings negatively. 

Using dependency based context helps to alleviate this problem. However, all existing syntactic context based methods  \cite{deps_paper,ext_paper,Li2018a} severely expand vocabulary size (as discussed in Section \ref{wordgcn_sec:intro}) which  limits their scalability to a large corpus. To eliminate this drawback, we propose \wordgcnMethod{} which employs Graph Convolution Networks 
%\cite{Defferrard2016,Kipf2016} 
to better encode syntactic information in embeddings. We prefer GCNs over other graph encoding architectures such as Tree LSTM \cite{tree_lstm} as GCNs do not restrict graphs to be trees and have been found to be more effective at capturing global information \cite{gcn_re_stanford}. Moreover, they give substantial speedup as they do not involve recursive operations which are difficult to  parallelize. The overall architecture is shown in Figure \ref{wordgcn_fig:wordgcn_overview}, for more details refer to Section \ref{wordgcn_sec:wordgcn_details}. 
%We describe it in more detail in 

Enriching word embeddings with semantic knowledge helps to improve their quality for several NLP tasks. %Semantic sources like PPDB \cite{ppdb_paper}, WordNet \cite{wordnet} and \cite{framenet} provide various types of relationships like \textit{synonymy}, \textit{antonymy}, \textit{hyponymy} etc. 
%two types of semantic constraints: \textit{symmetric} and \textit{asymmetric}. For instance, relations like \textit{synonym}, \textit{antonym} are symmetric in nature whereas relations like \textit{hyponym}, \textit{hypernym} are asymmetric. Moreover, each relation represents different constraints between words. 
%<<<<<<< HEAD
%Existing approaches are either incapable of utilizing these diverse relations or need to be explicitly modeled for exploiting them. In this chapter, we propose \wordgcnMethodSide{} which automatically learns to utilize multiple semantic constraints by modeling them as different edge types. 
%\wordgcnMethodSide{} can be used as a post-processing method similar to \citet{Faruqui2014,counter_paper}. We describe it in more detail in Section \ref{wordgcn_sec:rfgcn_details}.
%=======
Existing approaches are either incapable of utilizing these diverse relations or need to be explicitly modeled for exploiting them. 
In this chapter, we propose \wordgcnMethodSide{} which automatically learns to utilize multiple semantic constraints by modeling them as different edge types. 
%\wordgcnMethodSide{} 
It can be used as a post-processing method similar to \citet{Faruqui2014,counter_paper}. We describe it in more detail in Section \ref{wordgcn_sec:rfgcn_details}.
%>>>>>>> 69048e861a43230a562ad4fd5ee4d0dc73b478cf

% !TeX spellcheck = en_US

\begin{figure*}[t]
	\centering
	\includegraphics[scale=0.8]{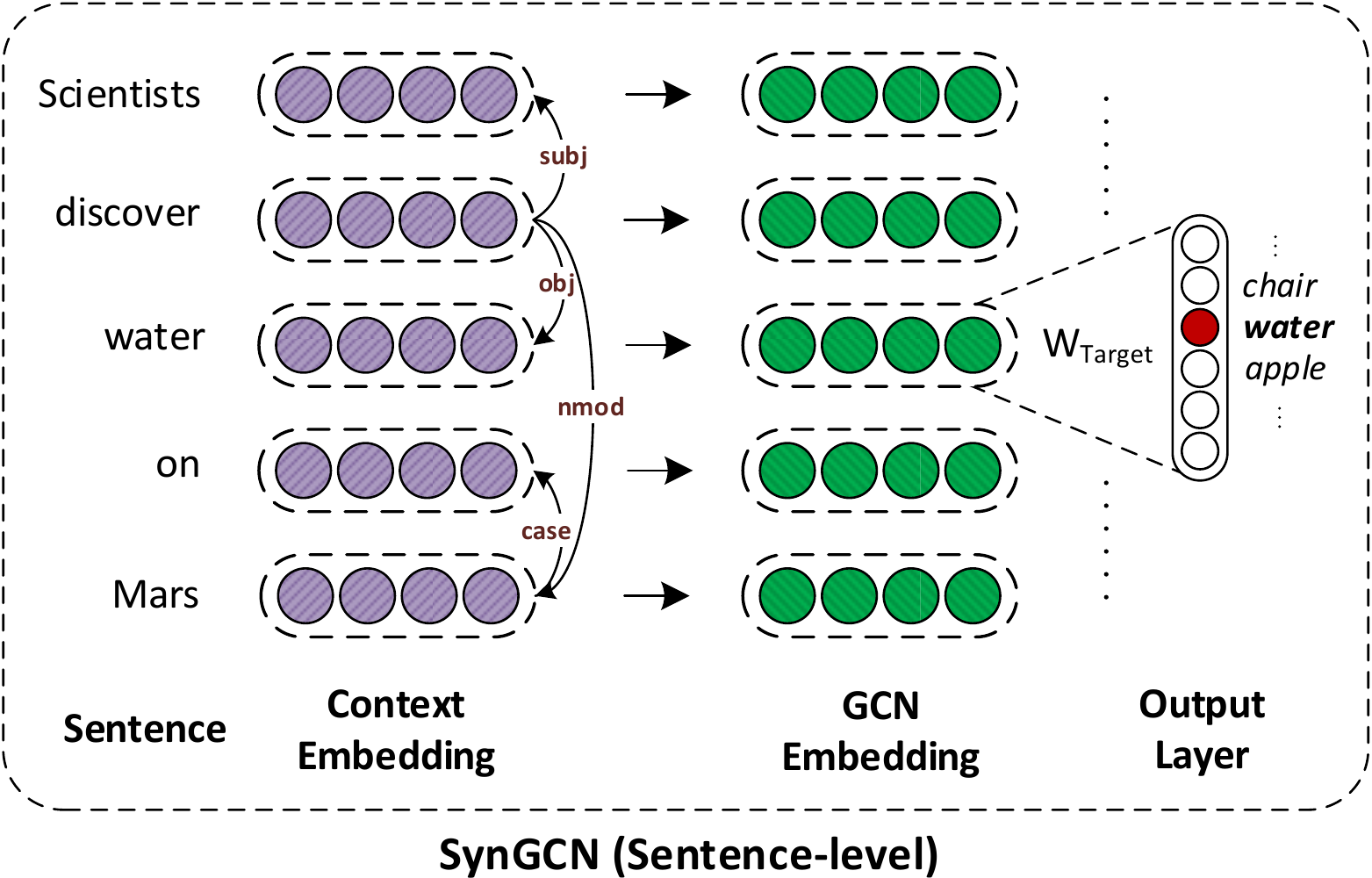}
	\caption{\label{wordgcn_fig:wordgcn_overview} \small Overview of \wordgcnMethod{}: 
		\wordgcnMethod{} employs Graph Convolution Network for utilizing dependency context for learning word embeddings. For each word in vocabulary, the model learns its representation by aiming to predict each word based on its dependency context encoded using GCNs. Please refer Section \ref{wordgcn_sec:wordgcn_details} for more details. 
		%\wordgcnMethodSide{} (\textbf{Right}) is our proposed framework for incorporating diverse semantic constraints in pre-trained embeddings. Double-headed edges denote edges in both directions. Please refer Section \ref{wordgcn_sec:rfgcn_details} for details.
	}  
\end{figure*}

\subsection{\wordgcnMethod{}}
\label{wordgcn_sec:wordgcn_details}
In this section, we provide a detailed description of our proposed method, \wordgcnMethod{}. %, the first component of \method{{. 
%our proposed method, \methodsynfull{} (\wordgcnMethod{})
Following \citet{Mikolov2013b,deps_paper,ext_paper}, we separately define target and context embeddings for each word in the vocabulary as parameters in the model.
For a given sentence $s = (w_1, w_2, \ldots, w_n)$, we first extract its dependency parse graph $\mathcal{G}_s= (\mathcal{V}_s, \mathcal{E}_s)$ using Stanford CoreNLP parser \cite{stanford_corenlp}. Here, $\mathcal{V}_s = \{w_1, w_2, \dots, w_n\}$ and $\mathcal{E}_s$ denotes the labeled directed dependency edges of the form $(w_i, w_j, l_{ij})$, where $l_{ij}$ is the dependency relation of $w_i$ to $w_j$. 

Similar to \citet{Mikolov2013b}'s continuous-bag-of-words (CBOW) model, which defines the context of a word $w_i$ as 
$\m{C}_{w_i} = \{w_{i+j} : -c \leq j \leq c, j \neq 0\}$ for a window of size $c$, we define the context as its neighbors in $\m{G}_s$, i.e., $\m{C}_{w_i} = \m{N}(w_i)$. Now, unlike CBOW which takes the sum of the context embedding of words in $\m{C}_{w_i}$ to predict $w_i$, we apply directed Graph Convolution Network (as defined in Section \ref{sec:directed_gcn}) on $\m{G}_s$ with context embeddings of words in $s$ as input features. Thus, for each word $w_i$ in $s$, we obtain a representation $h^{k+1}_i$ after $k$-layers of GCN using Equation \ref{eqn:gated_gcn} which we reproduce below for ease of readability (with one exception as described below).
\begin{equation*}
h^{k+1}_{i} = f  \left(\sum_{j \in \mathcal{N}(i)} g^{k}_{l_{ij}} \times \left({W}^{k}_{l_{ij}} h^{k}_{j} + b^{k}_{l_{ij}}\right)\right)
\end{equation*}

%\begin{equation}
%h^{k+1}_{i} = f \left(\sum_{{w_j}  \in \mathcal{N}({w_i})} \left(W^{k}_{l_{ij}} \cdot h^{k}_{j} + b^{k}_{l_{ij}}\right)\right).
%\label{eqn:main_details}
%\end{equation}
%
%Here, $W^{k}_{l_{ij}}$ and $b^{k}_{l_{ij}}$ are model parameters which are learned during training and rest of the terms are as defined in Section \ref{wordgcn_sec:gcn_directed}. Unlike the regular GCN formulation \cite{Kipf2016}, 
Please note that unlike in Equation \ref{eqn:gated_gcn}, we use 
	$\mathcal{N}({i})$ instead of $\mathcal{N_{+}}({i})$ in \wordgcnMethod{}, i.e., we do not include self-loops in $\mathcal{G}_s$.
This helps to avoid overfitting to the initial embeddings, which is undesirable in the case of \wordgcnMethod{} as it uses random initialization. We note that similar strategy has been followed by \citet{Mikolov2013b}. 
%This is done as %would lead to an undesirable scenario 
%otherwise it will be equivalent to including the target word as one of the context words. % in CBOW model. 
%In such cases, the model will %tries to 
%predict the target word using its own embedding rather than its context. Hence, this does not allow the model to learn meaningful word representations. 
Furthermore, to handle erroneous edges in automatically constructed dependency parse graph, we perform edge-wise gating (Section \ref{sec:directed_gcn}) to give importance to relevant edges and suppress the noisy ones. The embeddings obtained are then used to calculate the loss as described in Section \ref{wordgcn_sec:loss}.

%Contrary to the standard word-embedding approaches \cite{Mikolov2013b,glove} which rely on \textit{sequential context},
 \wordgcnMethod{} utilizes \textit{syntactic context} to learn more meaningful word representations. We validate this in Section \ref{wordgcn_sec:intrinsic_results}. Note that, the word vocabulary remains unchanged during the entire learning process, this makes \wordgcnMethod{} more scalable compared to the existing approaches.
 
 Note that, \wordgcnMethod{} is a generalization of CBOW model, as shown below.

%Compared to the existing dependency context based word embedding approaches \cite{deps_paper,ext_paper,Li2018a}, \wordgcnMethod{} %our method does not increase vocabulary size during the entire learning process which makes it more scalable.\reminder{Mentioned for the third time, we can skip it here as we are out of space as well.}  \reminder{The last part is repeated many times.}

\begin{theorem}
\label{th:cbow_generalization}
\wordgcnMethod{} is a generalization of Continuous-bag-of-words (CBOW) model.
\end{theorem}
\begin{proof}
The reduction can be obtained as follows. For a given sentence $s$, take the neighborhood of each word $w_i$ in $\m{G}_s$ as it sequential context, i.e., $\m{N}(w_i) = \{w_{i+j} : -c \leq j \leq c, j \neq 0 \} \ \forall w_i \in s$. Now, if the number of GCN layers are restricted to $1$ and the activation function is taken as identity $(f(x) =x)$, then Equation \ref{eqn:gated_gcn} reduces to
\[
h_{i} = \sum_{-c \leq j \leq c, j \neq 0} \left(g_{l_{ij}} \times \left({W}_{l_{ij}} h_{j} + b^{k}_{l_{ij}}\right)  \right). 
\]
Finally, $W^{k}_{l_{ij}}$ and $b^{k}_{l_{ij}}$ can be fixed to an identity matrix $(\mathbb{\bold{I}})$ and a zero vector $(\bold{0})$, respectively, and edge-wise gating $(g_{l_{ij}})$ can be set to $1$. This gives 
\[
h_{i} = \sum_{-c \leq j \leq c, j \neq 0} \left( \mathbb{\bold{I}} \cdot h_{j} + \bold{0} \right)= \sum_{-c \leq j \leq c, j \neq 0} h_{j} ,
\]
which is the hidden layer equation of CBOW model. 
\end{proof}

%\section{\methodsidefull{} Details}
%\subsection{\methodsidefull{}}

\begin{figure}[t]
	\centering
	\includegraphics[scale=0.8]{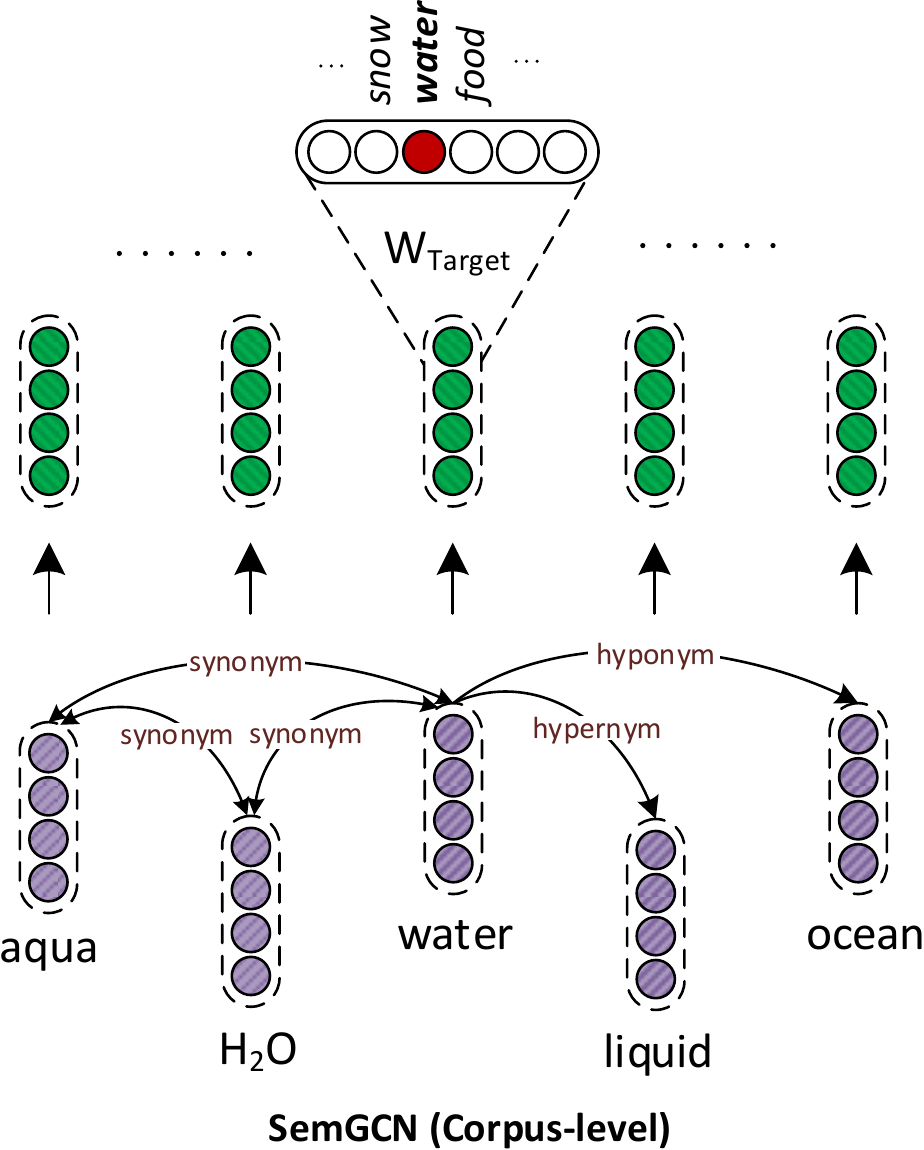}
	\caption{\label{wordgcn_fig:rfgcn_model}Overview of \wordgcnMethodSide{}, our proposed Graph Convolution based framework for incorporating diverse semantic information in learned embeddings. Double-headed edges denote two edges in both directions. Please refer to Section \ref{wordgcn_sec:rfgcn_details} for more details.}  
\end{figure}

\subsection{\wordgcnMethodSide{}}
\label{wordgcn_sec:rfgcn_details} 
%Incorporating semantic knowledge from sources like PPDB \cite{ppdb_paper} and WordNet \cite{wordnet} in word embeddings has been shown to improve their performance on several NLP tasks \cite{Faruqui2014}.\reminder{Repetition: Mention in section 4 as well. We can just start from the next line.} \reminder{The previous sentence is there in related work.} 
% describe \wordgcnMethodSide{}, \wordgcnMethod{}'s second component, which incorporates semantic knowledge in pre-trained embeddings. 
In this section, we propose another Graph Convolution based framework, \wordgcnMethodSide{}, for incorporating semantic knowledge in pre-trained word embeddings. 
Most of the existing approaches 
%utilize semantic constraints either jointly during training \cite{rc_net_joint2014,joint_app2015,japanese2018} or as a post-processing step \cite{Faruqui2014,counter_paper}. \reminder{The last sentence is also in realted work.}
%However, approaches 
like \citet{Faruqui2014,counter_paper} are restricted to handling symmetric relations like \textit{synonymy} and \textit{antonymy}. On the other hand, although recently proposed \cite{japanese2018} is capable of handling  asymmetric information, it still requires manually defined relation strength function which can be labor intensive and suboptimal. 

\wordgcnMethodSide{} is capable of incorporating both symmetric as well as asymmetric information jointly. Unlike \wordgcnMethod{}, \wordgcnMethodSide{} operates on a corpus-level directed labeled graph with words as nodes and edges representing semantic relationship among them from different sources. For instance, in Figure \ref{wordgcn_fig:rfgcn_model}, semantic relations such as \textit{hyponymy}, \textit{hypernymy} and \textit{synonymy} are represented together in a single graph. Symmetric information is handled by including a directed edge in both directions. Given the corpus level graph $\m{G}$, the training procedure is similar to that of \wordgcnMethod{}, i.e.,  predict the word $w$ based on its neighbors in $\m{G}$. Inspired by \citet{Faruqui2014}, we preserve the semantics encoded in pre-trained embeddings by initializing both target and context embeddings with given word representations and keeping target embeddings fixed during training. \wordgcnMethodSide{} uses Equation \ref{eqn:gated_gcn} to update node embeddings. Please note that in this case $\mathcal{N_{+}}(v)$ is used as the neighborhood definition to preserve the initial learned representation of the words.

%\section{Training Loss} 
\subsection{Training Details} 
\label{wordgcn_sec:loss}
Given the GCN representation $(h_t)$ of a word $(w_t)$, the training objective of \wordgcnMethod{} and \wordgcnMethodSide{} is to predict the target word given its neighbors in the graph. Formally, for each method we maximize the following objective\footnote{We also experimented with joint \wordgcnMethod{} and \wordgcnMethodSide{} model but our preliminary experiments gave suboptimal performance as compared to the sequential model. This can be attributed to the fact that syntactic information is orders of magnitude greater than the semantic information available. Hence, the semantic constraints are not effectively utilized. We leave the analysis of the joint model as a future work.}.
% \begingroup\makeatletter\def\f@size{10}\check@mathfonts	
\[
	E = \sum_{t=1}^{|V|} \log P(w_t | w^{t}_{1}, w^{t}_{2} \dots w^{t}_{N_t})
\] 
%\endgroup
where, $w_t$ is the target word and $w^{t}_{1}, w^{t}_{2} \dots w^{t}_{N_t}$ are its neighbors in the graph. 
%Since a word has different semantic properties depending upon weather it appears as a target word or as a neighboring word, we keep two embeddings (target embedding and neighbor embedding) for each word similar to \cite{word2vec}. 
The probability $P(w_t | w^{t}_{1}, w^{t}_{2} \dots w^{t}_{N_t})$ is calculated using the softmax function, defined as
 \begingroup\makeatletter\def\f@size{10}\check@mathfonts	
\[	
P(w_t | w^{t}_{1}, w^{t}_{2} \dots w^{t}_{N_t}) = \frac{\exp(v^T_{w_t} h_t)}{\sum_{i=1}^{|V|} \exp(v^T_{w_{i} }h_t)}.
\]
\endgroup
Hence, $E$ reduces to 
 \begingroup\makeatletter\def\f@size{10}\check@mathfonts	
\begin{equation}
E = \sum_{t=1}^{|V|}  \left( v^T_{w_t} h_t - \log \sum_{i=1}^{|V|} \exp(v^T_{w_{i} }h_t ) \right),
\label{wordgcn_eqn:full_softmax}
\end{equation}
\endgroup
where, $h_t$ is the GCN representation of the target word $w_t$ and $v_{w_t}$ is its target embedding.

The second term in Equation \ref{wordgcn_eqn:full_softmax} is computationally expensive as the summation needs to be taken over the entire vocabulary. This can be overcome using several approximations like noise-contrastive estimation \cite{Gutmann2010} and hierarchical softmax \cite{hierarchical_softmax1}. In our methods, we use negative sampling as used by \citet{Mikolov2013b}.

\section{Experiments}

\subsection{Experimental Setup}
\label{wordgcn_sec:experiments}

\subsubsection{Dataset and Training}
\label{wordgcn_sec:dataset}
In our experiments, we use Wikipedia\footnote{https://dumps.wikimedia.org/enwiki/20180301/} corpus for training the models.  After discarding too long and too short sentences, we get an average sentence length of nearly $20$ words. The corpus consists of 57 million sentences with 1.1 billion tokens and 1 billion syntactic dependencies. %extracted using Stanford CoreNLP.

\subsubsection{Baselines}
\label{wordgcn_sec:baselines}

For evaluating \wordgcnMethod{} (Section \ref{wordgcn_sec:wordgcn_details}), we compare against the following baselines:
\begin{itemize}[itemsep=1pt,parsep=0pt,topsep=2pt]
	\item \textbf{Word2vec} is continuous-bag-of-words model originally proposed by \citet{Mikolov2013b}.
	\item \textbf{GloVe} \cite{glove}, a log-bilinear regression model which leverages global co-occurrence statistics of corpus.
	\item \textbf{Deps} \cite{deps_paper} is a modification of skip-gram model which uses dependency context in place of sequential context. 
	\item \textbf{EXT} \cite{ext_paper} is an extension of Deps which utilizes second-order dependency context features. 	 
\end{itemize}

\noindent \wordgcnMethodSide{} (Section \ref{wordgcn_sec:rfgcn_details}) model is evaluated against the following methods:
\begin{itemize}[itemsep=1pt,parsep=0pt,topsep=2pt]

	\item \textbf{Retro-fit} \cite{Faruqui2014} is a post-processing procedure which uses similarity constraints from semantic knowledge sources. 
	\item \textbf{Counter-fit} \cite{counter_paper}, a method for injecting both antonym and synonym constraints into word embeddings.
	\item \textbf{JointReps} \cite{japanese2018}, a joint word representation learning method which simultaneously utilizes the corpus and KB.
%	\item \textbf{\wordgcnMethod{}} and \textbf{\wordgcnMethodSide{}} are the methods proposed in this chapter. Please refer Section \ref{wordgcn_sec:wordgcn_details} and Section \ref{wordgcn_sec:rfgcn_details} for more details.
%	\begin{itemize}
%		\item 
%	\end{itemize}
	%\item \textbf{\cite{Li2018a}} incorporates multi-order dependency based context with several essential constraints. 
\end{itemize}

\subsubsection{Evaluation method}
\label{wordgcn_sec:evaluation_tasks}
To evaluate the effectiveness of our proposed methods,  we compare them against the baselines on the following intrinsic and extrinsic tasks\footnote{Details of hyperparameters are in supplementary.}: 

\begin{itemize}[itemsep=2pt,parsep=0pt,topsep=2pt]
\item \textbf{Intrinsic Tasks:} \\	
\noindent \textbf{Word Similarity} is the task of evaluating closeness between semantically similar words.  Following \citet{ext_paper,glove}, we evaluate on 
Simlex-999 \cite{simlex_dataset}, 
WS353
\cite{ws353_dataset}, 
and RW \cite{rw_dataset} datasets.
%MTurk \cite{mturk_dataset}, RG65, and MEN \cite{men_dataset}. 
%This task evaluates the closeness of pair of related words in the embedded space. Average similarity scores are reported on \textit{WS353S, WS353R, MTurk, RG65, and MEN} datasets.

\noindent \textbf{Concept Categorization} involves grouping nominal concepts into natural categories. For instance, \textit{tiger} and \textit{elephant} should belong to \textit{mammal} class. In our experiments, we evalute on AP \cite{ap_dataset}, Battig \cite{battig_dataset}, BLESS \cite{bless_dataset}, ESSLI \cite{essli_dataset} datasets.

\noindent \textbf{Word Analogy} task is to predict word $b_{2}$, given three words $a_{1}$, $a_{2}$, and $b_{1}$, such that the relation $b_{1}:b_{2}$ is same as the relation $a_{1}:a_{2}$. We compare methods on MSR \cite{Mikolov2013b} and SemEval-2012 \cite{semeval_dataset}.

%\noindent \textbf{Question Classification} is the problem of categorizing questions into different types. We use TREC dataset \cite{trec_dataset} which comprises of six question types. %Out of 5,452 training questions, we take 10\% for validation.

%\noindent \textbf{News Categorization} involves identifying category of a given document. Following \cite{Faruqui2014}, we use document belonging to \textit{Computers:IBM} and \textit{Computers:Mac} from 20 Newsgroup dataset and model it as a binary classification task.
%We use 20 Newsgroups dataset for three different binary classification task following \cite{Faruqui2015}. Each task is about categorizing a document according to two related categories. For all three task we follow \cite{Subramanian2018} train/val/test splits. 

%\textbf{Noun Phrase Bracketing:} The task of Noun Phrase Bracketing to locate the noun phrases along with their start and end point within the sentence. We use NP bracketing task from \cite{Lazaridou2013} where a 3 word noun phrase need to be classified as left or right bracketed. The dataset is splitted into 10 fold and contain around 2200 noun phrases. We use first fold as validation and use rest as test set. We concatenate the embedding of the words to get the embedding of the noun phrase \cite{Faruqui2015}.
%
\begin{table*}[!t]
	\centering
	%	\resizebox{\textwidth}{!}{
	\begin{small}
		\resizebox{\columnwidth}{!}{
		\begin{tabular}{lcccc|cccc|cc}
			\toprule
			& \multicolumn{4}{c}{Word Similarity} & \multicolumn{4}{c}{Concept Categorization} & \multicolumn{2}{c}{Word Analogy} \\ 
			\cmidrule(r){2-5} \cmidrule(r){6-9} \cmidrule(r){10-11} 
			Method        & WS353S        & WS353R        & SimLex999     & RW            & AP            & Battig        & BLESS         & ESSLI     & SemEval2012   & MSR           \\
			\midrule
			Word2vec      & 71.4          & 52.6          & 38.0          & 30.0          & 63.2          & 43.3          & 77.8          & 63.0          & 18.9          & 44.0          \\
			GloVe         & 69.2          & \bf{53.4}     & 36.7          & 29.6          & 58.0          & 41.3          & 80.0          & 59.3          & 18.7          & 45.8          \\
			Deps          & 65.7          & 36.2          & 39.6          & 33.0          & 61.8          & 41.7          & 65.9          & 55.6          & 22.9          & 40.3          \\
			EXT           & 69.6          & 44.9          & 43.2          & 18.6          & 52.6          & 35.0          & 65.2          & 66.7          & 21.8          & 18.8          \\
			\midrule
			\wordgcnMethod{}  & \bf{73.2}     & 45.7          & \bf{45.5}     & \bf{33.7}     & \bf{69.3}     & \bf{45.2}     & \bf{85.2}     & \bf{70.4}     & \bf{23.4}     & \bf{52.8}     \\
			\bottomrule
		\end{tabular}
	}
	\end{small}
	%	}
	\caption{\label{wordgcn_tbl:intrinsic} \small \textbf{\wordgcnMethod{} Intrinsic Evaluation:} Performance on word similarity (Spearman correlation), concept categorization (cluster purity), and word analogy (Spearman correlation). Overall, \wordgcnMethod{} outperforms other existing approaches in $9$ out of $10$ settings. Please refer to Section \ref{wordgcn_sec:intrinsic_results} for more details. 
	}
\end{table*}

\item \textbf{Extrinsic Tasks:} \\
\noindent \textbf{Named Entity Recognition (NER)} is the task of locating and classifying entity mentions into categories like \textit{person}, \textit{organization} etc. We use \citet{e2e_coref_model}'s model on CoNLL-2003 dataset \cite{conll03_dataset} for evaluation.

\noindent \textbf{Question Answering} in Stanford Question Answering Dataset (\textbf{SQuAD}) \cite{squad_dataset} involves identifying answer to a question as a segment of text from a given passage. Following \citet{elmo_paper}, we evaluate using \citet{docqa_model}'s model. 

\noindent \textbf{Part-of-speech (POS) tagging} aims at associating with each word, a unique tag describing its syntactic role. For evaluating word embeddings, we use \citet{e2e_coref_model}'s model on Penn Treebank POS dataset \cite{penn_pos_dataset}.

\noindent \textbf{Co-reference Resolution (Coref)} involves identifying all expressions that refer to the same entity in the text. To inspect the effect of embeddings, we use \citet{e2e_coref_model}'s model on CoNLL-2012 shared task dataset  \cite{coref_conll12_dataset}.
\end{itemize}

% !TeX spellcheck = en_US
\subsection{Results}
\label{wordgcn_sec:results}

In this section, we attempt to answer the following questions.
\begin{itemize}[itemsep=2pt,topsep=2pt,parsep=1pt]
	\item[Q1.] Does \wordgcnMethod{} learn better word embeddings than existing approaches? (Section \ref{wordgcn_sec:intrinsic_results})
	%\item[Q2.] How effective is \wordgcnMethod{}'s embeddings on downstream tasks? (Section \ref{wordgcn_sec:extrinsic_results})
	\item[Q2.] Does \wordgcnMethodSide{} effectively handle diverse semantic information as compared to other methods? (Section \ref{wordgcn_sec:diverse_side_results})
	\item[Q3.] How does \wordgcnMethodSide{} perform compared to other methods when provided with the same semantic constraints? (Section \ref{wordgcn_sec:rfgcn_re_stanfordsults})
	\item[Q4.] Does dependency context based embedding encode complementary information compared to ELMo? (Section \ref{wordgcn_sec:elmo_comp})
\end{itemize}

\subsubsection{\wordgcnMethod{} Evaluation}
\label{wordgcn_sec:intrinsic_results}

The evaluation results on intrinsic tasks -- word similarity, concept categorization, and analogy --  are summarized in Table \ref{wordgcn_tbl:intrinsic}. We report Spearman correlation for word similarity and analogy tasks and cluster purity for concept categorization task. Overall, we find that  \wordgcnMethod{}, our proposed method, outperforms all the existing word embedding approaches in 9 out of 10 settings. The inferior performance of \wordgcnMethod{} and other dependency context based methods on WS353R dataset compared to sequential context based methods is consistent with the observation reported in \citet{deps_paper,ext_paper}. This is because the syntactic context based embeddings capture functional similarity rather than topical similarity (as discussed in Section \ref{wordgcn_sec:intro}). On average, we obtain around $1.5$\%,  $5.7$\%  and $7.5$\% absolute increase in performance on word similarity, concept categorization and analogy tasks compared to the best performing baseline. The results demonstrate that the learned embeddings from \wordgcnMethod{} more effectively capture semantic and syntactic properties of words.

We also evaluate the performance of different word embedding approaches on the downstream tasks as defined in Section \ref{wordgcn_sec:evaluation_tasks}. 
%Inspired from the appreciable improvement from ELMo \cite{elmo_paper} enhanced models, we also asses the performance of \wordgcnMethod{} when used along with ELMo. 
The experimental results are summarized in Table \ref{wordgcn_tbl:extrinsic}. Overall, we find that \wordgcnMethod{} either outperforms or performs comparably to other methods on all four tasks. Compared to the sequential context based methods, dependency based methods perform superior at question answering task as they effectively encode syntactic information. This is consistent with the observation of \citet{elmo_paper}. 

\begin{table}[t]
	\centering
%	\resizebox{\columnwidth}{!}{
\small
		\begin{tabular}{lcccc}
			\toprule
			Method          & POS                   & SQuAD                 & NER                   & Coref \\ 
			\midrule
			Word2vec        & 95.0$\pm$0.1          & 78.5$\pm$0.3          & 89.0$\pm$0.2          & 65.1$\pm$0.3 \\ 
			GloVe           & 94.6$\pm$0.1          & 78.2$\pm$0.2          & 89.1$\pm$0.1          & 64.9$\pm$0.2 \\ 
			Deps            & 95.0$\pm$0.1          & 77.8$\pm$0.3          & 88.6$\pm$0.3          & 64.8$\pm$0.1 \\ 
			EXT             & 94.9$\pm$0.2          & \bf{79.6$\pm$0.1}     & 88.0$\pm$0.1          & 64.8$\pm$0.1 \\ 
			\midrule
			\wordgcnMethod{}      	& \bf{95.4$\pm$0.1}     & \bf{79.6$\pm$0.2}     & \bf{89.5$\pm$0.1}     & \bf{65.8$\pm$0.1 }\\ 
			\bottomrule
		\end{tabular}	
%	}
	\caption{\label{wordgcn_tbl:extrinsic} \small \textbf{\wordgcnMethod{} Extrinsic Evaluation:} Comparison on parts-of-speech tagging (POS), question answering (SQuAD), named entity recognition (NER), and co-reference resolution (Coref). \wordgcnMethod{} performs comparable or outperforms all existing approaches on all tasks. Refer Section \ref{wordgcn_sec:intrinsic_results} for details.}
\end{table}

\begin{table*}[!t]
	\centering
	\resizebox{\textwidth}{!}{
		\begin{tabular}{l|ccc|ccc|ccc|ccc|ccc}
			\toprule
			\multicolumn{1}{c}{Init Embeddings (=X)}			& \multicolumn{3}{c}{Word2vec} 		& \multicolumn{3}{c}{GloVe} 		& \multicolumn{3}{c}{Deps} 			& \multicolumn{3}{c}{EXT} & \multicolumn{3}{c}{\wordgcnMethod} 		\\ 
			\cmidrule(r){1-1} \cmidrule(r){2-4} \cmidrule(r){5-7} \cmidrule(r){8-10} \cmidrule(r){11-13} \cmidrule(r){14-16}
			Datasets		          & WS353      & AP         & MSR       & WS353       & AP        & MSR       & WS353      & AP        & MSR       & WS353      & AP         & MSR        & WS353      & AP        & MSR    \\
			\midrule
			Performance of X     & 63.0      & 63.2       & 44.0      & 58.0       & 60.4      & 45.8      & 55.6      & 64.2      & 40.3      & 59.3      & 53.5       & 18.8       & 61.7      & 69.3      & 52.8 \\
			\midrule
			Retro-fit (X,1)       & 63.4      & \bf{67.8}  & \bf{46.7} & 58.5       & 61.1      & \bf{47.2} & 54.8      & 64.7      & 41.0      & 61.6      & 55.1       & 40.5       & 61.2      & 67.1      & 51.4   \\
			Counter-fit (X,2)     & 60.3      & 62.9       & 31.4      & 53.7       & 62.5      & 29.6      & 46.9      & 60.4      & 33.4      & 52.0      & 54.4       & 35.8       & 55.2      & 66.4      & 31.7  \\
			JointReps (X,4)     & 60.9      & 61.1       & 28.5      & 59.2       & 55.5      & 37.6      & 54.8      & 58.7      & 38.0      & 58.8      & 54.8       & 20.6       & 60.9      & 68.2      & 24.9  \\
			\midrule
			\wordgcnMethodSide{} (X,4)   & \bf{64.8} & \bf{67.8}  & 36.8      & \bf{63.3}  & \bf{63.2} & 44.1      & \bf{62.3} & \boxed{\bf{69.3}} & \bf{41.1} & \bf{62.9} & \bf{67.1}  & \bf{52.1}  & \boxed{\bf{65.3}} & \boxed{\bf{69.3}} & \boxed{\bf{54.4}} \\
			\bottomrule
		\end{tabular}
	}
	\caption{\label{wordgcn_tbl:rfgcn_intrinsic} \small \textbf{\wordgcnMethodSide{} Intrinsic Evaluation:} Evaluation of different methods for incorporating diverse semantic constraints initialized using various pre-trained embeddings (X). M(X, R) denotes the fine-tuned embeddings using method M taking X as initialization embeddings. R denotes the type of semantic relations used as defined in Section \ref{wordgcn_sec:diverse_side_results}. \wordgcnMethodSide{} outperforms other methods in $13$ our of $15$ settings. \wordgcnMethodSide{} with \wordgcnMethod{} gives the best performance across all tasks (highlighted using \boxed{\cdot}). Please refer Section \ref{wordgcn_sec:diverse_side_results} for details.}
\end{table*}

\begin{table}[t]
	\centering
	\small
%	\resizebox{\columnwidth}{!}{
		\begin{tabular}{lcccc}
			\toprule
			Method          & POS                   & SQuAD                 & NER                   & Coref \\ 
			\midrule
			X = \wordgcnMethod{}      	& 95.4$\pm$0.1    		& 79.6$\pm$0.2   		& \textbf{89.5$\pm$0.1} & 65.8$\pm$0.1 \\ 
			Retro-fit (X,1)      & 94.8$\pm$0.1          & 79.6$\pm$0.1          & 88.8$\pm$0.1          & 66.0$\pm$0.2 \\ 
			Counter-fit (X,2)    & 94.7$\pm$0.1          & 79.8$\pm$0.1          & 88.3$\pm$0.3          & 65.7$\pm$0.3 \\ 
			JointReps (X,4)      & 95.4$\pm$0.1          & 79.4$\pm$0.3          & 89.1$\pm$0.3          & 65.6$\pm$0.1 \\ 
			\midrule
			\wordgcnMethodSide{} (X,4)  & \textbf{95.5$\pm$0.1} & \bf{80.4$\pm$0.1}     & \textbf{89.5$\pm$0.1} & \textbf{66.1$\pm$0.1} \\ 
			\bottomrule
		\end{tabular}
%	}
	\caption{\label{wordgcn_tbl:rfgcn_ext_full} \small \textbf{\wordgcnMethodSide{} Extrinsic Evaluation:} Comparison of different methods for incorporating diverse semantic constraints in \wordgcnMethod{} embeddings on all extrinsic tasks. Refer Section \ref{wordgcn_sec:diverse_side_results} for details.}
\end{table}

\subsubsection{Evaluation with Diverse Semantic Information}
\label{wordgcn_sec:diverse_side_results}

In this section, we compare \wordgcnMethodSide{} against the methods listed in Section \ref{wordgcn_sec:baselines} for incorporating diverse semantic information in pre-trained embeddings. We use \textit{hypernym}, \textit{hyponym},  and \textit{antonym} relations from WordNet, and \textit{synonym} relations from PPDB as semantic information. %To all the baselines, 
For each method, we provide the semantic information that it can utilize, e.g., Retro-fit can only make use of \textit{synonym} relation\footnote{Experimental results controlling for semantic information are provided in Section \ref{wordgcn_sec:rfgcn_re_stanfordsults}.}. 
%We define the notation used in subsequent results below. %
In our results, \textbf{M(X, R)} denotes the fine-tuned embeddings obtained using method M while taking X as initialization embeddings. R denotes the types of semantic information used as defined below.
\begin{itemize}[itemsep=0pt,parsep=0pt,partopsep=0pt,leftmargin=*,topsep=0pt]
	\item \textbf{R=1:} Only synonym information. 
	\item \textbf{R=2:} Synonym and antonym information. 
	\item \textbf{R=4:} Synonym, antonym, hypernym and hyponym information. 
\end{itemize}
For instance, Counter-fit (GloVe, 2) represents GloVe embeddings fine-tuned by Counter-fit using synonym and antonym information.

Similar to Section \ref{wordgcn_sec:intrinsic_results}, the evaluation is performed on the three intrinsic tasks. Due to space limitations, we report results on one representative dataset per task. % on the most commonly used datasets \cite{Faruqui2014,Mikolov2013b} across different initializations. 
The results are summarized in Table \ref{wordgcn_tbl:rfgcn_intrinsic}. We find that in $13$ out of $15$ settings, \wordgcnMethodSide{} outperforms other methods. Overall, we observe that \wordgcnMethodSide{}, when initialized with \wordgcnMethod{}, gives the best performance on all the tasks (highlighted by \boxed{\cdot} in Table \ref{wordgcn_tbl:rfgcn_intrinsic}).

For comparing performance on the extrinsic  tasks, we first fine-tune \wordgcnMethod{} embeddings using different methods for incorporating semantic information. The embeddings obtained by this process are then evaluated on extrinsic tasks, as in %similar to 
Section \ref{wordgcn_sec:intrinsic_results}. The results are shown in Table \ref{wordgcn_tbl:rfgcn_ext_full}. 
%We use the same $M(X, R)$ notation to represent methods as in Section \ref{wordgcn_sec:diverse_side_results}. 
We observe that while the other methods do not always consistently give improvement over the baseline \wordgcnMethod{}, \wordgcnMethodSide{} is able to improve upon \wordgcnMethod{} in all settings (better or comparable).
%We observe that although other methods give some improvement on intrinsic tasks, they do not always help on the extrinsic tasks. On the contrary, \wordgcnMethodSide{} either helps in improving performance or keeps it same on both intrinsic and extrinsic tasks.  \wordgcnMethodSide{} gives  improvement on SQuAD and Coref tasks while on others it gives the same performance without degrading it. 
Overall, we observe that \wordgcnMethod{} along with \wordgcnMethodSide{} is the most suitable method for incorporating both syntactic and semantic information.

\begin{figure}[t]
	\centering
	\includegraphics[width=4in]{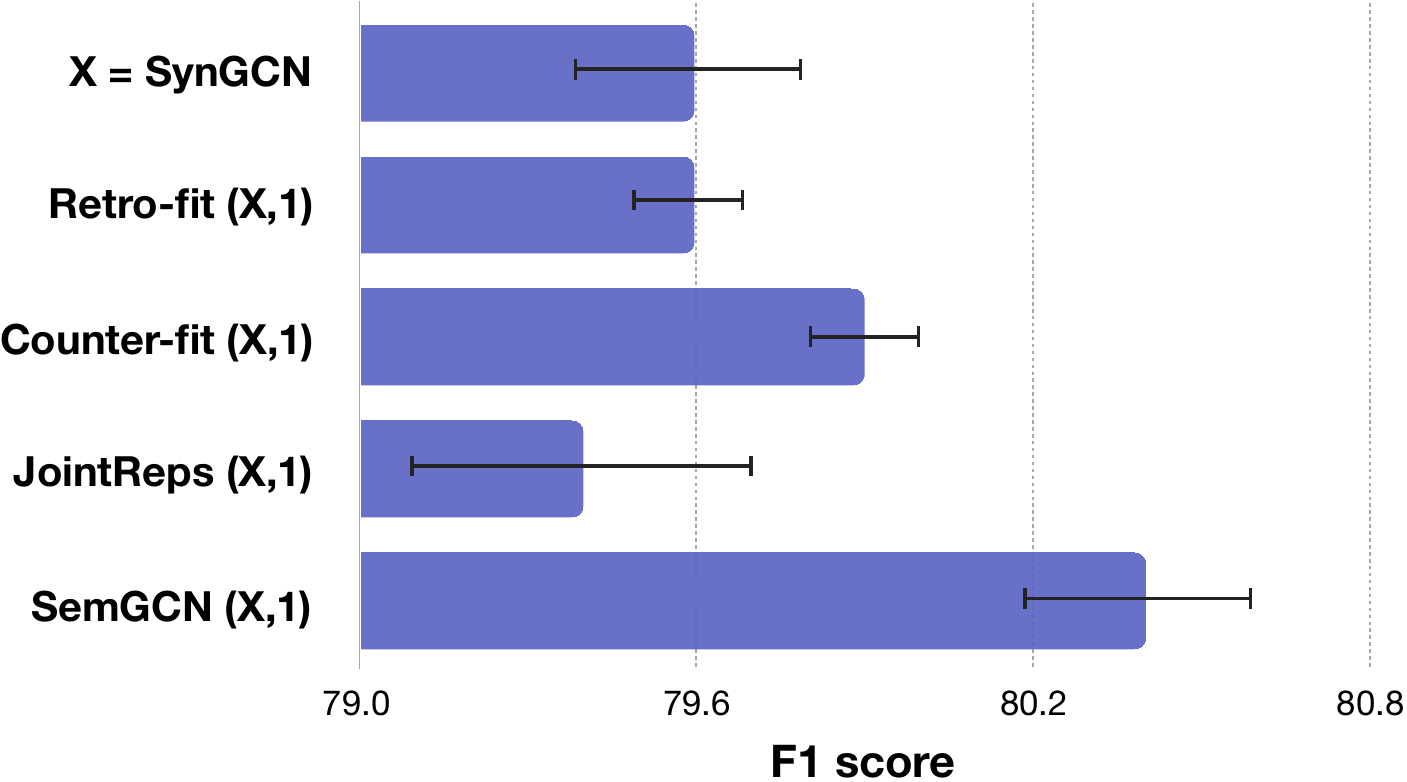}
	\caption{\label{wordgcn_fig:rfgcn_ppdb_ext} \small
		Comparison of different methods when provided with the same semantic information (synonym) for fine tuning \wordgcnMethod{} embeddings. Results denote the F1-score on SQuAD dataset. \wordgcnMethodSide{} gives considerable improvement in performance. Please refer Section \ref{wordgcn_sec:rfgcn_re_stanfordsults} for details.
	} 

\end{figure}

\subsubsection{Evaluation with Same Semantic Information}

\label{wordgcn_sec:rfgcn_re_stanfordsults}
In this section, we compare \wordgcnMethodSide{} against other baselines when provided with the same semantic information: synonyms from PPDB. Similar to Section \ref{wordgcn_sec:diverse_side_results}, we compare both on intrinsic and extrinsic tasks with different initializations. The evaluation results of fine-tuned \wordgcnMethod{} embeddings by different methods on SQuAD are shown in the Figure \ref{wordgcn_fig:rfgcn_ppdb_ext}. The remaining results are included in the supplementary (Table S1 and S2). We observe that compared to other methods, \wordgcnMethodSide{} is most effective at incorporating semantic constraints across all the initializations and outperforms others at both intrinsic and extrinsic tasks. 

\begin{table}[t]
	
	\centering
	\small
%	\resizebox{\columnwidth}{!}{
		\begin{tabular}{lcccc}
			\toprule
			Method          & POS                   & SQuAD                 & NER                   & Coref \\ 
			\midrule
			ELMo (E)     										& 96.1$\pm$0.1    		& 81.8$\pm$0.2   		& 90.3$\pm$0.3 & 67.8$\pm$0.1 \\ 
		E+SemGCN(SynGCN, 4)  & \textbf{96.2$\pm$0.1}          & \textbf{82.4$\pm$0.1}          & \textbf{90.9$\pm$0.1} & \textbf{68.3$\pm$0.1} \\ 

			\bottomrule
		\end{tabular}
%	}
	\caption{\label{wordgcn_tbl:elmo_comp}\small  Comparison of ELMo with \wordgcnMethod{} and \wordgcnMethodSide{} embeddings on multiple extrinsic tasks. For each task, models use a linear combination of the provided embeddings whose weights are learned. Results show that our proposed methods encode complementary information which is not captured by ELMo. Please refer Section \ref{wordgcn_sec:elmo_comp} for more details.}
\end{table}

\subsubsection{Comparison with ELMo}
\label{wordgcn_sec:elmo_comp}

Recently, ELMo \cite{elmo_paper} has been proposed which fine-tunes word embedding based on sentential context. In this section, we evaluate \wordgcnMethod{} and \wordgcnMethodSide{} when given along with pre-trained ELMo embeddings. The results are reported in Table \ref{wordgcn_tbl:elmo_comp}. The results show that dependency context based embeddings encode complementary information which is not captured by ELMo as it only relies on sequential context. Hence, our proposed methods serves as an effective combination with ELMo. 
\section{Conclusion}
\label{wordgcn_sec:conclusion}
%<<<<<<< HEAD
%In this chapter, we have proposed \wordgcnMethod{}, a graph convolution based approach which utilizes syntactic context for learning word representations. Unlike existing approaches, \wordgcnMethod{} does not suffer from the problem of vocabulary explosion and outperforms state-of-the-art word embedding methods. We also propose \wordgcnMethodSide{}, a framework for jointly incorporating diverse semantic information in pre-trained word embeddings. Overall, we find that the combination of \wordgcnMethod{} and \wordgcnMethodSide{}  gives the best overall performance.
%We make the source code of both the methods available to encourage reproducible research.
%=======
In this chapter, we have proposed \wordgcnMethod{}, a graph convolution based approach which utilizes syntactic context for learning word representations. \wordgcnMethod{} overcomes the problem of vocabulary explosion and outperforms state-of-the-art word embedding approaches on several intrinsic and extrinsic tasks. We also propose \wordgcnMethodSide{}, a framework for jointly incorporating diverse semantic information in pre-trained word embeddings. The combination of \wordgcnMethod{} and \wordgcnMethodSide{} gives the best overall performance.
We make the source code of both models available to encourage reproducible research.
%>>>>>>> 5946f1a3fdb1bdc6f568705fd235d633b506a263

%Recently, word  embeddings have been widely adopted across all NLP applications. However, most methods solely rely on linear context for learning word representation. While there exist approaches for utilizing signals from semantic knowledge sources, all of them deal with each word in isolation. In this chapter, we propose \wordgcnMethod{}, a Graph Convolution based word representation learning approach which provides a framework for exploiting multiple types of word relationships through a single model. Unlike prior works, \wordgcnMethod{} utilize dependency parse based context without exploding word vocabulary. To the best of our knowledge, this is the first application of recently proposed Graph Convolutional network for learning word representation. Through extensive experiments on various intrinsic and extrinsic tasks, we find that \wordgcnMethod{} outperforms exisiting word embedding approaches. We make \wordgcnMethod{}'s source code available to encourage reproducible research.*  and outperforms state-of-the-art word embedding approaches 

\part{\partThree{}}
\chapter{Improving Semi-Supervised Learning through Confidence-based Graph Convolutional Networks}
\label{chap_confgcn}

% !TeX spellcheck = en_US
\section{Introduction}
\label{confgcn_sec:intro}

In this chapter, we present our first enhancement of Graph Neural Networks for making them more effective at embedding heterogeneous real world graphs. 
Graphs are all around us, ranging from citation and social networks to knowledge graphs. Predicting properties of nodes in such graphs is often desirable. For example, given a citation network, we may want to predict the research area of an author. Making such predictions, especially in the semi-supervised setting, has been the focus of graph-based semi-supervised learning (SSL) \cite{subramanya2014graph}. In graph-based SSL, a small set of nodes are initially labeled. Starting with such supervision and while utilizing the rest of the graph structure, the initially unlabeled nodes are labeled. Conventionally, the graph structure has been  incorporated as an explicit regularizer which enforces a  smoothness constraint on the labels estimated on nodes  \cite{term1_just1,Belkin2006manifold,Weston2008deep}. 
%Other methods like node2vec \cite{node2vec} and DeepWalk \cite{deepwalk} used random walk on graphs to learn graph representations ignoring node features. 
%\reminderR{PPT: add a line about pre GCN graph neural methods} 
Recently proposed Graph Convolutional Networks (GCN) \cite{Defferrard2016,Kipf2016} provide a framework to apply deep neural networks to graph-structured data. GCNs have been employed successfully for improving performance on tasks such as  semantic role labeling \cite{gcn_srl}, machine translation \cite{gcn_nmt}, relation extraction \cite{reside,gcn_re_stanford}, document dating \cite{neuraldater}, shape segmentation \cite{yi2016syncspeccnn}, and action recognition \cite{huang2017deep}. GCN formulations for graph-based SSL have also attained state-of-the-art performance  \cite{Kipf2016,Liao2018graph,gat}. In this chapter, we also focus on the task of graph-based SSL using GCNs. %\reminderR{PPT: give non NLP examples as well. We may also mention that we focus on SSL over nodes in the graph.} 

GCN iteratively estimates embedding of nodes in the graph by aggregating embeddings of neighborhood nodes, while backpropagating errors from a target loss function. %This target loss function may be defined either at an individual node level, or over the whole graph. 
Finally, the learned node embeddings are used to estimate label scores on the nodes. In addition to the label scores, it is desirable to also have confidence estimates associated with them. Such confidence scores may be used to determine how much to trust the label scores estimated on a given node. While methods to estimate label score confidence in non-deep graph-based SSL has been previously proposed \cite{Orbach2012}, confidence-based GCN is still unexplored. 

\begin{figure*}[t]
	\centering
	\includegraphics[scale=0.7]{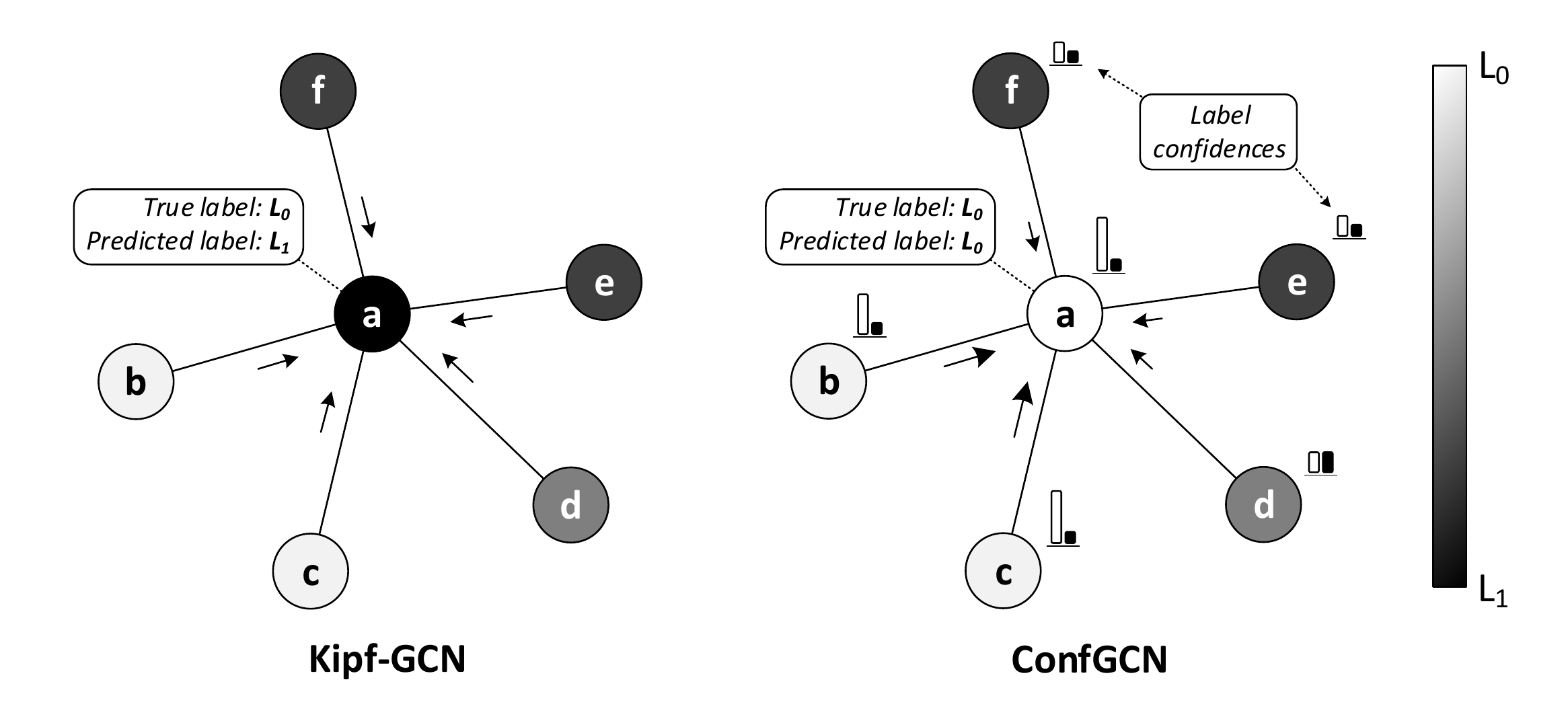}
	\caption{\label{confgcn_fig:motivation}\small Label prediction on node $a$ by Kipf-GCN and ConfGCN (this chapter). $L_0$ is $a$'s true label. Shade intensity of a node reflects the estimated score of label $L_1$ assigned to that node. %This is a subgraph from the Cora-ML dataset, where node $a$ is present in a heterogeneous neighborhood. 
	Since Kipf-GCN is not capable of estimating influence of one node on another, it is misled by the dominant label $L_1$ in node $a$'s neighborhood and thereby making the wrong assignment. ConfGCN, on the other hand, estimates confidences (shown by bars) over the label scores, and uses them to increase influence of nodes $b$ and $c$ to estimate the right label on $a$. Please see \refsec{confgcn_sec:intro} for details. 
	%This is a subgraph of the Cora-ML dataset, we contrast the difference between the predictions of Kipf-GCN and ConfGCN (proposed model). We show that our model effectively uses the label distributions and confidence scores to predict the labels correctly while Kipf-GCN gets biased by the majority. Please refer \refsec{confgcn_sec:intro} for more details.} %\reminder{PPT: Since GAT has anisotropic behavior, maybe we need to include it also in these motivational examples}
	}
\end{figure*}

In order to fill this important gap, we propose ConfGCN, a GCN framework for graph-based SSL. ConfGCN jointly estimates label scores on nodes, along with confidences over them.  One of the added benefits of confidence over node's label scores is that they may be used to subdue irrelevant nodes in a node's neighborhood, thereby controlling the number of effective neighbors for each node. In other words, this enables \emph{anisotropic} behavior in GCNs. Let us explain this through the example shown in \reffig{confgcn_fig:motivation}. In this figure, while node $a$ has true label $L_0$ (white), it is incorrectly classified as $L_1$ (black) by Kipf-GCN \cite{Kipf2016}\footnote{In this chapter, unless otherwise stated, we refer to Kipf-GCN whenever we mention GCN.}. This is because Kipf-GCN suffers from limitations of its  neighborhood aggregation scheme \cite{Xu2018}. For example, Kipf-GCN has no constraints on the number of nodes that can influence the representation of a given target node. In a $k$-layer Kipf-GCN model, each node is influenced by all the nodes in its $k$-hop neighborhood. %Concretely, the $k$-th layer embeddings of nodes, $H^{(k)}$ are updates as \[H^{(k)} = \sigma(A H^{(k-1)} W^{(k-1)})\] where $A$ is the adjacency matrix and $W^{(k-1)}$ are the $k$-th layer parameters.
However, in real world graphs, nodes are often present in  \emph{heterogeneous} neighborhoods, i.e., a node is often surrounded by nodes of other labels. For example, in \reffig{confgcn_fig:motivation}, node $a$ is surrounded by three nodes ($d$, $e$, and $f$) which are predominantly labeled $L_1$, while two nodes ($b$ and $c$) are labeled $L_0$. Please note that all of these are estimated label scores during GCN learning. In this case, it is desirable that node $a$ is more influenced by nodes $b$ and $c$ than the other three nodes. However, since Kipf-GCN doesn't discriminate among the neighboring nodes, it is swayed by the majority and thereby estimating the wrong label $L_1$ for node $a$. %In such cases, it is desirable that a node is more influenced by other nodes with the same label in its neighborhood. However, true labels of most nodes are not available during iterative updates of GCN, as in a semi-supervised node classification setting, true labels of only a subset of the nodes in the graph are available. Thus, there is a need to automatically estimate influence  of one node on another node during GCN learning. 

ConfGCN is able to overcome this problem by estimating confidences on each node's label scores. In \reffig{confgcn_fig:motivation}, such estimated confidences are shown by bars, with white and black bars denoting confidences in scores of labels $L_0$ and $L_1$, respectively. ConfGCN uses these label confidences to subdue nodes $d, e, f$ since they have low confidence for their label $L_1$ (shorter black bars), whereas nodes $b$ and $c$ are highly confident about their labels being $L_0$ (taller white bars). This leads to higher influence of $b$ and $c$ during aggregation, and thereby  ConfGCN correctly predicting the true label of node $a$ as $L_0$ with high confidence. This clearly demonstrates the benefit of label confidences and their utility in estimating node influences. Graph Attention Networks (GAT) \cite{gat}, a recently proposed method also provides a mechanism to estimate  influences by allowing nodes to attend to their neighborhood. However, as we shall see in \refsec{confgcn_sec:experiments}, ConfGCN, through its use of label confidences, is considerably more effective.

Our contributions in this chapter are as follows.

\begin{itemize}
	\item We propose ConfGCN, a Graph Convolutional Network (GCN) framework for semi-supervised learning which models label distribution and their confidences for each node in the graph. To the best of our knowledge, this is the first confidence-enabled formulation of GCNs.
	\item ConfGCN utilize label confidences to estimate influence of one node on another in a label-specific manner during neighborhood aggregation of GCN learning.
	\item Through extensive evaluation on multiple real-world datasets, we demonstrate ConfGCN effectiveness over state-of-the-art baselines.
\end{itemize}

ConfGCN's source code and datasets used in the chapter are available at \url{http://github.com/malllabiisc/ConfGCN}.
% !TeX spellcheck = en_US
\section{Related Work}
\label{confgcn_sec:related_work}
\textbf{Semi-Supervised learning (SSL) on graphs:} SSL on graphs is the problem of classifying nodes in a graph, where labels are available only for a small fraction of nodes. Conventionally, the graph structure is imposed by adding an explicit graph-based regularization term in the loss function \cite{term1_just1,Weston2008deep,Belkin2006manifold}.
Recently, implicit graph regularization via learned node representation has proven to be more effective. 
%The graph representations are learned which are then used for the semi-supervised task. 
This can be done either sequentially or in an end to end fashion. Methods like DeepWalk \cite{deepwalk}, node2vec \cite{node2vec}, and LINE \cite{line} first learn graph representations via sampled random walk on the graph or breadth first search traversal and then use the learned representation for node classification. On the contrary, Planetoid \cite{Yang2016revisiting} learns node embedding by jointly predicting the class labels and the neighborhood context in the graph.
Recently, \citet{Kipf2016} employs Graph Convolutional Networks (GCNs) to learn node representations.

\textbf{Graph Convolutional Networks (GCNs):}
The generalization of Convolutional Neural Networks to non-euclidean domains is proposed by \citet{Bruna2013} which formulates the spectral and spatial construction of GCNs. This is later improved through an efficient localized filter approximation \cite{Defferrard2016}. \citet{Kipf2016} provide a first-order formulation of GCNs and show its effectiveness for SSL on graphs. \citet{gcn_srl} propose GCNs for directed graphs and provide a mechanism for edge-wise gating to discard noisy edges during aggregation. This is further improved by \citet{gat} which allows nodes to attend to their neighboring nodes, implicitly providing different weights to different nodes. \citet{Liao2018graph} propose Graph Partition Neural Network (GPNN), an extension of GCNs to learn node representations on large graphs. GPNN first partitions the graph into subgraphs and then alternates between locally and globally propagating information across subgraphs. Recently, Lovasz Convolutional Networks \citet{lovasz_paper} is proposed for incorporating global graph properties in GCNs. An extensive survey of GCNs and their applications can be found in \citet{Bronstein2017}.
% whereas \citet{Li2015gated} employs a recurrent neural networks based model.

\textbf{Confidence Based Methods:} The natural idea of incorporating confidence in predictions has been explored by \citet{Li2006} for the task of active learning. \citet{Lei2014} proposes a confidence based framework for classification problems, where the classifier consists of two regions in the predictor space, one for confident classifications and other for ambiguous ones. In representation learning, uncertainty (inverse of confidence) is first utilized for word embeddings by \citet{Vilnis2014}. 
%Confidence is the opposite of uncertainty, which can be captured by learning higher order representations. \citet{Vilnis2014} represented words as Gaussian distribution. 
\citet{Athiwaratkun2018} further extend this idea to learn hierarchical word representation through encapsulation of probability distributions. 
\citet{Orbach2012} propose TACO (Transduction Algorithm with COnfidence), the first graph based method which learns
label distribution along with its uncertainty for semi-supervised node classification. 
%higher order embeddings for the nodes. For each node they learn the label distribution along with the uncertainty in label distribution, they model confidence as the reciprocal of uncertainty. 
\citet{graph2gauss} embeds graph nodes as Gaussian distribution using ranking based framework which allows to capture uncertainty of representation. They update node embeddings to maintain neighborhood ordering, i.e. 1-hop neighbors are more similar to 2-hop neighbors and so on. Gaussian embeddings have been used for collaborative filtering \cite{DosSantos2017} and topic modelling \cite{Das2015} as well.
% !TeX spellcheck = en_US

\section{Notation \& Problem Statement} 
\label{confgcn_sec:notation}

Let $\mathcal{G} = (\mathcal{V},\mathcal{E},\mathcal{X})$ be an undirected graph, where $\mathcal{V} = \mathcal{V}_l \cup \mathcal{V}_u$ is the union of labeled ($\mathcal{V}_l$) and unlabeled ($\mathcal{V}_u$) nodes in the graph with cardinalities $n_l $ and $n_u$, $\mathcal{E}$ is the set of edges and $\mathcal{X} \in \mathbb{R}^{(n_l + n_u) \times d}$ is the input node features.
%with cardinalities $n_l$ and $n_u$ respectively. 
The actual label of a node $v$ is denoted by a one-hot vector $Y_v \in \mathbb{R}^{m}$, where $m$ is the number of classes. Given $\mathcal{G}$ and seed labels $Y \in \mathbb{R}^{n_l \times m}$, the goal is to predict the labels of the unlabeled nodes. 
To incorporate confidence, we additionally estimate label distribution $\bm{\mu}_v \in \mathbb{R}^{m}$ and a diagonal co-variance matrix $\bm{\Sigma}_v \in \mathbb{R}^{m \times m},~\forall v \in \mathcal{V}$. Here, $\bm{\mu}_{v,i}$ denotes the score of label $i$ on node $v$, while $(\bm{\Sigma}_v)_{ii}$ denotes the variance in the estimation of $\bm{\mu}_{v,i}$. In other words, $(\bm{\Sigma}_{v}^{-1})_{ii}$ is ConfGCN's confidence in $\bm{\mu}_{v,i}$.

\section{Proposed Method: Confidence Based Graph \\ Convolutional Networks (ConfGCN)}
\label{confgcn_sec:details}

Following \cite{Orbach2012}, ConfGCN 
%
%In ConfGCN, for each node $v\in\mathcal{V}$, we estimate a label distribution $\bm{\mu}_v \in \mathbb{R}^{L}$ along with the diagonal co-variance matrix $\bm{\Sigma}_v \in \mathbb{R}^{L\times L}$. The value $(\bm{\mu}_v)_i$ denotes the probability of node $v$ belonging to class $i$, while $(\bm{\Sigma}_v)_{ii}$ denotes the variance in the estimation of $(\bm{\mu}_v)_i$. Following \citet{Orbach2012}, we use 
uses co-variance matrix based symmetric Mahalanobis distance for defining distance between two nodes in the graph. Formally, for any two given nodes $u$ and $v$, with label distributions $\bm{\mu}_u$ and $\bm{\mu}_v$  and co-variance matrices $\bm{\Sigma}_u$ and $\bm{\Sigma}_v$, distance between them is defined as follows.
%\small
\[
d_M(u,v) = (\bm{\mu}_u-\bm{\mu}_v)^T(\bm{\Sigma}_u^{-1}+\bm{\Sigma}_v^{-1})(\bm{\mu}_u-\bm{\mu}_v). 
\]
%\normalsize
Characteristic of the above distance metric is that if either of $\bm{\Sigma}_u$ or $\bm{\Sigma}_v$ has large eigenvalues, then the distance will be low irrespective of the closeness of $\bm{\mu}_u$ and $\bm{\mu}_v$. On the other hand, if $\bm{\Sigma}_u$ and $\bm{\Sigma}_v$ both have low eigenvalues,  then it requires $\bm{\mu}_u$ and $\bm{\mu}_v$ to be close for their distance to be low. Given the above properties, we define $r_{uv}$, the influence score of node $u$ on its neighboring node $v$ during GCN aggregation, as follows.
\[
r_{uv} = \frac{1}{d_M(u,v)} .
\]
This influence score gives more relevance to neighboring nodes with highly confident similar label, while reducing importance of nodes with low confident label scores. This results in ConfGCN acquiring anisotropic capability during neighborhood aggregation. For a node $v$, ConfGCN's  equation for updating embedding at the $k$-th layer is thus defined as follows.
%This allows to give more relevance to neighboring nodes with similar label distribution and high confidence and subdues nodes with low confidence, giving anisotropic behavior to GCN during aggregation. For any node $v$ in the graph, the GCN update equation at $k$-th layer is thus defined as 

\begin{equation}
\bm{h}^{k+1}_{v} = f\left(\sum_{u \in \m{N}(v)}  r_{uv} \times \left(\bm{W}^{k} \bm{h}^{k}_{u} + \bm{b}^{k}\right)\right), \forall v \in \m{V} .
\end{equation}

%
%Here, the notations are adopted from \refeqn{confgcn_eqn:gcn_undir_main} in \refsec{confgcn_sec:background}. 
The final node representation obtained from ConfGCN is used for predicting labels of the nodes in the graph as follows.

\[
\bm{\hat{Y}}_v = \mathrm{softmax}(\bm{W}^K \bm{h}^K_v + \bm{b}^K), ~\forall v \in \m{V}
\]

where, $K$ denotes the number of ConfGCN's layers. Finally, in order to learn label scores $\{\bm{\mu}_{v}\}$ and co-variance matrices $\{\bm{\Sigma}_v\}$ jointly with other parameters $\{\bm{W}^k, \bm{b}^k\}$, following \citet{Orbach2012}, we include the following two terms in ConfGCN's objective function.

%
%where, %$\theta = \{\bm{W}^k, \bm{b}^k, \bm{\mu}_v, \bm{\Sigma}_v\}$ denotes the parameters of ConfGCN, 
%$\lambda_* \in \mathbb{R}$ decides the importance of each term in $L(\theta)$.
%The first term, $L_{\text{cross}}$ is the standard cross-entropy loss for semi-supervised multi-class classification over all the labeled nodes ($\mathcal{V}_l$). 

\begin{table*}[!t]
	\centering
	\begin{tabular}{lcccccc}
		\toprule
		Dataset  & Nodes & Edges & Classes & Features & Label Mismatch & $\frac{|\mathcal{V}_l|}{|\mathcal{V}|}$ \\
		\midrule
		Cora 			  		& 2,708 		 & 5,429 		& 7  	 & 1,433 		&  0.002 & 0.052\\
		Cora-ML				 & 2,995		  & 8,416 		  & 7		& 2,879		  & 0.018 & 0.166\\
		Citeseer 		 		& 3,327 		& 4,372 		& 6 	& 3,703 	&	0.003 & 0.036 \\
		Pubmed 		  		 & 19,717 		  & 44,338 		& 3  	 & 500   	&	0.0  & 0.003\\
		%		NELL 		    	   & 65,755 		& 266,144 	 & 210  & 5,414 	&	0.0  & 0.001\\
		\bottomrule
	\end{tabular}
	\caption{\label{confgcn_tbl:dataset_statistics}Details of the datasets used in the chapter. Please refer \refsec{confgcn_sec:datasets} for more details.}
\end{table*} 

For enforcing neighboring nodes to be close to each other, we include $L_{\text{smooth}}$ defined as

\[L_{\text{smooth}} = \sum_{(u,v) \in \m{E}} (\bm{\mu}_u- \bm{\mu}_v)^T(\bm{\Sigma}^{-1}_u + \bm{\Sigma}^{-1}_v)(\bm{\mu}_u - \bm{\mu}_v). \]

To impose the desirable property that the label distribution of nodes in $\m{V}_l$ should be close to their input label distribution, we incorporate $L_{\text{label}}$ defined as

\[L_{\text{label}}  = \sum_{v \in \m{V}_l} (\bm{\mu}_v - \bm{Y}_v)^T(\bm{\Sigma}^{-1}_v + \dfrac{1}{\gamma}\bm{I})(\bm{\mu}_v - \bm{Y}_v)  .\]

Here, for input labels, we assume a fixed uncertainty of $\frac{1}{\gamma}\bm{I} \in \mathbb{R}^{L \times L}$, where $\gamma > 0$. 
We also include the following regularization term, $L_{\text{reg}}$ to constraint the co-variance matrix to be finite and positive. 

%\[L_{\text{reg}}  = \sum_{v \in \m{V}}\mathrm{Tr}\bm{\Sigma}_v  - \eta \sum_{v \in \m{V}} \log(\mathrm{det} \bm{\Sigma}_v)  ,\]
\[L_{\text{reg}}  = \sum_{v \in \m{V}}\mathrm{Tr} \ \mathrm{max}(-\bm{\Sigma}_v, 0)  ,\]

%for some $\eta  > 0$. The first term increases monotonically with the eigenvalues of $\bm{\Sigma}$ and the second term prevents them from becoming zero.
This regularization term enforces soft positivity constraint on co-variance matrix. 
Additionally in ConfGCN, we include the $L_{\text{const}}$ in the objective, to push the label distribution ($\bm{\mu}$) close to the final model prediction ($\bm{\hat{Y}}$).

\[L_{\text{const}}  = \sum_{v \in \m{V}} (\bm{\mu}_v - \hat{\bm{Y}}_v)^T(\bm{\mu}_v - \hat{\bm{Y}}_v) .\]

Finally, we include the standard cross-entropy loss for semi-supervised multi-class classification over all the labeled nodes ($\mathcal{V}_l$).

\[L_{\text{cross}} = - \sum_{v \in \mathcal{V}_l} \sum_{j=1}^{m} \bm{Y}_{vj} \log(\bm{\hat{Y}}_{vj})  .\]

The final objective for optimization is the linear combination of the above defined terms. 
%\begin{equation*}
%%\begin{align*}
%	L(\{\bm{W}^k, \bm{b}^k, \bm{\mu}_v, \bm{\Sigma}_v\}) = L_{\text{cross}} + \lambda_1 L_{\text{smooth}} + \lambda_2 L_{\text{label}} \\ + \lambda_3 L_{\text{const}} +  \lambda_4 L_{\text{reg}} \label{confgcn_eqn:main_obj}
%%\end{align*}
%\end{equation*}

%\begin{align}\label{confgcn_eqn:main_obj}
%\begin{split}
%L(\{\bm{W}^k, \bm{b}^k, \bm{\mu}_v, \bm{\Sigma}_v\}) =\; & L_{\text{cross}} + \lambda_1 L_{\text{smooth}} \\
%&    + \lambda_2 L_{\text{label}} + \lambda_3 L_{\text{const}} \\
%&    + \lambda_4 L_{\text{reg}}
%\end{split}
%\end{align}

\begin{equation*} \label{confgcn_eqn:main_obj}
\small
\begin{aligned}
L(\theta) = 
%L(\{\bm{W}^k, \bm{b}^k, \bm{\mu}_v, \bm{\Sigma}_v\}) = 
& \hspace{0 mm} - \sum_{i \in \mathcal{V}_l} \sum_{j=1}^{L} \bm{Y}_{ij} \log( \bm{\hat{Y}}_{ij}) \\    
& + \hspace{1 mm} \lambda_{\text{1}} \sum_{(u,v) \in \m{E}} (\bm{\mu}_u- \bm{\mu}_v)^T(\bm{\Sigma}^{-1}_u + \bm{\Sigma}^{-1}_v)(\bm{\mu}_u - \bm{\mu}_v) \\
& + \hspace{1 mm} \lambda_{\text{2}} \hspace{1.5 mm} \sum_{u \in \m{V}_l} (\bm{\mu}_u - \bm{Y}_u)^T(\bm{\Sigma}^{-1}_u + \dfrac{1}{\gamma}\bm{I})(\bm{\mu}_u - \bm{Y}_u)  \\
& + \hspace{1 mm} \lambda_{\text{3}} \hspace{2 mm} \sum_{v \in \m{V}} (\bm{\mu}_u - \hat{\bm{Y}}_u)^T(\bm{\mu}_u - \hat{\bm{Y}}_u)\\
%& + \hspace{1 mm} \lambda_{\text{4}} \hspace{2 mm} \sum_{v \in \m{V}}\mathrm{Tr}\bm{\Sigma}_v  - \lambda_{\text{5}} \sum_{v \in \m{V}} \log(\det \bm{\Sigma}_v) \\
& + \hspace{1 mm} \lambda_{\text{4}} \hspace{2 mm} \sum_{v \in \m{V}}\mathrm{Tr} \ \mathrm{max}(-\bm{\Sigma}_v, 0)  \\
\end{aligned}
%\begin{aligned}
%&\left.\vphantom{\begin{aligned}
%	L(\theta) =  
%	& \hspace{0 mm} - \sum_{i \in \mathcal{V}_l} \sum_{j=1}^{L} y_{ij} log(\hat{Y}_{ij})
%	\end{aligned}}\right\rbrace\quad L_{\text{cross}}\\   
%&\left.\vphantom{\begin{aligned}
%	& + \lambda_{cross} \sum_{(u,v) \in \m{E}} (\bm{\mu_u}- \bm{\mu_v})^T(\Sigma^{-1}_u + \Sigma^{-1}_v)(\bm{\mu_u} - \bm{\mu_v})
%	\end{aligned}}\right\rbrace\quad L_{\text{smooth}}\\   
%&\left.\vphantom{\begin{aligned}
%	& + \sum_{u \in \m{V}_l} (\bm{\mu_u} - y_u)^T(\Sigma^{-1}_u + \frac{1}{\gamma}I)(\bm{\mu_u} - y_u) 
%	\end{aligned}}\right\rbrace\quad L_{\text{label}}\\
%&\left.\vphantom{\begin{aligned}
%	& + \sum_{v \in \m{V}} (\bm{\mu_u} - \hat{y}_u)^T(\bm{\mu_u} - \hat{y}_u)
%	\end{aligned}}\right\rbrace\quad L_{\text{const}}\\
%&\left.\vphantom{\begin{aligned}
%	& + \alpha \sum_{v \in \m{V}} Tr\bm{\Sigma_v}  - \beta \sum_{v \in \m{V}} log(det \bm{\Sigma_v}) 
%	\end{aligned}}\right\rbrace\quad L_{\text{reg}}\\
%\end{aligned}
\end{equation*}

where, $\theta = \{\bm{W}^k, \bm{b}^k, \bm{\mu}_v, \bm{\Sigma}_v\}$ and $\lambda_i \in \mathbb{R}$, are the weights of the terms in the objective. We optimize $L(\theta)$ using stochastic gradient descent. We hypothesize that all the terms help in improving ConfGCN's performance and we validate this in \refsec{confgcn_sec:ablation_results}.

\section{Experiments}
\label{confgcn_sec:experiments}

\subsection{Experimental Setup}

\subsubsection{Datasets}
\label{confgcn_sec:datasets}

For evaluating the effectiveness of ConfGCN, we evaluate on several semi-supervised classification benchmarks. Following the experimental setup of \cite{Kipf2016,Liao2018graph}, we evaluate on Cora, Citeseer, and Pubmed datasets \cite{pubmed}. The dataset statistics is summarized in \reftbl{confgcn_tbl:dataset_statistics}. Label mismatch denotes the fraction of edges between nodes with different labels in the training data. The benchmark datasets commonly used for semi-supervised classification task have substantially low label mismatch rate. In order to examine models on datasets with more heterogeneous neighborhoods, %  for examining models on more real-world scenario, 
we also evaluate on Cora-ML dataset \cite{graph2gauss}. 
%NELL \citet{nell,Yang2016revisiting} dataset is a bipartite graph extracted from a knowledge graph 9,891 entity and 55,864 relation nodes. 

%Dataset statistics are shown in \reftbl{confgcn_tbl:dataset_statistics}. These datasets are highly diverse in terms of scale, label rate and feature dimensions. We use the same data splits as in \citet{graph2gauss} for Cora-ML dataset and for the rest, we use the data splits provided by \citet{Yang2016revisiting}. The citation network datasets i.e. Cora, Cora-ML, Citeseer and Pubmed \citet{pubmed} consist of documents as nodes and citation links as edges. NELL \citet{nell,Yang2016revisiting} dataset is a bipartite graph extracted from a knowledge graph consisting of 55,864 relation nodes and 9,891 entity nodes.

%\textbf{Citation Networks} (Cora, Citeseer, Pubmed and Cora-ML) 
All the four datasets are citation networks, where each document is represented using bag-of-words features  in the graph with undirected citation links between documents. The goal is to classify documents into one of the predefined classes. We use the data splits used by \cite{Yang2016revisiting} and follow similar setup for Cora-ML dataset. Following \cite{Kipf2016}, additional 500 labeled nodes are used for hyperparameter tuning.

%\textbf{NELL} dataset is extracted from Never-Ending Language Learner (NELL) \citet{nell} Knowledge graph consisting of entities as nodes and relations as directed labeled edges. Following \citet{Yang2016revisiting}, we assign separate relation nodes $r_1$ and $r_2$ for each entity pair ($e_1$, $r$, $e_2$) as ($e_1$, $r_1$) and ($e_2$, $r_2$). Each relation node is represented by a unique one-hot representation, effectively extending the feature vectors to be 61,278 dimensional. Similar to citation network, the task is to classify each node in the graph given only single labeled node per class.

\begin{table*}[t]
	\centering
	\begin{tabular}{lcccc}
		\toprule
		
		Method & Citeseer & Cora & Pubmed & Cora ML\\
		%		\cmidrule(r){1-1}	\cmidrule(r){2-2} \cmidrule(r){3-3} \cmidrule(r){4-4} \cmidrule(r){5-5}
		\midrule
		
		LP 	\cite{term1_just1}	& 45.3 		& 68.0 		& 63.0 		& -\\
		ManiReg \cite{Belkin2006manifold}	& 60.1 		& 59.5 		& 70.7 		& -\\
		SemiEmb \cite{Weston2008deep}		& 59.6 		& 59.0 		& 71.1 		& -\\
		Feat 	\cite{Yang2016revisiting}	& 57.2 		& 57.4 		& 69.8 		& -\\
		DeepWalk \cite{deepwalk}	& 43.2 		& 67.2 		& 65.3 		& -\\
		GGNN	\cite{Li2015gated}	& 68.1 		& 77.9 		& 77.2 		& -\\
		Planetoid \cite{Yang2016revisiting}	& 64.9 		& 75.7 		& 75.7 		& -\\
		Kipf-GCN \cite{Kipf2016}			& 69.4 $\pm$ 0.4 	& 80.9 $\pm$ 0.4 	& 76.8 $\pm$ 0.2 	& 85.7 $\pm$ 0.3\\
		G-GCN \cite{gcn_srl}				& 69.6 $\pm$ 0.5 		& 81.2 $\pm$ 0.4 		& 77.0 $\pm$ 0.3 		& 86.0 $\pm$ 0.2 \\
		GPNN \cite{Liao2018graph} 			& 68.1 $\pm$ 1.8		& 79.0 $\pm$ 1.7 		&  73.6 $\pm$ 0.5		& 69.4 $\pm$ 2.3 \\
		GAT \cite{gat} & 72.5 $\pm$ 0.7 		& \textbf{83.0 $\pm$ 0.7} 		& 79.0 $\pm$ 0.3 		& 83.0 $\pm$ 0.8\\
		\midrule
		ConfGCN (this work)	 & \textbf{72.7 $\pm$ 0.8} & 82.0 $\pm$ 0.3 & \textbf{79.5 $\pm$ 0.5} & \textbf{86.5 $\pm$ 0.3}\\
		\bottomrule
	\end{tabular}
	\caption{\label{confgcn_tbl:qualitative_results} Performance comparison of several methods for semi-supervised node classification on multiple benchmark datasets. ConfGCN performs consistently better across all the datasets. Baseline method performances on Citeseer, Cora and Pubmed datasets are taken from \citet{Liao2018graph,gat}. We consider only the top performing baseline methods on these datasets for evaluation on the Cora-ML dataset. Please refer \refsec{confgcn_sec:quantitative_results} for details.}
	
\end{table*}

\textbf{Hyperparameters:} We use the same data splits as described in \cite{Yang2016revisiting}, with a test set of 1000 labeled nodes for testing the prediction accuracy of ConfGCN and a validation set of 500 labeled nodes for optimizing the hyperparameters. The ranges of hyperparameters were adapted from previous literature \cite{Orbach2012,Kipf2016}. The model is trained using Adam \cite{adam_opt} with a learning rate of 0.01. The weight matrices along with ${\bf \mu}$  are initialized using Xavier initialization \cite{xavier_init} and ${\bf \Sigma}$ matrix is initialized with identity. To avoid numerical instability we model ${\bf \Sigma^{-1}}$ directly and compute ${\bf \Sigma}$ wherever required. Following \citet{Kipf2016}, we use two layers of GCN ($K$) for all the experiments in this chapter.

\subsubsection{Baselines}
\label{confgcn_sec:baselines}
For evaluating ConfGCN, we compare against the following baselines:
\begin{itemize}[]
	\item \textbf{Feat} \cite{Yang2016revisiting} takes only node features as input and ignores the graph structure.
	\item \textbf{ManiReg} \cite{Belkin2006manifold} is a framework for providing data-dependent geometric regularization.
	\item \textbf{SemiEmb} \cite{Weston2008deep} augments deep architectures with semi-supervised regularizers to improve their training.
	\item \textbf{LP} \cite{term1_just1} is an iterative iterative label propagation algorithm which propagates a nodes labels to its neighboring unlabeled nodes according to their proximity.
	\item \textbf{DeepWalk} \cite{deepwalk} learns node features by treating random walks in a graph as the equivalent of sentences.
	\item \textbf{Planetoid} \cite{Yang2016revisiting} provides a transductive and inductive framework for  jointly predicting class label and neighborhood context of a node in the graph. 
	\item \textbf{GCN} \cite{Kipf2016} is a variant of convolutional neural networks used for semi-supervised learning on graph-structured data.
	\item \textbf{G-GCN} \cite{gcn_srl} is a variant of GCN with edge-wise gating to discard noisy edges during aggregation.
	\item \textbf{GGNN} \cite{Li2015gated} is a generalization of RNN framework which can be used for graph-structured data.
	\item \textbf{GPNN} \cite{Liao2018graph} is a graph partition based algorithm which propagates information after partitioning large graphs into smaller subgraphs.
	\item \textbf{GAT} \cite{gat} is a graph attention based method which provides different weights to different nodes by allowing nodes to attend to their neighborhood.
\end{itemize}

%feat \citet{Yang2016revisiting}, which takes only features as input, manifold regularization (ManiReg) \citet{Belkin2006manifold}, semi-supervised embedding (SemiEmb) \citet{Weston2008deep}, label propagation (LP) \citet{term1_just1}, skip-gram based word-embeddings (DeepWalk) \citet{deepwalk}, transductive and inductive versions of Planetoid \citet{Yang2016revisiting}, Graph Convolutional Networks (GCN) \citet{Kipf2016}, Gated Graph sequence Neural Networks (GGNN) \citet{Li2015gated}.

%Further, we compare our model against Graph Partition Based Neural Networks (GPNN) which propagates information after partitioning large graphs into subgraphs.

% !TeX spellcheck = en_US
\begin{figure*}[t!]
	\begin{minipage}{0.5\linewidth}
		\centering
		\includegraphics[width=\linewidth]{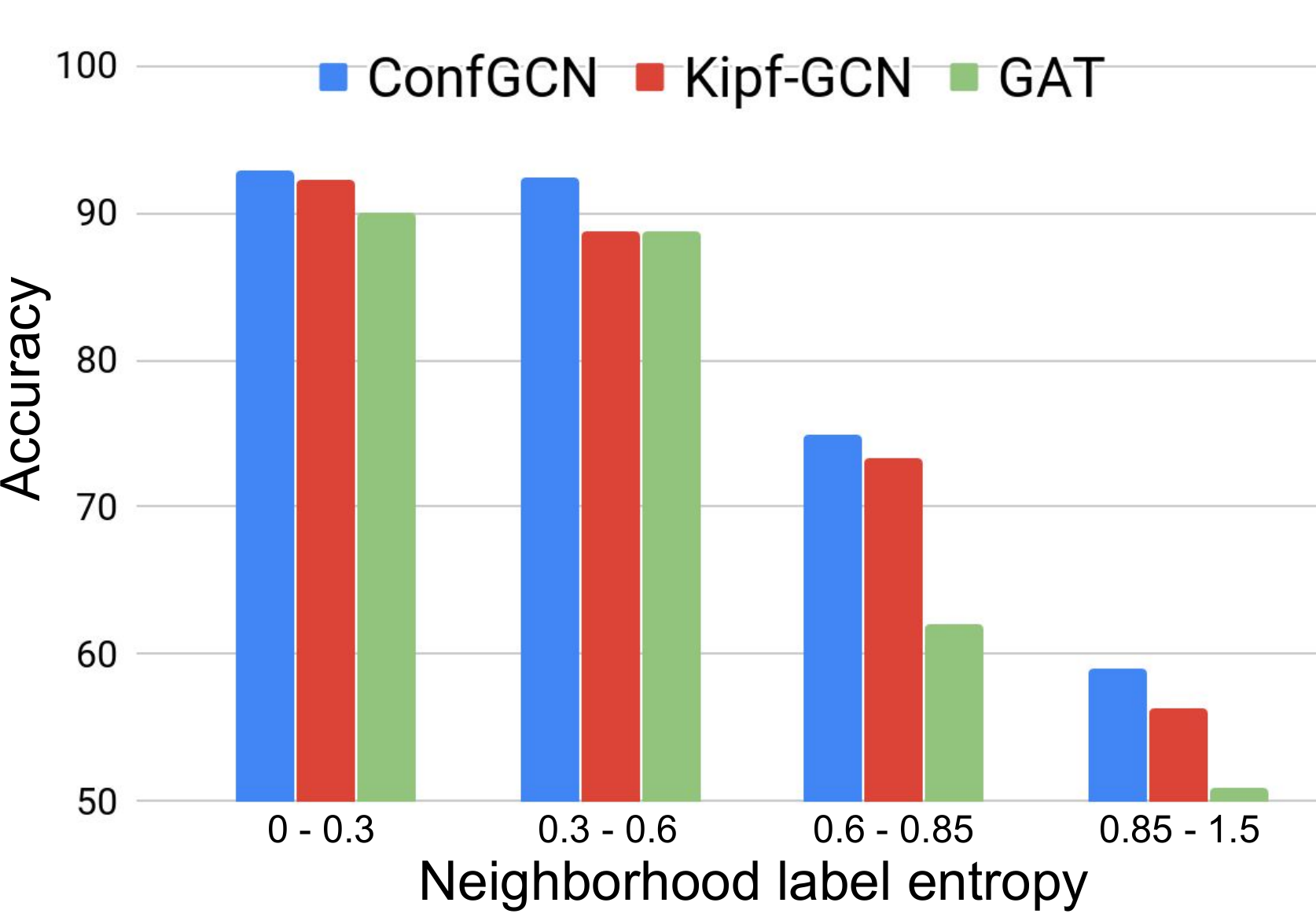}
		\subcaption{\label{confgcn_fig:entropy_plot}}
		%		\caption{\label{confgcn_fig:degree_plot} Plots of node classification accuracy vs. node degree. On $x$-axis we have quartiles of Node degree, i.e each bin has 25 \% of the samples in sorted order of Node degree. On $y$-axis we have the node classification accuracy of nodes whose node degree belong to that quantile.}
	\end{minipage} 
	\hfill
	\begin{minipage}{0.5\linewidth}
		\centering
		\includegraphics[width=\linewidth]{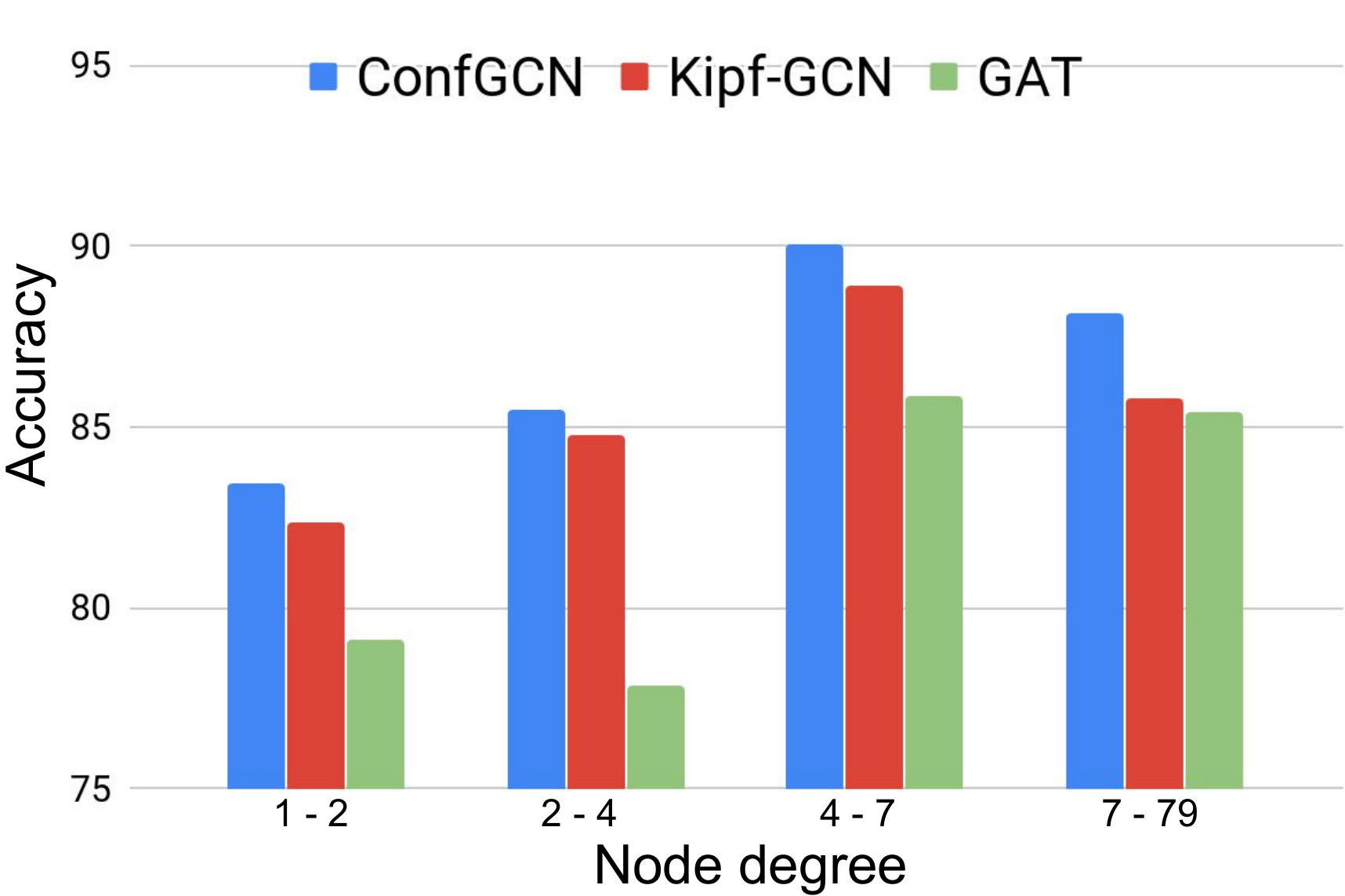}
		\subcaption{\label{confgcn_fig:degree_plot}}
		%		\caption{\label{confgcn_fig:entropy_plot} Plots of node classification accuracy vs. neighborhood entropy. On $x$-axis we have quartiles of entropy, i.e each bin has 25 \% of the samples in sorted order of entropy. On $y$-axis we have the node classification accuracy of nodes whose node entropy belong to that quantile.}
	\end{minipage}
	\caption{Plots of node classification accuracy vs. (a) neighborhood label entropy and (b) node degree. On $x$-axis, we plot quartiles of (a) neighborhood label entropy and (b) degree, i.e., each bin has 25\% of the samples in sorted order. 
		%On $y$-axis we have the node classification accuracy of nodes whose node entropy/degree belong to that quantile. 
		Overall, we observe that ConfGCN performs  better than Kipf-GCN and GAT at all levels of node entropy and degree. 
		%We also observe that the performance of Kipf-GCN and GAT degrade with the increase in neighborhood label  entropy.   %whereas ConfGCN performs consistently better. 
		%	Overall, we observe that the performance of Kipf-GCN and GAT degrades with the increase in neighborhood label  entropy and degree. In contrast, ConfGCN is able to avoid such degradation due to its estimation and use of confidence scores.  %whereas ConfGCN performs consistently better. 
		Please see \refsec{confgcn_sec:entropy} for details.}
\end{figure*}

\subsection{Results}
\label{confgcn_sec:results}

In this section, we attempt to answer the following questions:
\begin{itemize}[itemsep=2pt,topsep=4pt]
	\item[Q1.] How does ConfGCN compare against existing methods for the semi-supervised node classification task? (\refsec{confgcn_sec:quantitative_results})
	\item[Q2.] How do the performance of methods vary with increasing node degree and neighborhood label mismatch? (\refsec{confgcn_sec:entropy}) 
	%when there is high label mismatch amongst the labels in graphs? 
	\item[Q3.] How does increasing the number of layers effect ConfGCN's performance? (\refsec{confgcn_sec:layers_results})
	\item[Q4.] What is the effect of ablating different terms in ConfGCN's loss function? (\refsec{confgcn_sec:ablation_results})
%	\item[Q4.] How meaningful are $\mu$ and $\Sigma$ values learnt by ConfGCN? (\refsec{confgcn_sec:qualitatie_results})
	
\end{itemize}

\subsubsection{Node Classification}
\label{confgcn_sec:quantitative_results}

The evaluation results for semi-supervised node classification are summarized in \reftbl{confgcn_tbl:qualitative_results}. Results of all other baseline methods on Cora, Citeseer and Pubmed datasets are taken from \cite{Liao2018graph,gat} directly. For evaluation on the Cora-ML dataset, only top performing baselines from the other three datasets are considered. Overall, we find that ConfGCN outperforms all existing approaches consistently across all the datasets. 

This may be attributed to ConfGCN's ability to model nodes' label distribution along with the confidence scores which subdues the effect of noisy nodes during neighborhood aggregation. The lower performance of GAT \cite{gat} compared to Kipf-GCN on Cora-ML shows that computing attention based on the hidden representation of nodes is not much helpful in suppressing noisy neighborhood nodes. % as the representations lack label information. % or are over averaged or unstable. 
We also observe that the performance of GPNN \cite{Liao2018graph} suffers on the Cora-ML dataset. This is due to the fact that while propagating information between small subgraphs, the high label mismatch rate in Cora-ML (please see \reftbl{confgcn_tbl:dataset_statistics}) leads to wrong information propagation. Hence, during the global propagation step, this error is further magnified.
% For Planetoid \citet{Yang2016revisiting} we report the best variant of their model. 

\subsubsection{Effect of Node Entropy and Degree on Performance}
\label{confgcn_sec:entropy}
In this section, we provide an analysis of the performance of Kipf-GCN, GAT and ConfGCN for node classification on the Cora-ML dataset which has higher label mismatch rate. We use neighborhood label entropy  to quantify label mismatch, which for a node $u$ is defined as follows.
%For each node $u$  in the graph, we define entropy as 
\[
	\mathrm{Neighbor Label Entropy}(u) = - \sum_{l=1}^{L} p_{ul}  \log p_{ul}\]
\[
	\text{where,}\,\,\, p_{ul} = \frac{|\{v \in \mathcal{N}(u)~|~\mathrm{label}(v) = l \}|}{|\mathcal{N}(u)|}.
\]

Here, $\mathrm{label}(v)$ is the true label of node $v$. The results for neighborhood label entropy and node degree are summarized in Figures \ref{confgcn_fig:entropy_plot} and \ref{confgcn_fig:degree_plot}, respectively. On the x-axis of these figures, we plot quartiles of label entropy and degree, i.e., each bin has 25\% of the instances in sorted order. With increasing neighborhood label entropy, the node classification task  is expected to become more challenging. We indeed see this trend in \reffig{confgcn_fig:entropy_plot} where performances of all the methods degrade with increasing neighborhood label entropy. However, ConfGCN performs comparatively better than the existing state-of-art approaches at all levels of node entropy.

\begin{figure}[t]
	%\centering
	\begin{minipage}{3.25in}
		\includegraphics[width=3.25in]{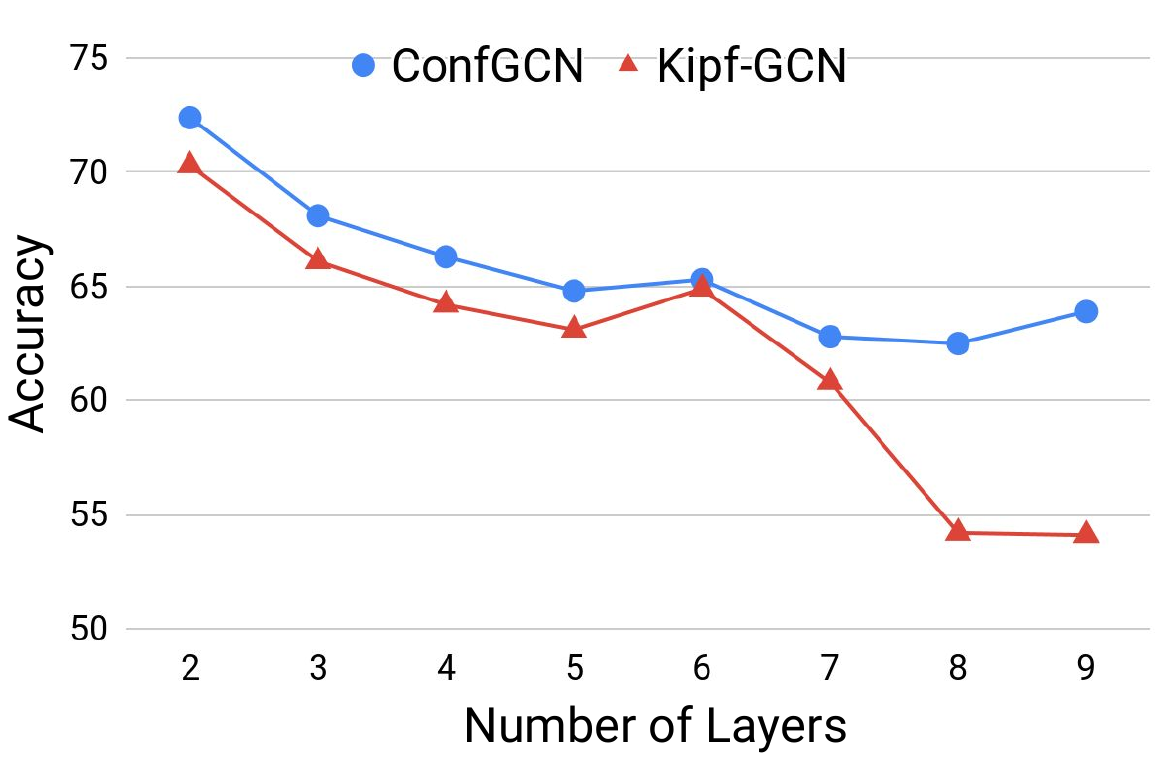}
		\subcaption{\label{confgcn_fig:layers} \small Evaluation of Kipf-GCN and ConfGCN on the citeseer dataset with increasing number of GCN layers. Overall, ConfGCN outperforms Kipf-GCN, and while both methods' performance degrade with increasing layers, ConfGCN's degradation is more gradual than Kipf-GCN's abrupt drop. Please see \refsec{confgcn_sec:layers_results} for details.}
	\end{minipage}
	\begin{minipage}{3.25in}
		\includegraphics[width=3.25in]{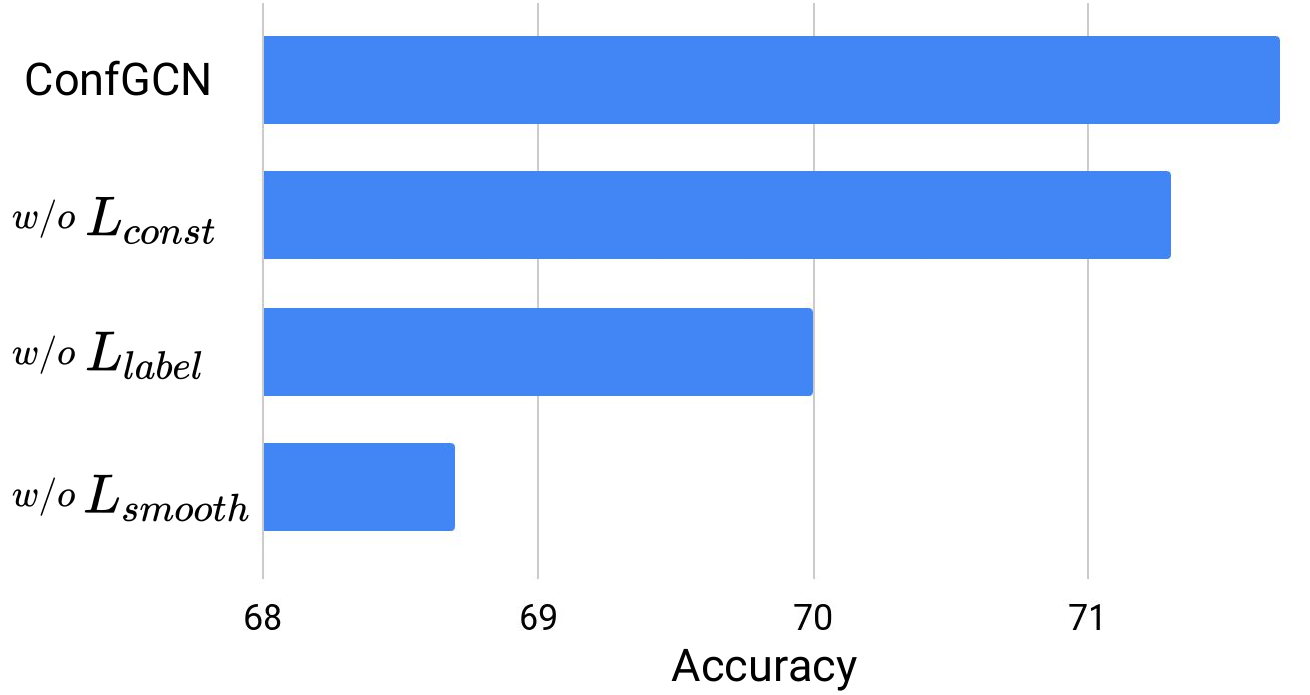}
		\subcaption{\label{confgcn_fig:ablation} \small Performance comparison of different ablated version of ConfGCN on the citeseer dataset. These results justify inclusion of the different terms in ConfGCN's objective function. Please see \refsec{confgcn_sec:ablation_results} for details.}
	\end{minipage}
\end{figure}

In case of node degree also (\reffig{confgcn_fig:degree_plot}), we find that ConfGCN performs better than Kipf-GCN and GAT at all quartiles of node degrees. Classifying sparsely connected nodes (first and second bins) is challenging as very little information is present in the node neighborhood. Performance improves with availability of moderate number of neighboring nodes (third bin), but further increase in degree (fourth bin) results in introduction of many potentially noisy neighbors, thereby affecting performance of all the methods. For higher degree nodes, ConfGCN gives an improvement of around 3\% over GAT and Kipf-GCN. This shows that ConfGCN, through its use of label confidences, is able to give higher influence score to relevant nodes in the neighborhood during aggregation while reducing importance of the noisy ones.

%The neighborhood label entropy of a node increases with the label mismatch amongst its neighbors. Thus, making it more difficult to get correctly classified. Similarly, node classification becomes challenging with increase in node degree, therefore, we also evaluate the performance of methods with increasing node degree. The results are summarized in Figures \ref{confgcn_fig:entropy_plot} and \ref{confgcn_fig:degree_plot}. We find that in general,  performance degrades with increase in node entropy. ConfGCN performs comparatively better than the existing state-of-art approaches at all levels of node entropy and degree. For higher degree nodes, ConfGCN gives an improvement of around 3\% than GAT and Kipf-GCN. This shows that ConfGCN is able to give higher influence score to relevant nodes in the neighborhood during aggregation while subduing the noisy ones.

%We use these graph properties and then compute accuracy on nodes in each quartile, i.e we sort the data in ascending order and then first 25 \% of the data lie in the first bin and so on.
%For each node $u$  in Cora-ML dataset, we compute the neighborhood entropy as 
%In \reffig{confgcn_fig:degree_plot}, we show that 

\subsubsection{Effect of Increasing Convolutional Layers}
\label{confgcn_sec:layers_results}

Recently, \citet{Xu2018} highlighted an unusual behavior of Kipf-GCN where its performance degrades significantly with increasing number of layers.  % on increasing the number of layers,  GCN performance degrades drastically. 
This is because of increase in the number of influencing nodes with increasing layers, resulting in ``averaging out" of information during aggregation.
For comparison, we evaluate the performance of Kipf-GCN and ConfGCN on citeseer dataset with increasing number of convolutional layers. The results are summarized in \reffig{confgcn_fig:layers}. We observe that Kipf-GCN's performance degrades drastically with increasing number of layers,  whereas ConfGCN's decrease in performance is more gradual. This shows that confidence based GCN helps in alleviating this problem. We also note that ConfGCN outperforms Kipf-GCN at all layer levels.%performs comparably better and consistent.

\subsubsection{Ablation Results}
\label{confgcn_sec:ablation_results}
In this section, we evaluate the different ablated version of ConfGCN by cumulatively eliminating terms from its objective function as defined in \refsec{confgcn_sec:details}. The results on citeseer dataset are summarized in \reffig{confgcn_fig:ablation}. Overall, we find that each term ConfGCN's loss function (\refeqn{confgcn_eqn:main_obj}) helps in improving its performance and the method performs best when all the terms are included. %the performance of ConfGCN highly depends on the presence of all the terms. $L_{const}$ and $L_{label}$ enable to correctly model label distribution (${\bm \mu}$) and co-variance (${\bm \Sigma}$) of nodes whereas, $L_{smooth}$ restricts co-variance matrix's eigenvalues to remain finite and positive. Thus, our hypothesis on including them in the objective function is justified.

%\section{Discussion}
%\label{confgcn_sec:discussion}
% !TeX spellcheck = en_US

\section{Conclusion}
\label{confgcn_sec:conclusion}

In this chapter, we present ConfGCN, a confidence based Graph Convolutional Network which estimates label scores along with their confidences jointly in a GCN-based setting. In ConfGCN, the influence of one node on another during aggregation is determined using the estimated confidences and label scores, thus inducing anisotropic behavior to GCN. We demonstrate the effectiveness of ConfGCN against state-of-the-art  methods for the semi-supervised node classification task and analyze its performance in different settings. We make ConfGCN's source code available. % for encouraging reproducible research. %by achieving state-of-the-art performance on standard benchmark datasets.

%Through extensive analysis and experiments on standard benchmarks, we find that ConfGCN is able to significantly outperform state-of-the-art baselines.
%We have made ConfGCN’s source code available to encourage reproducible research.
%We address few of the issues that conventional GCN based models face like over averaging and stability in leaned embeddings. ConfGCN utilizes label distribution and confidence score to compute relevance of nodes. Using these relevance scores we update node representations. This helps in reducing the effective neighborhood size for each node. 
\chapter{Composition-based Multi-Relational Graph Convolutional Networks for Relational Graphs}
\label{chap_compgcn}

\section{Introduction}
\label{compgcn_sec:introduction}

%Graphs are one of the most expressive data-structures which have been used to model a variety of problems. Traditional neural network architectures like Convolutional Neural Networks \citep{alexnet} and Recurrent Neural Networks \citep{lstm} are constrained to handle only Euclidean data. Recently, Graph Convolutional Networks (GCNs) \citep{Bruna2013,Defferrard2016} have been proposed to address this shortcoming, and have been successfully applied to several domains such as social networks \citep{graphsage}, knowledge graphs p\cite{r_gcn}, natural language processing \citep{gcn_srl}, drug discovery \citep{deep_chem} and natural sciences \citep{protein_protein_gcn}.

In this chapter, we address another important limitation of Graph Convolutional models. Most of the existing research on GCNs \citep{Kipf2016,graphsage,gat} have focused on learning representations of nodes in simple undirected graphs. A more general and pervasive class of graphs are multi-relational graphs\footnote{In this chapter, multi-relational graphs refer to graphs with edges that have labels and directions.}. A notable example of such graphs is knowledge graphs. Most of the existing GCN based approaches for handling relational graphs \citep{gcn_srl,r_gcn} suffer from over-parameterization and are limited to learning only node representations.
%To incorporate direction and relation information in such graphs, \citet{gcn_srl} and \citet{r_gcn} proposed extensions of GCNs. The former employs direction-specific convolution filters while the latter uses relation specific convolution filters.
%Although these models have shown improvements upon non-relational GCN methods, they are limited to learning embedding vectors for only the vertices of the graphs. 
Hence, such methods are not directly applicable for tasks that require relation embedding vectors such as link prediction.
%and they usually leverage additional models to learn these relation representations \cite{r_gcn}.
%Initial attempts for learning representations of relations in undirected graphs \cite{dual_primal_gcn} have shown performance gains on node classification.
Initial attempts at learning representations for relations in graphs \citep{dual_primal_gcn,graph2seq} have shown some performance gains on tasks like node classification and neural machine translation.

\begin{figure*}[t]	
	\centering	
	\includegraphics[width=\textwidth]{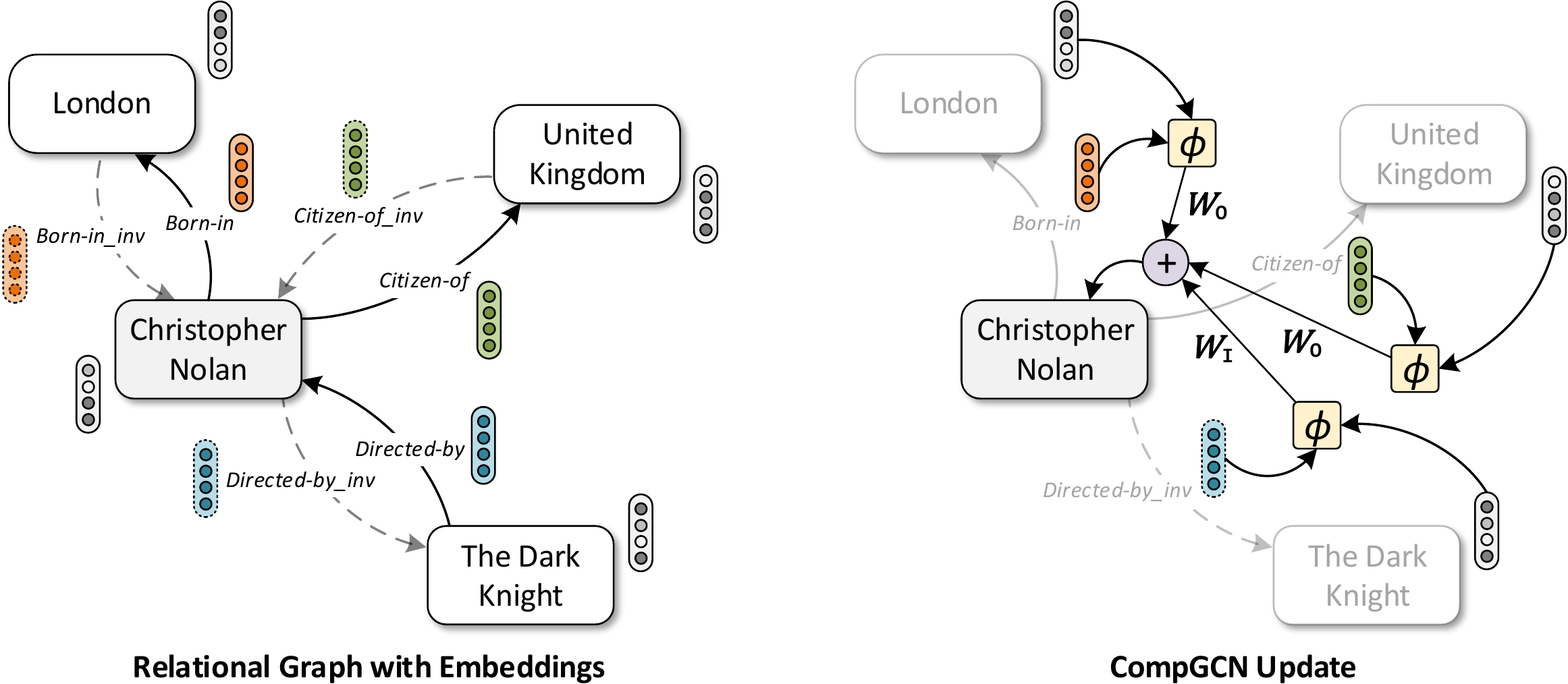}	
	\caption{\label{compgcn_fig:method_overview} Overview of \textsc{CompGCN}. Given node and relation embeddings, \textsc{CompGCN} performs a composition operation $\phi(\cdot)$ over each edge in the neighborhood of a central node (e.g. \textit{Christopher Nolan} above). The composed embeddings are then convolved with specific filters $\bm{W}_O$ and $\bm{W}_I$ for original and inverse relations respectively. We omit self-loop in the diagram for clarity. The message from all the neighbors are then aggregated to get an updated embedding of the central node. Also, the relation embeddings are transformed using a separate weight matrix. Please refer to Section \ref{compgcn_sec:details} for details.
	}	
\end{figure*}

There has been extensive research on embedding Knowledge Graphs (KG) \citep{survey2016nickel,survey2017} where representations of both nodes and relations are jointly learned. These methods are restricted to learning embeddings using link prediction objective. Even though GCNs can learn from task-specific objectives such as classification, their application has been largely restricted to non-relational graph setting. 
%However, there has not yet been an attempt to utilize KG embedding techniques to extend GCNs to multi-relational graphs.
%thus cannot be directly utilized for supervised tasks such as node classification on multi-relational graphs. GCNs on the other hand can be effectively utilized for such tasks.
%where representations for both vertices and relations are jointly learnt. But there has not yet been a attempt to incorporate these methods into the GCN framework.
Thus, there is a need for a framework which can utilize KG embedding techniques for learning task-specific node and relation embeddings.  In this chapter, we propose \textsc{CompGCN}, a novel GCN framework for multi-relational graphs which systematically leverages entity-relation composition operations from knowledge graph embedding techniques.
%that can systematically incorporate the advances in knowledge graph embedding techniques into graph convolutional methods.
\textsc{CompGCN} addresses the shortcomings of previously proposed GCN models by jointly learning vector representations for both nodes and relations in the graph. An overview of \textsc{CompGCN} is presented in Figure \ref{compgcn_fig:method_overview}.
%It uses non-parametric composition operators like subtraction, multiplication etc. which are motivated by knowledge graph embedding models \cite{transe, distmult, hole}. Such compositonal operations allow \textsc{CompGCN} to learn relation-aware vertex representations. Overall, \textsc{CompGCN} enables us to perform tasks like link prediction in knowledge graphs in an end-to-end manner.
The contributions of our work can be summarized as follows:

\begin{enumerate}[itemsep=2pt,parsep=0pt,partopsep=0pt,leftmargin=*,topsep=0.2pt]
	\item We propose \textsc{CompGCN}, a novel framework for incorporating multi-relational information in Graph Convolutional Networks which leverages a variety of composition operations from knowledge graph embedding techniques to jointly embed both nodes and relations in a graph.
	\item We demonstrate that \textsc{CompGCN} framework generalizes several existing multi-relational GCN methods (Proposition \ref{compgcn_prop:reduction}) and also scales with the increase in number of relations in the graph (Section \ref{compgcn_sec:results_basis}). 
%	Unlike most GCN based multi-relational graph embedding methods, \textsc{CompGCN} jointly learns embeddings of both nodes and relations in the graph.
	\item Through extensive experiments on tasks such as node classification, link prediction, and graph classification, we demonstrate the effectiveness of our proposed method.
\end{enumerate} 
%\noindent The source code of \textsc{CompGCN} and datasets used in this work have been made available.\footnote{Source code of \textsc{CompGCN}: \href{https://drive.google.com/open?id=1lNa7SumZXJtwctwIfpZ497Ob_bOoVALA}{link}}
\section{Related Work}
\label{compgcn_sec:related_work}

\textbf{Graph Convolutional Networks:} GCNs generalize Convolutional Neural Networks (CNNs) to non-Euclidean data. GCNs were first introduced by \cite{Bruna2013} and later made scalable through efficient localized filters in the spectral domain \citep{Defferrard2016}. A first-order approximation of GCNs using Chebyshev polynomials has been proposed by \cite{Kipf2016}. Recently, several of its extensions have also been formulated \citep{graphsage,gat,gin}. Most of the existing GCN methods follow \textit{Message Passing Neural Networks} (MPNN) framework \citep{mpnn} for node aggregation. 
Our proposed method can be seen as an instantiation of the MPNN framework. However, it is specialized for relational graphs.

%In NLP, GCNs have been utilized in many problems such as semantic role labeling \cite{gcn_srl}, event detection \cite{gcn_event}, and relation extraction \cite{re_over_pruned}. 

%However, most work in this area ignores relation information. A detailed description of GCNs can be found in \cite{Bronstein2017,gcn_survey}.

\textbf{GCNs for Multi-Relational Graph:} An extension of GCNs for relational graphs is proposed by \cite{gcn_srl}. However, they only consider direction-specific filters and ignore relations due to over-parameterization. \cite{r_gcn} address this shortcoming by proposing basis and block-diagonal decomposition of relation specific filters. \textit{Weighted Graph Convolutional Network}  \citep{sacn} utilizes learnable relational specific scalar weights during GCN aggregation. While these methods show performance gains on node classification and link prediction, they are limited to embedding only the nodes of the graph. 
%Comtemporary to our work, \cite{kbat} and \cite{vrgcn} have also proposed extensions of GCNs for embedding both nodes and relations in multi-relational graphs. The former propose a two-step training procedure for link prediction whereas we present an end-to-end differentiable model for the task. In comparison to the latter model which draws inspiration from \cite{transe} for embedding relations, \textsc{CompGCN} is a generalized framework that can leverage any KG composition operator. We provide a detailed comparison against both the method in Section \ref{compgcn_sec:results}. 
Contemporary to our work, \cite{vrgcn} have also proposed an extension of GCNs for embedding both nodes and relations in multi-relational graphs. However, our proposed method is a more generic framework which can leverage any KG composition operator. We compare against their method in Section \ref{compgcn_sec:results_link}.

%A detailed comparison of our proposed method and all the existing approaches is presented in Table \ref{compgcn_table:gcn_model_comp}.

\textbf{Knowledge Graph Embedding:} 
%Embedding Knowledge Graph is the task of learning a representation for all nodes and relation in the graphs. 
Knowledge graph (KG) embedding is a widely studied field \citep{survey2016nickel, survey2017} with application in tasks like link prediction and question answering \citep{kg_question_answering}. Most of KG embedding approaches define a score function and train node and relation embeddings such that valid triples are assigned a higher score than the invalid ones. Based on the type of score function, KG embedding method are classified as translational \citep{transe,transh}, semantic matching based \citep{distmult,hole} and neural network based \citep{ntn_kg,conve}. In our work, we evaluate the performance of \textsc{CompGCN} on link prediction with methods of all three types.

\section{\textsc{CompGCN} Details}
\label{compgcn_sec:details}

In this section, we provide a detailed description of our proposed method, \textsc{CompGCN}. The overall architecture is shown in Figure \ref{compgcn_fig:method_overview}. We represent a multi-relational graph by $\m{G}=(\m{V}, \m{R}, \m{E}, \bm{\m{X}},  \bm{\m{Z}})$ as defined in Section \ref{sec:directed_gcn} where $\bm{\m{Z}} \in \real{|\m{R}| \times d_0}$ denotes the initial relation features. Our model is motivated by the first-order approximation of GCNs using Chebyshev polynomials \citep{Kipf2016}. Following \citet{gcn_srl}, we also allow the information in a directed edge to flow along both directions. Hence, we extend $\m{E}$ and $\m{R}$ with corresponding inverse edges and relations, i.e., 
\begin{equation*}
\m{E'} = \m{E} \cup \{(v,u,r^{-1})~|~ (u,v,r) \in \m{E}\} \cup \{(u, u, \top)~|~u \in \m{V})\} ,
\end{equation*}
and $\m{R}' = \m{R} \cup \m{R}_{inv} \cup \{\top\}$, where $\m{R}_{inv} =\{ r^{-1} \hspace{2pt} | \hspace{2pt} r \in \m{R} \}$ denotes the inverse relations and $\top$ indicates the self loop.

\subsection{Relation-based Composition}
\label{compgcn_sec:details_relation}
Unlike most of the existing methods which embed only nodes in the graph, \textsc{CompGCN} learns a $d$-dimensional representation $\bm{h}_r \in \real{d}, \forall r \in \m{R}$ along with node embeddings $\bm{h}_v \in \real{d}, \forall v \in \m{V}$. Representing relations as vectors alleviates the problem of over-parameterization while applying GCNs on relational graphs. Further, it allows \textsc{CompGCN} to exploit any available relation features $(\bmm{Z})$ as initial representations. To incorporate relation embeddings into the GCN formulation, we leverage the entity-relation composition operations used in Knowledge Graph embedding approaches \citep{transe,survey2016nickel}, which are of the form
\[
\bm{e}_o = \phi(\bm{e}_s, \bm{e}_r).
\]
Here, $\phi:\mathbb{R}^{d} \times \mathbb{R}^{d} \to \mathbb{R}^{d}$ is a composition operator, $s$, $r$, and $o$ denote subject, relation and object in the knowledge graph and $\bm{e}_{(\cdot)} \in \real{d}$ denotes their corresponding embeddings. In this chapter, we restrict ourselves to non-parameterized operations like subtraction \citep{transe}, multiplication \citep{distmult} and circular-correlation \citep{hole}. However, \textsc{CompGCN} can be extended to parameterized operations like Neural Tensor Networks (NTN) \citep{ntn_kg} and ConvE \citep{conve}. We defer their analysis as future work.

As we show in Section \ref{compgcn_sec:results}, the choice of composition operation is important in deciding the quality of the learned embeddings. Hence, superior composition operations for Knowledge Graphs developed in future can be adopted to improve \textsc{CompGCN}'s performance further.

\subsection{\textsc{CompGCN} Update Equation}

The GCN update equation (Eq. \ref{eqn:kipf_main})) defined in Section \ref{sec:undirected_gcn} can be re-written as
\[
\bm{h}_{v} = f \Bigg(\sum_{ (u,r) \in \mathcal{N}(v)}  \bm{W}_{r} \bm{h}_{u} \Bigg),
\]
where $\m{N}(v)$ is a set of immediate neighbors of $v$ for its outgoing edges. Since this formulation suffers from over-parameterization, in \textsc{CompGCN} we perform composition ($\phi$) of a neighboring node $u$ with respect to its relation $r$ as defined above. This allows our model to be relation aware while being linear ($\mathcal{O}(|\m{R}|d)$) in the number of feature dimensions. Moreover, for treating original, inverse, and self edges differently, we define separate filters for each of them. The update equation of \textsc{CompGCN} is given as:
\begin{equation}
\label{compgcn_eq:node_update}
\bm{h}_{v} = f \Bigg(\sum_{ (u,r) \in \mathcal{N}(v)} \bm{W}_{\lambda(r)} \phi(\bm{x}_{u}, \bm{z}_r) \Bigg),
\end{equation}
where $\bm{x}_u, \bm{z}_r$ denotes initial features for node $u$ and relation $r$ respectively, $\bm{h}_{v}$ denotes the updated representation of node $v$, and $\bm{W}_{\lambda(r)} \in \real{d_1 \times d_0}$ is a relation-type specific parameter.  In \textsc{CompGCN}, we use direction specific weights, i.e., $\lambda({r}) = \mathrm{dir}(r)$, given as:
\begin{equation}
\label{compgcn_eqn:weight_def}
\bm{W}_{\mathrm{dir}(r)} =
\begin{cases} 
\bm{W}_O, & r \in \m{R}\\
\bm{W}_I, & r \in \m{R}_{inv}\\
\bm{W}_S, & r = \top  \ \text{\textit{(self-loop)}}
\end{cases}
\end{equation}

Further, in \textsc{CompGCN}, after the node embedding update defined in Eq. \ref{compgcn_eq:node_update}, the relation embeddings are also transformed as follows:
\begin{equation}
\label{compgcn_eq:rel_share}
\bm{h}_r = \bm{W}_{\mathrm{rel}} \bm{z}_r   ,
\end{equation}
where $\bm{W}_{\mathrm{rel}} \in \real{d_1 \times d_0}$ is a learnable transformation matrix which projects the relations to the same embedding space as nodes and allows them to be utilized in the next \textsc{CompGCN} layer. 

%where we do not have an initial relation feature matrix $\bmm{Z}_r$
To ensure that \textsc{CompGCN} scales with the increasing number of relations, we use a variant of the basis formulations proposed in \citet{r_gcn}. Instead of independently defining an embedding for each relation, they are expressed as a linear combination of a set of basis vectors. Formally, let $ \{\bm{v}_1, \bm{v}_2, ..., \bm{v}_{\m{B}} \}$ be a set of learnable basis vectors. Then, initial relation representation is given as:
\[
\bm{z}_r = \sum_{b=1}^{\mathcal{B}} \alpha_{{}_{br}} \bm{v}_b.
\]
Here, $\alpha_{{}_{br}} \in \real{}$ is relation and basis specific learnable scalar weight. Note that this is different from the formulation in \citet{r_gcn}, where a separate set of basis matrices is defined for each GCN layer. In \textsc{CompGCN}, basis vectors are defined only for the first layer, and the later layers share the relations through transformations according to Equation \ref{compgcn_eq:rel_share}.

We can extend the formulation of Equation \ref{compgcn_eq:node_update} to the case where we have $k$-stacked \textsc{CompGCN} layers. Let $\bm{h}_v^{k+1}$ denote the representation of a node $v$ obtained after $k$ layers which is defined as
\begin{equation}
\label{compgcn_eqn:main_upd}
\bm{h}_{v}^{k+1} = f \Bigg(\sum_{ (u,r) \in \mathcal{N}(v)} \bm{W}_{\lambda(r)}^k \phi(\bm{h}_{u}^k, \bm{h}_r^k) \Bigg) .
\end{equation}
Similarly, let $\bm{h}_r^{k+1}$ denote the representation of a relation $r$ after $k$ layers. Then,
\[
\bm{h}_r^{k+1} = \bm{W}_{\mathrm{rel}}^k \  \bm{h}_r^k.
\]
Here, $\bm{h}_v^0$ and $\bm{h}_r^0$ are the initial node ($\bm{x}_v$) and relation ($\bm{z}_r$) features respectively.

\begin{proposition}
\label{compgcn_prop:reduction}
\textsc{CompGCN} generalizes the following Graph Convolutional based methods: \textbf{Kipf-GCN} \citep{Kipf2016}, \textbf{Relational GCN} \citep{r_gcn}, \textbf{Directed GCN} \citep{gcn_srl}, and \textbf{Weighted GCN} \citep{sacn}. 
\end{proposition}
\begin{proof}
	For Kipf-GCN, this can be trivially obtained by making weights $(\bm{W}_{\lambda(r)})$ and composition function $(\phi)$ relation agnostic in Equation \ref{compgcn_eqn:main_upd}, i.e., $\bm{W}_{\lambda(r)} = \bm{W}$ and $\phi(\bm{h}_{u}, \bm{h}_r) = \bm{h}_u$. Similar reductions can be obtained for other methods as shown in Table \ref{compgcn_tbl:reduction}.
\end{proof}

\section{Experiments}
\subsection{Experimental Setup}

%\begin{table}[ht]
%	\begin{minipage}[b]{0.5\linewidth}
%		\centering
%		\begin{tabular}{lrrr}
%			\toprule
%			\textbf{Dataset} & \multicolumn{1}{c}{\textbf{FB15k-237}} & \multicolumn{1}{c}{\textbf{MUTAG}} &  \multicolumn{1}{c}{\textbf{AM}} \\
%			\midrule
%			Entities & 14,541 & 23,644 & 1,666,764 \\
%			Edges & 310,116 & 74,227 & 5,988,321 \\
%			Relations & 237 & 23 & 133 \\
%			Labeled & - & 340 & 1000 \\
%			Classes & - & 2 & 11 \\
%			\bottomrule
%		\end{tabular}
%		\caption{Student Database}
%		\label{compgcn_table:student}s
%	\end{minipage}\hfill
%	\begin{minipage}[b]{0.47\linewidth}
%		\centering
%		\includegraphics[width=\linewidth]{sections/compgcn/images/link_pred-crop.pdf}
%		\captionof{figure}{2-D scatterplot of the Student Database}
%		\label{compgcn_fig:image}
%	\end{minipage}
%\end{table}

%\begin{minipage}{0.45\linewidth}
%
%\end{minipage}
%\begin{minipage}{0.45\linewidth}
%	\centering
%	\includegraphics[width=0.8\linewidth]{sections/compgcn/images/link_pred-crop.pdf}
%\end{minipage}

\begin{table}[t]
	\centering
	\small
	\begin{tabular}{llc}
		\toprule
		  \multicolumn{1}{c}{\textbf{Methods}} & $\bm{W}^k_{\lambda(r)}$ & $\phi(\bm{h}^k_{u}, \bm{h}^k_r)$ \\ 
		\midrule
		Kipf-GCN \citep{Kipf2016} & $\bm{W}^k$ & $\bm{h}^k_u$ \\ 
	    Relational-GCN  \citep{r_gcn} & $\bm{W}^k_r$ & $\bm{h}^k_u$ \\ 
	    Directed-GCN \citep{gcn_srl} & $\bm{W}^k_{\mathrm{dir}(r)}$ & $\bm{h}^k_u$ \\ 
	    Weighted-GCN  \citep{sacn} & $\bm{W}^k$ & $\alpha^k_{r} \bm{h}^k_u$\\ 
%	    Vectorized Relational-GCN (VR-GCN) \citep{vrgcn} & $\bm{W}$ & $\alpha_{r} \bm{h}_u$\\ 
		\bottomrule
	\end{tabular}
	
	\caption{\label{compgcn_tbl:reduction}Reduction of \textsc{CompGCN} to several existing Graph Convolutional methods. Here, $\alpha^k_r$ is a relation specific scalar, $\bm{W}^k_r$ denotes a separate weight for each relation, and $\bm{W}^k_{\mathrm{dir}(r)}$ is as defined in Equation \ref{compgcn_eqn:weight_def}. Please refer to Proposition \ref{compgcn_prop:reduction} for more details.}
\end{table}

%\begin{wrapfigure}{R}{0.45\textwidth}
%	\begin{center}
%		\includegraphics[width=0.45\textwidth]{./sections/compgcn/images/link_pred-crop.pdf}
%	\end{center}
%	\caption{\label{compgcn_fig:link_pred} Knowledge Graph link prediction with \textsc{CompGCN} and other methods. \textsc{CompGCN} generates both entity and relation embedding as opposed to just entity embeddings for other models. For more details, please refer to Section \ref{compgcn_sec:exp_tasks}.
%		%\sv{This figure can be removed.} \su{we can keep this figure, but we should add NMT also.}
%	}
%\end{wrapfigure}

\subsubsection{Evaluation tasks}
\label{compgcn_sec:exp_tasks}
In our experiments, we evaluate \textsc{CompGCN} on the below-mentioned tasks.%\footnote{}

\begin{itemize}[itemsep=2pt,parsep=0pt,partopsep=0pt,leftmargin=*,topsep=2pt]

\item  \textbf{Link Prediction} is the task of inferring missing facts based on the known facts in Knowledge Graphs. In our experiments, we utilize FB15k-237 \citep{toutanova} and WN18RR \citep{conve} datasets for evaluation. Following \cite{transe}, we use filtered setting for evaluation and report Mean Reciprocal Rank (MRR), Mean Rank (MR) and Hits@N.
	
\item  \textbf{Node Classification} is the task of predicting the labels of nodes in a graph-based on node features and their connections. Similar to \cite{r_gcn}, we evaluate \textsc{CompGCN} on MUTAG (Node) and AM \citep{rdf2vec} datasets.  %We use $K$-layers of GCN based methods followed by a softmax activation to predict class labels.

\item \textbf{Graph Classification}, where, given a set of graphs and their corresponding labels, the goal is to learn a representation for each graph which is fed to a classifier for prediction. We evaluate on 2 bioinformatics dataset: MUTAG (Graph) and PTC \citep{graph_datasets}.  
%\item  \textbf{Semantic Role Labelling}
%Following the setup in \cite{gcn_srl}, we construct syntactic dependency trees and apply \textsc{CompGCN} on these graphs. For evaluation, we use the CoNLL 2012 dataset \cite{ontonotes}.

%\item  \textbf{Neural Machine Translation} is the task of translating a source sentence to a target sentence given translation pairs from a parallel corpus. In our experiments, we evaluate on two English-German datasets: Multi30k \cite{multi30k} and En-De News Commentary v11 dataset from WMT16 translation task\footnote{http://www.statmt.org/wmt16/translation-task.html}. Similar to \citet{gcn_mt}, we apply different GCN methods on syntactic dependency parse obtained from Stanford CoreNLP \cite{stanford_corenlp} over the output of RNN based encoder. Finally, an RNN based decoder is used for getting the target sentence.
\end{itemize} 
A summary statistics of the datasets used is provided in Table \ref{compgcn_table:rgcn_data}.
%\begin{figure}[t]	
%	\centering	
%	\includegraphics[width=0.5\columnwidth]{./sections/compgcn/images/link_pred-crop.pdf}	
%	\caption{\label{compgcn_fig:link_pred} Knowledge Graph link prediction with \textsc{CompGCN} and other methods. \textsc{CompGCN} generates both entity and relation embedding as opposed to just entity embeddings for other models. For more details, please refer to Section \ref{compgcn_sec:exp_tasks}.}
%		
%\end{figure}

\begin{table}[t!]
	\centering
	\small
	\begin{tabular}{lcccccc}
		\toprule
		&  \multicolumn{2}{c}{\bf Link Prediction} & \multicolumn{2}{c}{\bf Node Classification} & \multicolumn{2}{c}{\bf Graph Classification}\\ 
		\cmidrule(r){2-3} \cmidrule(r){4-5} \cmidrule(r){6-7} 
		& \multicolumn{1}{c}{FB15k-237} & \multicolumn{1}{c}{WN18RR} & \multicolumn{1}{c}{MUTAG (Node)} &  \multicolumn{1}{c}{AM} & \multicolumn{1}{c}{MUTAG (Graph)} & \multicolumn{1}{c}{PTC}\\
		\midrule
		Graphs   & 1 & 1 & 1 & 1 & 188 & 344 \\
		Entities   & 14,541 & 40,943 & 23,644 & 1,666,764 & 17.9 (Avg) & 25.5 (Avg)\\
		Edges 	  & 310,116 & 93,003 & 74,227 & 5,988,321 & 39.6 (Avg) & 29.5 (Avg)\\
		Relations & 237 & 11 & 23 & 133 & 4 & 4 \\
		%		Labeled   & - & - & 340 & 1000 & - & -\\
		Classes   & - & - & 2 & 11 & 2 & 2\\
		\bottomrule
	\end{tabular}
	
	\caption{\label{compgcn_table:rgcn_data}The details of the datasets used for node classification, link prediction, and graph classification tasks. Please refer to Section \ref{compgcn_sec:exp_tasks} for more details.}
\end{table}

\subsubsection{Baselines}
\label{compgcn_sec:exp_baselines}

Across all tasks, we compare against the following GCN methods for relational graphs: (1) Relational-GCN (\textbf{R-GCN}) \citep{r_gcn} which uses relation-specific weight matrices that are defined as a linear combinations of a set of basis matrices. (2) Directed-GCN (\textbf{D-GCN}) \citep{gcn_srl} has separate weight matrices for incoming edges, outgoing edges, and self-loops. It also has relation-specific biases. (3) Weighted-GCN (\textbf{W-GCN}) \citep{sacn} assigns a learnable scalar weight to each relation and multiplies an incoming "message" by this weight. Apart from this, we also compare with several task-specific baselines mentioned below.

\textbf{Link prediction:} For evaluating \textsc{CompGCN}, we compare against several non-neural and neural baselines: TransE \cite{transe}, DistMult \citep{distmult}, ComplEx \citep{complex}, R-GCN \citep{r_gcn}, KBGAN \citep{kbgan}, ConvE \citep{conve}, ConvKB \citep{convkb}, SACN \citep{sacn}, HypER \citep{hyper}, RotatE \citep{rotate}, ConvR \citep{convr}, and VR-GCN \citep{vrgcn}.

\textbf{Node and Graph Classification:} For node classification, following \cite{r_gcn}, we compare with Feat \citep{feat}, WL \citep{wl}, and RDF2Vec \citep{rdf2vec}. Finally, for graph classification, we evaluate against \textsc{PachySAN} \citep{pachysan}, Deep Graph CNN (DGCNN)  \citep{dgcnn}, and Graph Isomorphism Network (GIN) \citep{gin}.

\begin{table*}[t]
	\centering
	\begin{small}
		\resizebox{\textwidth}{!}{
			\begin{tabular}{lccccccccccc}
				\toprule
				& \multicolumn{5}{c}{\textbf{FB15k-237}} && \multicolumn{5}{c}{\textbf{WN18RR}}\\ 
				\cmidrule(r){2-6}  \cmidrule(r){8-12} 
				%		& \multicolumn{1}{c}{} && \multicolumn{3}{c}{Hits} & \multicolumn{1}{c}{} && \multicolumn{3}{c}{Hits} \\ 
				%		\cmidrule(r){4-6}   \cmidrule(r){9-11} 
				& MRR & MR &H@10 &  H@3 & H@1 && MRR & MR & H@10 & H@3 & H@1 \\
				\midrule
				TransE \citep{transe}  & .294 & 357 & .465 & - & - && .226 & 3384 & .501 & - & - \\
				DistMult \citep{distmult}	& .241 & 254 & .419 & .263 & .155 && .43 & 5110 & .49  & .44 & .39 \\
				ComplEx	\citep{complex}		& .247 & 339 & .428 & .275 & .158 && .44  & 5261 & .51  & .46 & .41 \\
				R-GCN \citep{r_gcn}		& .248 & -   & .417 & & .151 && -    & -    & -    & & -  \\
				KBGAN \citep{kbgan}		& .278 & -   & .458 & & -    && .214 & -    & .472 & - & -\\
				ConvE \citep{conve}		& .325 & 244 & .501 & .356 & .237 &&  .43 & 4187 & .52  & .44 & .40  \\
				ConvKB \citep{convkb} & .243 & 311 & .421 & .371 & .155 && .249 & 3324 & .524 & .417 & .057 \\
				SACN \citep{sacn} 		& .35  & -   & .54  & .39 & .26  && .47  & -    & .54 & .48 & .43 \\
				HypER \citep{hyper} 	& .341 & 250 & .520 & .376 & .252 && .465 & 5798 & .522 & .477 & .436 \\
				RotatE \citep{rotate}		& .338 & \textbf{177} & .533 & .375 & .241 && .476 & \textbf{3340} & \textbf{.571} & .492 & .428 \\
				ConvR \citep{convr} 			& .350 & - & .528 & .385 & .261 && .475 & - & .537 & .489 & \textbf{.443} \\
				%				KBAT \citep{kbat} 			& .350 & - & .528 & .385 & .261 & .475 & - & .537 & .489 & \textbf{.443} \\
				%				KBAT (Sach) \citep{kbat} 			& .265 & 298 & .439 & .296 & .178 && \\
				VR-GCN \citep{vrgcn} 	& .248 & - & .432 & .272 & .159 && - & - & - & - & - \\
				%				DihEdral \citep{dihedral} & .320 & - & .502 & .353 & .230 && .486 & - & .557 & .505 & .442 \\ 
				%				TuckER \cite{tucker} & .358 & - & .544 & .394 & .266 & .470 & - & .526 & .482 & .443 \\
				
				\midrule 
				\textsc{CompGCN} (Proposed Method)		& \textbf{.355} & 197 & \textbf{.535} & \textbf{.390} & \textbf{.264} &&  \textbf{.479} & 3533 & .546  & \textbf{.494} & \textbf{.443}  \\
				\bottomrule
				\addlinespace
			\end{tabular}
		}
		\caption{\label{compgcn_tbl:link_pred} \small Link prediction performance of \textsc{CompGCN} and several recent models on FB15k-237 and WN18RR datasets. The results of all the baseline methods are taken directly from the previous papers. We find that \textsc{CompGCN} outperforms all the existing methods on $4$ out of $5$ metrics on FB15k-237 and $3$ out of $5$ metrics on WN18RR. Please refer to Section \ref{compgcn_sec:results_link} for more details. } 
	\end{small}
\end{table*}

\subsection{Results}
\label{compgcn_sec:results}

%\begin{wrapfigure}{R}{0.45\textwidth}
%	\begin{center}
%		\includegraphics[width=0.45\textwidth]{./sections/compgcn/images/link_pred-crop.pdf}
%	\end{center}
%	\caption{\label{compgcn_fig:link_pred} Knowledge Graph link prediction with \textsc{CompGCN} and other methods. \textsc{CompGCN} generates both entity and relation embedding as opposed to just entity embeddings for other models. For more details, please refer to Section \ref{compgcn_sec:exp_tasks}.
%		%\sv{This figure can be removed.} \su{we can keep this figure, but we should add NMT also.}
%	}
%\end{wrapfigure}

%\begin{wrapfigure}{!t}{0.5\textwidth}
%	
%	\centering
%	\includegraphics[width=0.5\columnwidth]{./sections/compgcn/images/basis_plot-crop.pdf}	
%	
%	\caption{\label{compgcn_fig:basis_plot} Comparison of performance of \textsc{CompGCN} with different number of relation basis vectors on link prediction task. We report the obtained MRR on FB15k-237 dataset. Overall, we find that \textsc{CompGCN} gives comparable performance even with limited parameters. Refer to Section \ref{compgcn_sec:results_basis} for details.}
%\end{wrapfigure}

In this section, we attempt to answer the following questions.
\begin{itemize}[itemsep=1pt,topsep=2pt,parsep=0pt,partopsep=0pt]
	\item[Q1.] How does \textsc{CompGCN} perform on link prediction compared to existing methods? (\ref{compgcn_sec:results_link})
	\item[Q2.] What is the effect of using different GCN encoders and choice of the compositional operator in \textsc{CompGCN} on link prediction performance? (\ref{compgcn_sec:results_link})
	\item[Q3.] Does \textsc{CompGCN} scale with the number of relations in the graph? (\ref{compgcn_sec:results_basis})
	\item[Q4.] How does \textsc{CompGCN} perform on node and graph classification tasks? (\ref{compgcn_sec:res_node_nmt})
%	\item[Q4.] How does \textsc{CompGCN} perform on different relation categories for the link prediction task? (Section \ref{compgcn_sec:results_rel_cat})
\end{itemize}

\subsubsection{Performance Comparison on Link Prediction}
\label{compgcn_sec:results_link}

In this section, we evaluate the performance of \textsc{CompGCN} and the baseline methods listed in Section \ref{compgcn_sec:exp_baselines} on link prediction task. The results on \datafbn{} and \datawnn{} datasets are presented in Table \ref{compgcn_tbl:link_pred}. The scores of baseline methods are taken directly from the previous papers \citep{rotate,kbgan,sacn,hyper,convr,vrgcn}. However, for ConvKB, we generate the results using the corrected evaluation code\footnote{https://github.com/KnowledgeBaseCompleter/eval-ConvKB}. Overall, we find that \textsc{CompGCN} outperforms all the existing methods in $4$ out of $5$ metrics on \datafbn{} and in $3$ out of $5$ metrics on \datawnn{} dataset. We note that the best performing baseline RotatE uses rotation operation in complex domain. The same operation can be utilized in a complex variant of our proposed method to improve its performance further. We defer this as future work.

%Following \cite{r_gcn,sacn}, we use GCN based methods as an encoder along with a standard link-prediction score function. %as shown in Figure \ref{compgcn_fig:link_pred}. 
%\textsc{CompGCN} and VR-GCN provide both entity and relation embeddings for the link prediction objective while other methods provide only entity representation and relation embeddings are defined independently.

\begin{table*}[!t]
	\centering
	%	\begin{small}
	\resizebox{\columnwidth}{!}{
		
		\begin{tabular}{m{13em}ccccccccccc}
			\toprule
			
			\multirow{1}{*}{\textbf{Scoring Function (=X)}} $\bm{\rightarrow}$ & \multicolumn{3}{c}{\textbf{TransE}} && \multicolumn{3}{c}{\textbf{DistMult}} && \multicolumn{3}{c}{\textbf{ConvE}} \\ 
			\cmidrule(r){2-4}  \cmidrule(r){6-8} \cmidrule(r){10-12}
			\textbf{Methods} $\bm{\downarrow}$ & MRR & MR & H@10 && MRR & MR & H@10 && MRR & MR & H@10 \\
			\midrule
			X	& 0.294	& 357	& 0.465	&& 0.241	& 354	& 0.419	&& 0.325	& 244	& 0.501 \\
			%			X + MPNN & 0.297 & 240	& 0.460	&& 0.299 & 256 & 0.461	&& 0.261	& 309	& 0.403 \\
			X + D-GCN & 0.299	& 351 & 0.469	&& 0.321	& 225 & 0.497	&&  0.344 & 200 & 0.524\\	
			X + R-GCN & 0.281	& 325	& 0.443	&& 0.324	& 230	& 0.499	&& 0.342	& 197	& 0.524 \\
			X + W-GCN & 0.267	& 1520	& 0.444	&& 0.324	& 229	& 0.504	&& 0.344	& 201	& 0.525 \\
			\midrule
			X + \textsc{CompGCN} (Sub)	& 0.335	& \textbf{194}	& 0.514	&& 0.336	& 231	& 0.513	&& 0.352	& 199	& 0.530 \\
			X + \textsc{CompGCN} (Mult)	& \textbf{0.337}	& 233	& 0.515	&& \textbf{0.338	}& \textbf{200}	& \textbf{0.518	}&& 0.353	& 216	& 0.532 \\
			X + \textsc{CompGCN} (Corr)	& 0.336	& 214	& \textbf{0.518}	&& 0.335	& 227	& 0.514	&& \boxed{\textbf{0.355}} & 197	& \boxed{\textbf{0.535}} \\
			\midrule
			X + \textsc{CompGCN} ($\m{B} = 50$) & 0.330 & 203 & 0.502 && 0.333 & 210 & 0.512 && 0.350 & \boxed{\textbf{193}} & 0.530 \\
			\bottomrule
		\end{tabular}
	}
	\caption{\label{compgcn_tbl:link_pred_results} \small Performance on link prediction task evaluated on FB15k-237 dataset. X + M (Y) denotes that method M is used for obtaining entity (and relation) embeddings with X as the scoring function. In the case of \textsc{CompGCN}, Y denotes the composition operator used. $\m{B}$ indicates the number of relational basis vectors used. Overall, we find that \textsc{CompGCN} outperforms all the existing methods across different scoring functions. ConvE + \textsc{CompGCN} (Corr) gives the best performance across all settings (highlighted using \boxed{\cdot}). Please refer to Section \ref{compgcn_sec:results_link} for more details.}   
	
	%	\end{small}
\end{table*}

\begin{figure*}[t]
	\centering
	\begin{minipage}{.485\textwidth}
		\centering
		\includegraphics[width=1.\linewidth]{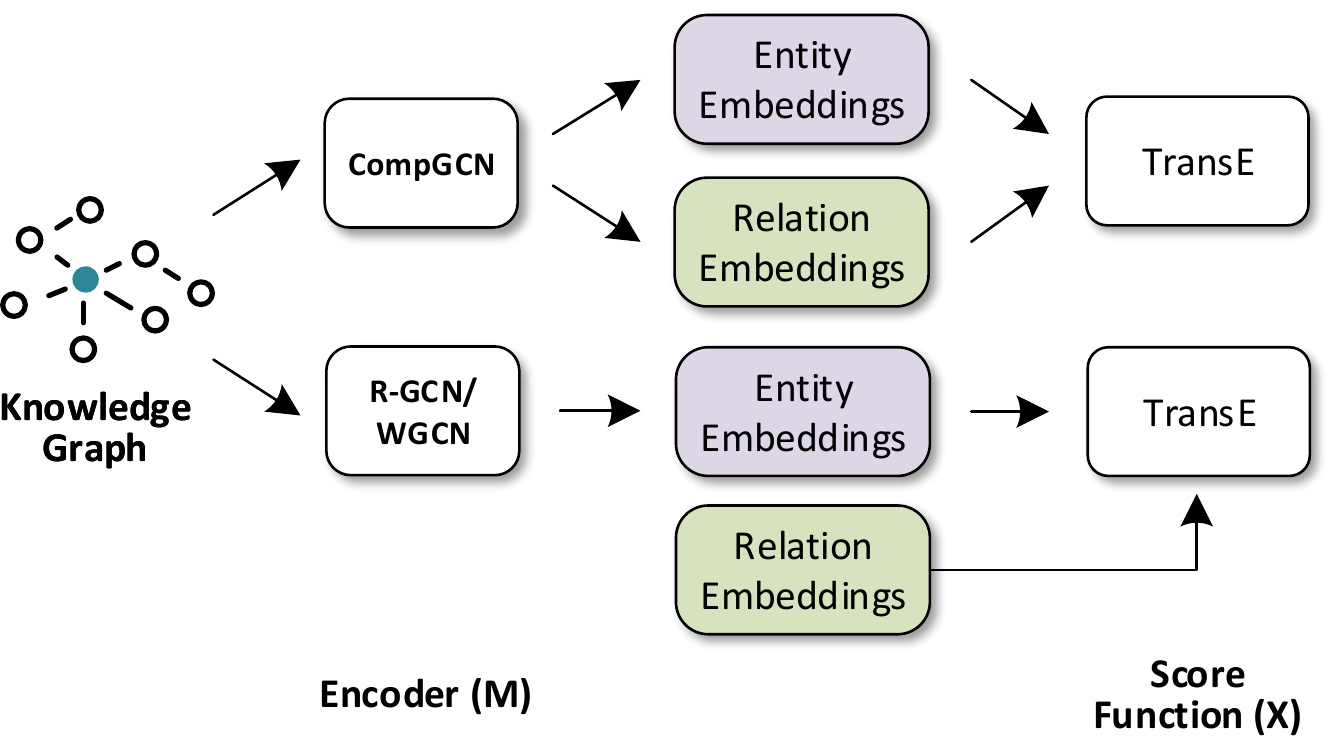}
		\caption{\label{compgcn_fig:link_pred} \small Knowledge Graph link prediction with \textsc{CompGCN} and other methods. \textsc{CompGCN} generates both entity and relation embedding as opposed to just entity embeddings for other models. For more details, please refer to Section \ref{compgcn_sec:results_gcn_encoder}}
	\end{minipage} \quad
	\begin{minipage}{0.47\textwidth}
		\centering
		\includegraphics[width=.99\linewidth]{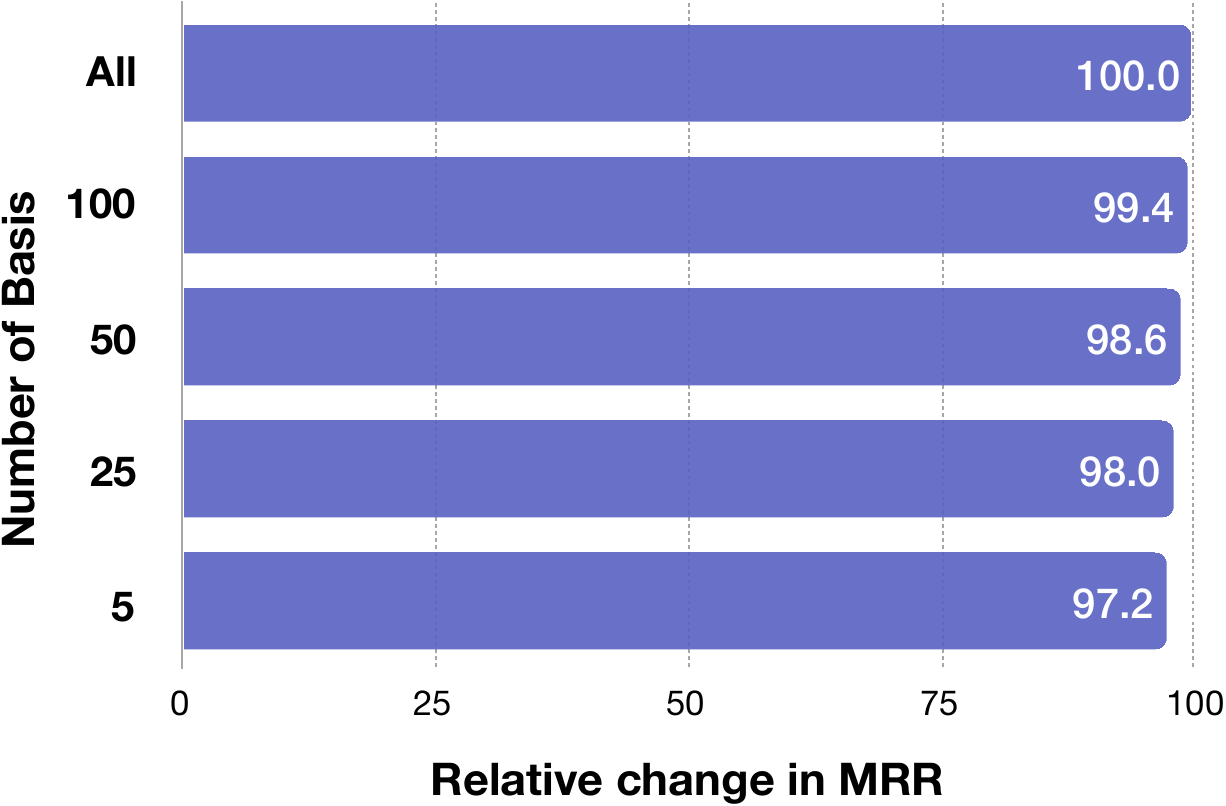}
		\caption{\label{compgcn_fig:basis_plot} \small Performance of \textsc{CompGCN} with different number of relation basis vectors on link prediction task. We report the relative change in MRR on FB15k-237 dataset. Overall, \textsc{CompGCN} gives comparable performance even with limited parameters. Refer to Section \ref{compgcn_sec:results_basis} for details.}
	\end{minipage} 
\end{figure*}

\subsubsection{Comparison of Different GCN Encoders on Link Prediction Performance}
\label{compgcn_sec:results_gcn_encoder}

Next, we evaluate the effect of using different GCN methods as an encoder along with a representative score function (shown in Figure \ref{compgcn_fig:link_pred}) from each category: TransE (translational), DistMult (semantic-based), and ConvE (neural network-based).  In our results, \textbf{X + M (Y)}  denotes that method \textbf{M} is used for obtaining entity embeddings (and relation embeddings in the case of \textsc{CompGCN}) with \textbf{X} as the score function as depicted in Figure \ref{compgcn_fig:link_pred}. \textbf{Y} denotes the composition operator in the case of \textsc{CompGCN}. We evaluate \textsc{CompGCN} on three non-parametric composition operators inspired from TransE \citep{transe}, DistMult \citep{distmult}, and HolE \citep{hole} defined as
\begin{itemize}[itemsep=1pt,topsep=1pt,parsep=0pt,partopsep=0pt,leftmargin=5.5mm]
	\item \textbf{Subtraction (Sub):} $\phi(\bm{e}_s, \bm{e}_r) = \bm{e}_s - \bm{e}_r.$
	\item \textbf{Multiplication (Mult):} $\phi(\bm{e}_s, \bm{e}_r) = \bm{e}_s * \bm{e}_r.$
	\item \textbf{Circular-correlation (Corr):} $\phi(\bm{e}_s, \bm{e}_r) \text{=} \bm{e}_s \star \bm{e}_r$
\end{itemize}

The overall results are summarized in Table \ref{compgcn_tbl:link_pred_results}. Similar to \citet{r_gcn}, we find that utilizing Graph Convolutional based method as encoder gives a substantial improvement in performance for most types of score functions. We observe that although all the baseline GCN methods lead to some degradation with TransE score function, no such behavior is observed for \textsc{CompGCN}. On average, \textsc{CompGCN} obtains around $6$\%, $4$\% and $3$\% relative increase in MRR with TransE, DistMult, and ConvE objective respectively compared to the best performing baseline. The superior performance of \textsc{CompGCN} can be attributed to the fact that it learns both entity and relation embeddings jointly thus providing more expressive power in learned representations. Overall, we find that \textsc{CompGCN} with ConvE (highlighted using \boxed{\cdot}) is the best performing method for link prediction.%\footnote{We further analyze the best performing method for different relation categories in Appendix \ref{compgcn_sec:results_rel_cat}.}

\noindent \textbf{Effect of composition Operator:} 
The results on link prediction with different composition operators are presented in Table \ref{compgcn_tbl:link_pred_results}. 
%Overall, we find the performance varies with scoring functions across different composition operators. 
We find that with DistMult score function, multiplication operator (Mult) gives the best performance while with ConvE, circular-correlation surpasses all other operators. Overall, we observe that more complex operators like circular-correlation outperform or perform comparably to simpler operators such as subtraction.

\subsubsection{Parameter Efficiency of \textsc{CompGCN}} 
\label{compgcn_sec:results_basis}

In this section, we analyze the performance of \textsc{CompGCN} on changing the number of relation basis vectors ($\m{B}$) as defined in Section \ref{compgcn_sec:details}. For this, we evaluate the best performing model for link prediction (ConvE + \textsc{CompGCN} (Corr)) with a variable number of basis vectors. The results are summarized in Figure \ref{compgcn_fig:basis_plot}. We find that our model performance improves with the increasing number of basis vectors. We note that with $\m{B}=100$, the performance of the model becomes comparable to the case where all relations have their individual embeddings. 
In Table \ref{compgcn_tbl:link_pred_results}, we report the results for the best performing model across all score function with $\m{B}$ set to $50$. We note that the parameter-efficient variant also gives a comparable performance and outperforms the baselines in all settings. This demonstrates that \textsc{CompGCN} is scalable with the increasing number of relations and thus can be utilized for larger graphs effectively.

\subsubsection{Evaluation on Node and Graph Classification}
\label{compgcn_sec:res_node_nmt}

In this section, we evaluate \textsc{CompGCN} on node and graph classification tasks on datasets as described in Section \ref{compgcn_sec:exp_tasks}. The experimental results are presented in Table \ref{compgcn_tbl:nodeclass_results}. For node classification task, we report accuracy on test split provided by \cite{node_class_splits}, whereas for graph classification, following \cite{graph_datasets} and \cite{gin}, we report the average and standard deviation of validation accuracies across the 10 folds cross-validation.  Overall, we find that \textsc{CompGCN} outperforms all the baseline methods on node classification and gives a comparable performance on graph classification task. This demonstrates the effectiveness of incorporating relations using \textsc{CompGCN} over the existing GCN based models. On node classification, compared to the best performing baseline, we obtain an average improvement of $3$\% across both datasets while on graph classification, we obtain an improvement of $3$\% on PCT dataset.

\begin{table}[!t]
	\centering
	\small
	\resizebox{\textwidth}{!}{
		\begin{tabular}{lcc}
			\toprule
			& \textbf{MUTAG (Node)} & \textbf{AM} \\
			\midrule
			Feat$^*$		& 77.9 & 66.7 \\ 
			WL$^*$				& 80.9 & 87.4 \\
			RDF2Vec$^*$	  & 67.2 & 88.3 \\
			%		MPNN & 73.1 $\pm$ 3.6 & 84.2 $\pm$ 2.6 \\ 
			R-GCN$^*$		& 73.2	&89.3 \\
			SynGCN	& 74.8 $\pm$ 5.5	& 86.2 $\pm$ 1.9	\\
			WGCN	& 77.9 $\pm$ 3.2	& 90.2 $\pm$ 0.9	\\
			\midrule
			\textsc{CompGCN}	& \textbf{85.3 $\pm$ 1.2}	& \textbf{90.6 $\pm$ 0.2}	\\
			\bottomrule
			\addlinespace
		\end{tabular}
		\quad
		\begin{tabular}{lcc}
			\toprule
			& \textbf{MUTAG (Graph)} & \textbf{PCT} \\
			\midrule
			\sc{PachySAN}$^\dagger$			& \textbf{92.6 $\pm$ 4.2} & 60.0 $\pm$ 4.8 \\ 
			DGCNN$^\dagger$				& 85.8 & 58.6 \\
			GIN$^\dagger$	  	& 89.4 $\pm$ 4.7 & 64.6 $\pm$ 7.0 \\
			%	MPNN & 72.8 $\pm$ 13.1 & 65.0 $\pm$ 10.9 \\ 
			R-GCN & 82.3 $\pm$ 9.2	& 67.8 $\pm$ 13.2 \\
			SynGCN & 79.3 $\pm$ 10.3	& 69.4 $\pm$ 11.5 \\
			WGCN & 78.9 $\pm$ 12.0	& 67.3 $\pm$ 12.0 \\
			\midrule
			\textsc{CompGCN}	& 89.0 $\pm$ 11.1	& \textbf{71.6 $\pm$ 12.0}	\\
			\bottomrule
			\addlinespace
		\end{tabular}
	}
	\caption{\label{compgcn_tbl:nodeclass_results} \small Performance comparison on node classification (\textbf{Left}) and graph classification (\textbf{Right}) tasks.  $*$ and $\dagger$ indicate that results are directly taken from \cite{r_gcn} and \cite{gin} respectively. Overall, we find that \textsc{CompGCN} either outperforms or performs comparably compared to the existing methods. Please refer to Section \ref{compgcn_sec:res_node_nmt} for more details.}  \vspace{-5mm}
\end{table}

%\subsection{Effect of Number of Relational Basis}
%\label{compgcn_sec:results_basis}
%In this section, we analyze the performance of \textsc{CompGCN} on changing the number of relation basis vectors ($\m{B}$). For this, we evaluate the best performing model for link prediction (ConvE + \textsc{CompGCN} (Cor)) with a variable number of basis vectors. The results are summarized in Figure \ref{compgcn_fig:basis_plot}. We find that our model performance improves with increasing number of basis vectors. We note that with $100$ basis vectors, the performance of the model becomes comparable to the case where all relations have their individual embeddings. We also observe that with only $5$ basis vectors, \textsc{CompGCN} is able to perform comparably to the existing baselines (Table \ref{compgcn_tbl:link_pred_results}). This shows that our model performs competitively even with less basis vectors and hence can scale to graphs where the number of relations are large. 

%\begin{figure}[t]	
%	\centering	
%	\includegraphics[width=0.5\columnwidth]{./sections/compgcn/images/basis_plot-crop.pdf}	
%	\caption{\label{compgcn_fig:basis_plot} Comparison of performance of \textsc{CompGCN} with different number of relation basis vectors on link prediction task. We report the obtained MRR on FB15k-237 dataset. Overall, we find that \textsc{CompGCN} gives comparable performance even with limited parameters. Refer to Section \ref{compgcn_sec:results_basis} for details.}	
%\end{figure}

\section{Conclusion}
\label{compgcn_sec:conclusion}

In this chapter, we proposed \textsc{CompGCN}, a novel Graph Convolutional based framework for multi-relational graphs which leverages a variety of composition operators from Knowledge Graph embedding techniques to jointly embed nodes and relations in a  graph. Our method generalizes several existing multi-relational GCN methods. Moreover, our method alleviates the problem of over-parameterization by sharing relation embeddings across layers and using basis decomposition. Through extensive experiments on knowledge graph link prediction, node classification, and graph classification tasks, we showed the effectiveness of \textsc{CompGCN} over existing GCN methods and demonstrated its scalability with increasing number of relations.

\chapter{Conclusion and Future Work}
\label{chap:conclusion}

The first part of the thesis explored two different ways of addressing the sparsity problem in Knowledge Graphs (KG). We begin with alleviating it through canonicalization, which involves identifying and merging identical nodes in a given KG. For this, we proposed CESI (Canonicalization using Embeddings and Side Information), a novel method for canonicalizing Open KBs using learned embeddings and side information.  CESI solves a joint objective to learn noun and relation phrase embeddings while utilizing relevant side information in a principled manner. These learned embeddings are then clustered together to obtain canonicalized noun and relation phrase clusters. The second approach is Relation Extraction which involves using unstructured text for extracting facts to densify KGs. We propose RESIDE, a novel neural network-based model which makes principled use of relevant side information, such as entity type and relation alias, from Knowledge Base, for improving distantly supervised relation extraction. RESIDE employs Graph Convolution Networks for encoding syntactic information of sentences and is robust to limited side information. 
%Finally, the third solution is to use the existing facts to infer new facts in the KG; the task is termed as link prediction. For this, we propose InteractE, a novel knowledge graph embedding method which alleviates the limitations of ConvE by capturing additional heterogeneous feature interactions. InteractE achieves this by utilizing three central ideas, namely feature permutation, checkered feature reshaping, and circular convolution.
 Through experimental results, we demonstrated the effectiveness of all the proposed methods on several benchmark datasets.

In the second part of the thesis, we showed the effectiveness of utilizing recently proposed Graph Convolutional Networks (GCNs) for exploiting several graph structures in NLP. We demonstrated their effectiveness on two important problems: Document Timestamping and learning word embeddings. For the first problem, we proposed NeuralDater, a GCN based method for document dating which exploits syntactic and temporal structures in the document in a principled way. To the best of our knowledge, this is the first application of deep learning techniques for the problem of document dating. For the second problem, we proposed SynGCN, a graph convolution-based approach which utilizes syntactic context for learning word representations. SynGCN overcomes the problem of vocabulary explosion and outperforms state-of-the-art word embedding approaches on several intrinsic and extrinsic tasks. We also propose SemGCN, a framework for jointly incorporating diverse semantic information in pre-trained word embeddings. The combination of SynGCN and SemGCN gives the best overall performance. 

Finally, in the third part of the thesis, we addressed two significant limitations of existing Graph Convolutional Network-based methods. First, we address the issue of noisy representation of hub nodes in GCNs because of neighborhood aggregation scheme which puts no constraint on the influence neighborhood of a node. For this, we present ConfGCN, confidence based Graph Convolutional Network, which estimates label scores along with their confidences jointly in a GCN based setting. In ConfGCN, the influence of one node on another during aggregation is determined using the estimated confidences and label scores, thus inducing anisotropic behavior to GCN. Apart from this also extend existing GCN models for multi-relational graphs, which are a more pervasive class of graphs for modeling data. We propose CompGCN, a novel Graph Convolutional based framework for multi-relational graphs which leverages a variety of composition operators from Knowledge Graph embedding techniques to embed nodes and relations in a graph jointly. Our method generalizes several existing multi-relational GCN methods. Moreover, our method alleviates the problem of over-parameterization by sharing relation embeddings across layers and using basis decomposition. Through extensive experiments on several tasks, we demonstrated the effectiveness of our proposed solutions. 

\textbf{Future Works:} An exciting future direction for addressing sparsity in Knowledge Graphs is to utilize contextualized embedding methods such as ELMo \cite{elmo_paper} and BERT \cite{bert} instead of GloVe for obtaining the representation of noun and relation phrases in Open KG canonicalization. Contextualized embedding approaches have been shown to give superior performance than standard word2vec based embeddings for a variety of tasks. However, utilizing them for canonicalization has not been explored so far. Another future work includes extending our proposed model RESIDE to utilize more types of side information. KGs are a vast storehouse of facts, among which a lot of them can be utilized for further improving RE. 
%In Chapter \ref{chap_interacte}, we explored how increasing heterogeneous interaction can help to improve link prediction performance for ConvE method. The same idea can be extended for further improving other state-of-the-art link prediction methods.
%In our proposed model, RESIDE we restricted ourselves to use only two type of side information for improving relation extraction, i.e., entity type and relation alias. However, KGs are rich in knowledge and contain more information which can also be utilized for improving RE. 
%existing knowledge graph embedding approaches to further improve performance on the other two tasks, i.e., canonicalization and relation extraction. Moreover, one can also utilize recently proposed contextualized embedding approaches like ELMo \cite{elmo_paper} and BERT \cite{bert} for improving performance on problems like Relation Extraction and Document Timestamping. Further, the proposed idea in InteractE of increasing heterogeneous interactions in embeddings can be extended to several of the existing KG embedding approaches to improve them further. 
For further enhancing performance on document timestamping problem, one can explore utilizing knowledge graphs which contain information about world events. Exploiting world knowledge for the task is also more close to how a human would approach the problem. In Chapter \ref{chap_wordgcn}, we explored extending word2vec using Graph Convolutional Networks. However, recently, contextualized embedding methods have been shown to be much more effective. Thus, one can use similar ideas to extend models such as BERT using GCNs for utilizing syntactic information for learning better representation. Finally, existing GCN methods suffer from several other limitations, which have been highlighted by \citet{Xu2018} in his work.

% \backmatter % book mode only
% \appendix
% \include{Appendix1/appendix1}
% \include{Appendix2/appendix2}

\bibliographystyle{plainnat}
%\bibliographystyle{Classes/CUEDbiblio}
%\bibliographystyle{Classes/jmb}
%\bibliographystyle{Classes/jmb} % bibliography style
% \renewcommand{\bibname}{References} % changes default name Bibliography to References
% \bibliography{References/references} % References file
\bibliography{references.bib}

\begin{thebibliography}{225}
\providecommand{\natexlab}[1]{#1}
\providecommand{\url}[1]{\texttt{#1}}
\expandafter\ifx\csname urlstyle\endcsname\relax
  \providecommand{\doi}[1]{doi: #1}\else
  \providecommand{\doi}{doi: \begingroup \urlstyle{rm}\Url}\fi

\bibitem[Allan et~al.(1998)Allan, Papka, and Lavrenko]{event_detection}
James Allan, Ron Papka, and Victor Lavrenko.
\newblock On-line new event detection and tracking.
\newblock In \emph{Proceedings of the 21st Annual International ACM SIGIR
  Conference on Research and Development in Information Retrieval}, SIGIR '98,
  pages 37--45, New York, NY, USA, 1998. ACM.
\newblock ISBN 1-58113-015-5.
\newblock \doi{10.1145/290941.290954}.
\newblock URL \url{http://doi.acm.org/10.1145/290941.290954}.

\bibitem[Almuhareb(2006)]{ap_dataset}
Abdulrahman Almuhareb.
\newblock Attributes in lexical acquisition.
\newblock 2006.

\bibitem[Alsuhaibani et~al.(2018)Alsuhaibani, Bollegala, Maehara, and
  Kawarabayashi]{japanese2018}
Mohammed Alsuhaibani, Danushka Bollegala, Takanori Maehara, and Ken-ichi
  Kawarabayashi.
\newblock Jointly learning word embeddings using a corpus and a knowledge base.
\newblock \emph{PLOS ONE}, 13\penalty0 (3):\penalty0 1--26, 03 2018.
\newblock \doi{10.1371/journal.pone.0193094}.
\newblock URL \url{https://doi.org/10.1371/journal.pone.0193094}.

\bibitem[Angeli et~al.(2015)Angeli, Johnson~Premkumar, and
  Manning]{stanford_openie}
Gabor Angeli, Melvin~Jose Johnson~Premkumar, and Christopher~D. Manning.
\newblock Leveraging linguistic structure for open domain information
  extraction.
\newblock In \emph{Proceedings of the 53rd Annual Meeting of the Association
  for Computational Linguistics and the 7th International Joint Conference on
  Natural Language Processing (Volume 1: Long Papers)}, pages 344--354,
  Beijing, China, July 2015. Association for Computational Linguistics.
\newblock \doi{10.3115/v1/P15-1034}.
\newblock URL \url{https://www.aclweb.org/anthology/P15-1034}.

\bibitem[Athiwaratkun and Wilson(2018)]{Athiwaratkun2018}
Ben Athiwaratkun and Andrew~Gordon Wilson.
\newblock On modeling hierarchical data via probabilistic order embeddings.
\newblock In \emph{International Conference on Learning Representations}, 2018.
\newblock URL \url{https://openreview.net/forum?id=HJCXZQbAZ}.

\bibitem[Auer et~al.(2007)Auer, Bizer, Kobilarov, Lehmann, Cyganiak, and
  Ives]{Auer:2007:DNW:1785162.1785216}
S\"{o}ren Auer, Christian Bizer, Georgi Kobilarov, Jens Lehmann, Richard
  Cyganiak, and Zachary Ives.
\newblock Dbpedia: A nucleus for a web of open data.
\newblock In \emph{Proceedings of the 6th International The Semantic Web and
  2Nd Asian Conference on Asian Semantic Web Conference}, ISWC'07/ASWC'07,
  pages 722--735, Berlin, Heidelberg, 2007. Springer-Verlag.
\newblock ISBN 3-540-76297-3, 978-3-540-76297-3.
\newblock URL \url{http://dl.acm.org/citation.cfm?id=1785162.1785216}.

\bibitem[Bahdanau et~al.(2014)Bahdanau, Cho, and Bengio]{bahdanau2014}
Dzmitry Bahdanau, Kyunghyun Cho, and Yoshua Bengio.
\newblock Neural machine translation by jointly learning to align and
  translate.
\newblock \emph{arXiv e-prints}, abs/1409.0473, September 2014.
\newblock URL \url{https://arxiv.org/abs/1409.0473}.

\bibitem[Baker et~al.(1998)Baker, Fillmore, and Lowe]{framenet}
Collin~F. Baker, Charles~J. Fillmore, and John~B. Lowe.
\newblock The berkeley framenet project.
\newblock In \emph{Proceedings of the 36th Annual Meeting of the Association
  for Computational Linguistics}, ACL '98, pages 86--90, 1998.
\newblock \doi{10.3115/980845.980860}.
\newblock URL \url{https://doi.org/10.3115/980845.980860}.

\bibitem[Bala\v{z}evi\'c et~al.(2019)Bala\v{z}evi\'c, Allen, and
  Hospedales]{hyper}
Ivana Bala\v{z}evi\'c, Carl Allen, and Timothy~M Hospedales.
\newblock Hypernetwork knowledge graph embeddings.
\newblock In \emph{International Conference on Artificial Neural Networks},
  2019.

\bibitem[Banerjee and Pedersen(2002)]{Banerjee2002}
Satanjeev Banerjee and Ted Pedersen.
\newblock \emph{An Adapted Lesk Algorithm for Word Sense Disambiguation Using
  WordNet}, pages 136--145.
\newblock Springer Berlin Heidelberg, Berlin, Heidelberg, 2002.
\newblock ISBN 978-3-540-45715-2.
\newblock \doi{10.1007/3-540-45715-1_11}.
\newblock URL \url{https://doi.org/10.1007/3-540-45715-1_11}.

\bibitem[Banko et~al.(2007)Banko, Cafarella, Soderland, Broadhead, and
  Etzioni]{textrunner}
Michele Banko, Michael~J. Cafarella, Stephen Soderland, Matt Broadhead, and
  Oren Etzioni.
\newblock Open information extraction from the web.
\newblock In \emph{Proceedings of the 20th International Joint Conference on
  Artifical Intelligence}, IJCAI'07, pages 2670--2676, San Francisco, CA, USA,
  2007. Morgan Kaufmann Publishers Inc.
\newblock URL \url{http://dl.acm.org/citation.cfm?id=1625275.1625705}.

\bibitem[Baroni and Lenci(2010)]{battig_dataset}
Marco Baroni and Alessandro Lenci.
\newblock Distributional memory: A general framework for corpus-based
  semantics.
\newblock \emph{Comput. Linguist.}, 36\penalty0 (4):\penalty0 673--721,
  December 2010.
\newblock ISSN 0891-2017.
\newblock \doi{10.1162/coli_a_00016}.
\newblock URL \url{http://dx.doi.org/10.1162/coli_a_00016}.

\bibitem[Baroni and Lenci(2011)]{bless_dataset}
Marco Baroni and Alessandro Lenci.
\newblock How we blessed distributional semantic evaluation.
\newblock In \emph{Proceedings of the GEMS 2011 Workshop on GEometrical Models
  of Natural Language Semantics}, GEMS '11, pages 1--10, Stroudsburg, PA, USA,
  2011. Association for Computational Linguistics.
\newblock ISBN 978-1-937284-16-9.
\newblock URL \url{http://dl.acm.org/citation.cfm?id=2140490.2140491}.

\bibitem[Baroni et~al.(2008)Baroni, Evert, and Lenci]{essli_dataset}
Marco Baroni, Stefan Evert, and Alessandro Lenci.
\newblock Esslli 2008 workshop on distributional lexical semantics.
\newblock Association for Logic, Language and Information, 2008.

\bibitem[Bastings et~al.(2017)Bastings, Titov, Aziz, Marcheggiani, and
  Simaan]{gcn_nmt}
Joost Bastings, Ivan Titov, Wilker Aziz, Diego Marcheggiani, and Khalil Simaan.
\newblock Graph convolutional encoders for syntax-aware neural machine
  translation.
\newblock In \emph{Proceedings of the 2017 Conference on Empirical Methods in
  Natural Language Processing}, pages 1957--1967, Copenhagen, Denmark,
  September 2017. Association for Computational Linguistics.
\newblock URL \url{https://www.aclweb.org/anthology/D17-1209}.

\bibitem[Beck et~al.(2018)Beck, Haffari, and Cohn]{graph2seq}
Daniel Beck, Gholamreza Haffari, and Trevor Cohn.
\newblock Graph-to-sequence learning using gated graph neural networks.
\newblock In Iryna Gurevych and Yusuke Miyao, editors, \emph{ACL 2018 - The
  56th Annual Meeting of the Association for Computational Linguistics}, pages
  273--283. Association for Computational Linguistics (ACL), 2018.
\newblock ISBN 9781948087322.

\bibitem[Belkin et~al.(2006)Belkin, Niyogi, and Sindhwani]{Belkin2006manifold}
Mikhail Belkin, Partha Niyogi, and Vikas Sindhwani.
\newblock Manifold regularization: A geometric framework for learning from
  labeled and unlabeled examples.
\newblock \emph{J. Mach. Learn. Res.}, 7:\penalty0 2399--2434, December 2006.
\newblock ISSN 1532-4435.
\newblock URL \url{http://dl.acm.org/citation.cfm?id=1248547.1248632}.

\bibitem[Bengio et~al.(2013)Bengio, Courville, and Vincent]{app_ner}
Y.~Bengio, A.~Courville, and P.~Vincent.
\newblock Representation learning: A review and new perspectives.
\newblock \emph{IEEE Transactions on Pattern Analysis and Machine
  Intelligence}, 35\penalty0 (8):\penalty0 1798--1828, Aug 2013.
\newblock ISSN 0162-8828.
\newblock \doi{10.1109/TPAMI.2013.50}.

\bibitem[Bengio et~al.(2003)Bengio, Ducharme, Vincent, and Janvin]{Bengio2003}
Yoshua Bengio, R{\'e}jean Ducharme, Pascal Vincent, and Christian Janvin.
\newblock A neural probabilistic language model.
\newblock \emph{J. Mach. Learn. Res.}, 3:\penalty0 1137--1155, March 2003.
\newblock ISSN 1532-4435.
\newblock URL \url{http://dl.acm.org/citation.cfm?id=944919.944966}.

\bibitem[Bojchevski and Günnemann(2018)]{graph2gauss}
Aleksandar Bojchevski and Stephan Günnemann.
\newblock Deep gaussian embedding of graphs: Unsupervised inductive learning
  via ranking.
\newblock In \emph{International Conference on Learning Representations}, 2018.
\newblock URL \url{https://openreview.net/forum?id=r1ZdKJ-0W}.

\bibitem[Bollacker et~al.(2008)Bollacker, Evans, Paritosh, Sturge, and
  Taylor]{freebase}
Kurt Bollacker, Colin Evans, Praveen Paritosh, Tim Sturge, and Jamie Taylor.
\newblock Freebase: A collaboratively created graph database for structuring
  human knowledge.
\newblock In \emph{Proceedings of the 2008 ACM SIGMOD International Conference
  on Management of Data}, SIGMOD '08, pages 1247--1250, New York, NY, USA,
  2008. ACM.
\newblock ISBN 978-1-60558-102-6.
\newblock \doi{10.1145/1376616.1376746}.
\newblock URL \url{http://doi.acm.org/10.1145/1376616.1376746}.

\bibitem[Bordes et~al.(2013)Bordes, Usunier, Garcia-Duran, Weston, and
  Yakhnenko]{transe}
Antoine Bordes, Nicolas Usunier, Alberto Garcia-Duran, Jason Weston, and Oksana
  Yakhnenko.
\newblock Translating embeddings for modeling multi-relational data.
\newblock In C.~J.~C. Burges, L.~Bottou, M.~Welling, Z.~Ghahramani, and K.~Q.
  Weinberger, editors, \emph{Advances in Neural Information Processing Systems
  26}, pages 2787--2795. Curran Associates, Inc., 2013.
\newblock URL
  \url{http://papers.nips.cc/paper/5071-translating-embeddings-for-modeling-multi-relational-data.pdf}.

\bibitem[Bordes et~al.(2014{\natexlab{a}})Bordes, Chopra, and
  Weston]{kg_question_answering}
Antoine Bordes, Sumit Chopra, and Jason Weston.
\newblock Question answering with subgraph embeddings.
\newblock In \emph{Proceedings of the 2014 Conference on Empirical Methods in
  Natural Language Processing ({EMNLP})}, pages 615--620, Doha, Qatar, October
  2014{\natexlab{a}}. Association for Computational Linguistics.
\newblock \doi{10.3115/v1/D14-1067}.
\newblock URL \url{https://www.aclweb.org/anthology/D14-1067}.

\bibitem[Bordes et~al.(2014{\natexlab{b}})Bordes, Chopra, and Weston]{qa_kg_1}
Antoine Bordes, Sumit Chopra, and Jason Weston.
\newblock Question answering with subgraph embeddings.
\newblock In \emph{Proceedings of the 2014 Conference on Empirical Methods in
  Natural Language Processing (EMNLP)}, pages 615--620. Association for
  Computational Linguistics, 2014{\natexlab{b}}.
\newblock \doi{10.3115/v1/D14-1067}.
\newblock URL \url{http://aclweb.org/anthology/D14-1067}.

\bibitem[Bordes et~al.(2014{\natexlab{c}})Bordes, Weston, and Usunier]{qa_kg_2}
Antoine Bordes, Jason Weston, and Nicolas Usunier.
\newblock Open question answering with weakly supervised embedding models.
\newblock In Toon Calders, Floriana Esposito, Eyke H{\"u}llermeier, and Rosa
  Meo, editors, \emph{Machine Learning and Knowledge Discovery in Databases},
  pages 165--180, Berlin, Heidelberg, 2014{\natexlab{c}}. Springer Berlin
  Heidelberg.

\bibitem[Br\"{o}cheler et~al.(2010)Br\"{o}cheler, Mihalkova, and
  Getoor]{Brocheler:2010:PSL:3023549.3023558}
Matthias Br\"{o}cheler, Lilyana Mihalkova, and Lise Getoor.
\newblock Probabilistic similarity logic.
\newblock In \emph{Proceedings of the Twenty-Sixth Conference on Uncertainty in
  Artificial Intelligence}, UAI'10, pages 73--82, Arlington, Virginia, United
  States, 2010. AUAI Press.
\newblock ISBN 978-0-9749039-6-5.
\newblock URL \url{http://dl.acm.org/citation.cfm?id=3023549.3023558}.

\bibitem[Bronstein et~al.(2017)Bronstein, Bruna, LeCun, Szlam, and
  Vandergheynst]{Bronstein2017}
M.~M. Bronstein, J.~Bruna, Y.~LeCun, A.~Szlam, and P.~Vandergheynst.
\newblock Geometric deep learning: Going beyond euclidean data.
\newblock \emph{IEEE Signal Processing Magazine}, 34\penalty0 (4):\penalty0
  18--42, July 2017.
\newblock ISSN 1053-5888.
\newblock \doi{10.1109/MSP.2017.2693418}.

\bibitem[Bruna et~al.(2013)Bruna, Zaremba, Szlam, and LeCun]{Bruna2013}
Joan Bruna, Wojciech Zaremba, Arthur Szlam, and Yann LeCun.
\newblock Spectral networks and locally connected networks on graphs.
\newblock \emph{CoRR}, abs/1312.6203, 2013.
\newblock URL \url{http://arxiv.org/abs/1312.6203}.

\bibitem[Bruna et~al.(2014)Bruna, Zaremba, Szlam, and Lecun]{gcn_first_work}
Joan Bruna, Wojciech Zaremba, Arthur Szlam, and Yann Lecun.
\newblock Spectral networks and locally connected networks on graphs.
\newblock In \emph{International Conference on Learning Representations
  (ICLR2014), CBLS, April 2014}, 2014.

\bibitem[Cai and Wang(2018)]{kbgan}
Liwei Cai and William~Yang Wang.
\newblock {KBGAN}: Adversarial learning for knowledge graph embeddings.
\newblock In \emph{Proceedings of the 2018 Conference of the North {A}merican
  Chapter of the Association for Computational Linguistics: Human Language
  Technologies}, pages 1470--1480, 2018.
\newblock URL \url{https://www.aclweb.org/anthology/N18-1133}.

\bibitem[Callan et~al.(2009)Callan, Hoy, Yoo, and Zhao]{callan2009clueweb09}
Jamie Callan, Mark Hoy, Changkuk Yoo, and Le~Zhao.
\newblock Clueweb09 data set, 2009.

\bibitem[Chambers(2012)]{Chambers:2012:LDT:2390524.2390539}
Nathanael Chambers.
\newblock Labeling documents with timestamps: Learning from their time
  expressions.
\newblock In \emph{Proceedings of the 50th Annual Meeting of the Association
  for Computational Linguistics: Long Papers - Volume 1}, ACL '12, pages
  98--106, Stroudsburg, PA, USA, 2012. Association for Computational
  Linguistics.
\newblock URL \url{http://dl.acm.org/citation.cfm?id=2390524.2390539}.

\bibitem[Chambers and Jurafsky(2008)]{Chambers:2008:JCI:1613715.1613803}
Nathanael Chambers and Dan Jurafsky.
\newblock Jointly combining implicit constraints improves temporal ordering.
\newblock In \emph{Proceedings of the Conference on Empirical Methods in
  Natural Language Processing}, EMNLP '08, pages 698--706, Stroudsburg, PA,
  USA, 2008. Association for Computational Linguistics.
\newblock URL \url{http://dl.acm.org/citation.cfm?id=1613715.1613803}.

\bibitem[Chambers et~al.(2007)Chambers, Wang, and
  Jurafsky]{Chambers:2007:CTR:1557769.1557820}
Nathanael Chambers, Shan Wang, and Dan Jurafsky.
\newblock Classifying temporal relations between events.
\newblock In \emph{Proceedings of the 45th Annual Meeting of the ACL on
  Interactive Poster and Demonstration Sessions}, ACL '07, pages 173--176,
  Stroudsburg, PA, USA, 2007. Association for Computational Linguistics.
\newblock URL \url{http://dl.acm.org/citation.cfm?id=1557769.1557820}.

\bibitem[Chambers et~al.(2014)Chambers, Cassidy, McDowell, and
  Bethard]{Chambers14}
Nathanael Chambers, Taylor Cassidy, Bill McDowell, and Steven Bethard.
\newblock Dense event ordering with a multi-pass architecture.
\newblock \emph{Transactions of the Association of Computational Linguistics},
  2:\penalty0 273--284, 2014.
\newblock URL \url{http://www.aclweb.org/anthology/Q14-1022}.

\bibitem[Chang and Manning(2012)]{sutime_paper}
Angel~X. Chang and Christopher Manning.
\newblock Sutime: A library for recognizing and normalizing time expressions.
\newblock In \emph{Proceedings of the Eighth International Conference on
  Language Resources and Evaluation (LREC-2012)}. European Language Resources
  Association (ELRA), 2012.
\newblock URL \url{http://www.aclweb.org/anthology/L12-1122}.

\bibitem[Chen et~al.(2018)Chen, Ma, and Xiao]{fastgcn}
Jie Chen, Tengfei Ma, and Cao Xiao.
\newblock Fast{GCN}: Fast learning with graph convolutional networks via
  importance sampling.
\newblock In \emph{International Conference on Learning Representations}, 2018.
\newblock URL \url{https://openreview.net/forum?id=rytstxWAW}.

\bibitem[Cho et~al.(2014)Cho, van Merrienboer, Gulcehre, Bahdanau, Bougares,
  Schwenk, and Bengio]{gru_paper}
Kyunghyun Cho, Bart van Merrienboer, Caglar Gulcehre, Dzmitry Bahdanau, Fethi
  Bougares, Holger Schwenk, and Yoshua Bengio.
\newblock Learning phrase representations using rnn encoder--decoder for
  statistical machine translation.
\newblock In \emph{Proceedings of the 2014 Conference on Empirical Methods in
  Natural Language Processing (EMNLP)}, pages 1724--1734. Association for
  Computational Linguistics, 2014.
\newblock \doi{10.3115/v1/D14-1179}.
\newblock URL \url{http://www.aclweb.org/anthology/D14-1179}.

\bibitem[Christensen et~al.(2011)Christensen, Mausam, Soderland, and
  Etzioni]{ollie1}
Janara Christensen, Mausam, Stephen Soderland, and Oren Etzioni.
\newblock An analysis of open information extraction based on semantic role
  labeling.
\newblock In \emph{Proceedings of the Sixth International Conference on
  Knowledge Capture}, K-CAP '11, pages 113--120, New York, NY, USA, 2011. ACM.
\newblock ISBN 978-1-4503-0396-5.
\newblock \doi{10.1145/1999676.1999697}.
\newblock URL \url{http://doi.acm.org/10.1145/1999676.1999697}.

\bibitem[Clark and Gardner(2018)]{docqa_model}
Christopher Clark and Matt Gardner.
\newblock Simple and effective multi-paragraph reading comprehension.
\newblock In \emph{Proceedings of the 56th Annual Meeting of the Association
  for Computational Linguistics (Volume 1: Long Papers)}, pages 845--855.
  Association for Computational Linguistics, 2018.
\newblock URL \url{http://aclweb.org/anthology/P18-1078}.

\bibitem[Collobert et~al.(2011)Collobert, Weston, Bottou, Karlen, Kavukcuoglu,
  and Kuksa]{Collobert2011}
Ronan Collobert, Jason Weston, L{\'e}on Bottou, Michael Karlen, Koray
  Kavukcuoglu, and Pavel Kuksa.
\newblock Natural language processing (almost) from scratch.
\newblock \emph{J. Mach. Learn. Res.}, 12:\penalty0 2493--2537, November 2011.
\newblock ISSN 1532-4435.
\newblock URL \url{http://dl.acm.org/citation.cfm?id=1953048.2078186}.

\bibitem[Dakka et~al.(2008)Dakka, Gravano, and Ipeirotis]{ir_time_dakka}
Wisam Dakka, Luis Gravano, and Panagiotis~G. Ipeirotis.
\newblock Answering general time sensitive queries.
\newblock In \emph{Proceedings of the 17th ACM Conference on Information and
  Knowledge Management}, CIKM '08, pages 1437--1438, New York, NY, USA, 2008.
  ACM.
\newblock ISBN 978-1-59593-991-3.
\newblock \doi{10.1145/1458082.1458320}.
\newblock URL \url{http://doi.acm.org/10.1145/1458082.1458320}.

\bibitem[Das et~al.(2015)Das, Zaheer, and Dyer]{Das2015}
Rajarshi Das, Manzil Zaheer, and Chris Dyer.
\newblock Gaussian lda for topic models with word embeddings.
\newblock In \emph{Proceedings of the 53rd Annual Meeting of the Association
  for Computational Linguistics and the 7th International Joint Conference on
  Natural Language Processing (Volume 1: Long Papers)}, pages 795--804.
  Association for Computational Linguistics, 2015.
\newblock \doi{10.3115/v1/P15-1077}.
\newblock URL \url{http://www.aclweb.org/anthology/P15-1077}.

\bibitem[{de Jong} et~al.(2005){de Jong}, Rode, and Hiemstra]{history_time}
{Franciska M.G.} {de Jong}, H.~Rode, and Djoerd Hiemstra.
\newblock \emph{Temporal Language Models for the Disclosure of Historical
  Text}, pages 161--168.
\newblock KNAW, 9 2005.
\newblock ISBN 90-6984-456-7.
\newblock Imported from EWI/DB PMS [db-utwente:inpr:0000003683].

\bibitem[Defays(1977)]{defays1977efficient}
Daniel Defays.
\newblock An efficient algorithm for a complete link method.
\newblock \emph{The Computer Journal}, 20\penalty0 (4):\penalty0 364--366,
  1977.

\bibitem[Defferrard et~al.(2016{\natexlab{a}})Defferrard, Bresson, and
  Vandergheynst]{Defferrard2016}
Micha{\"{e}}l Defferrard, Xavier Bresson, and Pierre Vandergheynst.
\newblock Convolutional neural networks on graphs with fast localized spectral
  filtering.
\newblock \emph{CoRR}, abs/1606.09375, 2016{\natexlab{a}}.
\newblock URL \url{http://arxiv.org/abs/1606.09375}.

\bibitem[Defferrard et~al.(2016{\natexlab{b}})Defferrard, Bresson, and
  Vandergheynst]{Defferrard:2016:CNN:3157382.3157527}
Micha\"{e}l Defferrard, Xavier Bresson, and Pierre Vandergheynst.
\newblock Convolutional neural networks on graphs with fast localized spectral
  filtering.
\newblock In \emph{Proceedings of the 30th International Conference on Neural
  Information Processing Systems}, NIPS'16, pages 3844--3852, USA,
  2016{\natexlab{b}}. Curran Associates Inc.
\newblock ISBN 978-1-5108-3881-9.
\newblock URL \url{http://dl.acm.org/citation.cfm?id=3157382.3157527}.

\bibitem[{Delli Bovi} et~al.(2015){Delli Bovi}, {Espinosa Anke}, and
  Navigli]{dellibovi-espinosaanke-navigli:2015:EMNLP}
Claudio {Delli Bovi}, Luis {Espinosa Anke}, and Roberto Navigli.
\newblock Knowledge base unification via sense embeddings and disambiguation.
\newblock In \emph{Proceedings of the 2015 Conference on Empirical Methods in
  Natural Language Processing}, pages 726--736, Lisbon, Portugal, September
  2015. Association for Computational Linguistics.
\newblock URL \url{http://aclweb.org/anthology/D15-1084}.

\bibitem[Dettmers et~al.(2018)Dettmers, Pasquale, Pontus, and Riedel]{conve}
Tim Dettmers, Minervini Pasquale, Stenetorp Pontus, and Sebastian Riedel.
\newblock Convolutional 2d knowledge graph embeddings.
\newblock In \emph{Proceedings of the 32th AAAI Conference on Artificial
  Intelligence}, pages 1811--1818, February 2018.
\newblock URL \url{https://arxiv.org/abs/1707.01476}.

\bibitem[Devlin et~al.(2018)Devlin, Chang, Lee, and Toutanova]{bert}
Jacob Devlin, Ming-Wei Chang, Kenton Lee, and Kristina Toutanova.
\newblock Bert: Pre-training of deep bidirectional transformers for language
  understanding.
\newblock \emph{arXiv preprint arXiv:1810.04805}, 2018.

\bibitem[Dong et~al.(2014)Dong, Gabrilovich, Heitz, Horn, Lao, Murphy,
  Strohmann, Sun, and Zhang]{kg_incomp1}
Xin Dong, Evgeniy Gabrilovich, Geremy Heitz, Wilko Horn, Ni~Lao, Kevin Murphy,
  Thomas Strohmann, Shaohua Sun, and Wei Zhang.
\newblock Knowledge vault: A web-scale approach to probabilistic knowledge
  fusion.
\newblock In \emph{Proceedings of the 20th ACM SIGKDD International Conference
  on Knowledge Discovery and Data Mining}, KDD '14, pages 601--610, New York,
  NY, USA, 2014. ACM.
\newblock ISBN 978-1-4503-2956-9.
\newblock \doi{10.1145/2623330.2623623}.
\newblock URL \url{http://doi.acm.org/10.1145/2623330.2623623}.

\bibitem[Dos~Santos et~al.(2017)Dos~Santos, Piwowarski, and
  Gallinari]{DosSantos2017}
Ludovic Dos~Santos, Benjamin Piwowarski, and Patrick Gallinari.
\newblock Gaussian embeddings for collaborative filtering.
\newblock In \emph{Proceedings of the 40th International ACM SIGIR Conference
  on Research and Development in Information Retrieval}, SIGIR '17, pages
  1065--1068, New York, NY, USA, 2017. ACM.
\newblock ISBN 978-1-4503-5022-8.
\newblock \doi{10.1145/3077136.3080722}.
\newblock URL \url{http://doi.acm.org/10.1145/3077136.3080722}.

\bibitem[D'Souza and Ng(2013)]{N13-1112}
Jennifer D'Souza and Vincent Ng.
\newblock Classifying temporal relations with rich linguistic knowledge.
\newblock In \emph{Proceedings of the 2013 Conference of the North American
  Chapter of the Association for Computational Linguistics: Human Language
  Technologies}, pages 918--927. Association for Computational Linguistics,
  2013.
\newblock URL \url{http://www.aclweb.org/anthology/N13-1112}.

\bibitem[Fader et~al.(2011)Fader, Soderland, and Etzioni]{reverb}
Anthony Fader, Stephen Soderland, and Oren Etzioni.
\newblock Identifying relations for open information extraction.
\newblock In \emph{Proceedings of the Conference of Empirical Methods in
  Natural Language Processing ({EMNLP} '11)}, Edinburgh, Scotland, UK, July
  27-31 2011.

\bibitem[Faruqui et~al.(2014)Faruqui, Dodge, Jauhar, Dyer, Hovy, and
  Smith]{Faruqui2014}
Manaal Faruqui, Jesse Dodge, Sujay~Kumar Jauhar, Chris Dyer, Eduard~H. Hovy,
  and Noah~A. Smith.
\newblock Retrofitting word vectors to semantic lexicons.
\newblock \emph{CoRR}, abs/1411.4166, 2014.
\newblock URL \url{http://arxiv.org/abs/1411.4166}.

\bibitem[Feng et~al.(2017)Feng, Guo, Qin, Liu, and Liu]{feng2017effective}
Xiaocheng Feng, Jiang Guo, Bing Qin, Ting Liu, and Yongjie Liu.
\newblock Effective deep memory networks for distant supervised relation
  extraction.
\newblock In \emph{Proceedings of the Twenty-Sixth International Joint
  Conference on Artificial Intelligence, {IJCAI-17}}, pages 4002--4008, 2017.
\newblock \doi{10.24963/ijcai.2017/559}.
\newblock URL \url{https://doi.org/10.24963/ijcai.2017/559}.

\bibitem[Finkel et~al.(2005)Finkel, Grenager, and
  Manning]{finkel2005incorporating}
Jenny~Rose Finkel, Trond Grenager, and Christopher Manning.
\newblock Incorporating non-local information into information extraction
  systems by gibbs sampling.
\newblock In \emph{Proceedings of the 43rd annual meeting on association for
  computational linguistics}, ACL '05, pages 363--370, Stroudsburg, PA, USA,
  2005. Association for Computational Linguistics, Association for
  Computational Linguistics.
\newblock \doi{10.3115/1219840.1219885}.
\newblock URL \url{https://doi.org/10.3115/1219840.1219885}.

\bibitem[Finkelstein et~al.(2001)Finkelstein, Gabrilovich, Matias, Rivlin,
  Solan, Wolfman, and Ruppin]{ws353_dataset}
Lev Finkelstein, Evgeniy Gabrilovich, Yossi Matias, Ehud Rivlin, Zach Solan,
  Gadi Wolfman, and Eytan Ruppin.
\newblock Placing search in context: The concept revisited.
\newblock In \emph{Proceedings of the 10th International Conference on World
  Wide Web}, WWW '01, pages 406--414, New York, NY, USA, 2001. ACM.
\newblock ISBN 1-58113-348-0.
\newblock \doi{10.1145/371920.372094}.
\newblock URL \url{http://doi.acm.org/10.1145/371920.372094}.

\bibitem[Fout et~al.(2017)Fout, Byrd, Shariat, and
  Ben-Hur]{gcn_protein_interaction}
Alex Fout, Jonathon Byrd, Basir Shariat, and Asa Ben-Hur.
\newblock Protein interface prediction using graph convolutional networks.
\newblock In \emph{Proceedings of the 31st International Conference on Neural
  Information Processing Systems}, NIPS'17, pages 6533--6542, USA, 2017. Curran
  Associates Inc.
\newblock ISBN 978-1-5108-6096-4.
\newblock URL \url{http://dl.acm.org/citation.cfm?id=3295222.3295399}.

\bibitem[Gabrilovich et~al.(2013)Gabrilovich, Ringgaard, and
  Subramanya]{gabrilovich2013facc1}
Evgeniy Gabrilovich, Michael Ringgaard, and Amarnag Subramanya.
\newblock Facc1: Freebase annotation of clueweb corpora, version 1.
\newblock \emph{Release date}, pages 06--26, 2013.

\bibitem[Gal\'{a}rraga et~al.(2014)Gal\'{a}rraga, Heitz, Murphy, and
  Suchanek]{Galarraga:2014:COK:2661829.2662073}
Luis Gal\'{a}rraga, Geremy Heitz, Kevin Murphy, and Fabian~M. Suchanek.
\newblock Canonicalizing open knowledge bases.
\newblock In \emph{Proceedings of the 23rd ACM International Conference on
  Conference on Information and Knowledge Management}, CIKM '14, pages
  1679--1688, New York, NY, USA, 2014. ACM.
\newblock ISBN 978-1-4503-2598-1.
\newblock \doi{10.1145/2661829.2662073}.
\newblock URL \url{http://doi.acm.org/10.1145/2661829.2662073}.

\bibitem[Gal\'{a}rraga et~al.(2013)Gal\'{a}rraga, Teflioudi, Hose, and
  Suchanek]{Galarraga:2013:AAR:2488388.2488425}
Luis~Antonio Gal\'{a}rraga, Christina Teflioudi, Katja Hose, and Fabian
  Suchanek.
\newblock Amie: Association rule mining under incomplete evidence in
  ontological knowledge bases.
\newblock In \emph{Proceedings of the 22Nd International Conference on World
  Wide Web}, WWW '13, pages 413--422, New York, NY, USA, 2013. ACM.
\newblock ISBN 978-1-4503-2035-1.
\newblock \doi{10.1145/2488388.2488425}.
\newblock URL \url{http://doi.acm.org/10.1145/2488388.2488425}.

\bibitem[Gentleman and Sande(1966)]{fft}
W.~M. Gentleman and G.~Sande.
\newblock Fast fourier transforms: For fun and profit.
\newblock In \emph{Proceedings of the November 7-10, 1966, Fall Joint Computer
  Conference}, AFIPS '66 (Fall), pages 563--578, New York, NY, USA, 1966. ACM.
\newblock \doi{10.1145/1464291.1464352}.
\newblock URL \url{http://doi.acm.org/10.1145/1464291.1464352}.

\bibitem[Gilmer et~al.(2017)Gilmer, Schoenholz, Riley, Vinyals, and Dahl]{mpnn}
Justin Gilmer, Samuel~S. Schoenholz, Patrick~F. Riley, Oriol Vinyals, and
  George~E. Dahl.
\newblock Neural message passing for quantum chemistry.
\newblock In \emph{Proceedings of the 34th International Conference on Machine
  Learning - Volume 70}, ICML'17, pages 1263--1272. JMLR.org, 2017.
\newblock URL \url{http://dl.acm.org/citation.cfm?id=3305381.3305512}.

\bibitem[Glorot and Bengio(2010)]{xavier_init}
Xavier Glorot and Yoshua Bengio.
\newblock Understanding the difficulty of training deep feedforward neural
  networks.
\newblock In Yee~Whye Teh and Mike Titterington, editors, \emph{Proceedings of
  the Thirteenth International Conference on Artificial Intelligence and
  Statistics}, volume~9 of \emph{Proceedings of Machine Learning Research},
  pages 249--256, Chia Laguna Resort, Sardinia, Italy, 13--15 May 2010. PMLR.
\newblock URL \url{http://proceedings.mlr.press/v9/glorot10a.html}.

\bibitem[Goodfellow et~al.(2016)Goodfellow, Bengio, and
  Courville]{deep_learning_book}
Ian Goodfellow, Yoshua Bengio, and Aaron Courville.
\newblock \emph{Deep Learning}.
\newblock MIT Press, 2016.
\newblock \url{http://www.deeplearningbook.org}.

\bibitem[Graves et~al.(2013)Graves, r.~Mohamed, and Hinton]{rnn_speech_recog}
A.~Graves, A.~r.~Mohamed, and G.~Hinton.
\newblock Speech recognition with deep recurrent neural networks.
\newblock In \emph{2013 IEEE International Conference on Acoustics, Speech and
  Signal Processing}, pages 6645--6649, May 2013.
\newblock \doi{10.1109/ICASSP.2013.6638947}.

\bibitem[Grover and Leskovec(2016)]{node2vec}
Aditya Grover and Jure Leskovec.
\newblock node2vec: Scalable feature learning for networks.
\newblock In \emph{Proceedings of the 22nd ACM SIGKDD International Conference
  on Knowledge Discovery and Data Mining}, 2016.

\bibitem[Gutmann and Hyvärinen(2010)]{Gutmann2010}
Michael Gutmann and Aapo Hyvärinen.
\newblock Noise-contrastive estimation: A new estimation principle for
  unnormalized statistical models.
\newblock In Yee~Whye Teh and Mike Titterington, editors, \emph{Proceedings of
  the Thirteenth International Conference on Artificial Intelligence and
  Statistics}, volume~9 of \emph{Proceedings of Machine Learning Research},
  pages 297--304, Chia Laguna Resort, Sardinia, Italy, 13--15 May 2010. PMLR.
\newblock URL \url{http://proceedings.mlr.press/v9/gutmann10a.html}.

\bibitem[Hamilton et~al.(2017)Hamilton, Ying, and Leskovec]{graphsage}
William~L. Hamilton, Rex Ying, and Jure Leskovec.
\newblock Inductive representation learning on large graphs.
\newblock In \emph{NIPS}, 2017.

\bibitem[He et~al.(2018)He, Chen, Li, Zhang, Zhang, and Zhang]{see_paper}
Zhengqiu He, Wenliang Chen, Zhenghua Li, Meishan Zhang, Wei Zhang, and Min
  Zhang.
\newblock See: Syntax-aware entity embedding for neural relation extraction,
  2018.
\newblock URL
  \url{https://www.aaai.org/ocs/index.php/AAAI/AAAI18/paper/view/16362}.

\bibitem[Hill et~al.(2015)Hill, Reichart, and Korhonen]{simlex_dataset}
Felix Hill, Roi Reichart, and Anna Korhonen.
\newblock Simlex-999: Evaluating semantic models with genuine similarity
  estimation.
\newblock \emph{Comput. Linguist.}, 41\penalty0 (4):\penalty0 665--695,
  December 2015.
\newblock ISSN 0891-2017.
\newblock \doi{10.1162/COLI_a_00237}.
\newblock URL \url{http://dx.doi.org/10.1162/COLI_a_00237}.

\bibitem[Hinton et~al.(2012)Hinton, Deng, Yu, Dahl, r.~Mohamed, Jaitly, Senior,
  Vanhoucke, Nguyen, Sainath, and Kingsbury]{microsoft_speech}
G.~Hinton, L.~Deng, D.~Yu, G.~E. Dahl, A.~r.~Mohamed, N.~Jaitly, A.~Senior,
  V.~Vanhoucke, P.~Nguyen, T.~N. Sainath, and B.~Kingsbury.
\newblock Deep neural networks for acoustic modeling in speech recognition: The
  shared views of four research groups.
\newblock \emph{IEEE Signal Processing Magazine}, 29\penalty0 (6):\penalty0
  82--97, Nov 2012.
\newblock ISSN 1053-5888.
\newblock \doi{10.1109/MSP.2012.2205597}.

\bibitem[Hochreiter and Schmidhuber(1997)]{lstm}
Sepp Hochreiter and J\"{u}rgen Schmidhuber.
\newblock Long short-term memory.
\newblock \emph{Neural Comput.}, 9\penalty0 (8):\penalty0 1735--1780, November
  1997.
\newblock ISSN 0899-7667.
\newblock \doi{10.1162/neco.1997.9.8.1735}.
\newblock URL \url{http://dx.doi.org/10.1162/neco.1997.9.8.1735}.

\bibitem[Hoffmann et~al.(2011)Hoffmann, Zhang, Ling, Zettlemoyer, and
  Weld]{hoffmann2011knowledge}
Raphael Hoffmann, Congle Zhang, Xiao Ling, Luke Zettlemoyer, and Daniel~S Weld.
\newblock Knowledge-based weak supervision for information extraction of
  overlapping relations.
\newblock In \emph{Proceedings of the 49th Annual Meeting of the Association
  for Computational Linguistics: Human Language Technologies-Volume 1}, pages
  541--550. Association for Computational Linguistics, 2011.

\bibitem[Huang et~al.(2017)Huang, Wan, Probst, and Van~Gool]{huang2017deep}
Zhiwu Huang, Chengde Wan, Thomas Probst, and Luc Van~Gool.
\newblock Deep learning on lie groups for skeleton-based action recognition.
\newblock In \emph{Proceedings of the 2017 IEEE Conference on Computer Vision
  and Pattern Recognition (CVPR)}, pages 1243--1252. IEEE computer Society,
  2017.

\bibitem[{Jat} et~al.(2018){Jat}, {Khandelwal}, and {Talukdar}]{bgwa_paper}
S.~{Jat}, S.~{Khandelwal}, and P.~{Talukdar}.
\newblock {Improving Distantly Supervised Relation Extraction using Word and
  Entity Based Attention}.
\newblock \emph{ArXiv e-prints}, April 2018.

\bibitem[Ji et~al.(2017)Ji, Liu, He, and Zhao]{entity_description}
Guoliang Ji, Kang Liu, Shizhu He, and Jun Zhao.
\newblock Distant supervision for relation extraction with sentence-level
  attention and entity descriptions.
\newblock In \emph{AAAI}, 2017.

\bibitem[Ji et~al.(2015)Ji, Yun, Yanardag, Matsushima, and
  Vishwanathan]{Ji2015}
Shihao Ji, Hyokun Yun, Pinar Yanardag, Shin Matsushima, and S.~V.~N.
  Vishwanathan.
\newblock Wordrank: Learning word embeddings via robust ranking.
\newblock \emph{CoRR}, abs/1506.02761, 2015.
\newblock URL \url{http://arxiv.org/abs/1506.02761}.

\bibitem[Jiang et~al.(2019)Jiang, Wang, and Wang]{convr}
Xiaotian Jiang, Quan Wang, and Bin Wang.
\newblock Adaptive convolution for multi-relational learning.
\newblock In \emph{Proceedings of the 2019 Conference of the North {A}merican
  Chapter of the Association for Computational Linguistics: Human Language
  Technologies}, 2019.
\newblock URL \url{https://www.aclweb.org/anthology/N19-1103"}.

\bibitem[Johnson et~al.(2017)Johnson, Schuster, Le, Krikun, Wu, Chen, Thorat,
  Vi{\'e}gas, Wattenberg, Corrado, Hughes, and Dean]{multilingual_google}
Melvin Johnson, Mike Schuster, Quoc~V. Le, Maxim Krikun, Yonghui Wu, Zhifeng
  Chen, Nikhil Thorat, Fernanda Vi{\'e}gas, Martin Wattenberg, Greg Corrado,
  Macduff Hughes, and Jeffrey Dean.
\newblock Google's multilingual neural machine translation system: Enabling
  zero-shot translation.
\newblock \emph{Transactions of the Association for Computational Linguistics},
  5:\penalty0 339--351, 2017.
\newblock URL \url{http://aclweb.org/anthology/Q17-1024}.

\bibitem[Jurgens et~al.(2012)Jurgens, Turney, Mohammad, and
  Holyoak]{semeval_dataset}
David~A. Jurgens, Peter~D. Turney, Saif~M. Mohammad, and Keith~J. Holyoak.
\newblock Semeval-2012 task 2: Measuring degrees of relational similarity.
\newblock In \emph{Proceedings of the First Joint Conference on Lexical and
  Computational Semantics - Volume 1: Proceedings of the Main Conference and
  the Shared Task, and Volume 2: Proceedings of the Sixth International
  Workshop on Semantic Evaluation}, SemEval '12, pages 356--364, Stroudsburg,
  PA, USA, 2012. Association for Computational Linguistics.
\newblock URL \url{http://dl.acm.org/citation.cfm?id=2387636.2387693}.

\bibitem[Kanhabua and N{\o}rv{\aa}g(2008)]{temporal_entropy}
Nattiya Kanhabua and Kjetil N{\o}rv{\aa}g.
\newblock Improving temporal language models for determining time of
  non-timestamped documents.
\newblock In \emph{International Conference on Theory and Practice of Digital
  Libraries}, pages 358--370. Springer, 2008.

\bibitem[Kiela et~al.(2015)Kiela, Hill, and Clark]{joint_app2015}
Douwe Kiela, Felix Hill, and Stephen Clark.
\newblock Specializing word embeddings for similarity or relatedness.
\newblock In \emph{Proceedings of the 2015 Conference on Empirical Methods in
  Natural Language Processing}, pages 2044--2048. Association for Computational
  Linguistics, 2015.
\newblock \doi{10.18653/v1/D15-1242}.
\newblock URL \url{http://aclweb.org/anthology/D15-1242}.

\bibitem[Kim(2014)]{yoon_kim}
Yoon Kim.
\newblock Convolutional neural networks for sentence classification.
\newblock In \emph{Proceedings of the 2014 Conference on Empirical Methods in
  Natural Language Processing (EMNLP)}, pages 1746--1751. Association for
  Computational Linguistics, 2014.
\newblock \doi{10.3115/v1/D14-1181}.
\newblock URL \url{http://www.aclweb.org/anthology/D14-1181}.

\bibitem[Kingma and Ba(2014)]{adam_opt}
Diederik Kingma and Jimmy Ba.
\newblock Adam: A method for stochastic optimization.
\newblock 12 2014.

\bibitem[Kipf and Welling(2016)]{Kipf2016}
Thomas~N. Kipf and Max Welling.
\newblock Semi-supervised classification with graph convolutional networks.
\newblock \emph{CoRR}, abs/1609.02907, 2016.
\newblock URL \url{http://arxiv.org/abs/1609.02907}.

\bibitem[Kipf and Welling(2017)]{kipf2016semi}
Thomas~N. Kipf and Max Welling.
\newblock Semi-supervised classification with graph convolutional networks.
\newblock In \emph{International Conference on Learning Representations
  (ICLR)}, 2017.

\bibitem[Kobren et~al.(2017)Kobren, Monath, Krishnamurthy, and
  McCallum]{Kobren:2017:HAE:3097983.3098079}
Ari Kobren, Nicholas Monath, Akshay Krishnamurthy, and Andrew McCallum.
\newblock A hierarchical algorithm for extreme clustering.
\newblock In \emph{Proceedings of the 23rd ACM SIGKDD International Conference
  on Knowledge Discovery and Data Mining}, KDD '17, pages 255--264, New York,
  NY, USA, 2017. ACM.
\newblock ISBN 978-1-4503-4887-4.
\newblock \doi{10.1145/3097983.3098079}.
\newblock URL \url{http://doi.acm.org/10.1145/3097983.3098079}.

\bibitem[Komninos and Manandhar(2016)]{ext_paper}
Alexandros Komninos and Suresh Manandhar.
\newblock Dependency based embeddings for sentence classification tasks.
\newblock In \emph{Proceedings of the 2016 Conference of the North {A}merican
  Chapter of the Association for Computational Linguistics: Human Language
  Technologies}, pages 1490--1500, San Diego, California, June 2016.
  Association for Computational Linguistics.
\newblock \doi{10.18653/v1/N16-1175}.
\newblock URL \url{https://www.aclweb.org/anthology/N16-1175}.

\bibitem[Kotsakos et~al.(2014)Kotsakos, Lappas, Kotzias, Gunopulos, Kanhabua,
  and N{\o}rv{\aa}g]{Kotsakos:2014:BAD:2600428.2609495}
Dimitrios Kotsakos, Theodoros Lappas, Dimitrios Kotzias, Dimitrios Gunopulos,
  Nattiya Kanhabua, and Kjetil N{\o}rv{\aa}g.
\newblock A burstiness-aware approach for document dating.
\newblock In \emph{Proceedings of the 37th International ACM SIGIR Conference
  on Research \&\#38; Development in Information Retrieval}, SIGIR '14, pages
  1003--1006, New York, NY, USA, 2014. ACM.
\newblock ISBN 978-1-4503-2257-7.
\newblock \doi{10.1145/2600428.2609495}.
\newblock URL \url{http://doi.acm.org/10.1145/2600428.2609495}.

\bibitem[Krishnamurthy and
  Mitchell(2011)]{Krishnamurthy:2011:NPD:2002472.2002545}
Jayant Krishnamurthy and Tom~M. Mitchell.
\newblock Which noun phrases denote which concepts?
\newblock In \emph{Proceedings of the 49th Annual Meeting of the Association
  for Computational Linguistics: Human Language Technologies - Volume 1}, HLT
  '11, pages 570--580, Stroudsburg, PA, USA, 2011. Association for
  Computational Linguistics.
\newblock ISBN 978-1-932432-87-9.
\newblock URL \url{http://dl.acm.org/citation.cfm?id=2002472.2002545}.

\bibitem[Krizhevsky et~al.(2012)Krizhevsky, Sutskever, and Hinton]{alexnet}
Alex Krizhevsky, Ilya Sutskever, and Geoffrey~E. Hinton.
\newblock Imagenet classification with deep convolutional neural networks.
\newblock In \emph{Proceedings of the 25th International Conference on Neural
  Information Processing Systems - Volume 1}, NIPS'12, pages 1097--1105, USA,
  2012. Curran Associates Inc.
\newblock URL \url{http://dl.acm.org/citation.cfm?id=2999134.2999257}.

\bibitem[Kuzey et~al.(2016)Kuzey, Setty, Str\"{o}tgen, and
  Weikum]{temponym_paper}
Erdal Kuzey, Vinay Setty, Jannik Str\"{o}tgen, and Gerhard Weikum.
\newblock As time goes by: Comprehensive tagging of textual phrases with
  temporal scopes.
\newblock In \emph{Proceedings of the 25th International Conference on World
  Wide Web}, WWW '16, pages 915--925, Republic and Canton of Geneva,
  Switzerland, 2016. International World Wide Web Conferences Steering
  Committee.
\newblock ISBN 978-1-4503-4143-1.
\newblock \doi{10.1145/2872427.2883055}.
\newblock URL \url{https://doi.org/10.1145/2872427.2883055}.

\bibitem[Lappas et~al.(2009)Lappas, Arai, Platakis, Kotsakos, and
  Gunopulos]{Lappas:2009:BSD:1557019.1557075}
Theodoros Lappas, Benjamin Arai, Manolis Platakis, Dimitrios Kotsakos, and
  Dimitrios Gunopulos.
\newblock On burstiness-aware search for document sequences.
\newblock In \emph{Proceedings of the 15th ACM SIGKDD International Conference
  on Knowledge Discovery and Data Mining}, KDD '09, pages 477--486, New York,
  NY, USA, 2009. ACM.
\newblock ISBN 978-1-60558-495-9.
\newblock \doi{10.1145/1557019.1557075}.
\newblock URL \url{http://doi.acm.org/10.1145/1557019.1557075}.

\bibitem[LeCun et~al.(1999)LeCun, Haffner, Bottou, and Bengio]{cnn_paper}
Yann LeCun, Patrick Haffner, L{\'e}on Bottou, and Yoshua Bengio.
\newblock Object recognition with gradient-based learning.
\newblock In \emph{Shape, Contour and Grouping in Computer Vision}, pages
  319--, London, UK, UK, 1999. Springer-Verlag.
\newblock ISBN 3-540-66722-9.
\newblock URL \url{http://dl.acm.org/citation.cfm?id=646469.691875}.

\bibitem[Lee et~al.(2013)Lee, Chang, Peirsman, Chambers, Surdeanu, and
  Jurafsky]{sieve_architecture}
Heeyoung Lee, Angel Chang, Yves Peirsman, Nathanael Chambers, Mihai Surdeanu,
  and Dan Jurafsky.
\newblock Deterministic coreference resolution based on entity-centric,
  precision-ranked rules.
\newblock \emph{Comput. Linguist.}, 39\penalty0 (4):\penalty0 885--916,
  December 2013.
\newblock ISSN 0891-2017.

\bibitem[Lee et~al.(2018)Lee, He, and Zettlemoyer]{e2e_coref_model}
Kenton Lee, Luheng He, and Luke Zettlemoyer.
\newblock Higher-order coreference resolution with coarse-to-fine inference.
\newblock In \emph{Proceedings of the 2018 Conference of the North American
  Chapter of the Association for Computational Linguistics: Human Language
  Technologies, Volume 2 (Short Papers)}, pages 687--692. Association for
  Computational Linguistics, 2018.
\newblock URL \url{http://aclweb.org/anthology/N18-2108}.

\bibitem[Lei(2014)]{Lei2014}
Jing Lei.
\newblock Classification with confidence.
\newblock \emph{Biometrika}, 101\penalty0 (4):\penalty0 755--769, 2014.
\newblock \doi{10.1093/biomet/asu038}.
\newblock URL \url{http://dx.doi.org/10.1093/biomet/asu038}.

\bibitem[Levy and Goldberg(2014)]{deps_paper}
Omer Levy and Yoav Goldberg.
\newblock Dependency-based word embeddings.
\newblock In \emph{Proceedings of the 52nd Annual Meeting of the Association
  for Computational Linguistics (Volume 2: Short Papers)}, pages 302--308.
  Association for Computational Linguistics, 2014.
\newblock \doi{10.3115/v1/P14-2050}.
\newblock URL \url{http://www.aclweb.org/anthology/P14-2050}.

\bibitem[Li et~al.(2018{\natexlab{a}})Li, Li, Song, and Lin]{Li2018a}
Chen Li, Jianxin Li, Yangqiu Song, and Ziwei Lin.
\newblock Training and evaluating improved dependency-based word embeddings.
\newblock In \emph{AAAI}, 2018{\natexlab{a}}.

\bibitem[Li and Sethi(2006)]{Li2006}
Mingkun Li and Ishwar~K. Sethi.
\newblock Confidence-based active learning.
\newblock \emph{IEEE Trans. Pattern Anal. Mach. Intell.}, 28\penalty0
  (8):\penalty0 1251--1261, August 2006.
\newblock ISSN 0162-8828.
\newblock \doi{10.1109/TPAMI.2006.156}.
\newblock URL \url{https://doi.org/10.1109/TPAMI.2006.156}.

\bibitem[Li and Croft(2003)]{ir_time_li}
Xiaoyan Li and W.~Bruce Croft.
\newblock Time-based language models.
\newblock In \emph{Proceedings of the Twelfth International Conference on
  Information and Knowledge Management}, CIKM '03, pages 469--475, New York,
  NY, USA, 2003. ACM.
\newblock ISBN 1-58113-723-0.
\newblock \doi{10.1145/956863.956951}.
\newblock URL \url{http://doi.acm.org/10.1145/956863.956951}.

\bibitem[Li et~al.(2015)Li, Tarlow, Brockschmidt, and Zemel]{Li2015gated}
Yujia Li, Daniel Tarlow, Marc Brockschmidt, and Richard Zemel.
\newblock Gated graph sequence neural networks.
\newblock \emph{arXiv preprint arXiv:1511.05493}, 2015.

\bibitem[Li et~al.(2018{\natexlab{b}})Li, Vinyals, Dyer, Pascanu, and
  Battaglia]{gcn_comp_synthesis}
Yujia Li, Oriol Vinyals, Chris Dyer, Razvan Pascanu, and Peter~W. Battaglia.
\newblock Learning deep generative models of graphs.
\newblock \emph{ArXiv}, abs/1803.03324, 2018{\natexlab{b}}.

\bibitem[Liao et~al.(2018)Liao, Brockschmidt, Tarlow, Gaunt, Urtasun, and
  Zemel]{Liao2018graph}
Renjie Liao, Marc Brockschmidt, Daniel Tarlow, Alexander Gaunt, Raquel Urtasun,
  and Richard~S. Zemel.
\newblock Graph partition neural networks for semi-supervised classification,
  2018.
\newblock URL \url{https://openreview.net/forum?id=rk4Fz2e0b}.

\bibitem[Lin and Pantel(2001)]{Lin:2001:DSI:502512.502559}
Dekang Lin and Patrick Pantel.
\newblock Dirt @sbt@discovery of inference rules from text.
\newblock In \emph{Proceedings of the Seventh ACM SIGKDD International
  Conference on Knowledge Discovery and Data Mining}, KDD '01, pages 323--328,
  New York, NY, USA, 2001. ACM.
\newblock ISBN 1-58113-391-X.
\newblock \doi{10.1145/502512.502559}.
\newblock URL \url{http://doi.acm.org/10.1145/502512.502559}.

\bibitem[Lin et~al.(2012)Lin, Mausam, and
  Etzioni]{Lin:2012:ELW:2391200.2391216}
Thomas Lin, Mausam, and Oren Etzioni.
\newblock Entity linking at web scale.
\newblock In \emph{Proceedings of the Joint Workshop on Automatic Knowledge
  Base Construction and Web-scale Knowledge Extraction}, AKBC-WEKEX '12, pages
  84--88, Stroudsburg, PA, USA, 2012. Association for Computational
  Linguistics.
\newblock URL \url{http://dl.acm.org/citation.cfm?id=2391200.2391216}.

\bibitem[Lin et~al.(2016)Lin, Shen, Liu, Luan, and Sun]{lin2016neural}
Yankai Lin, Shiqi Shen, Zhiyuan Liu, Huanbo Luan, and Maosong Sun.
\newblock Neural relation extraction with selective attention over instances.
\newblock In \emph{Proceedings of the 54th Annual Meeting of the Association
  for Computational Linguistics (Volume 1: Long Papers)}, volume~1, pages
  2124--2133, 2016.

\bibitem[Ling and Weld(2012)]{figer_paper}
Xiao Ling and Daniel~S. Weld.
\newblock Fine-grained entity recognition.
\newblock In \emph{Proceedings of the Twenty-Sixth AAAI Conference on
  Artificial Intelligence}, AAAI'12, pages 94--100. AAAI Press, 2012.
\newblock URL \url{http://dl.acm.org/citation.cfm?id=2900728.2900742}.

\bibitem[Liu et~al.(2017)Liu, Wang, Chang, and Sui]{softlabel_paper}
Tianyu Liu, Kexiang Wang, Baobao Chang, and Zhifang Sui.
\newblock A soft-label method for noise-tolerant distantly supervised relation
  extraction.
\newblock In \emph{Proceedings of the 2017 Conference on Empirical Methods in
  Natural Language Processing}, pages 1790--1795. Association for Computational
  Linguistics, 2017.
\newblock URL \url{http://aclweb.org/anthology/D17-1189}.

\bibitem[Liu et~al.(2014)Liu, Liu, Xu, and Zhao]{Liu2014ExploringFE}
Yang Liu, Kang Liu, Liheng Xu, and Jun Zhao.
\newblock Exploring fine-grained entity type constraints for distantly
  supervised relation extraction.
\newblock In \emph{COLING}, 2014.

\bibitem[Llid{\'o} et~al.(2001)Llid{\'o}, Berlanga, and
  Aramburu]{temp_reasoner2}
D.~Llid{\'o}, R.~Berlanga, and M.~J. Aramburu.
\newblock Extracting temporal references to assign document event-time
  periods*.
\newblock In Heinrich~C. Mayr, Jiri Lazansky, Gerald Quirchmayr, and Pavel
  Vogel, editors, \emph{Database and Expert Systems Applications}, pages
  62--71, Berlin, Heidelberg, 2001. Springer Berlin Heidelberg.
\newblock ISBN 978-3-540-44759-7.

\bibitem[Llorens et~al.(2015)Llorens, Chambers, UzZaman, Mostafazadeh, Allen,
  and Pustejovsky]{tempeval15}
Hector Llorens, Nathanael Chambers, Naushad UzZaman, Nasrin Mostafazadeh, James
  Allen, and James Pustejovsky.
\newblock Semeval-2015 task 5: Qa tempeval-evaluating temporal information
  understanding with question answering.
\newblock In \emph{Proceedings of the 9th International Workshop on Semantic
  Evaluation (SemEval 2015)}, pages 792--800, 2015.

\bibitem[Luong et~al.(2013)Luong, Socher, and Manning]{rw_dataset}
Minh-Thang Luong, Richard Socher, and Christopher~D. Manning.
\newblock Better word representations with recursive neural networks for
  morphology.
\newblock In \emph{CoNLL}, Sofia, Bulgaria, 2013.

\bibitem[Ma and Hovy(2016)]{pos_tagging}
Xuezhe Ma and Eduard~H. Hovy.
\newblock End-to-end sequence labeling via bi-directional lstm-cnns-crf.
\newblock \emph{CoRR}, abs/1603.01354, 2016.

\bibitem[{Ma} et~al.(2015){Ma}, {Crook}, {Sarikaya}, and
  {Fosler-Lussier}]{kg_in_dialog}
Y.~{Ma}, P.~A. {Crook}, R.~{Sarikaya}, and E.~{Fosler-Lussier}.
\newblock Knowledge graph inference for spoken dialog systems.
\newblock In \emph{2015 IEEE International Conference on Acoustics, Speech and
  Signal Processing (ICASSP)}, pages 5346--5350, April 2015.
\newblock \doi{10.1109/ICASSP.2015.7178992}.

\bibitem[Mani and Wilson(2000)]{temp_reasoner1}
Inderjeet Mani and George Wilson.
\newblock Robust temporal processing of news.
\newblock In \emph{Proceedings of the 38th Annual Meeting on Association for
  Computational Linguistics}, ACL '00, pages 69--76, Stroudsburg, PA, USA,
  2000. Association for Computational Linguistics.
\newblock \doi{10.3115/1075218.1075228}.
\newblock URL \url{https://doi.org/10.3115/1075218.1075228}.

\bibitem[Manning et~al.(2008)Manning, Raghavan, and
  Sch\"{u}tze]{Manning:2008:IIR:1394399}
Christopher~D. Manning, Prabhakar Raghavan, and Hinrich Sch\"{u}tze.
\newblock \emph{Introduction to Information Retrieval}.
\newblock Cambridge University Press, New York, NY, USA, 2008.
\newblock ISBN 0521865719, 9780521865715.

\bibitem[Manning et~al.(2014)Manning, Surdeanu, Bauer, Finkel, Bethard, and
  McClosky]{stanford_corenlp}
Christopher~D. Manning, Mihai Surdeanu, John Bauer, Jenny Finkel, Steven~J.
  Bethard, and David McClosky.
\newblock The {Stanford} {CoreNLP} natural language processing toolkit.
\newblock In \emph{Association for Computational Linguistics (ACL) System
  Demonstrations}, pages 55--60, 2014.
\newblock URL \url{http://www.aclweb.org/anthology/P/P14/P14-5010}.

\bibitem[Marcheggiani and Titov(2017)]{gcn_srl}
Diego Marcheggiani and Ivan Titov.
\newblock Encoding sentences with graph convolutional networks for semantic
  role labeling.
\newblock In \emph{Proceedings of the 2017 Conference on Empirical Methods in
  Natural Language Processing}, pages 1506--1515. Association for Computational
  Linguistics, 2017.
\newblock URL \url{http://aclweb.org/anthology/D17-1159}.

\bibitem[Marcus et~al.(1994)Marcus, Kim, Marcinkiewicz, MacIntyre, Bies,
  Ferguson, Katz, and Schasberger]{penn_pos_dataset}
Mitchell Marcus, Grace Kim, Mary~Ann Marcinkiewicz, Robert MacIntyre, Ann Bies,
  Mark Ferguson, Karen Katz, and Britta Schasberger.
\newblock The penn treebank: Annotating predicate argument structure.
\newblock In \emph{Proceedings of the Workshop on Human Language Technology},
  HLT '94, pages 114--119, Stroudsburg, PA, USA, 1994. Association for
  Computational Linguistics.
\newblock ISBN 1-55860-357-3.
\newblock \doi{10.3115/1075812.1075835}.
\newblock URL \url{https://doi.org/10.3115/1075812.1075835}.

\bibitem[Mausam et~al.(2012)Mausam, Schmitz, Bart, Soderland, and
  Etzioni]{ollie}
Mausam, Michael Schmitz, Robert Bart, Stephen Soderland, and Oren Etzioni.
\newblock Open language learning for information extraction.
\newblock In \emph{Proceedings of the 2012 Joint Conference on Empirical
  Methods in Natural Language Processing and Computational Natural Language
  Learning}, EMNLP-CoNLL '12, pages 523--534, Stroudsburg, PA, USA, 2012.
  Association for Computational Linguistics.
\newblock URL \url{http://dl.acm.org/citation.cfm?id=2390948.2391009}.

\bibitem[Mausam(2016)]{ollie3}
Mausam Mausam.
\newblock Open information extraction systems and downstream applications.
\newblock In \emph{Proceedings of the Twenty-Fifth International Joint
  Conference on Artificial Intelligence}, IJCAI'16, pages 4074--4077. AAAI
  Press, 2016.
\newblock ISBN 978-1-57735-770-4.
\newblock URL \url{http://dl.acm.org/citation.cfm?id=3061053.3061220}.

\bibitem[Mikolov et~al.(2013{\natexlab{a}})Mikolov, Sutskever, Chen, Corrado,
  and Dean]{word2vec}
Tomas Mikolov, Ilya Sutskever, Kai Chen, Greg Corrado, and Jeffrey Dean.
\newblock Distributed representations of words and phrases and their
  compositionality.
\newblock In \emph{Proceedings of the 26th International Conference on Neural
  Information Processing Systems - Volume 2}, NIPS'13, pages 3111--3119, USA,
  2013{\natexlab{a}}. Curran Associates Inc.
\newblock URL \url{http://dl.acm.org/citation.cfm?id=2999792.2999959}.

\bibitem[Mikolov et~al.(2013{\natexlab{b}})Mikolov, Yih, and
  Zweig]{Mikolov2013b}
Tomas Mikolov, Scott Wen-tau Yih, and Geoffrey Zweig.
\newblock Linguistic regularities in continuous space word representations.
\newblock In \emph{Proceedings of the 2013 Conference of the North American
  Chapter of the Association for Computational Linguistics: Human Language
  Technologies (NAACL-HLT-2013)}. Association for Computational Linguistics,
  May 2013{\natexlab{b}}.
\newblock URL
  \url{https://www.microsoft.com/en-us/research/publication/linguistic-regularities-in-continuous-space-word-representations/}.

\bibitem[Miller(1995)]{wordnet}
George~A. Miller.
\newblock Wordnet: A lexical database for english.
\newblock \emph{Commun. ACM}, 38\penalty0 (11):\penalty0 39--41, November 1995.
\newblock ISSN 0001-0782.
\newblock \doi{10.1145/219717.219748}.
\newblock URL \url{http://doi.acm.org/10.1145/219717.219748}.

\bibitem[Mintz et~al.(2009)Mintz, Bills, Snow, and
  Jurafsky]{distant_supervision2009}
Mike Mintz, Steven Bills, Rion Snow, and Dan Jurafsky.
\newblock Distant supervision for relation extraction without labeled data.
\newblock In \emph{Proceedings of the Joint Conference of the 47th Annual
  Meeting of the ACL and the 4th International Joint Conference on Natural
  Language Processing of the AFNLP: Volume 2-Volume 2}, pages 1003--1011.
  Association for Computational Linguistics, 2009.

\bibitem[Mirza and Tonelli(2014)]{E14-1033}
Paramita Mirza and Sara Tonelli.
\newblock Classifying temporal relations with simple features.
\newblock In \emph{Proceedings of the 14th Conference of the European Chapter
  of the Association for Computational Linguistics}, pages 308--317.
  Association for Computational Linguistics, 2014.
\newblock \doi{10.3115/v1/E14-1033}.
\newblock URL \url{http://www.aclweb.org/anthology/E14-1033}.

\bibitem[Mirza and Tonelli(2016)]{catena_paper}
Paramita Mirza and Sara Tonelli.
\newblock Catena: Causal and temporal relation extraction from natural language
  texts.
\newblock In \emph{Proceedings of COLING 2016, the 26th International
  Conference on Computational Linguistics: Technical Papers}, pages 64--75. The
  COLING 2016 Organizing Committee, 2016.
\newblock URL \url{http://www.aclweb.org/anthology/C16-1007}.

\bibitem[Mitchell et~al.(2015)Mitchell, Cohen, Hruschka, Talukdar, Betteridge,
  Carlson, Dalvi, Gardner, Kisiel, Krishnamurthy, Lao, Mazaitis, Mohamed,
  Nakashole, Platanios, Ritter, Samadi, Settles, Wang, Wijaya, Gupta, Chen,
  Saparov, Greaves, and Welling]{NELL-aaai15}
T.~Mitchell, W.~Cohen, E.~Hruschka, P.~Talukdar, J.~Betteridge, A.~Carlson,
  B.~Dalvi, M.~Gardner, B.~Kisiel, J.~Krishnamurthy, N.~Lao, K.~Mazaitis,
  T.~Mohamed, N.~Nakashole, E.~Platanios, A.~Ritter, M.~Samadi, B.~Settles,
  R.~Wang, D.~Wijaya, A.~Gupta, X.~Chen, A.~Saparov, M.~Greaves, and
  J.~Welling.
\newblock Never-ending learning.
\newblock In \emph{Proceedings of the Twenty-Ninth AAAI Conference on
  Artificial Intelligence (AAAI-15)}, 2015.

\bibitem[Mitchell et~al.(2018)Mitchell, Cohen, Hruschka, Talukdar, Yang,
  Betteridge, Carlson, Dalvi, Gardner, Kisiel, Krishnamurthy, Lao, Mazaitis,
  Mohamed, Nakashole, Platanios, Ritter, Samadi, Settles, Wang, Wijaya, Gupta,
  Chen, Saparov, Greaves, and Welling]{nell}
T.~Mitchell, W.~Cohen, E.~Hruschka, P.~Talukdar, B.~Yang, J.~Betteridge,
  A.~Carlson, B.~Dalvi, M.~Gardner, B.~Kisiel, J.~Krishnamurthy, N.~Lao,
  K.~Mazaitis, T.~Mohamed, N.~Nakashole, E.~Platanios, A.~Ritter, M.~Samadi,
  B.~Settles, R.~Wang, D.~Wijaya, A.~Gupta, X.~Chen, A.~Saparov, M.~Greaves,
  and J.~Welling.
\newblock Never-ending learning.
\newblock \emph{Commun. ACM}, 61\penalty0 (5):\penalty0 103--115, April 2018.
\newblock ISSN 0001-0782.
\newblock \doi{10.1145/3191513}.
\newblock URL \url{http://doi.acm.org/10.1145/3191513}.

\bibitem[Monti et~al.(2018)Monti, Shchur, Bojchevski, Litany, G{\"{u}}nnemann,
  and Bronstein]{dual_primal_gcn}
Federico Monti, Oleksandr Shchur, Aleksandar Bojchevski, Or~Litany, Stephan
  G{\"{u}}nnemann, and Michael~M. Bronstein.
\newblock Dual-primal graph convolutional networks.
\newblock \emph{CoRR}, abs/1806.00770, 2018.
\newblock URL \url{http://arxiv.org/abs/1806.00770}.

\bibitem[Morin and Bengio(2005)]{hierarchical_softmax1}
Frederic Morin and Yoshua Bengio.
\newblock Hierarchical probabilistic neural network language model.
\newblock In \emph{Proceedings of the Tenth International Workshop on
  Artificial Intelligence and Statistics, {AISTATS} 2005, Bridgetown, Barbados,
  January 6-8, 2005}, 2005.
\newblock URL \url{http://www.gatsby.ucl.ac.uk/aistats/fullpapers/208.pdf}.

\bibitem[Mrk{\v{s}}i{\'{c}} et~al.(2016)Mrk{\v{s}}i{\'{c}}, {\'O}~S{\'e}aghdha,
  Thomson, Ga{\v{s}}i{\'{c}}, Rojas-Barahona, Su, Vandyke, Wen, and
  Young]{counter_paper}
Nikola Mrk{\v{s}}i{\'{c}}, Diarmuid {\'O}~S{\'e}aghdha, Blaise Thomson, Milica
  Ga{\v{s}}i{\'{c}}, Lina~M. Rojas-Barahona, Pei-Hao Su, David Vandyke,
  Tsung-Hsien Wen, and Steve Young.
\newblock Counter-fitting word vectors to linguistic constraints.
\newblock In \emph{Proceedings of the 2016 Conference of the North American
  Chapter of the Association for Computational Linguistics: Human Language
  Technologies}, pages 142--148. Association for Computational Linguistics,
  2016.
\newblock \doi{10.18653/v1/N16-1018}.
\newblock URL \url{http://www.aclweb.org/anthology/N16-1018}.

\bibitem[Nagarajan et~al.(2017)Nagarajan, Jat, and Talukdar]{candis_paper}
Tushar Nagarajan, Sharmistha Jat, and Partha Talukdar.
\newblock Candis: Coupled {\&} attention-driven neural distant supervision.
\newblock \emph{CoRR}, abs/1710.09942, 2017.

\bibitem[Nakashole et~al.(2012)Nakashole, Weikum, and
  Suchanek]{Nakashole:2012:PTR:2390948.2391076}
Ndapandula Nakashole, Gerhard Weikum, and Fabian Suchanek.
\newblock Patty: A taxonomy of relational patterns with semantic types.
\newblock In \emph{Proceedings of the 2012 Joint Conference on Empirical
  Methods in Natural Language Processing and Computational Natural Language
  Learning}, EMNLP-CoNLL '12, pages 1135--1145, Stroudsburg, PA, USA, 2012.
  Association for Computational Linguistics.
\newblock URL \url{http://dl.acm.org/citation.cfm?id=2390948.2391076}.

\bibitem[Nguyen et~al.(2018)Nguyen, Nguyen, Nguyen, and Phung]{convkb}
Dai~Quoc Nguyen, Tu~Dinh Nguyen, Dat~Quoc Nguyen, and Dinh Phung.
\newblock A novel embedding model for knowledge base completion based on
  convolutional neural network.
\newblock In \emph{Proceedings of the 2018 Conference of the North American
  Chapter of the Association for Computational Linguistics: Human Language
  Technologies, Volume 2 (Short Papers)}, pages 327--333. Association for
  Computational Linguistics, 2018.
\newblock \doi{10.18653/v1/N18-2053}.
\newblock URL \url{http://aclweb.org/anthology/N18-2053}.

\bibitem[Nguyen and Grishman(2018)]{gcn_event}
Thien Nguyen and Ralph Grishman.
\newblock Graph convolutional networks with argument-aware pooling for event
  detection, 2018.
\newblock URL
  \url{https://www.aaai.org/ocs/index.php/AAAI/AAAI18/paper/view/16329}.

\bibitem[{Nickel} et~al.(2016){Nickel}, {Murphy}, {Tresp}, and
  {Gabrilovich}]{survey2016nickel}
M.~{Nickel}, K.~{Murphy}, V.~{Tresp}, and E.~{Gabrilovich}.
\newblock A review of relational machine learning for knowledge graphs.
\newblock \emph{Proceedings of the IEEE}, 104\penalty0 (1):\penalty0 11--33,
  Jan 2016.
\newblock ISSN 0018-9219.
\newblock \doi{10.1109/JPROC.2015.2483592}.

\bibitem[Nickel et~al.(2011)Nickel, Tresp, and
  Kriegel]{Nickel:2011:TMC:3104482.3104584}
Maximilian Nickel, Volker Tresp, and Hans-Peter Kriegel.
\newblock A three-way model for collective learning on multi-relational data.
\newblock In \emph{Proceedings of the 28th International Conference on
  International Conference on Machine Learning}, ICML'11, pages 809--816, USA,
  2011. Omnipress.
\newblock ISBN 978-1-4503-0619-5.
\newblock URL \url{http://dl.acm.org/citation.cfm?id=3104482.3104584}.

\bibitem[Nickel et~al.(2016)Nickel, Rosasco, and Poggio]{hole}
Maximilian Nickel, Lorenzo Rosasco, and Tomaso Poggio.
\newblock Holographic embeddings of knowledge graphs.
\newblock In \emph{Proceedings of the Thirtieth AAAI Conference on Artificial
  Intelligence}, AAAI'16, pages 1955--1961. AAAI Press, 2016.
\newblock URL \url{http://dl.acm.org/citation.cfm?id=3016100.3016172}.

\bibitem[Niepert et~al.(2016)Niepert, Ahmed, and Kutzkov]{pachysan}
Mathias Niepert, Mohamed Ahmed, and Konstantin Kutzkov.
\newblock Learning convolutional neural networks for graphs.
\newblock In \emph{Proceedings of the 33rd International Conference on
  International Conference on Machine Learning - Volume 48}, ICML'16, pages
  2014--2023. JMLR.org, 2016.
\newblock URL \url{http://dl.acm.org/citation.cfm?id=3045390.3045603}.

\bibitem[Nimishakavi et~al.(2016)Nimishakavi, Saini, and Talukdar]{D16-1040}
Madhav Nimishakavi, Uday~Singh Saini, and Partha Talukdar.
\newblock Relation schema induction using tensor factorization with side
  information.
\newblock In \emph{Proceedings of the 2016 Conference on Empirical Methods in
  Natural Language Processing}, pages 414--423. Association for Computational
  Linguistics, 2016.
\newblock \doi{10.18653/v1/D16-1040}.
\newblock URL \url{http://www.aclweb.org/anthology/D16-1040}.

\bibitem[Olson et~al.(1999)Olson, Bostic, Seltzer, and
  Berkeley]{ir_time_usenix}
MA~Olson, K~Bostic, MI~Seltzer, and DB~Berkeley.
\newblock Usenix annual technical conference, freenix track, 1999.

\bibitem[Orbach and Crammer(2012)]{Orbach2012}
Matan Orbach and Koby Crammer.
\newblock Graph-based transduction with confidence.
\newblock In \emph{ECML/PKDD}, 2012.

\bibitem[Parker et~al.(2011{\natexlab{a}})Parker, Graff, Kong, Chen, and
  Maeda]{gigaword5th}
Robert Parker, David Graff, Junbo Kong, Ke~Chen, and Kazuaki Maeda.
\newblock English gigaword fifth edition ldc2011t07. dvd.
\newblock \emph{Philadelphia: Linguistic Data Consortium}, 2011{\natexlab{a}}.

\bibitem[Parker et~al.(2011{\natexlab{b}})Parker, Graff, Kong, Chen, and
  Maeda]{parker2011english}
Robert Parker, David Graff, Junbo Kong, Ke~Chen, and Kazuaki Maeda.
\newblock English gigaword fifth edition, linguistic data consortium.
\newblock Technical report, Technical report, Technical Report. Linguistic Data
  Consortium, Philadelphia, 2011{\natexlab{b}}.

\bibitem[Paulheim and F\"{u}mkranz(2012)]{feat}
Heiko Paulheim and Johannes F\"{u}mkranz.
\newblock Unsupervised generation of data mining features from linked open
  data.
\newblock In \emph{Proceedings of the 2Nd International Conference on Web
  Intelligence, Mining and Semantics}, WIMS '12, pages 31:1--31:12, New York,
  NY, USA, 2012. ACM.
\newblock ISBN 978-1-4503-0915-8.
\newblock \doi{10.1145/2254129.2254168}.
\newblock URL \url{http://doi.acm.org/10.1145/2254129.2254168}.

\bibitem[Pavlick et~al.(2015{\natexlab{a}})Pavlick, Rastogi, Ganitkevitch,
  Durme, and Callison{-}Burch]{ppdb2}
Ellie Pavlick, Pushpendre Rastogi, Juri Ganitkevitch, Benjamin~Van Durme, and
  Chris Callison{-}Burch.
\newblock {PPDB} 2.0: Better paraphrase ranking, fine-grained entailment
  relations, word embeddings, and style classification.
\newblock In \emph{Proceedings of the 53rd Annual Meeting of the Association
  for Computational Linguistics and the 7th International Joint Conference on
  Natural Language Processing of the Asian Federation of Natural Language
  Processing, {ACL} 2015, July 26-31, 2015, Beijing, China, Volume 2: Short
  Papers}, pages 425--430, 2015{\natexlab{a}}.
\newblock URL \url{http://aclweb.org/anthology/P/P15/P15-2070.pdf}.

\bibitem[Pavlick et~al.(2015{\natexlab{b}})Pavlick, Rastogi, Ganitkevitch,
  Van~Durme, and Callison-Burch]{ppdb_paper}
Ellie Pavlick, Pushpendre Rastogi, Juri Ganitkevitch, Benjamin Van~Durme, and
  Chris Callison-Burch.
\newblock Ppdb 2.0: Better paraphrase ranking, fine-grained entailment
  relations, word embeddings, and style classification.
\newblock In \emph{In Proceedings of the Association for Computational
  Linguistics (ACL-2015)}, pages 425--430. Association for Computational
  Linguistics, 2015{\natexlab{b}}.
\newblock \doi{10.3115/v1/P15-2070}.
\newblock URL \url{http://www.aclweb.org/anthology/P15-2070}.

\bibitem[Pennington et~al.(2014)Pennington, Socher, and Manning]{glove}
Jeffrey Pennington, Richard Socher, and Christopher~D. Manning.
\newblock Glove: Global vectors for word representation.
\newblock In \emph{Empirical Methods in Natural Language Processing (EMNLP)},
  pages 1532--1543, 2014.
\newblock URL \url{http://www.aclweb.org/anthology/D14-1162}.

\bibitem[Perozzi et~al.(2014)Perozzi, Al-Rfou, and Skiena]{deepwalk}
Bryan Perozzi, Rami Al-Rfou, and Steven Skiena.
\newblock Deepwalk: Online learning of social representations.
\newblock In \emph{Proceedings of the 20th ACM SIGKDD International Conference
  on Knowledge Discovery and Data Mining}, KDD '14, pages 701--710, New York,
  NY, USA, 2014. ACM.
\newblock ISBN 978-1-4503-2956-9.
\newblock \doi{10.1145/2623330.2623732}.
\newblock URL \url{http://doi.acm.org/10.1145/2623330.2623732}.

\bibitem[Peters et~al.(2018)Peters, Neumann, Iyyer, Gardner, Clark, Lee, and
  Zettlemoyer]{elmo_paper}
Matthew~E. Peters, Mark Neumann, Mohit Iyyer, Matt Gardner, Christopher Clark,
  Kenton Lee, and Luke Zettlemoyer.
\newblock Deep contextualized word representations.
\newblock In \emph{Proc. of NAACL}, 2018.

\bibitem[Pradhan et~al.(2012)Pradhan, Moschitti, Xue, Uryupina, and
  Zhang]{coref_conll12_dataset}
Sameer Pradhan, Alessandro Moschitti, Nianwen Xue, Olga Uryupina, and Yuchen
  Zhang.
\newblock Conll-2012 shared task: Modeling multilingual unrestricted
  coreference in ontonotes.
\newblock In \emph{Joint Conference on EMNLP and CoNLL - Shared Task}, CoNLL
  '12, pages 1--40, Stroudsburg, PA, USA, 2012. Association for Computational
  Linguistics.
\newblock URL \url{http://dl.acm.org/citation.cfm?id=2391181.2391183}.

\bibitem[Pujara et~al.(2013)Pujara, Miao, Getoor, and
  Cohen]{Pujara:2013:KGI:2717129.2717164}
Jay Pujara, Hui Miao, Lise Getoor, and William Cohen.
\newblock Knowledge graph identification.
\newblock In \emph{Proceedings of the 12th International Semantic Web
  Conference - Part I}, ISWC '13, pages 542--557, New York, NY, USA, 2013.
  Springer-Verlag New York, Inc.
\newblock ISBN 978-3-642-41334-6.
\newblock \doi{10.1007/978-3-642-41335-3_34}.
\newblock URL \url{http://dx.doi.org/10.1007/978-3-642-41335-3_34}.

\bibitem[Pustejovsky et~al.(2003)Pustejovsky, Hanks, Sauri, See, Gaizauskas,
  Setzer, Radev, Sundheim, Day, Ferro, et~al.]{timebank03}
James Pustejovsky, Patrick Hanks, Roser Sauri, Andrew See, Robert Gaizauskas,
  Andrea Setzer, Dragomir Radev, Beth Sundheim, David Day, Lisa Ferro, et~al.
\newblock The timebank corpus.
\newblock In \emph{Corpus linguistics}, volume 2003, page~40. Lancaster, UK.,
  2003.

\bibitem[Rajpurkar et~al.(2016)Rajpurkar, Zhang, Lopyrev, and
  Liang]{squad_dataset}
Pranav Rajpurkar, Jian Zhang, Konstantin Lopyrev, and Percy Liang.
\newblock Squad: 100,000+ questions for machine comprehension of text.
\newblock In \emph{Proceedings of the 2016 Conference on Empirical Methods in
  Natural Language Processing}, pages 2383--2392. Association for Computational
  Linguistics, 2016.
\newblock \doi{10.18653/v1/D16-1264}.
\newblock URL \url{http://aclweb.org/anthology/D16-1264}.

\bibitem[Ramsundar et~al.(2019)Ramsundar, Eastman, Walters, Pande, Leswing, and
  Wu]{deep_chem}
Bharath Ramsundar, Peter Eastman, Patrick Walters, Vijay Pande, Karl Leswing,
  and Zhenqin Wu.
\newblock \emph{Deep Learning for the Life Sciences}.
\newblock O'Reilly Media, 2019.
\newblock
  \url{https://www.amazon.com/Deep-Learning-Life-Sciences-Microscopy/dp/1492039837}.

\bibitem[Ratinov et~al.(2011)Ratinov, Roth, Downey, and
  Anderson]{Ratinov:2011:LGA:2002472.2002642}
Lev Ratinov, Dan Roth, Doug Downey, and Mike Anderson.
\newblock Local and global algorithms for disambiguation to wikipedia.
\newblock In \emph{Proceedings of the 49th Annual Meeting of the Association
  for Computational Linguistics: Human Language Technologies - Volume 1}, HLT
  '11, pages 1375--1384, Stroudsburg, PA, USA, 2011. Association for
  Computational Linguistics.
\newblock ISBN 978-1-932432-87-9.
\newblock URL \url{http://dl.acm.org/citation.cfm?id=2002472.2002642}.

\bibitem[Riedel et~al.(2010)Riedel, Yao, and McCallum]{riedel2010modeling}
Sebastian Riedel, Limin Yao, and Andrew McCallum.
\newblock Modeling relations and their mentions without labeled text.
\newblock In \emph{Joint European Conference on Machine Learning and Knowledge
  Discovery in Databases}, pages 148--163. Springer, 2010.

\bibitem[Ristoski and Paulheim(2016)]{rdf2vec}
Petar Ristoski and Heiko Paulheim.
\newblock Rdf2vec: Rdf graph embeddings for data mining.
\newblock In \emph{International Semantic Web Conference}, pages 498--514.
  Springer, 2016.

\bibitem[Ristoski et~al.(2016)Ristoski, de~Vries, and
  Paulheim]{node_class_splits}
Petar Ristoski, Gerben Klaas~Dirk de~Vries, and Heiko Paulheim.
\newblock A collection of benchmark datasets for systematic evaluations of
  machine learning on the semantic web.
\newblock In Paul Groth, Elena Simperl, Alasdair Gray, Marta Sabou, Markus
  Kr{\"o}tzsch, Freddy Lecue, Fabian Fl{\"o}ck, and Yolanda Gil, editors,
  \emph{The Semantic Web -- ISWC 2016}, pages 186--194, Cham, 2016. Springer
  International Publishing.
\newblock ISBN 978-3-319-46547-0.

\bibitem[Saha et~al.(2017)Saha, Pal, and Mausam]{ollie2}
Swarnadeep Saha, Harinder Pal, and Mausam.
\newblock Bootstrapping for numerical open ie.
\newblock In \emph{Proceedings of the 55th Annual Meeting of the Association
  for Computational Linguistics (Volume 2: Short Papers)}, pages 317--323.
  Association for Computational Linguistics, 2017.
\newblock \doi{10.18653/v1/P17-2050}.
\newblock URL \url{http://www.aclweb.org/anthology/P17-2050}.

\bibitem[Schlichtkrull et~al.(2017)Schlichtkrull, Kipf, Bloem, Berg, Titov, and
  Welling]{r_gcn}
Michael Schlichtkrull, Thomas~N Kipf, Peter Bloem, Rianne van~den Berg, Ivan
  Titov, and Max Welling.
\newblock Modeling relational data with graph convolutional networks.
\newblock \emph{arXiv preprint arXiv:1703.06103}, 2017.

\bibitem[Sen et~al.(2008)Sen, Namata, Bilgic, Getoor, Gallagher, and
  Eliassi{-}Rad]{pubmed}
Prithviraj Sen, Galileo Namata, Mustafa Bilgic, Lise Getoor, Brian Gallagher,
  and Tina Eliassi{-}Rad.
\newblock Collective classification in network data.
\newblock \emph{{AI} Magazine}, 29\penalty0 (3):\penalty0 93--106, 2008.

\bibitem[Shang et~al.(2019)Shang, Tang, Huang, Bi, He, and Zhou]{sacn}
Chao Shang, Yun Tang, Jing Huang, Jinbo Bi, Xiaodong He, and Bowen Zhou.
\newblock End-to-end structure-aware convolutional networks for knowledge base
  completion, 2019.

\bibitem[Shervashidze et~al.(2011)Shervashidze, Schweitzer, van Leeuwen,
  Mehlhorn, and Borgwardt]{wl}
Nino Shervashidze, Pascal Schweitzer, Erik~Jan van Leeuwen, Kurt Mehlhorn, and
  Karsten~M. Borgwardt.
\newblock Weisfeiler-lehman graph kernels.
\newblock \emph{J. Mach. Learn. Res.}, 12:\penalty0 2539--2561, November 2011.
\newblock ISSN 1532-4435.
\newblock URL \url{http://dl.acm.org/citation.cfm?id=1953048.2078187}.

\bibitem[{Shuman} et~al.(2013){Shuman}, {Narang}, {Frossard}, {Ortega}, and
  {Vandergheynst}]{emerging_field_graph_signal}
D.~I. {Shuman}, S.~K. {Narang}, P.~{Frossard}, A.~{Ortega}, and
  P.~{Vandergheynst}.
\newblock The emerging field of signal processing on graphs: Extending
  high-dimensional data analysis to networks and other irregular domains.
\newblock \emph{IEEE Signal Processing Magazine}, 30\penalty0 (3):\penalty0
  83--98, May 2013.
\newblock \doi{10.1109/MSP.2012.2235192}.

\bibitem[{Singh} et~al.(2018){Singh}, {Nath}, and {Kumar}]{speech_survey}
A.~P. {Singh}, R.~{Nath}, and S.~{Kumar}.
\newblock A survey: Speech recognition approaches and techniques.
\newblock In \emph{2018 5th IEEE Uttar Pradesh Section International Conference
  on Electrical, Electronics and Computer Engineering (UPCON)}, pages 1--4, Nov
  2018.
\newblock \doi{10.1109/UPCON.2018.8596954}.

\bibitem[Socher et~al.(2013{\natexlab{a}})Socher, Bauer, Manning, and
  Ng]{parsing_app}
Richard Socher, John Bauer, Christopher~D. Manning, and Andrew~Y. Ng.
\newblock {Parsing With Compositional Vector Grammars}.
\newblock In \emph{{ACL}}. 2013{\natexlab{a}}.

\bibitem[Socher et~al.(2013{\natexlab{b}})Socher, Chen, Manning, and
  Ng]{ntn_kg}
Richard Socher, Danqi Chen, Christopher~D Manning, and Andrew Ng.
\newblock Reasoning with neural tensor networks for knowledge base completion.
\newblock In C.~J.~C. Burges, L.~Bottou, M.~Welling, Z.~Ghahramani, and K.~Q.
  Weinberger, editors, \emph{Advances in Neural Information Processing Systems
  26}, pages 926--934. Curran Associates, Inc., 2013{\natexlab{b}}.
\newblock URL
  \url{http://papers.nips.cc/paper/5028-reasoning-with-neural-tensor-networks-for-knowledge-base-completion.pdf}.

\bibitem[Sood(2016)]{OMV93V_2016}
Gaurav Sood.
\newblock Parsed dmoz data, 2016.
\newblock URL \url{http://dx.doi.org/10.7910/DVN/OMV93V}.

\bibitem[Spielman(2007)]{spectral_graph_theory}
Daniel~A. Spielman.
\newblock Spectral graph theory and its applications.
\newblock In \emph{Proceedings of the 48th Annual IEEE Symposium on Foundations
  of Computer Science}, FOCS '07, pages 29--38, Washington, DC, USA, 2007. IEEE
  Computer Society.
\newblock ISBN 0-7695-3010-9.
\newblock \doi{10.1109/FOCS.2007.66}.
\newblock URL \url{http://dx.doi.org/10.1109/FOCS.2007.66}.

\bibitem[Spitkovsky and Chang(2012)]{SPITKOVSKY12.266}
Valentin~I. Spitkovsky and Angel~X. Chang.
\newblock A cross-lingual dictionary for english wikipedia concepts.
\newblock In Nicoletta Calzolari~(Conference Chair), Khalid Choukri, Thierry
  Declerck, Mehmet~Uğur Doğan, Bente Maegaard, Joseph Mariani, Asuncion
  Moreno, Jan Odijk, and Stelios Piperidis, editors, \emph{Proceedings of the
  Eight International Conference on Language Resources and Evaluation
  (LREC'12)}, Istanbul, Turkey, may 2012. European Language Resources
  Association (ELRA).
\newblock ISBN 978-2-9517408-7-7.

\bibitem[Strubell et~al.(2018)Strubell, Verga, Andor, Weiss, and
  McCallum]{lisa_paper}
Emma Strubell, Patrick Verga, Daniel Andor, David Weiss, and Andrew McCallum.
\newblock Linguistically-informed self-attention for semantic role labeling.
\newblock In \emph{Proceedings of the 2018 Conference on Empirical Methods in
  Natural Language Processing}, pages 5027--5038. Association for Computational
  Linguistics, 2018.
\newblock URL \url{http://aclweb.org/anthology/D18-1548}.

\bibitem[Subramanya and Talukdar(2014)]{subramanya2014graph}
Amarnag Subramanya and Partha~Pratim Talukdar.
\newblock Graph-based semi-supervised learning.
\newblock \emph{Synthesis Lectures on Artificial Intelligence and Machine
  Learning}, 8\penalty0 (4):\penalty0 1--125, 2014.

\bibitem[Suchanek et~al.(2007)Suchanek, Kasneci, and Weikum]{yago}
Fabian~M. Suchanek, Gjergji Kasneci, and Gerhard Weikum.
\newblock {Yago: A Core of Semantic Knowledge}.
\newblock In \emph{16th International Conference on the World Wide Web}, pages
  697--706, 2007.

\bibitem[Sun et~al.(2019)Sun, Deng, Nie, and Tang]{rotate}
Zhiqing Sun, Zhi-Hong Deng, Jian-Yun Nie, and Jian Tang.
\newblock Rotate: Knowledge graph embedding by relational rotation in complex
  space.
\newblock In \emph{International Conference on Learning Representations}, 2019.
\newblock URL \url{https://openreview.net/forum?id=HkgEQnRqYQ}.

\bibitem[Surdeanu et~al.(2012)Surdeanu, Tibshirani, Nallapati, and
  Manning]{surdeanu2012multi}
Mihai Surdeanu, Julie Tibshirani, Ramesh Nallapati, and Christopher~D Manning.
\newblock Multi-instance multi-label learning for relation extraction.
\newblock In \emph{Proceedings of the 2012 joint conference on empirical
  methods in natural language processing and computational natural language
  learning}, pages 455--465. Association for Computational Linguistics, 2012.

\bibitem[Sutskever et~al.(2014)Sutskever, Vinyals, and Le]{seq2seq}
Ilya Sutskever, Oriol Vinyals, and Quoc~V. Le.
\newblock Sequence to sequence learning with neural networks.
\newblock In \emph{Proceedings of the 27th International Conference on Neural
  Information Processing Systems - Volume 2}, NIPS'14, pages 3104--3112,
  Cambridge, MA, USA, 2014. MIT Press.
\newblock URL \url{http://dl.acm.org/citation.cfm?id=2969033.2969173}.

\bibitem[Tai et~al.(2015)Tai, Socher, and Manning]{tree_lstm}
Kai~Sheng Tai, Richard Socher, and Christopher~D. Manning.
\newblock Improved semantic representations from tree-structured long
  short-term memory networks.
\newblock In \emph{Proceedings of the 53rd Annual Meeting of the Association
  for Computational Linguistics and the 7th International Joint Conference on
  Natural Language Processing (Volume 1: Long Papers)}, pages 1556--1566.
  Association for Computational Linguistics, 2015.
\newblock \doi{10.3115/v1/P15-1150}.
\newblock URL \url{http://aclweb.org/anthology/P15-1150}.

\bibitem[Tan et~al.(2005)Tan, Steinbach, and Kumar]{Tan:2005:IDM:1095618}
Pang-Ning Tan, Michael Steinbach, and Vipin Kumar.
\newblock \emph{Introduction to Data Mining, (First Edition)}.
\newblock Addison-Wesley Longman Publishing Co., Inc., Boston, MA, USA, 2005.
\newblock ISBN 0321321367.

\bibitem[Tang et~al.(2015)Tang, Qu, Wang, Zhang, Yan, and Mei]{line}
Jian Tang, Meng Qu, Mingzhe Wang, Ming Zhang, Jun Yan, and Qiaozhu Mei.
\newblock Line: Large-scale information network embedding.
\newblock In \emph{WWW}. ACM, 2015.

\bibitem[Tjong Kim~Sang and De~Meulder(2003)]{conll03_dataset}
Erik~F. Tjong Kim~Sang and Fien De~Meulder.
\newblock Introduction to the conll-2003 shared task: Language-independent
  named entity recognition.
\newblock In \emph{Proceedings of the Seventh Conference on Natural Language
  Learning at HLT-NAACL 2003 - Volume 4}, CONLL '03, pages 142--147,
  Stroudsburg, PA, USA, 2003. Association for Computational Linguistics.
\newblock \doi{10.3115/1119176.1119195}.
\newblock URL \url{https://doi.org/10.3115/1119176.1119195}.

\bibitem[Toutanova and Chen(2015)]{toutanova}
Kristina Toutanova and Danqi Chen.
\newblock Observed versus latent features for knowledge base and text
  inference.
\newblock In \emph{Proceedings of the 3rd Workshop on Continuous Vector Space
  Models and their Compositionality}, pages 57--66, 2015.

\bibitem[Trani et~al.(2014)Trani, Ceccarelli, Lucchese, Orlando, and
  Perego]{Trani:2014:DOS:2878453.2878558}
Salvatore Trani, Diego Ceccarelli, Claudio Lucchese, Salvatore Orlando, and
  Raffaele Perego.
\newblock Dexter 2.0: An open source tool for semantically enriching data.
\newblock In \emph{Proceedings of the 2014 International Conference on Posters
  \&\#38; Demonstrations Track - Volume 1272}, ISWC-PD'14, pages 417--420,
  Aachen, Germany, Germany, 2014. CEUR-WS.org.
\newblock URL \url{http://dl.acm.org/citation.cfm?id=2878453.2878558}.

\bibitem[Trouillon et~al.(2016)Trouillon, Welbl, Riedel, Gaussier, and
  Bouchard]{complex}
Th{\'e}o Trouillon, Johannes Welbl, Sebastian Riedel, \'{E}ric Gaussier, and
  Guillaume Bouchard.
\newblock Complex embeddings for simple link prediction.
\newblock In \emph{Proceedings of the 33rd International Conference on
  International Conference on Machine Learning - Volume 48}, ICML'16, pages
  2071--2080. JMLR.org, 2016.
\newblock URL \url{http://dl.acm.org/citation.cfm?id=3045390.3045609}.

\bibitem[UzZaman et~al.(2013)UzZaman, Llorens, Derczynski, Allen, Verhagen, and
  Pustejovsky]{tempeval13}
Naushad UzZaman, Hector Llorens, Leon Derczynski, James Allen, Marc Verhagen,
  and James Pustejovsky.
\newblock Semeval-2013 task 1: Tempeval-3: Evaluating time expressions, events,
  and temporal relations.
\newblock In \emph{Second Joint Conference on Lexical and Computational
  Semantics (* SEM), Volume 2: Proceedings of the Seventh International
  Workshop on Semantic Evaluation (SemEval 2013)}, volume~2, pages 1--9, 2013.

\bibitem[van~der Maaten and Hinton(2008)]{maaten2008visualizing}
L.J.P. van~der Maaten and G.E. Hinton.
\newblock Visualizing high-dimensional data using t-sne.
\newblock 2008.

\bibitem[Vashishth et~al.(2018{\natexlab{a}})Vashishth, Dasgupta, Ray, and
  Talukdar]{neuraldater}
Shikhar Vashishth, Shib~Sankar Dasgupta, Swayambhu~Nath Ray, and Partha
  Talukdar.
\newblock Dating documents using graph convolution networks.
\newblock In \emph{Proceedings of the 56th Annual Meeting of the Association
  for Computational Linguistics (Volume 1: Long Papers)}, pages 1605--1615.
  Association for Computational Linguistics, 2018{\natexlab{a}}.
\newblock URL \url{http://aclweb.org/anthology/P18-1149}.

\bibitem[Vashishth et~al.(2018{\natexlab{b}})Vashishth, Jain, and
  Talukdar]{cesi_paper}
Shikhar Vashishth, Prince Jain, and Partha Talukdar.
\newblock {CESI}: Canonicalizing open knowledge bases using embeddings and side
  information.
\newblock In \emph{Proceedings of the 2018 World Wide Web Conference}, WWW '18,
  pages 1317--1327, Republic and Canton of Geneva, Switzerland,
  2018{\natexlab{b}}. International World Wide Web Conferences Steering
  Committee.
\newblock ISBN 978-1-4503-5639-8.
\newblock \doi{10.1145/3178876.3186030}.
\newblock URL \url{https://doi.org/10.1145/3178876.3186030}.

\bibitem[Vashishth et~al.(2018{\natexlab{c}})Vashishth, Joshi, Prayaga,
  Bhattacharyya, and Talukdar]{reside}
Shikhar Vashishth, Rishabh Joshi, Sai~Suman Prayaga, Chiranjib Bhattacharyya,
  and Partha Talukdar.
\newblock Reside: Improving distantly-supervised neural relation extraction
  using side information.
\newblock In \emph{Proceedings of the 2018 Conference on Empirical Methods in
  Natural Language Processing}, pages 1257--1266. Association for Computational
  Linguistics, 2018{\natexlab{c}}.
\newblock URL \url{http://aclweb.org/anthology/D18-1157}.

\bibitem[Vashishth et~al.(2019)Vashishth, Yadav, Bhandari, and
  Talukdar]{confgcn}
Shikhar Vashishth, Prateek Yadav, Manik Bhandari, and Partha Talukdar.
\newblock Confidence-based graph convolutional networks for semi-supervised
  learning.
\newblock In Kamalika Chaudhuri and Masashi Sugiyama, editors,
  \emph{Proceedings of Machine Learning Research}, volume~89 of
  \emph{Proceedings of Machine Learning Research}, pages 1792--1801. PMLR,
  16--18 Apr 2019.
\newblock URL \url{http://proceedings.mlr.press/v89/vashishth19a.html}.

\bibitem[Veli{\v{c}}kovi{\'{c}} et~al.(2018)Veli{\v{c}}kovi{\'{c}}, Cucurull,
  Casanova, Romero, Li{\`{o}}, and Bengio]{gat}
Petar Veli{\v{c}}kovi{\'{c}}, Guillem Cucurull, Arantxa Casanova, Adriana
  Romero, Pietro Li{\`{o}}, and Yoshua Bengio.
\newblock {Graph Attention Networks}.
\newblock \emph{International Conference on Learning Representations}, 2018.
\newblock URL \url{https://openreview.net/forum?id=rJXMpikCZ}.
\newblock accepted as poster.

\bibitem[Verhagen et~al.(2007)Verhagen, Gaizauskas, Schilder, Hepple, Katz, and
  Pustejovsky]{tempeval07}
Marc Verhagen, Robert Gaizauskas, Frank Schilder, Mark Hepple, Graham Katz, and
  James Pustejovsky.
\newblock Semeval-2007 task 15: Tempeval temporal relation identification.
\newblock In \emph{Proceedings of the 4th international workshop on semantic
  evaluations}, pages 75--80. Association for Computational Linguistics, 2007.

\bibitem[Verhagen et~al.(2010)Verhagen, Sauri, Caselli, and
  Pustejovsky]{tempeval10}
Marc Verhagen, Roser Sauri, Tommaso Caselli, and James Pustejovsky.
\newblock Semeval-2010 task 13: Tempeval-2.
\newblock In \emph{Proceedings of the 5th international workshop on semantic
  evaluation}, pages 57--62. Association for Computational Linguistics, 2010.

\bibitem[Vilnis and McCallum(2014)]{Vilnis2014}
Luke Vilnis and Andrew McCallum.
\newblock Word representations via gaussian embedding.
\newblock \emph{CoRR}, abs/1412.6623, 2014.
\newblock URL \url{http://arxiv.org/abs/1412.6623}.

\bibitem[Vrande\v{c}i\'{c} and Kr\"{o}tzsch(2014)]{wikidata_paper}
Denny Vrande\v{c}i\'{c} and Markus Kr\"{o}tzsch.
\newblock Wikidata: A free collaborative knowledgebase.
\newblock \emph{Commun. ACM}, 57\penalty0 (10):\penalty0 78--85, September
  2014.
\newblock ISSN 0001-0782.
\newblock \doi{10.1145/2629489}.
\newblock URL \url{http://doi.acm.org/10.1145/2629489}.

\bibitem[Wan(2007)]{text_summ_time}
Xiaojun Wan.
\newblock Timedtextrank: Adding the temporal dimension to multi-document
  summarization.
\newblock In \emph{Proceedings of the 30th Annual International ACM SIGIR
  Conference on Research and Development in Information Retrieval}, SIGIR '07,
  pages 867--868, New York, NY, USA, 2007. ACM.
\newblock ISBN 978-1-59593-597-7.
\newblock \doi{10.1145/1277741.1277949}.
\newblock URL \url{http://doi.acm.org/10.1145/1277741.1277949}.

\bibitem[{Wang} et~al.(2017){Wang}, {Mao}, {Wang}, and {Guo}]{survey2017}
Q.~{Wang}, Z.~{Mao}, B.~{Wang}, and L.~{Guo}.
\newblock Knowledge graph embedding: A survey of approaches and applications.
\newblock \emph{IEEE Transactions on Knowledge and Data Engineering},
  29\penalty0 (12):\penalty0 2724--2743, Dec 2017.
\newblock ISSN 1041-4347.
\newblock \doi{10.1109/TKDE.2017.2754499}.

\bibitem[Wang et~al.(2014)Wang, Zhang, Feng, and Chen]{transh}
Zhen Wang, Jianwen Zhang, Jianlin Feng, and Zheng Chen.
\newblock Knowledge graph embedding by translating on hyperplanes.
\newblock In \emph{Proceedings of the Twenty-Eighth AAAI Conference on
  Artificial Intelligence}, AAAI'14, pages 1112--1119. AAAI Press, 2014.
\newblock URL \url{http://dl.acm.org/citation.cfm?id=2893873.2894046}.

\bibitem[Weston et~al.(2008)Weston, Ratle, and Collobert]{Weston2008deep}
Jason Weston, Fr{\'e}d{\'e}ric Ratle, and Ronan Collobert.
\newblock Deep learning via semi-supervised embedding.
\newblock In \emph{Proceedings of the 25th International Conference on Machine
  Learning}, ICML '08, pages 1168--1175, New York, NY, USA, 2008. ACM.
\newblock ISBN 978-1-60558-205-4.
\newblock \doi{10.1145/1390156.1390303}.
\newblock URL \url{http://doi.acm.org/10.1145/1390156.1390303}.

\bibitem[Winkler(1999)]{winkler1999state}
William~E Winkler.
\newblock The state of record linkage and current research problems.
\newblock In \emph{Statistical Research Division, US Census Bureau}. Citeseer,
  1999.

\bibitem[Xu et~al.(2014)Xu, Bai, Bian, Gao, Wang, Liu, and
  Liu]{rc_net_joint2014}
Chang Xu, Yalong Bai, Jiang Bian, Bin Gao, Gang Wang, Xiaoguang Liu, and
  Tie-Yan Liu.
\newblock Rc-net: A general framework for incorporating knowledge into word
  representations.
\newblock November 2014.
\newblock URL
  \url{https://www.microsoft.com/en-us/research/publication/rc-net-a-general-framework-for-incorporating-knowledge-into-word-representations/}.

\bibitem[Xu et~al.(2018)Xu, Li, Tian, Sonobe, Kawarabayashi, and
  Jegelka]{Xu2018}
Keyulu Xu, Chengtao Li, Yonglong Tian, Tomohiro Sonobe, Ken-ichi Kawarabayashi,
  and Stefanie Jegelka.
\newblock Representation learning on graphs with jumping knowledge networks.
\newblock In Jennifer Dy and Andreas Krause, editors, \emph{Proceedings of the
  35th International Conference on Machine Learning}, volume~80 of
  \emph{Proceedings of Machine Learning Research}, pages 5453--5462,
  Stockholmsmässan, Stockholm Sweden, 10--15 Jul 2018. PMLR.
\newblock URL \url{http://proceedings.mlr.press/v80/xu18c.html}.

\bibitem[Xu et~al.(2019)Xu, Hu, Leskovec, and Jegelka]{gin}
Keyulu Xu, Weihua Hu, Jure Leskovec, and Stefanie Jegelka.
\newblock How powerful are graph neural networks?
\newblock In \emph{International Conference on Learning Representations}, 2019.
\newblock URL \url{https://openreview.net/forum?id=ryGs6iA5Km}.

\bibitem[Yadav et~al.(2019)Yadav, Nimishakavi, Yadati, Vashishth, Rajkumar, and
  Talukdar]{lovasz_paper}
Prateek Yadav, Madhav Nimishakavi, Naganand Yadati, Shikhar Vashishth, Arun
  Rajkumar, and Partha Talukdar.
\newblock Lovasz convolutional networks.
\newblock In Kamalika Chaudhuri and Masashi Sugiyama, editors,
  \emph{Proceedings of Machine Learning Research}, volume~89 of
  \emph{Proceedings of Machine Learning Research}, pages 1978--1987. PMLR,
  16--18 Apr 2019.
\newblock URL \url{http://proceedings.mlr.press/v89/yadav19a.html}.

\bibitem[Yaghoobzadeh et~al.(2017)Yaghoobzadeh, Adel, and
  Sch{\"u}tze]{typeinfo2017}
Yadollah Yaghoobzadeh, Heike Adel, and Hinrich Sch{\"u}tze.
\newblock Noise mitigation for neural entity typing and relation extraction.
\newblock In \emph{Proceedings of the 15th Conference of the European Chapter
  of the Association for Computational Linguistics: Volume 1, Long Papers},
  pages 1183--1194. Association for Computational Linguistics, 2017.
\newblock URL \url{http://aclweb.org/anthology/E17-1111}.

\bibitem[Yanardag and Vishwanathan(2015)]{graph_datasets}
Pinar Yanardag and S.V.N. Vishwanathan.
\newblock Deep graph kernels.
\newblock In \emph{Proceedings of the 21th ACM SIGKDD International Conference
  on Knowledge Discovery and Data Mining}, KDD '15, pages 1365--1374, New York,
  NY, USA, 2015. ACM.
\newblock ISBN 978-1-4503-3664-2.
\newblock \doi{10.1145/2783258.2783417}.
\newblock URL \url{http://doi.acm.org/10.1145/2783258.2783417}.

\bibitem[Yang et~al.(2014)Yang, Yih, He, Gao, and Deng]{distmult}
Bishan Yang, Wen{-}tau Yih, Xiaodong He, Jianfeng Gao, and Li~Deng.
\newblock Embedding entities and relations for learning and inference in
  knowledge bases.
\newblock \emph{CoRR}, abs/1412.6575, 2014.
\newblock URL \url{http://arxiv.org/abs/1412.6575}.

\bibitem[Yang et~al.(2016)Yang, Cohen, and Salakhutdinov]{Yang2016revisiting}
Zhilin Yang, William~W. Cohen, and Ruslan Salakhutdinov.
\newblock Revisiting semi-supervised learning with graph embeddings.
\newblock In \emph{Proceedings of the 33rd International Conference on
  International Conference on Machine Learning - Volume 48}, ICML'16, pages
  40--48. JMLR.org, 2016.
\newblock URL \url{http://dl.acm.org/citation.cfm?id=3045390.3045396}.

\bibitem[{Yao} et~al.(2018){Yao}, {Mao}, and {Luo}]{gcn_text_class}
Liang {Yao}, Chengsheng {Mao}, and Yuan {Luo}.
\newblock {Graph Convolutional Networks for Text Classification}.
\newblock \emph{ArXiv e-prints}, art. arXiv:1809.05679, September 2018.

\bibitem[Yates and Etzioni(2009)]{Yates:2009:UMD:1622716.1622724}
Alexander Yates and Oren Etzioni.
\newblock Unsupervised methods for determining object and relation synonyms on
  the web.
\newblock \emph{J. Artif. Int. Res.}, 34\penalty0 (1):\penalty0 255--296, March
  2009.
\newblock ISSN 1076-9757.
\newblock URL \url{http://dl.acm.org/citation.cfm?id=1622716.1622724}.

\bibitem[Ye et~al.(2019)Ye, Li, Fang, Zang, and Wang]{vrgcn}
Rui Ye, Xin Li, Yujie Fang, Hongyu Zang, and Mingzhong Wang.
\newblock A vectorized relational graph convolutional network for
  multi-relational network alignment.
\newblock In \emph{Proceedings of the Twenty-Eighth International Joint
  Conference on Artificial Intelligence, {IJCAI-19}}, pages 4135--4141.
  International Joint Conferences on Artificial Intelligence Organization, 7
  2019.
\newblock \doi{10.24963/ijcai.2019/574}.
\newblock URL \url{https://doi.org/10.24963/ijcai.2019/574}.

\bibitem[Yi et~al.(2016)Yi, Su, Guo, and Guibas]{yi2016syncspeccnn}
Li~Yi, Hao Su, Xingwen Guo, and Leonidas Guibas.
\newblock Syncspeccnn: Synchronized spectral cnn for 3d shape segmentation.
\newblock \emph{arXiv preprint arXiv:1612.00606}, 2016.

\bibitem[Yoshikawa et~al.(2009)Yoshikawa, Riedel, Asahara, and
  Matsumoto]{P09-1046}
Katsumasa Yoshikawa, Sebastian Riedel, Masayuki Asahara, and Yuji Matsumoto.
\newblock Jointly identifying temporal relations with markov logic.
\newblock In \emph{Proceedings of the Joint Conference of the 47th Annual
  Meeting of the ACL and the 4th International Joint Conference on Natural
  Language Processing of the AFNLP}, pages 405--413. Association for
  Computational Linguistics, 2009.
\newblock URL \url{http://www.aclweb.org/anthology/P09-1046}.

\bibitem[Yu and Dredze(2014)]{Mo2014}
Mo~Yu and Mark Dredze.
\newblock Improving lexical embeddings with semantic knowledge.
\newblock In \emph{Proceedings of the 52nd Annual Meeting of the Association
  for Computational Linguistics (Volume 2: Short Papers)}, pages 545--550.
  Association for Computational Linguistics, 2014.
\newblock \doi{10.3115/v1/P14-2089}.
\newblock URL \url{http://www.aclweb.org/anthology/P14-2089}.

\bibitem[Zeng et~al.(2014)Zeng, Liu, Lai, Zhou, and Zhao]{zeng2014relation}
Daojian Zeng, Kang Liu, Siwei Lai, Guangyou Zhou, and Jun Zhao.
\newblock Relation classification via convolutional deep neural network.
\newblock In \emph{Proceedings of COLING 2014, the 25th International
  Conference on Computational Linguistics: Technical Papers}, pages 2335--2344,
  2014.

\bibitem[Zeng et~al.(2015)Zeng, Liu, Chen, and Zhao]{zeng2015distant}
Daojian Zeng, Kang Liu, Yubo Chen, and Jun Zhao.
\newblock Distant supervision for relation extraction via piecewise
  convolutional neural networks.
\newblock In \emph{Proceedings of the 2015 Conference on Empirical Methods in
  Natural Language Processing}, pages 1753--1762, 2015.

\bibitem[Zhang et~al.(2016)Zhang, Yuan, Lian, Xie, and Ma]{kb-recommender}
Fuzheng Zhang, Nicholas~Jing Yuan, Defu Lian, Xing Xie, and Wei-Ying Ma.
\newblock Collaborative knowledge base embedding for recommender systems.
\newblock In \emph{Proceedings of the 22Nd ACM SIGKDD International Conference
  on Knowledge Discovery and Data Mining}, KDD '16, pages 353--362, New York,
  NY, USA, 2016. ACM.
\newblock ISBN 978-1-4503-4232-2.
\newblock \doi{10.1145/2939672.2939673}.
\newblock URL \url{http://doi.acm.org/10.1145/2939672.2939673}.

\bibitem[Zhang et~al.(2018{\natexlab{a}})Zhang, Cui, Neumann, and Chen]{dgcnn}
Muhan Zhang, Zhicheng Cui, Marion Neumann, and Yixin Chen.
\newblock An end-to-end deep learning architecture for graph classification.
\newblock In \emph{AAAI}, pages 4438--4445, 2018{\natexlab{a}}.

\bibitem[Zhang et~al.(2018{\natexlab{b}})Zhang, Qi, and
  Manning]{gcn_re_stanford}
Yuhao Zhang, Peng Qi, and Christopher~D. Manning.
\newblock Graph convolution over pruned dependency trees improves relation
  extraction.
\newblock In \emph{Proceedings of the 2018 Conference on Empirical Methods in
  Natural Language Processing}, pages 2205--2215. Association for Computational
  Linguistics, 2018{\natexlab{b}}.
\newblock URL \url{http://aclweb.org/anthology/D18-1244}.

\bibitem[{Zhao} et~al.(2019){Zhao}, {Zheng}, {Xu}, and {Wu}]{obj_detect_survey}
Z.~{Zhao}, P.~{Zheng}, S.~{Xu}, and X.~{Wu}.
\newblock Object detection with deep learning: A review.
\newblock \emph{IEEE Transactions on Neural Networks and Learning Systems},
  30\penalty0 (11):\penalty0 3212--3232, Nov 2019.
\newblock \doi{10.1109/TNNLS.2018.2876865}.

\bibitem[Zhu et~al.(2003)Zhu, Ghahramani, and Lafferty]{term1_just1}
Xiaojin Zhu, Zoubin Ghahramani, and John Lafferty.
\newblock Semi-supervised learning using gaussian fields and harmonic
  functions.
\newblock In \emph{Proceedings of the Twentieth International Conference on
  International Conference on Machine Learning}, ICML'03, pages 912--919. AAAI
  Press, 2003.
\newblock ISBN 1-57735-189-4.
\newblock URL \url{http://dl.acm.org/citation.cfm?id=3041838.3041953}.

\end{thebibliography}
\end{document}